\documentclass{article}
\usepackage{iclr2027_conference,times}

\usepackage{amsmath,amsfonts,bm}

\def\eqref#1{equation~\ref{#1}}

\def\plaineqref#1{\ref{#1}}

\def\1{\bm{1}}

\DeclareMathAlphabet{\mathsfit}{\encodingdefault}{\sfdefault}{m}{sl}
\SetMathAlphabet{\mathsfit}{bold}{\encodingdefault}{\sfdefault}{bx}{n}

\newcommand{\E}{\mathbb{E}}

\newcommand{\R}{\mathbb{R}}

\newcommand{\Var}{\mathrm{Var}}

\usepackage{hyperref}
\usepackage{url}
\usepackage{booktabs}
\usepackage{array}
\usepackage{amsmath,amssymb}
\usepackage{graphicx}
\usepackage{placeins}
\usepackage{float}
\usepackage{xcolor}
\usepackage{tikz}
\usepackage{amsthm}
\usepackage{mathtools}
\usepackage{enumitem}
\usepackage[listings,breakable]{tcolorbox}
\newtcblisting{prompt}[1]{
  breakable, listing only, title={#1},
  listing options={basicstyle=\ttfamily, breaklines=true,
    literate={’}{{'}}1 {‘}{{`}}1 {“}{{"}}1 {”}{{"}}1},
}

\usepackage{etoc}
\usepackage[sectionbib]{bibunits}
\defaultbibliographystyle{iclr2027_conference}
\defaultbibliography{iclr2027_conference}

\title{The Problem Is the Problem: \\Towards Scalable Mathematical Discovery}

\author{Zeyu Zheng$^{1}$\thanks{Equal contribution.},\ Shengtong Zhang$^{2}$\footnotemark[1],\ Jeremy Avigad$^{1}$,\ Prasad Tetali$^{1}$,\ Sean Welleck$^{1}$\\
$^1$Carnegie Mellon University \quad
$^2$Anysphere Co.\\
\texttt{\{zeyuzhen, avigad, ptetali, swelleck\}@andrew.cmu.edu}, \\
\texttt{shengtong@anysphere.co} \\
\url{https://github.com/zeyu-zheng/FAR}
}

\newcommand{\GapP}{\ensuremath{\mathsf{GapP}}}

\iclrpreprintcopy
\theoremstyle{plain}
\newtheorem{theorem}{Theorem}[subsection]
\newtheorem{lemma}[theorem]{Lemma}
\newtheorem{proposition}[theorem]{Proposition}
\newtheorem{corollary}[theorem]{Corollary}

\newtheorem{definition}[theorem]{Definition}
\theoremstyle{remark}
\newtheorem{remark}[theorem]{Remark}

\newtheorem{example}[theorem]{Example}

\begin{document}

\maketitle

\begin{abstract}
AI systems are increasingly capable of contributing to mathematical research. In research practice, frontier-model reasoning is a limited resource, and expert mathematical review is even more sharply constrained. Allocating these scarce resources well is therefore central to making AI-assisted mathematical discovery efficient. In most current AI-for-math workflows, human effort is concentrated at the beginning and end, in selecting suitable research problems and later reviewing the resulting artifacts. These two stages are becoming bottlenecks for research-level mathematics. We address them by proposing a new human-AI discovery paradigm. The human input is no longer a single problem selected in advance, but a research direction in which the experts have interest and expertise. The system then searches a broad literature corpus for candidate problems in that direction. Inspired by search and recommender systems, we build \textit{Find, Attempt, and Recommend (FAR)}, a literature-to-review cascade that automates the search for suitable problems and focuses human attention on artifacts that have passed several stages of filtering. In a combinatorics pilot, the pipeline starts from 5{,}245 combinatorics papers, recovers 6{,}453 candidate conjectures or open problems, and filters them to 4{,}717 apparently well-posed and still-open conjectures. Subsequent reasoning and automated triage stages surface 598 potential resolutions\footnote{Available at \url{https://probxiv.com}.} and select 77 items for author-team review. Among them, we identify many interesting discoveries, including results on conjectures and questions of Davies--Jenssen--Perkins--Roberts, Erd\H{o}s--Straus, Ikenmeyer--Pak--Panova, and Lund--Saraf--Wolf. These results demonstrate the effectiveness of this new mode of human-AI collaboration for mathematical discovery.
\end{abstract}

\section{Introduction}

AI systems are increasingly capable of contributing to mathematical research, with recent progress on mathematical reasoning benchmarks~\citep{hendrycks2021measuring,zheng2021minif2f,guo2025deepseek,shao2025deepseekmath}, formal theorem proving~\citep{trinh2024solving,chervonyi2025gold,hubert2026olympiad,ren2025deepseek,xin2025scaling,chen2025seed,seed2026seed2}, and selected research-level problems~\citep{openai2026openai,alon2026remarks,tsoukalas2026advancing,team2026seed2}. Most of these systems operate with a problem-level interface, in which a researcher supplies a theorem, conjecture, or formal goal, the system attempts it, and the output is checked.
This interface is useful, but it leaves out an essential part of research: deciding which problems are worth attempting in the first place. 

We study the decision of which problems to attempt in terms of effort allocation for AI-assisted mathematical discovery. Frontier-model reasoning and expert mathematical review are scarce, and their value depends on which conjectures receive them. Rather than concentrating reasoning and review effort on a few problems chosen in advance, we propose a new workflow in which experts specify a research direction, and the AI system 
automatically finds problems to focus effort on (Figure~\ref{fig:paradigm}).
We build \textit{Find, Attempt, and Recommend (FAR)}, a literature-to-review cascade. Given a broad mathematical topic, \textit{FAR} finds relevant open conjectures from a large literature corpus, attempts to prove or disprove them, and recommends promising conjecture-resolution pairs for expert review.

\begin{figure}[t]
\centering
\includegraphics[width=\textwidth]{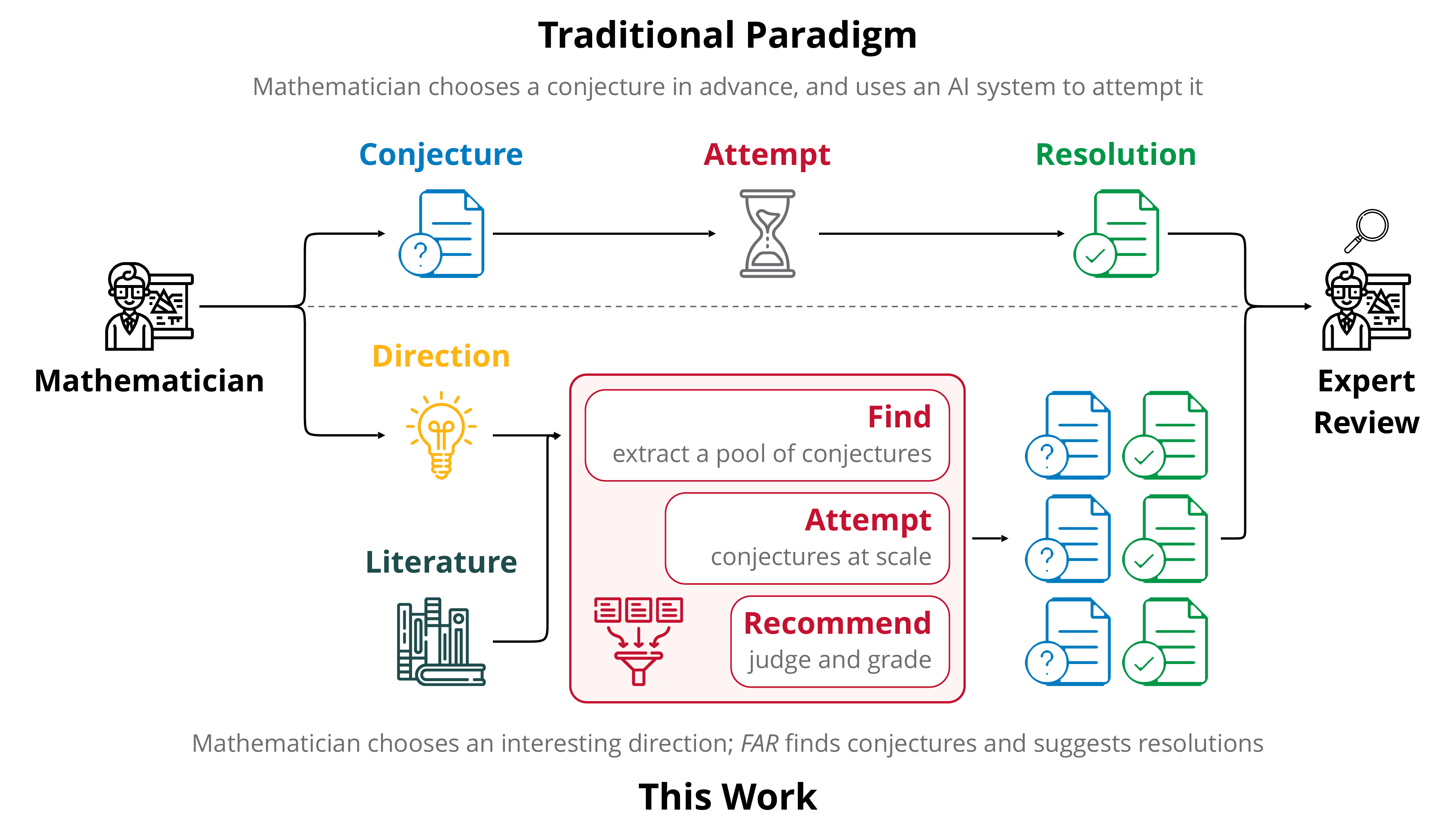}
\caption{\textbf{From choosing a problem to choosing a direction.} The upper part shows the problem-level interface; the lower part shows our approach. Figure~\ref{fig:process-v2} gives the details of \textit{FAR}.}
\label{fig:paradigm}
\end{figure}

We instantiate the workflow in combinatorics. Starting from 51{,}110 mathematics papers, the pipeline identifies 5{,}245 combinatorics papers, extracts 6{,}453 candidate conjectures or open problems from 2{,}742 papers, and filters these to 4{,}717 apparently well-posed and still-open conjectures. A first broad attempt run covers all 4{,}717 of them. Automated triage surfaces 598 potential resolutions, and a final selection step chooses 77 items for internal author-team review. We manually checked 15 of the resulting artifacts, chosen by our own interest. They include counterexamples to conjectures of Davies--Jenssen--Perkins--Roberts and of Lund--Saraf--Wolf, a proof of Ikenmeyer--Pak--Panova's conjecture on symmetric-group characters, and an answer to a question of Erd\H{o}s--Straus on divisibility among binomial coefficients.

We study various strategies for allocating a budget of solving attempts in our pipeline.
We frame allocation as a constrained optimization problem, and derive strategies for maximizing notions of quality and importance.
We show that these strategies can lead to more successful artifacts than with a uniform baseline.
Furthermore, the optimal strategy depends on the objective: for instance, the strategy differs if we wish to maximize the number of successful artifacts versus maximizing the importance across all artifacts.

In summary, our contributions are as follows:
\begin{itemize}
    \item We formulate scalable AI-assisted mathematical discovery as an effort-allocation problem over a pool of interesting mathematical problems, leading to a human--AI collaboration paradigm for using frontier-model reasoning and expert mathematical review efficiently.
    \item Inspired by search and recommender systems, we introduce \textit{Find, Attempt, and Recommend (FAR)}, which builds this problem pool from the mathematical literature and turns model attempts into reviewable artifacts.
    \item We instantiate the workflow in a combinatorics pilot and analyze the resulting literature-to-review funnel, including strategies for allocating model attempts.
    \item We obtain author-reviewed solutions to problems from the combinatorics literature, spanning proofs, counterexamples, and answers to open-ended questions. The write-ups are collected in Appendix~\ref{app:solutions}.
\end{itemize}

\section{Motivation and Related Work}
\label{sec:motivation}

\subsection{AI for Mathematical Research}
\label{sec:ai-for-math}

Recent AI systems have made rapid progress on the part of mathematical research that begins after a problem or objective has been specified. For example, FunSearch and AlphaEvolve search for new mathematical constructions, programs, and algorithms~\citep{romera2024mathematical,novikov2025alphaevolve}. AlphaProof Nexus studies formal proof search on open research-level problems~\citep{tsoukalas2026advancing}. Aletheia, Rethlas, QED, and other recent pipelines explore autonomous or semi-autonomous workflows for attempting open mathematical problems~\citep{feng2026towards,ju2026automated,an2026qed,ju2026automated,peng2026pipelinemath}. Frontier models have also contributed to individually selected problems of long-standing research interest, including OpenAI's disproof of the Unit Distance Problem and the Fable-assisted counterexample to the Jacobian conjecture~\citep{openai2026openai,alon2026remarks,jacobian2026counterexample,bukh2025oddtown}. 
Related work on an AI co-mathematician~\citep{zheng2026ai} develops a collaborative framework in which AI agents pursue parallel workstreams while mathematicians steer the research process. Taken together, these works point toward what Tao describes as a transition from proof scarcity to proof abundance~\citep{tao2026proofabundance}.

Mathematical research, however, rarely starts from an isolated problem statement. Before trying to solve a problem, researchers often need to study a direction of interest, find problems that matter within it, and decide whether those problems are worth sustained thought. These steps are ordinary parts of mathematical work, but they mostly sit outside the scope of today's AI-for-math systems. In current uses of AI for mathematics, the human work before and after model reasoning---problem selection and expert review---is becoming the narrow part of the pipeline.

Our approach moves the starting point of AI assistance to this earlier stage of mathematical research. Rather than selecting a  problem for the system, mathematicians specify a research direction. The system searches an available literature corpus, recovers and attempts candidate problems in that direction at scale, and returns a small set of artifacts for mathematical review. This shift is analogous to the move from task-driven agents, which act on a specified task, toward more proactive systems that help surface what tasks are worth pursuing~\citep{zhou2026externalization,lu2025proactive,thinkingmachines2026interactionmodels}. The human-specified direction bounds the system's initiative and aligns the search with the mathematicians' interests and domain expertise. Mathematicians remain responsible for validating and reporting any resulting mathematical claims~\citep{alper2026leiden,shan2026position}.

\subsection{Constructing an Attemptable Pool}
\label{sec:attemptable-pool}

A research direction does not by itself provide the system with a set of attemptable problems. Mathematical benchmarks such as PutnamBench, FrontierMath, and FirstProof provide clean, self-contained problem pools~\citep{tsoukalas2024putnambench,glazer2024frontiermath,abouzaid2026first}. In software engineering, another major domain for LLM agents~\citep{yang2024swe}, GitHub provides a centralized and structured collection of repositories, issues, documentation, and executable software from which benchmarks such as SWE-bench and ProgramBench can be constructed~\citep{jimenez2024swe,yang2026programbench}.

Mathematics does have valuable collections, including the Open Problem Garden, AIM Problem Lists, and Formal Conjectures~\citep{openproblemgarden,aimproblemlists,firsching2026formal}. Their coverage is selective, however, and they are not the primary infrastructure in which mathematical problems and their surrounding research context are recorded. Much of this information remains dispersed across the literature. A conjecture may appear as a numbered statement, a question, a remark, an unresolved case, or a sentence embedded in local notation, and its status may change after publication. Constructing an attemptable pool of problems therefore requires recovering candidate statements from their source context and checking whether they remain well posed and unresolved.

We draw on search and recommender systems to organize this process. Building a pool of attemptable problems is akin to candidate retrieval, i.e., recovering statements from a large corpus using imperfect signals and filtering them for provenance, well-posedness, and current status~\citep{belkin1992information,liu2009learning,liu2022neural}. 
We also draw on the idea of recommendation cascades, in which progressively more selective stages reduce a large collection of candidates to a small set of artifacts for expert attention~\citep{ricci2010introduction,covington2016deep,wang2011cascade,chen2017efficient,zhu2026contexting}. Our techniques also relate to literature-based discovery, which searches published knowledge for research opportunities not visible from a single paper~\citep{swanson1986undiscovered,swanson1986fish}.

Our objective differs from automated conjecturing, a complementary line of work that creates new mathematical conjectures rather than surfacing existing ones.
This line includes automated theory formation, the Ramanujan Machine, and the data-driven TxGraffiti system~\citep{colton2012automated,raayoni2021generating,davila2026automated}. Recent LLM work has used generated conjectures to expand formal training data and couple conjecturing with proving, as in LeanConjecturer and STP, while Moonshine makes conjecture generation the organizing objective of an autonomous mathematical research agent~\citep{onda2025leanconjecturer,dong2025stp,chen2026moonshine}. We instead recover unresolved statements that authors have already placed in the literature. Each candidate retains its source paper, statement text, local context, and status evidence, so that later attempts and reviews can be checked against what the source actually claimed. After status checking, the result is an attemptable pool $\mathcal{P}$ of source-grounded conjectures that appear well posed and still open.

\subsection{Effort Allocation under Uncertainty}
\label{sec:allocation}

Let $\mathcal{U}$ be the universe of mathematical questions that can be expressed in natural language. The human practice of mathematical research can be viewed as a large effort-allocation process over $\mathcal{U}$. Mathematicians search this space for questions worth exploring, and then try to answer them or make progress on them. As AI agents become primary sources of mathematical attempts, allocating agent compute raises a similar problem. At the same time, the Leiden Declaration emphasizes that mathematicians remain responsible for validating and reporting mathematical claims~\citep{alper2026leiden}, so the allocation of expert review effort also matters. The attemptable pool $\mathcal{P}$ described in Section~\ref{sec:attemptable-pool} forms  a small part of $\mathcal{U}$, but it is a natural starting point because its conjectures and open problems have already been selected, stated, and discussed by mathematicians.

Even for human mathematicians, deciding where to spend effort is difficult. Prior impressions of difficulty often differ from difficulty in hindsight. Some simply stated problems resist solution for a long time, while some long-standing problems eventually admit unexpectedly simple arguments. In an AI-assisted setting there is an additional source of uncertainty: what is difficult for human mathematicians need not be difficult in the same way for the current model. 
In turn, an important question is how to use limited model attempts to discover which conjectures from the attemptable pool $\mathcal{P}$ are likely to produce artifacts that are worth expert review.

\begin{figure}[H]
\centering
\includegraphics[width=0.6\textwidth]{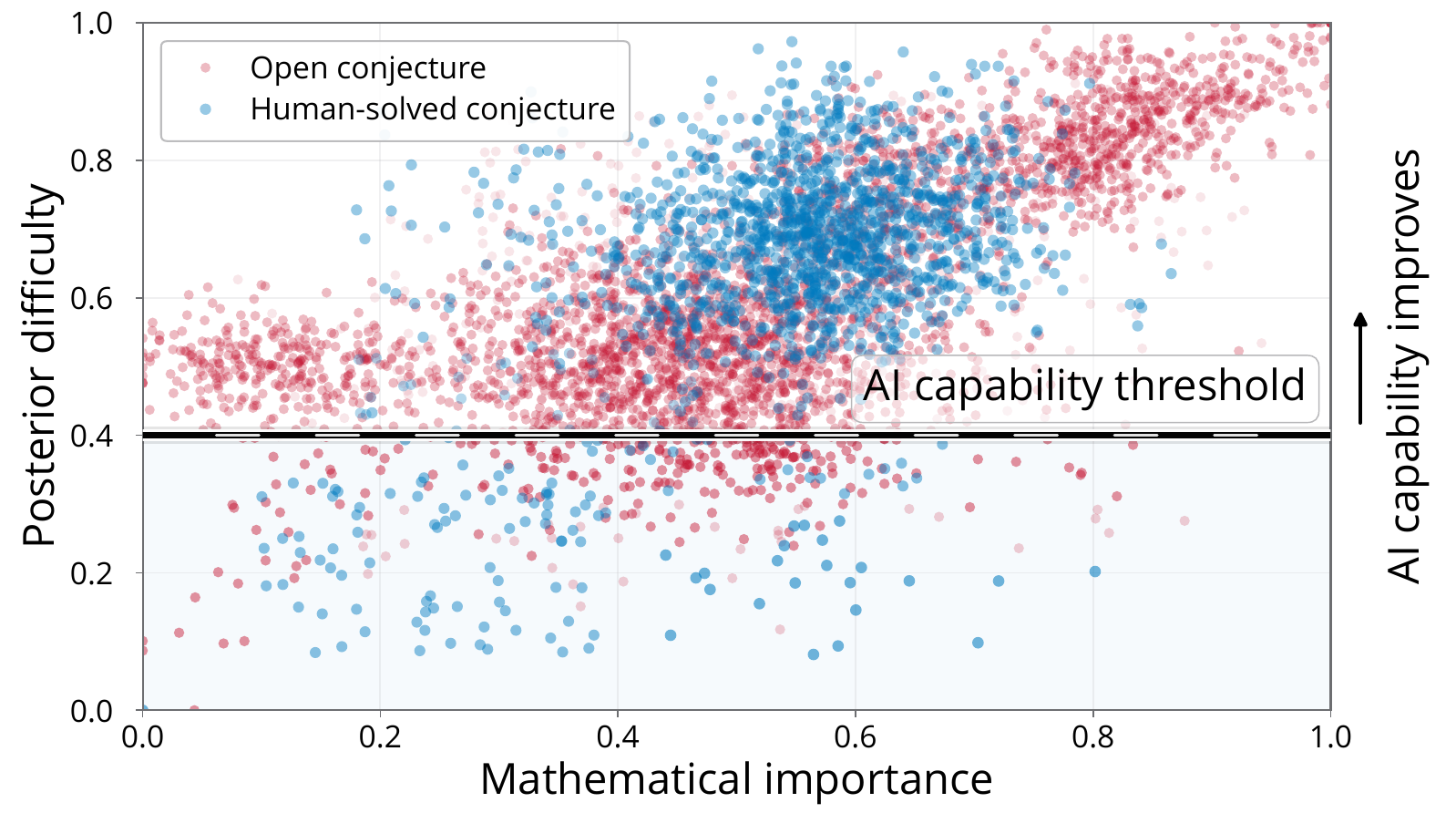}
\caption{\textbf{Schematic view of the current reachable region.} The figure is illustrative, and both axes should be read qualitatively. 
A single pass over the pool probes which conjectures can produce artifacts worth review under the current model. As model capability improves, more conjectures may become reachable.}
\label{fig:threshold-v2}
\end{figure}

The question of allocating limited attempts to a set of problems invites a simple bandit interpretation~\citep{bubeck2012regret,lattimore2019bandit}.
From this perspective, a conjecture $c\in\mathcal{P}$ is an arm, and spending reasoning and review effort on it is a pull. A pull yields an assessed outcome $Y(c)$, which may be no reliable result, a known resolution, or a candidate proof or counterexample.
In our experiments we pull each arm only once, as in the initialization stage of the UCB algorithm~\citep{auer2002finite}, and then use subsequent steps to concentrate human review on a much smaller set of outputs. Further investigating bandit algorithms for our setting is left for future work. 
Our effort allocation problem differs from the systems discussed in Section~\ref{sec:ai-for-math}, which amount to  concentrating reasoning effort on one or a few problems that are selected in advance. 
Finally, we note that effort allocation is dynamic: as models improve, conjectures that previously produced no useful progress may enter the current system's reachable region, thereby demanding different allocations of effort than the older models.
Figure~\ref{fig:threshold-v2} illustrates this reachable region.

\section{Method}
\label{sec:process}

Our method \textit{Find, Attempt, and Recommend (FAR)} instantiates the allocation view above as a literature-to-review pipeline. 
Analogous to search and recommender systems~\citep[e.g.,][]{wang2011cascade,chen2017efficient,liu2022neural}, \textit{FAR} progressively narrows a large collection through increasingly costly stages,
with later stages receiving
more compute per item.  
As illustrated in Figure~\ref{fig:process-v2}, given a literature corpus and a research direction, \textit{FAR} \emph{finds} relevant open problems and filters them into an attemptable pool $\mathcal{P}$ (the Label, Extract, and Check stages), \emph{attempts} candidate resolutions (Solve), and \emph{recommends} selected conjecture-resolution pairs for expert review (Judge).

\begin{figure}[H]
\centering
\includegraphics[width=1\textwidth]{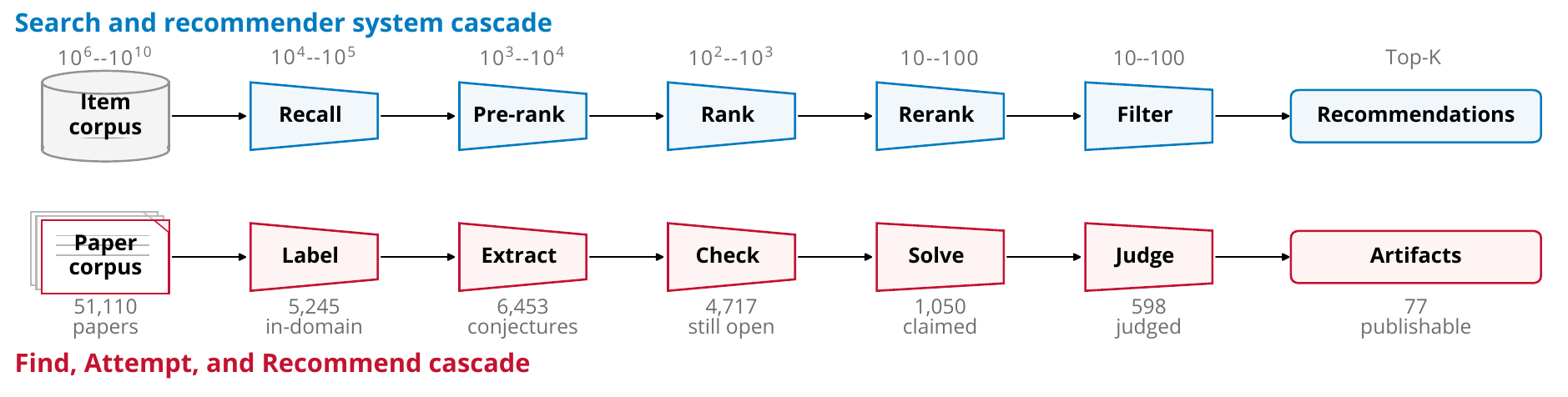}

\caption{\textbf{From papers to recommendations for expert review.} The upper row shows a search or recommender pipeline that recalls and filters candidates from a large corpus. Its numbers indicate typical orders of magnitude. The lower row shows the analogous \textit{FAR} pipeline. Numbers in the lower row are counts from our pilot run detailed in Section~\ref{sec:pilot-v2}.}
\label{fig:process-v2}
\end{figure}

\subsection{Finding Relevant Open Problems}
\label{sec:find}

The finding stage takes a literature corpus and a research direction as input and returns an attemptable pool $\mathcal{P}$ of relevant open problems. It identifies papers in the research direction,  extracts unresolved statements from them, and finally checks whether those statements are valid and still open. The corpus may be arXiv, a topic-specific collection, or another large-scale source of mathematical literature. The prompts for the three stages are given in Appendix~\ref{app:prompt-templates}.

\paragraph{Finding relevant papers via labeling.}
Before the pipeline runs, a mathematician fixes a research direction at whatever granularity they need. An agent then reads each paper, labels whether it lies in that direction, and keeps the papers that do. Labeling plays the part of recall in a retrieval cascade. It sees the entire corpus, so it runs the cheapest model in the pipeline.

\paragraph{Extracting and recovering conjectures.}
An agent extracts unresolved statements from the papers that survive labeling. A paper may yield several such statements or none. These may be labeled conjectures, questions, or open problems, or they may appear in prose. The extractor excludes future work that poses no specific mathematical question, and statements resolved within the same paper. Beyond those exclusions it is permissive, since a statement it passes over cannot be recovered later while a spurious one is dropped by the next stage. Each extracted statement becomes a candidate, recorded with its source paper. Table~\ref{tab:extraction-example} gives a representative source-to-candidate example.

\paragraph{Checking validity and status.}
An agent searches for later work on each candidate and records the supporting sources as status evidence. It labels the candidate \emph{open} when its source paper states a concrete unresolved problem and no credible resolution is found, \emph{solved} when credible evidence resolves it, and \emph{invalid} when the extracted text does not state a concrete open problem. Open candidates form $\mathcal{P}$, meaning  that the pool contains conjectures that appear relevant, well posed, and still open under the available evidence. Like the eligibility filters that the drop items that a recommender can no longer serve, this stage removes the candidates that turn out to be solved or invalid.

\begin{table}[H]
\caption{A recovered candidate, from source text to the pool.}
\label{tab:extraction-example}
\centering
\small
\begin{tabular}{p{0.22\linewidth}p{0.70\linewidth}}
\toprule
Field & Value \\
\midrule
Source paper & Ikenmeyer, Pak, and Panova, \emph{Positivity of the Symmetric Group Characters is as Hard as the Polynomial Time Hierarchy} \\
Extracted label & Conjecture 5.3.2 \\
Extracted statement & The problem \textsc{ComputeCharBinary} is $\GapP$-complete under many-one reductions. \\
Status & Open. The check found no credible resolution, and records that the completeness question is still unsettled. \\
\bottomrule
\end{tabular}
\end{table}

\subsection{Attempting for Candidate Resolutions}
\label{sec:attempt}

Every conjecture in the pool $\mathcal{P}$ receives one attempt, which is the unit of effort that this stage allocates.
An attempt may be a single agent run, a longer harness, a multi-agent workflow, or repeated sampling under a selection rule. 
Performing one attempt per conjecture simply allocates the attempt budget uniformly across problems in the pool.
As discussed in Section~\ref{sec:allocation}, this can be seen as the initialization step of a bandit algorithm, and
we study other strategies in Section \ref{sec:analysis}.

To perform an attempt, an agent is given each conjecture as extracted, together with its source paper, so that its notation is read against the text that introduced it. 
Although the checking stage already searched for whether the problem is open, the checking stage ran a weaker model.
Determining that a conjecture is equivalent to something that is already settled can take non-trivial reasoning, and hence may benefit from the stronger model used in the attempting stage. The agent therefore searches before it attempts, and labels the outcome \texttt{KNOWN} when a credible source already resolves the conjecture, \texttt{NEW} when it produces a complete proof or counterexample of its own, \texttt{FIX} when the statement as written is defective and the minimal repair it proposes cannot be settled, and \texttt{NONE} when it reaches none of these. The prompt is given in Appendix~\ref{app:prompt-templates}. Only \texttt{NEW} outcomes go on to the next stage.

The \textit{find} stages narrowed down the pool of problems, meaning that the \textit{attempt} stage can use more resources per problem.
For the attempt stage we run the most capable model in the pipeline, and it returns a set $\mathcal{Y}$ pairing every conjecture in $\mathcal{P}$ with an outcome.

\subsection{Recommendation for Expert Review}
\label{sec:human-review}

Expert review is the scarcest resource in the pipeline, and only a few of the conjecture-resolution pairs in $\mathcal{Y}$ can receive it. 
The \textit{recommend} stage decides which pairs receive expert review.
It judges whether the result is correct and whether it is significant enough to publish, similar to what a referee determines in human peer review.
The judging and grading in this stage are akin to the final filters in a search or recommender cascade. Here, a mathematician is the user that the results are served to.

\paragraph{Judging.}
One or more agents check each \texttt{NEW} outcome in $\mathcal{Y}$ for correctness, asking whether it addresses the statement that it targets, whether the argument or construction is complete, and whether every step holds. Each agent marks the outcome \texttt{PASS} or \texttt{FAIL}, and an outcome passes only if every agent passes it.

\paragraph{Recommending for review.}
A second agent takes the outcomes that passed and sorts each one as already known, as new but too minor to stand alone, or as substantial enough to publish on its own. Deciding the first requires a fresh literature search, since an existing resolution may have escaped both earlier stages. The last group forms the set $\mathcal{A}$ of artifacts. The judging and grading prompts are given in Appendix~\ref{app:prompt-templates}.

\paragraph{Expert review.}
By this point $\mathcal{A}$ is small, and everything in it lies in the direction that the mathematician set at the start. The mathematician(s) read the artifacts that interest them, check that each is correct and not already known, and write up those that hold.

\section{A Pilot Run in Combinatorics}
\label{sec:pilot-v2}

We instantiate \textit{Find, Attempt, and Recommend} in combinatorics, a domain whose results the authors have the expertise to verify. The remainder of this section reports the setup of the run, the outcomes it produced, and an analysis of those outcomes.

\subsection{Setup}

For this pilot we assemble a corpus of 51{,}110 mathematics papers from \textit{OpenAlex} metadata and source links~\citep{priem2022openalex}. The pipeline is not tied to that source, and arXiv or another large-scale collection of mathematical literature would serve as well. Labeling keeps 5{,}245 papers. Extraction recovers 6{,}453 candidates from 2{,}742 of them, and checking leaves 4{,}717 conjectures drawn from 2{,}206 papers. These form the attemptable pool $\mathcal{P}$.

The stages run different models. Labeling uses \texttt{gpt-oss-120b}, extraction \texttt{gemini-3.5-flash}, and checking \texttt{gemini-3.1-pro} with web search. Attempting, judging, and grading all use \texttt{gpt-5.5} at \texttt{xhigh} reasoning effort, and each conjecture in $\mathcal{P}$ received one attempt, instantiated here as a single run of the \texttt{opencode} agent in a working directory holding the paper and the statement. Each claimed resolution was put to three independent judges. This is the ordering Section~\ref{sec:process} describes, with progressively more capable models as the set narrows and the task becomes harder.

\subsection{Outcomes}
\label{sec:outcomes}

Every conjecture in $\mathcal{P}$ receives one attempt and one of four outcomes. Only \texttt{NEW} outcomes are judged, and only those that pass are graded. The tree below gives the count at each step.

\begin{center}
\small
\begin{tikzpicture}[inner sep=0pt]
\draw[black,line width=.7pt] (0.30,-0.14) -- (0.30,-1.69);
\draw[black,line width=.7pt] (0.30,-0.37) -- (0.82,-0.37);
\draw[black,line width=.7pt] (0.30,-0.81) -- (0.82,-0.81);
\draw[black,line width=.7pt] (0.30,-1.25) -- (0.82,-1.25);
\draw[black,line width=.7pt] (0.30,-1.69) -- (0.82,-1.69);
\draw[black,line width=.7pt] (1.75,-1.90) -- (1.75,-2.57);
\draw[black,line width=.7pt] (1.75,-2.13) -- (2.27,-2.13);
\draw[black,line width=.7pt] (1.75,-2.57) -- (2.27,-2.57);
\draw[black,line width=.7pt] (3.20,-2.78) -- (3.20,-3.89);
\draw[black,line width=.7pt] (3.20,-3.01) -- (3.72,-3.01);
\draw[black,line width=.7pt] (3.20,-3.45) -- (3.72,-3.45);
\draw[black,line width=.7pt] (3.20,-3.89) -- (3.72,-3.89);
\node[anchor=base east] at (0.95,0.00) {4{,}717};
\node[anchor=base west] at (1.10,0.00) {conjectures attempted};
\node[anchor=base west] at (9.55,0.00) {$=\mathcal{P}$};
\node[anchor=base east] at (1.95,-0.44) {2{,}905};
\node[anchor=base west] at (2.10,-0.44) {\texttt{NONE}};
\node[anchor=base west] at (3.82,-0.44) {no result};
\node[anchor=base east] at (1.95,-0.88) {443};
\node[anchor=base west] at (2.10,-0.88) {\texttt{KNOWN}};
\node[anchor=base west] at (3.82,-0.88) {already resolved in the literature};
\node[anchor=base east] at (1.95,-1.32) {319};
\node[anchor=base west] at (2.10,-1.32) {\texttt{FIX}};
\node[anchor=base west] at (3.82,-1.32) {statement defective as written};
\node[anchor=base east] at (1.95,-1.76) {1{,}050};
\node[anchor=base west] at (2.10,-1.76) {\texttt{NEW}};
\node[anchor=base west] at (3.82,-1.76) {claimed resolution};
\node[anchor=base east] at (3.13,-2.20) {452};
\node[anchor=base west] at (3.27,-2.20) {\texttt{FAIL}};
\node[anchor=base west] at (4.40,-2.20) {fails judging};
\node[anchor=base east] at (3.13,-2.64) {598};
\node[anchor=base west] at (3.27,-2.64) {\texttt{PASS}};
\node[anchor=base west] at (4.40,-2.64) {passes judging};
\node[anchor=base east] at (4.58,-3.08) {75};
\node[anchor=base west] at (4.72,-3.08) {already known};
\node[anchor=base east] at (4.58,-3.52) {446};
\node[anchor=base west] at (4.72,-3.52) {too minor to stand alone};
\node[anchor=base east] at (4.58,-3.96) {77};
\node[anchor=base west] at (4.72,-3.96) {publishable};
\node[anchor=base west] at (9.55,-2.64) {$=\mathcal{J}$};
\node[anchor=base west] at (9.55,-3.96) {$=\mathcal{A}$};
\end{tikzpicture}
\end{center}

The run produced claimed resolutions for 1{,}050 of the 4{,}717 conjectures. Judging accepted 598 of them, which we write $\mathcal{J}$, and grading left 77 of those to form $\mathcal{A}$.

The authors reviewed 15 of the artifacts in $\mathcal{A}$ that they found particularly interesting, and every one of them is mathematically correct. One, an asymptotic bound on a divisor-difference problem of Erd\H{o}s recorded in Guy's miscellany, had been settled a few months before our run by a route that none of the searches in the cascade turned up. Section~\ref{sec:results-v2} discusses several of these results, and Appendix~\ref{app:solutions} collects the write-ups.

\subsection{Analysis: Allocating Effort for Discovery}
\label{sec:analysis}

Beyond the mathematics that the run produced, the outcomes of the run let us study alternative effort allocation strategies to the uniform strategy that we used in the pilot.

To do so, we use the outcomes of each stage (namely, the conjectures $\mathcal{J}$ that ended up passing judging and the final artifacts $\mathcal{A}$) along with \textit{difficulty} and \textit{importance} scores that were collected during the run.
Concretely, the agent's prompt in the checking stage asked it to include two estimates:
\begin{itemize}[leftmargin=1.6em,itemsep=1pt,topsep=3pt,parsep=0pt]
    \item a \emph{difficulty} $d$, anchored at $0$ for a conjecture whose resolution would be an unpublishable exercise and at $1$ for one publishable in a top journal;
    \item an \emph{importance} $i$, anchored at $0$ for a statement with no substantive mathematical content and at $1$ for a Fields-Medal-level problem.
\end{itemize}
In both cases the prompt asks for the scores to follow a roughly normal distribution centered on $0.5$. The scores are fixed before any reasoning budget is spent, and the run then attempted every conjecture in $\mathcal{P}$, so the two scores can be read against outcomes that they preceded.

\subsubsection{Validity of the Agent's Difficulty and Importance Scores}

We first study the validity of the difficulty and importance scores, before using them within allocation strategies. For a group of conjectures $\mathcal{S}\subseteq\mathcal{P}$, write
\begin{itemize}[leftmargin=1.6em,itemsep=1pt,topsep=3pt,parsep=0pt]
    \item $\delta(\mathcal{S})=|\mathcal{S}\setminus\mathcal{J}|\,/\,|\mathcal{S}|$, the fraction of $\mathcal{S}$ with no accepted resolution;
    \item $\iota(\mathcal{S})=|\mathcal{A}\cap\mathcal{S}|\,/\,|\mathcal{J}\cap\mathcal{S}|$, the fraction of its accepted resolutions graded publishable.
\end{itemize}
The $\delta(S)$ metric is viewed as an empirical estimate of the difficulty, and $\iota(S)$ as an empirical estimate of the importance based on the outcomes of the run.
Therefore, we use these metrics to validate the model-estimated difficulty score $d$ and importance score $i$.
The tree in Section~\ref{sec:outcomes} gives both over the whole pool, $\delta(\mathcal{P})=4{,}119/4{,}717$ and $\iota(\mathcal{P})=77/598$.
Reading $d$ and $i$ as maps $\mathcal{P}\to[0,1]$, Figure~\ref{fig:solve-rate} plots $\delta(d^{-1}[a,b))$ in (a) and $\iota(i^{-1}[a,b))$ in (b), over the intervals $[a,b)$ marked on each axis.

\begin{figure}[H]
\centering
\begin{minipage}{0.49\textwidth}\centering
\includegraphics[width=\textwidth]{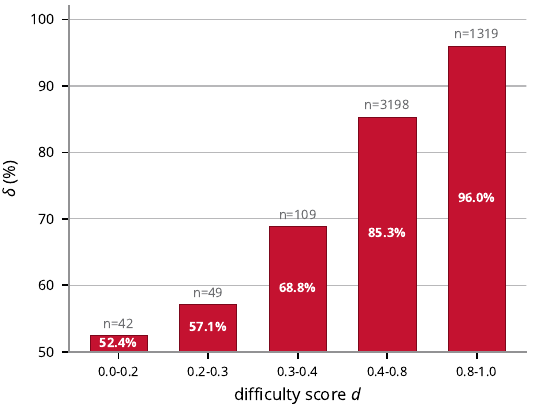}\\[-2pt]{\small (a) difficulty $d$ against $\delta$}
\end{minipage}\hfill
\begin{minipage}{0.49\textwidth}\centering
\includegraphics[width=\textwidth]{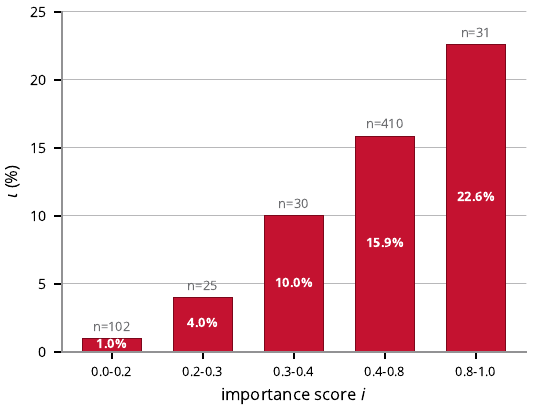}\\[-2pt]{\small (b) importance $i$ against $\iota$}
\end{minipage}
\caption{\textbf{Each score against the quantity it judges.} Panel (a) plots $\delta(d^{-1}[a,b))$ on an axis starting at $50\%$, panel (b) plots $\iota(i^{-1}[a,b))$. $n$ counts attempts in (a) and accepted resolutions in (b). Candidates that a later status recheck reclassified as solved or invalid are excluded.}
\label{fig:solve-rate}
\end{figure}

Figure~\ref{fig:solve-rate} shows that $\delta$, the fraction of attempts with no accepted resolution, and $\iota$, the publishable fraction of the accepted ones, both correlate positively with the difficulty and importance scores the agent assigned during the checking stage. We measure each association by the area under the ROC curve (AUC)~\citep{hanley1982meaning}, the probability that the score ranks a randomly chosen conjecture that has the outcome above a randomly chosen conjecture that does not, ties counting half. Difficulty achieves an AUC of $0.69$ against having no accepted resolution, and importance $0.60$ against being graded publishable among the accepted. Appendix~\ref{app:scores} repeats the figure with intervals attached and gives these statistics in full. A Mann-Whitney test~\citep{mann1947test} puts the first at $p<10^{-40}$ and the second at $p=0.008$, both showing strong positive correlation. We therefore use them below in analyzing allocation strategies.

We also note that the two scores are not independent of each other. Their Spearman rank correlation is $0.83$, and the figure below illustrates that dependence. Over 60\% of the conjectures lie strictly above the diagonal, while almost nothing falls strictly below. In other words, a problem that the agent calls important is almost always one that it also calls hard. This reflects a selection effect, since the pool holds only problems that are still open, and an important question stays open only while it remains hard.

\begin{figure}[H]
\centering
\includegraphics[width=0.5\textwidth]{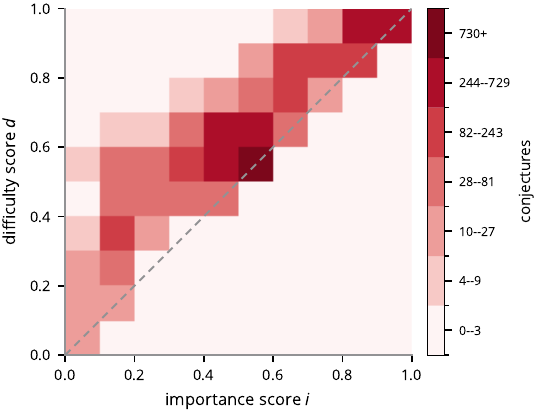}
\end{figure}
\FloatBarrier

The two scores do not carry the same information, however. We demonstrate this by comparing each conjecture only against others that received the same importance score. For each value $v$ of $i$, write $\mathrm{AUC}_v$ for the AUC of $d$ within the conjectures scored $v$, and $n_v$ for the number of pairs it compares. The stratified AUC is then,
\begin{equation*}
\mathrm{AUC}_{\mathrm{strat}}=\frac{\sum_{v} n_v\,\mathrm{AUC}_v}{\sum_{v} n_v}=0.56,\qquad p<10^{-5}.
\end{equation*}
Therefore, among conjectures the LLM scored equally important, the difficulty score still effectively separates the attempts that returned an accepted resolution from those that did not.

\subsubsection{Where to Spend a Limited Budget}

Section~\ref{sec:allocation} posed the question: once agents are a primary source of mathematical attempts, how should their compute be allocated? We take it up here in a simple setting. Suppose the attempting stage of Section~\ref{sec:attempt} is given a budget of $B$ attempts, and each conjecture receives at most one attempt. 
Writing $\mathcal{S}$ for the set of conjectures that the stage attempts and $f$ for the value of the artifacts that it produces, the stage should maximize the following:
\begin{equation}
\begin{aligned}
\text{maximize}_{\;\mathcal{S}\subseteq\mathcal{P}}\quad & \mathbb{E}\big[\,f(\mathcal{A}\cap\mathcal{S})\,\big] \\
\text{subject to}\quad & |\mathcal{S}|=B,
\end{aligned}
\label{eq:allocation}
\end{equation}
where the expectation is over the outcomes of the attempts. Three choices of $f$ are natural:
\begin{itemize}[leftmargin=1.6em,itemsep=2pt,topsep=3pt,parsep=0pt]
    \item $f_1=|\mathcal{A}\cap\mathcal{S}|$: the total number of publishable artifacts.
    \item $f_2=\sum_{c\in\mathcal{A}\cap\mathcal{S}}i(c)$: the total importance of those artifacts.
    \item $f_3=\max_{c\in\mathcal{A}\cap\mathcal{S}}i(c)$: the importance of the single best artifact.
\end{itemize}

\paragraph{Optimal strategies.}
Consider the outcome of a single attempt on conjecture $c$ as random. Then $\mathbb{E}\,\delta(c)$ is the probability that it returns no accepted resolution, $\mathbb{E}\,\iota(c)$ the probability that a resolution it does return is graded publishable, and $p(c)=\big(1-\mathbb{E}\,\delta(c)\big)\,\mathbb{E}\,\iota(c)$ the probability that $c$ ends in $\mathcal{A}$. If $\mathbb{E}\,\delta(c)$ and $\mathbb{E}\,\iota(c)$ are known in advance, we can solve the optimization problem~(\plaineqref{eq:allocation}) exactly for $f_1$ and $f_2$, and within a constant factor for $f_3$.

For $f_1$ and $f_2$, by linearity of expectation:
\[
\mathbb{E}\,[f_1(\mathcal{A}\cap\mathcal{S})]=\sum_{c\in\mathcal{S}}p(c),
\qquad
\mathbb{E}\,[f_2(\mathcal{A}\cap\mathcal{S})]=\sum_{c\in\mathcal{S}}i(c)\,p(c).
\]
Sorting $\mathcal{P}$ by $p(c)$, respectively by $i(c)\,p(c)$, and keeping the first $B$ therefore solves~(\plaineqref{eq:allocation}).

The argument for $f_3$ runs through the classical problem of monotone submodular maximization~\citep{krause2014submodular}. Formally, a set function $F$ on subsets of $\mathcal{P}$ is \emph{submodular} if $F(\mathcal{S})+F(\mathcal{T})\geq F(\mathcal{S}\cup\mathcal{T})+F(\mathcal{S}\cap\mathcal{T})$ for all $\mathcal{S},\mathcal{T}\subseteq\mathcal{P}$, and \emph{monotone} if $F(\mathcal{S})\leq F(\mathcal{T})$ whenever $\mathcal{S}\subseteq\mathcal{T}$.

Maximizing a monotone submodular $F$ under a cardinality constraint contains maximum coverage as a special case and is therefore NP-hard. It is known, however, that a simple greedy algorithm achieves at least a $(1-1/e)$ fraction of the optimum~\citep{nemhauser1978analysis}, and that no polynomial algorithm beats it unless P$\,=\,$NP~\citep{feige1998threshold}. The algorithm starts with $\mathcal{S}_0=\emptyset$, and at each iteration $j$ it adds the element of largest gain $F(c\mid\mathcal{S}_{j-1})=F(\mathcal{S}_{j-1}\cup\{c\})-F(\mathcal{S}_{j-1})$, so that
\[
\mathcal{S}_j=\mathcal{S}_{j-1}\cup\Big\{\operatorname*{arg\,max}_{c\in\mathcal{P}\setminus\mathcal{S}_{j-1}}F(c\mid\mathcal{S}_{j-1})\Big\},
\qquad j=1,\dots,B.
\]

We show that $\mathcal{S}\mapsto\mathbb{E}\,[f_3(\mathcal{A}\cap\mathcal{S})]$ is monotone submodular, which is what the guarantee above requires. The map $f_3$ is monotone by definition. For submodularity, suppose without loss of generality that $f_3(\mathcal{U})\geq f_3(\mathcal{V})$; then $f_3(\mathcal{U}\cup\mathcal{V})=f_3(\mathcal{U})$ and $f_3(\mathcal{U}\cap\mathcal{V})\leq f_3(\mathcal{V})$, so $f_3(\mathcal{U})+f_3(\mathcal{V})\geq f_3(\mathcal{U}\cup\mathcal{V})+f_3(\mathcal{U}\cap\mathcal{V})$. The map $\mathcal{S}\mapsto\mathcal{A}\cap\mathcal{S}$ preserves unions and intersections, so it carries both properties over to $\mathcal{S}\mapsto f_3(\mathcal{A}\cap\mathcal{S})$. Taking expectations then gives the same two properties for $\mathcal{S}\mapsto\mathbb{E}\,[f_3(\mathcal{A}\cap\mathcal{S})]$.

\paragraph{Building a strategy from the scores.}
In practice, neither $\mathbb{E}\,\delta(c)$ nor $\mathbb{E}\,\iota(c)$ are known. Before any reasoning is spent, the only signals available are the model-based difficulty and importance scores that the pool building stage recorded. By the arguments above, if we can approximate $p$ from these scores then we obtain near-optimal strategies for $f_1$ and $f_2$. For $f_3$ an approximation of $p$ is not enough, since each greedy step needs the value of $F(c\mid\mathcal{S})$, which depends on the joint distribution of the attempt outcomes.
Hence we propose an alternative algorithm for $f_3$ that, like those for $f_1$ and $f_2$, needs only an approximation of $p$ and does well in our run.

We fit the two factors of $p$ by least squares, $1-\delta$ on $d$ over $\mathcal{P}$ and $\iota$ on $i$ over $\mathcal{J}$, and write $\hat{\delta}$ and $\hat{\iota}$ for the resulting estimates of $\mathbb{E}\,\delta$ and $\mathbb{E}\,\iota$. We write $\hat{\delta}$ and $\hat{\iota}$ for the two estimates of $\mathbb{E}\,\delta$ and $\mathbb{E}\,\iota$. Their product $\hat{p}=(1-\hat{\delta})\,\hat{\iota}$ estimates $p$. In theory, ranking $\mathcal{P}$ on $\hat{p}$ and taking the top $B$ is near-optimal for $f_1$, and ranking on $i\,\hat{p}$ is near-optimal for $f_2$. 
For $f_3$, we 
choose a top fraction of $\mathcal{P}$ based on importance, and then
rank on $\hat{p}$.
The least squares fits underlying these strategies do not require the whole pool,
so a round may attempt a small part of the pool first, fit and compare strategies on the returned outcomes, and allocate the remaining budget with the strategy that performed best.

\paragraph{Empirical results.}

We compare how these strategies perform on the outcomes that our pilot run produced. 
As the baseline strategy, we take $\mathcal{S}$ to be a uniform random subset of $\mathcal{P}$ of size $B$. 
Against it we compare ranking the whole of $\mathcal{P}$ on $\hat{p}$, ranking it on $i\,\hat{p}$, and ranking on $\hat{p}$ inside a top fraction of $\mathcal{P}$ by importance (inside such a restriction $i$ varies too little to reorder anything, so ranking there on $i\,\hat{p}$ gives an almost same strategy). Figure~\ref{fig:budget} shows the baseline, the two rankings, and the two restrictions that returned the most artifacts.

Each strategy is evaluated by $5$-fold cross-validation. We split $\mathcal{P}$ into five parts at random, fit both factors on four of them, rank the fifth by the resulting $\hat{p}$ and keep its top $B/5$, and combine the five selections, so that no conjecture is ever ranked by a fit that saw its own outcome. Each point of the figure averages $1000$ such partitions, at $B=10$, $25$, $50$, $75$, $100$, $200$ and $300$. Appendix~\ref{app:budget} tabulates every plotted point.

\begin{figure}[H]
\centering
\includegraphics[width=\textwidth]{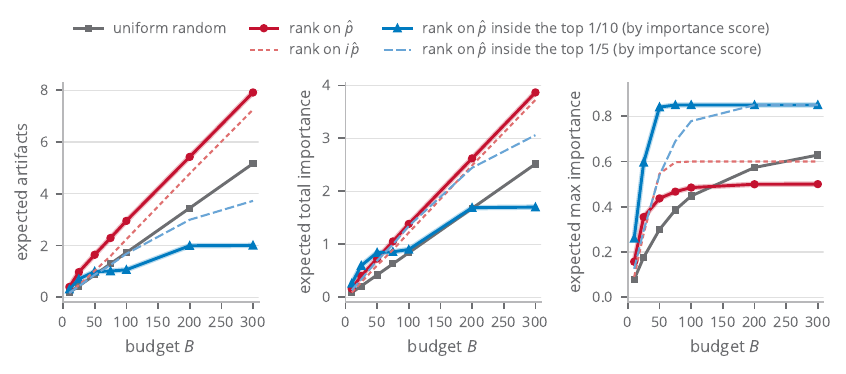}
\caption{\textbf{Allocation strategies against the budget.} Each fit is made on four fifths of $\mathcal{P}$ and applied to the remaining fifth, from which $B/5$ conjectures are drawn. The five selections together make one set of $B$ conjectures, and each point averages what that set returns over $1000$ random partitions. Ties are broken at random, and the uniform baseline is computed exactly from its closed form. Here, $B$ only counts the allocated attempts.}
\label{fig:budget}
\end{figure}
\FloatBarrier

On $f_1$, the run follows the theoretical analysis closely. Ranking $\mathcal{P}$ on $\hat{p}$ outperforms every other strategy at every budget. On $f_2$, the analysis prescribes ranking on $i\,\hat{p}$, but ranking on $\hat{p}$ instead beats it narrowly at every budget. On $f_3$, neither ranking of the whole pool does well, and at the larger budgets both even fall behind the uniform random baseline. What works best in practice is to discard all but the most important part of $\mathcal{P}$ and rank what remains on $\hat{p}$. Of the fractions we tried, keeping the most important $1/10$ did best on this run. Note that the results above depend on the scoring model, on the sources that the pool was built from, and on the models in the cascade.
Nevertheless, the results suggest that using strategies derived from model-based importance and difficulty scores, it is possible to allocate the budget more effectively than with a uniform strategy in our problem setting.
Furthermore, the optimal allocation strategy differs between $f_1$ and $f_3$, providing evidence that the best strategy can depend on the objective that the mathematician is optimizing for. 

\section{Selected Reviewed Results}
\label{sec:results-v2}

We manually reviewed fifteen of the artifacts from the pilot run, chosen by our own interest, and found no mathematical error in any of them. We discuss some of them below, which show three different degrees of novelty. Full write-ups of these results and of the others we reviewed are collected in Appendix~\ref{app:solutions}.

\subsection{A Known Result Graded as New}

\paragraph{Sets with no large divisor difference (Appendix~\ref{sol:erdos}).}
Erd\H{o}s~\citep{guy1983unsolved} asks how large a set of integers can be if no two of its elements have a large difference dividing the larger one. Formally, for $t\ge1$, let $F(n;t)$ be the largest size of a set $A\subseteq\{1,\dots,n\}$ in which no two elements $x<y$ satisfy both $(y-x)\mid y$ and $y-x\ge t$. The question is whether $F(n;t)\le(\tfrac12+o(1))n$ for every fixed $t$.

The artifact returned for this problem proves that $F(n;t)/n\to\tfrac12$. The odd numbers give the lower bound, since the difference of two odd numbers is even and cannot divide the larger. For the upper bound it fixes a finite set $P$ of odd primes exceeding $t$ with $\sum_{p\in P}1/p$ large. For each $p\in P$ the set cannot hold both $2rp$ and $(2r-1)p$, whose difference is $p\ge t$ and divides $2rp$, so the even elements of $A$ are divisible by no more primes of $P$, in total, than the odd integers outside $A$ are. A second-moment estimate on each parity class turns this into the matching upper bound.

The recommend stage judged the proof correct and graded it as a publishable result. Reviewing it ourselves we found the mathematics correct but the result already known. Four months before our run a proof of the same statement had been recorded on the Erd\H{o}s problems site, obtained by Liam Price with ChatGPT-5.2~\citep{erdosproblems635}, and Tao also observed there that the bound follows quickly from an inequality of \citet{elliott2012probabilistic}. The cascade never found this record.

\subsection{A Connection Not Previously Made}

\paragraph{Small unions of lines closing a route to the Nikodym bound (Appendix~\ref{sol:finkakeya}).}
For a set $L$ of affine lines in $\mathbb{F}_q^3$ let $P(L)=\bigcup_{\ell\in L}\ell$ be the union of its lines. \citet{lund2018finite} conjecture that for every constant $C>0$ and every $\alpha$ with $\alpha(q)/q\to\infty$, a set $L$ of at least $Cq^3$ lines in which no plane contains $\alpha(q)$ lines must satisfy $|P(L)|\ge(1-o(1))q^3$. They show that this would give an optimal bound on the size of a Nikodym set in three dimensions, and \citet{tao2025new} still cites it as open in that role.

The artifact returned for this problem disproves the conjecture for every odd $q$. It takes the affine paraboloid $z=x^2-\nu y^2$ with $\nu$ a nonsquare, and keeps at each of its points the $(q+1)/2$ tangent lines on which $x^2-\nu y^2-z$ is always a square. The family has $q^2(q+1)/2$ lines, no affine plane holds more than $q+1$ of them, and its union has exactly $q^2(q+1)/2$ points, a density of $1/2+o(1)$. The same family also refutes the stronger Conjecture 1.5 of that paper.

We found the construction mathematically correct, and verified the counts exhaustively for $q\le13$. The recommend stage, however, informed us that the construction already exists in other contexts. In projective language the family is a classical object of finite geometry, the half-tangent partition of an elliptic quadric~\citep{bruen1999construction,cossidente2017new}. The contribution here is to link the existing construction to this conjecture.

\subsection{Results with No Precedent Found}

The three results below have each been checked by an author or by a domain expert, and are new so far as we could determine. They represent three kinds of discovery. The first proves a conjecture, the second refutes one by counterexample, and the third answers an open-ended question.

\paragraph{Many-one $\GapP$-completeness for binary symmetric group characters (Appendix~\ref{sol:pak}).}
For partitions $\lambda$ and $\mu$ of an integer $n$, write $\chi^\lambda(\mu)$ for the irreducible character value of the symmetric group. \citet{ikenmeyer2024positivity} consider the problem of computing $\chi^\lambda(\mu)$ from $\lambda$ and $\mu$ given as lists of parts in binary. They prove that it is $\GapP$-complete under Turing reductions and conjecture that it is $\GapP$-complete under many-one reductions.

The artifact returned for it proves that conjecture. It reduces the difference of two counts of exact covers of a finite set to a single character value at a two-row partition $\lambda=(n-s,s)$, and checks membership in $\GapP$ separately. A two-row character value is itself a difference, $N_\mu(s)-N_\mu(s-1)$, where $N_\mu(t)$ is the number of ways to choose parts of $\mu$ summing to $t$.

\paragraph{Maximum versus average independent set size in triangle-free graphs (Appendix~\ref{sol:perkins_gang}).}
For a graph $G$ let $\alpha(G)$ be its independence number and $\bar\alpha(G)$ the expected size of a uniformly random independent set. \citet{davies2018average} conjecture that $\alpha(G)/\bar\alpha(G)\ge2-o_d(1)$ for every triangle-free $G$ of minimum degree $d$, and show that this would give $R(3,k)\le(\tfrac12+o(1))k^2/\log k$. It is restated as open in recent surveys of the hard-core model and Ramsey theory~\citep{davies2025hard,morris2026some}.

The artifact returned as a counterexample the Cartesian product $C_5\,\square\,K_{m,m}$, which is triangle-free and $(m+2)$-regular. An independent set of the product picks an independent subset of $C_5$ at each vertex of $K_{m,m}$, and the two sides must pick disjoint subsets. This makes $\alpha=4m$ and the expected size computable exactly, and the ratio tends to $24/13$. Replacing $C_5$ by the circulant $C_{13}(1,5)$ further lowers the limit to $32/19$.

\paragraph{Divisibility among binomial coefficients (Appendix~\ref{sol:erdos_straus}).}
For a fixed integer $n\ge2$, \citet{erdos1977products} ask for the natural density $d^*(n)$ of those $m$ that admit some $k$ with $1\le k\le m-n$ and $\binom{n+k}{n}\mid\binom{m+k}{k}$, that is, the proportion of such $m$ among the integers up to $x$ as $x\to\infty$. They settle $n=1$ themselves and record that $n=2$ ``seems much more difficult to decide''.

The returned artifact shows that $d^*(n)=1$ for every fixed $n\ge2$. Given a bound $B>n$ it takes $k_B=\prod_{p\le B}p^{e_p}$ with $e_p$ least such that $p^{e_p}>n$, so that by Kummer's theorem every prime factor of $N_B=\binom{n+k_B}{n}$ exceeds $B$. Each such prime $q$ therefore exceeds $n$, and so divides exactly one of $k_B+1,\dots,k_B+n$, say $k_B+i(q)$. By Legendre's formula the condition $m\bmod q\ge i(q)$ makes $\binom{m+k_B}{k_B}$ divisible by the full power of $q$ that divides $N_B$. The Chinese remainder theorem turns these residue conditions into a set of density $\prod_q(1-i(q)/q)$ that tends to $1$ as $B\to\infty$.

\section{Conclusion}
\label{sec:conclusion-v2}

We have argued that an important--yet understudied--aspect of AI for mathematics  concerns selecting which problems to solve.
We introduced \textit{FAR}, a literature-to-review cascade that recovers open problems from a corpus, attempts each of them, and recommends promising research outcomes for expert review. In a combinatorics pilot, \textit{FAR} recovered 6{,}453 candidate statements, checked 4{,}717 of them into an attemptable pool, and returned 77 artifacts graded as substantial enough to publish, 15 of which the authors reviewed and found to be correct. 
Using the outcomes of our run, we studied strategies for allocating effort, i.e., selecting a subset of conjectures to attempt given a budget.
We framed this problem as constrained optimization and derived strategies that are favorable both theoretically and in practice.
We showed that simply asking an agent in our pipeline to output difficulty and importance scores  before any reasoning is spent yields scores that can be used effectively within the allocation strategies.
%
We hope that \textit{FAR} is a starting point for new techniques that help  explore mathematics, which consists of not just solving problems, but of deciding which problems are worth spending effort on in the first place.

\section*{Acknowledgments}

We thank Sylvester W. Zhang for his help with expert review. Shengtong Zhang thanks Anysphere co. for providing compute for this project. This work was supported in part by NSF Grant DMS-2434614 and DARPA ExpMath Grant HR0011262E028.

\bibliography{iclr2027_conference}
\bibliographystyle{iclr2027_conference}

\renewcommand{\bibsection}{\paragraph{References.}}

\appendix

\clearpage
\section{Prompt Templates and Operational Details}
\label{app:prompt-templates}

This appendix includes the prompts used in our implementation. The released code also carries the schema validators and the retry logic.

\subsection{Finding Relevant Open Problems}
In the pilot the direction was ``combinatorics''.

\begin{prompt}{Label}
Return one JSON object with this schema:
{
  "comment": "...",
  "in_direction": false
}

Rules:
- `comment` must name the paper's primary subject in a few words.
- `in_direction` must be a JSON boolean.
- Use `true` when the paper's primary content lies in the research direction.
- Use `false` when it does not, when the content is not mathematical, or when the paper appears mislabeled.

Research direction: {direction}

Paper content:
{text}
\end{prompt}

\begin{prompt}{Extract}
Return one JSON object with this schema:
{
  "title": "...",
  "authors": ["...", "..."],
  "decision_basis": "...",
  "has_open_conjecture": false,
  "conjectures": [
    {
      "conjecture_label": "...",
      "conjecture_text": "...",
      "conjecture_section": "..."
    }
  ]
}

Rules:
- `title` must be a non-empty string.
- `authors` must be a JSON array of non-empty author-name strings.
- `decision_basis` must be one short English sentence.
- `has_open_conjecture` must be a JSON boolean.
- `conjectures` must be a JSON array. If `has_open_conjecture` is false, it must be `[]`.
- Set `has_open_conjecture` to true iff the paper contains at least one explicit unresolved mathematical statement.
- Count these as hits:
  1. labeled `Conjecture` / `Question` / `Open Problem`
  2. sentences with markers like `open question`, `open problem`, `open issue`, `remains unknown whether`, or `we suspect ... although we have been unable to establish ...`
  3. a direct statement that a specific mathematical property, existence claim, or classification problem `still remains an open issue`
- Do NOT count:
  1. generic future work that does not pose a specific mathematical question
  2. results that have already been proved or resolved within the paper itself
- If a sentence says a specific claim or property is `still an open issue`, count it even if it is not written as a formal question.
- If `has_open_conjecture` is true, extract only the explicit unresolved statements themselves, not nearby speculation.
- `conjecture_label` should use the paper's label when present, otherwise use a short fallback like `Unlabeled open problem 1`.
- `conjecture_text` should copy the paper's unresolved statement as faithfully as possible and preserve notation.
- `conjecture_section` should be the visible section/subsection title, or `""` if unavailable.

Paper content:
{text}
\end{prompt}

\begin{prompt}{Check}
Return one JSON object with this schema:
{
  "sources": [
    {"title": "...", "url": "...", "claim": "..."}
  ],
  "reason": "...",
  "status": "solved",
  "importance": 0.5,
  "difficulty": 0.5
}

Rules:
- Verify the candidate's current status using current web information.
- `status` must be one of: `open`, `solved`, `invalid`.
- Use `open` when the candidate is a concrete open problem in the source and no credible solved evidence is found.
- Use `solved` when a credible source appears to solve it.
- Use `invalid` when it is not a concrete open problem in the source.
- `sources` should list only sources directly supporting the status.
- each `claim` must be what that source says about this candidate.
- for `solved`, `sources` must name at least one source that resolves the candidate.
- `reason` must be one concise English sentence.
- `importance` must be a number in [0, 1] for the candidate itself: candidates with no substantive mathematical content should be scored 0; Fields-Medal-level problems should be scored 1; most ordinary research problems should follow a roughly normal distribution centered around 0.5.
- `difficulty` must be a number in [0, 1]: solving it would be an unpublishable exercise should be scored 0; solving it would be publishable in a top journal (Annals, Inventiones, JAMS, Acta) should be scored 1; most problems should follow a roughly normal distribution centered around 0.5.
- For `solved` or `invalid`, set `importance` and `difficulty` to 0.

Paper title: {title}
Paper authors: {authors}
Candidate label: {conjecture_label}
Candidate section: {conjecture_section}
Candidate text:
{conjecture_text}
\end{prompt}

\subsection{Attempting for Candidate Resolutions}

\begin{prompt}{Solve, system prompt}
You are a research-level mathematical reasoner. This is a test to see how well you can craft non-trivial, novel and creative proofs given a math problem.

Given a natural-language problem, conjecture, or paper metadata, reconstruct the most likely formal mathematical statement and resolve it.

First, state the reconstructed conjecture precisely, including all hypotheses, definitions, notation, quantifiers, ambient category, and axiom system when relevant. Explain briefly what information supports this reconstruction. If the reconstruction is ambiguous, list the plausible formalizations and choose one to analyze, explicitly noting the ambiguity.

Do not treat the fact that the source labels the statement open, conjectural, unresolved, or a problem as a reason to stop. The task is to attack the statement mathematically. However, do not lower the standard of proof. Never present an incomplete, heuristic, or speculative argument as a complete proof.

Before committing to a proof, test the statement against degenerate, extremal, low-dimensional, finite, infinite, and standard model examples appropriate to the field. Look actively for counterexamples as well as proofs.

If the literal statement is false because of a degenerate, boundary, vacuous, or typo-like case, do not stop after giving the counterexample. Instead:

- State the literal counterexample clearly and explain why it falsifies the literal statement.
- Diagnose whether the failure appears to come from a small formulation defect, such as a missing nonzero/nonempty/nontrivial assumption, a wrong inequality direction, an omitted endpoint condition, a missing connectedness or finiteness hypothesis, a confusion between strict and non-strict inequalities, a missing regularity condition, or a convention mismatch.
- Propose the minimal natural repair or repairs to the statement, using the fewest and most standard changes consistent with the paper’s terminology, surrounding context, and apparent mathematical intent.
- Check that the proposed repair is not merely ad hoc, vacuous, or so weakened that it no longer captures the intended conjecture.
- Retest the repaired statement against the original counterexample and nearby degenerate cases.
- Then prove or refute the most plausible repaired statement.

A complete answer must be a rigorous proof or a rigorous counterexample.

Present the reasoning in a locally checkable form: definitions, lemmas, propositions, and proofs. For every invoked theorem, verify its hypotheses in the present setting. Track dependencies of constants, choices, witnesses, bases, subsequences, exceptional sets, embeddings, isomorphisms, and parameters.

If the proof or counterexample is known in the literature, state that honestly and provide a reliable reference. Distinguish exact resolutions from stronger theorems, weaker partial results, equivalent reformulations, and merely related work. Do not invent references.

After the proof or counterexample, include a verification audit confirming that the formalized statement matches the reconstructed conjecture, that no extra assumptions were introduced, that all theorem hypotheses were checked, and that the conclusion exactly matches the target statement.

Response format:

The first line must be exactly one of: KNOWN, NEW, FIX, NONE.
- KNOWN: a reliable existing source in the literature already proves the conjecture or gives a counterexample/disproof. Cite the source.
- NEW: your answer gives a complete resolution that is not presented as known literature. Use NEW for either a complete proof that the conjecture is true or a complete counterexample/disproof that the conjecture is false.
- FIX: you have identified a small formulation defect and proposed a minimal natural repair, but you are unable to prove or refute the repaired statement. Use FIX to indicate that you have done this.
- NONE: you found neither a known resolution nor a reliable complete proof/counterexample despite all efforts.
Then use these sections exactly:
Problem:
Result:
Citation:
\end{prompt}

\begin{prompt}{Solve, user prompt}
Read input.json in the current directory. It contains the paper title, authors, paper text, the sources a status check turned up, and target conjecture. The target conjecture is in conjecture.text.
Resolve that target conjecture and return only the required labeled answer.
\end{prompt}

\subsection{Judging and Grading}

\texttt{KNOWN} and \texttt{NEW} outcomes are sent to the judge, and its label routes the outcome rather than certifying it. Only the results it accepts as new reach the grader.

\begin{prompt}{Judge, system prompt}
You are a strict referee for natural-language mathematics proofs.  This is a test to see how well you can referee a proposed natural-language mathematics proof given a math problem.

Check the claimed resolution or disproof against the target conjecture supplied in the user task. 

Accept only if the claimed resolution or disproof attacks the correct statement and is mathematically rigorous and complete.
A valid counterexample or disproof may pass if it rigorously disproves the conjecture.
Reject if it has fatal proof gaps, hallucinated dependencies, hidden assumptions, or a mismatch between the stated theorem and the original conjecture.
Do not reject merely because the original paper called the conjecture open.
In the case when the claimed resolution or disproof is NEW, you should also conduct a very thorough literature search using the web search tool to see if a similar or stronger result already exists in the literature.

On the first line, write exactly one word: PASS or FAIL or KNOWN.
- PASS: the claimed resolution is mathematically complete and attacks the correct statement, and in the case of NEW, a similar or stronger result does not exist in the literature despite your best search efforts.
- KNOWN: the claimed resolution is NEW, but a similar or stronger result already exists in the literature.
- FAIL: if neither of the above conditions are met.
Then briefly explain your verdict, including the most important gap if you fail it.
\end{prompt}

\begin{prompt}{Judge, user prompt}
Read input.json and solution.md in the current directory. input.json contains paper metadata, the paper text, the sources a status check turned up, and the target conjecture. solution.md contains the claimed resolution to check.
Return only PASS or FAIL or KNOWN followed by your explanation.
\end{prompt}

\begin{prompt}{Grade, system prompt}
You are a senior combinatorics referee performing a final quality-control pass on a result that a prover produced and a judge already accepted as a correct resolution.

Your job is NOT to re-verify correctness from scratch (assume the proof is correct unless a literature search clearly contradicts it). Your job is to classify the result by its novelty and publishable significance, so a human can triage it afterwards.

Do two things:
1. Literature check. 
Conduct a very thorough web search to determine whether the resolution, or a similar or stronger statement, is already known in the literature. Go beyond just searching for papers that cite the original paper; you should search for all open-access notes, surveys, forums, and other sources that might contain the result. The prover and earlier judges may have missed an existing reference; catching such cases is a primary goal of this pass.
2. Significance grading. If the result is genuinely not in the literature, assess how significant it is as a contribution to combinatorics: how hard, how novel, how interesting to the community, and what venue it would plausibly merit.

On the first line, write exactly one token: KNOWN, TYPE1, TYPE2, or TYPE3.
- KNOWN: the result (or a similar or stronger result) is in fact already known in the literature, despite the prover and earlier judges treating it as new. Cite the reference.
- TYPE1: genuinely new but minor and unpublishable on its own (e.g. a routine exercise, a trivial special case, an immediate corollary of standard results).
- TYPE2: genuinely new and substantial enough to support a standalone paper in a standard combinatorics or mathematics journal.
- TYPE3: genuinely new and strong enough to merit publication in a top combinatorics journal (a major advance, a resolved well-known conjecture, or a result of broad interest).

These boundaries are deliberately rough; when uncertain between two grades, pick the lower one and explain the uncertainty.

After the first line, use these sections exactly:
Classification rationale:
Literature check:
Citation:
\end{prompt}

\begin{prompt}{Grade, user prompt}
Read input.json, solution.md, and judge.md in the current directory. input.json contains the paper metadata, the paper text, the sources a status check turned up, and the target conjecture, solution.md contains the resolution that was accepted as new, and judge.md contains the verdicts of the earlier judges.
Classify the result and return only KNOWN, TYPE1, TYPE2, or TYPE3 on the first line, followed by the required sections.
\end{prompt}

The artifacts put forward for expert review are the \texttt{TYPE2} and \texttt{TYPE3} items.

\clearpage
\section{Analysis Details}
\label{app:analysis}

This appendix gives the data behind the two figures of Section~\ref{sec:analysis}.

\subsection{Score validity}
\label{app:scores}

Figure~\ref{fig:score-validity} gives the same two associations as Figure~\ref{fig:solve-rate} at a uniform bin width of $0.1$. Each point is the measured rate in its bin and each bar its $95\%$ Wilson interval.

\begin{figure}[H]
\centering
\includegraphics[width=\textwidth]{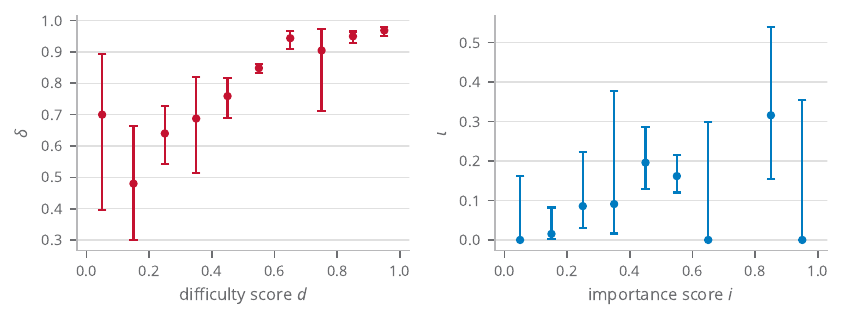}
\caption{\textbf{Each score against the quantity it judges, with Wilson intervals.} Panel (a) gives $\delta$ against the difficulty score, panel (b) gives $\iota$ against the importance score.}
\label{fig:score-validity}
\end{figure}

\subsection{Allocation curves}
\label{app:budget}

The tables below give the value of each point plotted in Figure~\ref{fig:budget}. Every entry is an average over $1000$ random partitions of the pool into five parts, with both factors of $\hat p$ fitted on four of them and the remaining part ranked by the result, from which $B/5$ conjectures are taken. The five selections together form the set of $B$ conjectures that the entry scores. Ties are broken at random, and the uniform baseline is evaluated from its closed form rather than sampled.

\begin{table}[H]
\centering\small\setlength{\tabcolsep}{3.5pt}
\begin{tabular}{lrrrrrrr}
\toprule
strategy & $B=10$ & $B=25$ & $B=50$ & $B=75$ & $B=100$ & $B=200$ & $B=300$ \\
\midrule
uniform random & 0.17 & 0.43 & 0.86 & 1.29 & 1.72 & 3.44 & 5.17 \\
rank on $\hat p$ & 0.40 & 0.97 & 1.64 & 2.29 & 2.95 & 5.42 & 7.92 \\
rank on $i\,\hat p$ & 0.15 & 0.47 & 1.02 & 1.59 & 2.24 & 4.78 & 7.25 \\
rank on $\hat p$ inside the top $1/10$ & 0.30 & 0.70 & 0.99 & 1.01 & 1.06 & 1.99 & 2.00 \\
rank on $\hat p$ inside the top $1/5$ & 0.17 & 0.42 & 0.86 & 1.25 & 1.68 & 2.99 & 3.72 \\
\bottomrule
\end{tabular}
\caption{Expected number of artifacts, the objective $f_1$.}
\end{table}

\begin{table}[H]
\centering\small\setlength{\tabcolsep}{3.5pt}
\begin{tabular}{lrrrrrrr}
\toprule
strategy & $B=10$ & $B=25$ & $B=50$ & $B=75$ & $B=100$ & $B=200$ & $B=300$ \\
\midrule
uniform random & 0.08 & 0.21 & 0.42 & 0.63 & 0.84 & 1.68 & 2.51 \\
rank on $\hat p$ & 0.16 & 0.40 & 0.73 & 1.05 & 1.38 & 2.62 & 3.87 \\
rank on $i\,\hat p$ & 0.09 & 0.28 & 0.60 & 0.89 & 1.22 & 2.49 & 3.72 \\
rank on $\hat p$ inside the top $1/10$ & 0.26 & 0.60 & 0.84 & 0.86 & 0.90 & 1.69 & 1.70 \\
rank on $\hat p$ inside the top $1/5$ & 0.13 & 0.34 & 0.69 & 1.00 & 1.34 & 2.44 & 3.06 \\
\bottomrule
\end{tabular}
\caption{Expected total importance of the artifacts returned, the objective $f_2$.}
\end{table}

\begin{table}[H]
\centering\small\setlength{\tabcolsep}{3.5pt}
\begin{tabular}{lrrrrrrr}
\toprule
strategy & $B=10$ & $B=25$ & $B=50$ & $B=75$ & $B=100$ & $B=200$ & $B=300$ \\
\midrule
uniform random & 0.078 & 0.176 & 0.299 & 0.386 & 0.448 & 0.574 & 0.629 \\
rank on $\hat p$ & 0.157 & 0.354 & 0.437 & 0.467 & 0.485 & 0.499 & 0.500 \\
rank on $i\,\hat p$ & 0.090 & 0.280 & 0.546 & 0.597 & 0.600 & 0.600 & 0.600 \\
rank on $\hat p$ inside the top $1/10$ & 0.259 & 0.596 & 0.841 & 0.850 & 0.850 & 0.850 & 0.850 \\
rank on $\hat p$ inside the top $1/5$ & 0.126 & 0.294 & 0.540 & 0.690 & 0.779 & 0.850 & 0.850 \\
\bottomrule
\end{tabular}
\caption{Expected maximum importance among the artifacts returned, the objective $f_3$.}
\end{table}
\FloatBarrier

\clearpage
\section{Reviewed Solutions}
\label{app:solutions}

This appendix collects the write-ups of the results the authors reviewed, ordered alphabetically by the mathematicians who posed the problems.

\etocsetnexttocdepth{subsection}
\etocsettocstyle{}{}
\localtableofcontents

\subsection[(Ananchuen--Caccetta) Paley graphs with a prescribed adjacency property]{Paley graphs with a prescribed adjacency property}
\label{sol:ellipticcurves}
\begingroup
\begin{bibunit}
\def\F{\mathbb{F}}
\def\cG{\mathcal{G}}

This problem concerns the number of vertices of a Paley graph that are adjacent to one prescribed vertex and to neither of two others.
Ananchuen and Caccetta proved that the Paley graph on $q$ vertices has at least $k$ such vertices, for every choice of the three, as soon as $q>(1+2\sqrt{2k})^{2}$, and conjectured that this threshold is exact.
We show that it is not: the Paley graph on $5^{5}=3125$ vertices has at least $377$ such vertices for every choice, although $3125<(1+2\sqrt{754})^{2}$.
The proof is a quadratic character count in which the only obstruction is an elliptic curve over $\F_{5^{5}}$ of trace $110$, and Waterhouse's classification of elliptic-curve traces forbids a nonzero trace divisible by the characteristic over a field of odd degree in characteristic greater than $3$.

\subsubsection{Introduction}

Let $m,n$ be nonnegative integers and let $k$ be a positive integer.
Following \citet{ananchuen1992graphs}, a graph $G$ has \emph{property $P(m,n,k)$} if for every pair of disjoint sets $A,B\subseteq V(G)$ with $|A|=m$ and $|B|=n$ there are at least $k$ vertices outside $A\cup B$ that are adjacent to every vertex of $A$ and to no vertex of $B$.
We write $\cG(m,n,k)$ for the class of graphs with property $P(m,n,k)$.
These classes are quantitative refinements of the adjacency axioms of \citet{blass1979properties}, who showed by a probabilistic argument that almost all graphs have property $P(n,n,1)$, from which the same follows for every $P(m,n,k)$; a graph is $n$-existentially closed exactly when it lies in $\cG(i,n-i,1)$ for all $0\le i\le n$.
Explicit examples are much harder to come by, and the extremal question of how few vertices a graph in $\cG(m,n,k)$ can have was raised by \citet{exoo1981adjacency} and studied systematically by \citet{ananchuen1992graphs}.

For a prime power $q\equiv1\pmod4$, the \emph{Paley graph} $G_q$ has vertex set $\F_q$, two distinct vertices $u,v$ being adjacent exactly when $u-v$ is a nonzero square in $\F_q$; this is well defined because $-1$ is a square in $\F_q$.
Paley graphs are the standard supply of explicit graphs with prescribed adjacency properties.
\citet{blass1981paley} showed that $G_p\in\cG(n,n,1)$ for every prime $p\equiv1\pmod4$ with $p>n^{2}2^{4n}$, and \citet{ananchuen1993adjacency} obtained thresholds for the full range of parameters by estimating the relevant character sums, among them
\[
q>\bigl(1+2\sqrt{2k}\bigr)^{2}
\quad\Longrightarrow\quad
G_q\in\cG(1,2,k).
\]
This same threshold reappears in \citet{ananchuen1995note}, where it is deduced, together with the companion conclusion $G_q\in\cG(2,1,k)$, from the $n$-parameter implication
\[
q>\bigl\{(n-2)2^{n}+2\bigr\}\sqrt q+(n+2k-1)2^{n}-2n-1
\quad\Longrightarrow\quad
G_q\in\cG(1,n,k)\cap\cG(n,1,k),
\]
which at $n=2$ reduces to the threshold above.

For $n=1$ the $n$-parameter threshold is exactly right.
Indeed it then reads $q>4k-3$, which for $q\equiv1\pmod4$ means $q\ge4k+1$, while \citet{exoo1981adjacency} proved that every graph in $\cG(1,1,k)$ has at least $4k+1$ vertices; hence $G_q\in\cG(1,1,k)$ if and only if $q\ge4k+1$ \citep{ananchuen1995note}.
\citet[Remark~2]{ananchuen1995note} leave the analogous exactness for $n=2$ as a conjecture, writing:

\medskip
\noindent
\emph{We have verified, by computer, that if $q\equiv1\pmod4$ is a prime power less than or equal to $1009$ and $k$ is a positive integer with $q<\bigl(1+2\sqrt{2k}\bigr)^{2}$, then $G_q\notin\cG(1,2,k)$.
We conjecture that this is true for all $q$.}
\medskip

\noindent
It is convenient to restate the conjecture as a single inequality.
The Paley graph is vertex-transitive under translation, so in testing property $P(1,2,k)$ we may always take $A=\{0\}$.
Writing $\eta$ for the quadratic character of $\F_q$, extended by $\eta(0)=0$, the number of vertices to be counted for $B=\{b,c\}$ is
\begin{equation}
\label{ellipticcurves:eq:defN}
N(b,c)=\#\bigl\{x\in\F_q\setminus\{0,b,c\}:\eta(x)=1,\ \eta(x-b)=\eta(x-c)=-1\bigr\}.
\end{equation}
Thus $G_q\in\cG(1,2,k)$ if and only if $k\le N_{\min}(q)$, where
\[
N_{\min}(q)=\min\bigl\{N(b,c):b,c\in\F_q^{*},\ b\ne c\bigr\}.
\]
Since $q<(1+2\sqrt{2k})^{2}$ is equivalent to $8k>(\sqrt q-1)^{2}$, the conjecture asserts precisely that
\begin{equation}
\label{ellipticcurves:eq:reform}
8\,N_{\min}(q)\le(\sqrt q-1)^{2}
\qquad\text{for every prime power }q\equiv1\pmod4 .
\end{equation}
Ananchuen and Caccetta go on, in the same remark, to choose the three vertices in their character-sum estimate so that it yields an upper bound as well, concluding that $G_q\notin\cG(1,2,k)$ whenever $q<\bigl(-1+2\sqrt{2(k+1)}\bigr)^{2}$.
The conjecture, they note, therefore has content only in the window
\[
\bigl(-1+2\sqrt{2(k+1)}\bigr)^{2}\le q\le\bigl(1+2\sqrt{2k}\bigr)^{2},
\]
an interval of length $(8+o(1))\sqrt{2k}$.
Our counterexample lies inside it: for $k=377$ the window is $2915.01\ldots\le q\le3126.83\ldots$, and $q=3125$.

\citet{ananchuen2001adjacency} and \citet{ananchuen2006cubic} carry the character-sum method of the note \citep{ananchuen1995note} over to the graphs built from cubic and quartic residues, and \citet{bonato2009search} surveys the explicit constructions of $n$-existentially closed graphs.
The note's companion paper \citep{australia1994constructing} takes up $\cG(1,2,1)$ directly, exhibiting a graph in that class of every order at least $10$ except $11$ and noting that $\cG(1,2,1)$ contains no graph of any other order.
We are aware of no earlier counterexample to the conjecture and, beyond the two bounds above, of no partial result on the exactness of the threshold for $\cG(1,2,k)$.

\begin{theorem}
\label{ellipticcurves:thm:main}
The Paley graph $G_{3125}$ on $5^{5}$ vertices belongs to $\cG(1,2,377)$, while
\[
3125<\bigl(1+2\sqrt{2\cdot377}\bigr)^{2}.
\]
In particular \eqref{ellipticcurves:eq:reform} fails at $q=5^{5}$, so the conjecture of \citet[Remark~2]{ananchuen1995note} is false.
\end{theorem}

\begin{remark}
\label{ellipticcurves:rem:computation}
The constant $377$ is best possible for this $q$.
Since $x^{5}-x+1$ is irreducible over $\F_{5}$, we may realize $\F_{3125}$ as $\F_{5}[a]$ with $a^{5}=a-1$; take $b=1$ and $c=4+4a+a^{3}$.
In the notation of Lemma~\ref{ellipticcurves:lem:count} below, $b$ and $c$ are squares while $b-c$ is not, so $R(b,c)=8$, and a computation in $\F_{5^{5}}$ gives $S(b,c)=-102$; hence
\[
8N(b,c)=3126-102-8=3016.
\]
Therefore $N(b,c)=377$, which with Theorem~\ref{ellipticcurves:thm:main} gives $N_{\min}(3125)=377$ and $G_{3125}\notin\cG(1,2,378)$.
The margin is thin: $8\,N_{\min}(3125)=3016$ exceeds $(\sqrt{3125}-1)^{2}=3014.19\ldots$ by less than $2$.
\end{remark}

Expanding the three character conditions in \eqref{ellipticcurves:eq:defN} turns $8N(b,c)$ into $q+1+S(b,c)-R(b,c)$, where
\[
S(b,c)=\sum_{x\in\F_q}\eta\bigl(x(x-b)(x-c)\bigr)
\]
is minus the trace of the elliptic curve $y^{2}=x(x-b)(x-c)$ and $R(b,c)\in\{0,4,8\}$ is a correction coming from the three points $0,b,c$.
For $q=3125$ Hasse's bound gives $S(b,c)\ge-111$ and hence $N(b,c)\ge376$, one short of what is needed.
The crucial observation is that the equality case $N(b,c)=376$ forces the curve to have trace exactly $110$, a nonzero multiple of the characteristic, and over $\F_{5^{5}}$ no elliptic curve has such a trace.

\subsubsection{The character count}

Throughout the rest of this section $q$ is a prime power with $q\equiv1\pmod4$, and $\eta$ is the quadratic character of $\F_q$, extended by $\eta(0)=0$.
Thus $\eta(x)=1$ if $x$ is a nonzero square, $\eta(x)=-1$ if $x$ is a nonsquare, and $\eta(-1)=1$.
In particular distinct $u,v\in\F_q$ are adjacent in $G_q$ if and only if $\eta(u-v)=1$.

Only one character sum evaluation is needed:
\begin{equation}
\label{ellipticcurves:eq:quad}
\sum_{x\in\F_q}\eta\bigl((x-u)(x-v)\bigr)=-1
\qquad(u\ne v).
\end{equation}
Indeed, substituting $x=u+(v-u)y$ turns the left side into $\sum_{y}\eta\bigl((v-u)^{2}y(y-1)\bigr)=\sum_{y}\eta\bigl(y(y-1)\bigr)$.
The term $y=0$ vanishes, and for $y\ne0$ we have $\eta\bigl(y(y-1)\bigr)=\eta\bigl(y^{2}(1-y^{-1})\bigr)=\eta(1-y^{-1})$.
As $y$ runs over $\F_q^{*}$ the element $1-y^{-1}$ runs over $\F_q\setminus\{1\}$, so the sum equals $\sum_{w\ne1}\eta(w)=-\eta(1)=-1$.

\begin{lemma}
\label{ellipticcurves:lem:count}
Let $b,c\in\F_q^{*}$ be distinct and let $N(b,c)$ be as in \eqref{ellipticcurves:eq:defN}.
Then
\begin{equation}
\label{ellipticcurves:eq:count}
8N(b,c)=q+1+S(b,c)-R(b,c),
\end{equation}
where
\[
S(b,c)=\sum_{x\in\F_q}\eta\bigl(x(x-b)(x-c)\bigr)
\]
and
\[
R(b,c)=\bigl(1-\eta(b)\bigr)\bigl(1-\eta(c)\bigr)
+\bigl(1+\eta(b)\bigr)\bigl(1-\eta(b-c)\bigr)
+\bigl(1+\eta(c)\bigr)\bigl(1-\eta(b-c)\bigr).
\]
Moreover $R(b,c)\in\{0,4,8\}$.
\end{lemma}

\begin{proof}
Put
\[
T=\sum_{x\in\F_q}\bigl(1+\eta(x)\bigr)\bigl(1-\eta(x-b)\bigr)\bigl(1-\eta(x-c)\bigr).
\]
For $x\notin\{0,b,c\}$ each of $\eta(x)$, $\eta(x-b)$, $\eta(x-c)$ is $\pm1$, so the summand equals $8$ when $x$ is one of the points counted by $N(b,c)$ and equals $0$ otherwise.
The three excluded points contribute
\[
\bigl(1-\eta(-b)\bigr)\bigl(1-\eta(-c)\bigr),
\qquad
\bigl(1+\eta(b)\bigr)\bigl(1-\eta(b-c)\bigr),
\qquad
\bigl(1+\eta(c)\bigr)\bigl(1-\eta(c-b)\bigr)
\]
respectively, and these sum to $R(b,c)$ because $\eta(-u)=\eta(u)$.
Hence
\[
T=8N(b,c)+R(b,c).
\]

On the other hand, expanding the product gives
\begin{align*}
T=\sum_{x}1
&+\sum_{x}\eta(x)-\sum_{x}\eta(x-b)-\sum_{x}\eta(x-c)\\
&+\sum_{x}\eta\bigl((x-b)(x-c)\bigr)-\sum_{x}\eta\bigl(x(x-b)\bigr)-\sum_{x}\eta\bigl(x(x-c)\bigr)
+S(b,c),
\end{align*}
all sums being over $x\in\F_q$.
The first sum is $q$, and the three single-character sums vanish because $\eta$ is nonprincipal.
By \eqref{ellipticcurves:eq:quad} each of the three quadratic sums equals $-1$, so together they contribute $-1+1+1=1$.
Therefore $T=q+1+S(b,c)$, and comparing the two expressions for $T$ gives \eqref{ellipticcurves:eq:count}.

It remains to determine the possible values of $R(b,c)$.
Write $\alpha=\eta(b)$, $\beta=\eta(c)$ and $\gamma=\eta(b-c)$, all of which lie in $\{\pm1\}$, so that
\[
R(b,c)=(1-\alpha)(1-\beta)+(1+\alpha)(1-\gamma)+(1+\beta)(1-\gamma).
\]
If $\gamma=1$ then $R(b,c)=(1-\alpha)(1-\beta)$, which is $4$ when $\alpha=\beta=-1$ and $0$ otherwise.
If $\gamma=-1$ then $R(b,c)=(1-\alpha)(1-\beta)+2(1+\alpha)+2(1+\beta)$, which is $8$ when $\alpha=\beta=1$ and $4$ in the three remaining cases.
This exhausts all possibilities.
\end{proof}

\subsubsection{Traces of elliptic curves}

Two facts about elliptic curves over finite fields are needed.
The first is Hasse's bound \citep[Chapter~V, Theorem~1.1]{silverman2009arithmetic}: if $E$ is an elliptic curve over $\F_q$, then its trace $t=q+1-\#E(\F_q)$ satisfies $|t|\le2\sqrt q$.
The second is the classification of the integers that occur as traces, due to \citet[Theorem~4.1]{waterhouse1969abelian}.

\begin{theorem}
\label{ellipticcurves:thm:waterhouse}
Let $p$ be a prime, let $r\ge1$, and let $t$ be an integer with $|t|\le2p^{r/2}$.
There is an elliptic curve over $\F_{p^{r}}$ of trace $t$ if and only if at least one of the following holds:
\begin{enumerate}
\item $\gcd(t,p)=1$;
\item $r$ is even and $t=\pm2p^{r/2}$;
\item $r$ is even, $p\not\equiv1\pmod3$ and $t=\pm p^{r/2}$;
\item $r$ is odd, $p\in\{2,3\}$ and $t=\pm p^{(r+1)/2}$;
\item $t=0$, and either $r$ is odd or $r$ is even with $p\not\equiv1\pmod4$.
\end{enumerate}
\end{theorem}

The same classification is restated by \citet{schoof1987nonsingular}, where it is the starting point for counting the isomorphism classes of elliptic curves in a fixed isogeny class.
We use it only through the following consequence.

\begin{corollary}
\label{ellipticcurves:cor:trace}
Let $p>3$ be a prime and let $r$ be odd.
If $E$ is an elliptic curve over $\F_{p^{r}}$ whose trace $t$ is divisible by $p$, then $t=0$.
\end{corollary}

\begin{proof}
Of the five cases of Theorem~\ref{ellipticcurves:thm:waterhouse}, the first is excluded by $p\mid t$, the second and third by $r$ being odd, and the fourth by $p>3$.
Only the fifth remains, and it gives $t=0$.
\end{proof}

\subsubsection{Proof of the main theorem}

\begin{proof}[Proof of Theorem~\ref{ellipticcurves:thm:main}]
Set $q=5^{5}=3125$.
By vertex-transitivity it suffices to prove that $N(b,c)\ge377$ for every pair of distinct $b,c\in\F_q^{*}$, since $N(b,c)$ counts exactly the vertices adjacent to $0$ and to neither $b$ nor $c$.

Fix such a pair and consider
\[
E_{b,c}:\quad y^{2}=x(x-b)(x-c).
\]
The cubic $x(x-b)(x-c)$ has three distinct roots and $\operatorname{char}\F_q=5\ne2$, so $E_{b,c}$ is nonsingular, that is, an elliptic curve.
For each $x\in\F_q$ the number of $y\in\F_q$ with $y^{2}=x(x-b)(x-c)$ is $1+\eta\bigl(x(x-b)(x-c)\bigr)$, so summing over $x$ and adding the point at infinity gives
\[
\#E_{b,c}(\F_q)=q+1+S(b,c).
\]
Hence the trace of $E_{b,c}$ is $t=q+1-\#E_{b,c}(\F_q)=-S(b,c)$, and Hasse's bound gives
\[
|S(b,c)|\le2\sqrt q=50\sqrt5<112,
\]
so that $S(b,c)\ge-111$.
Combining this with Lemma~\ref{ellipticcurves:lem:count} and $R(b,c)\le8$ yields
\[
8N(b,c)=q+1+S(b,c)-R(b,c)\ge3126-111-8=3007,
\]
and therefore $N(b,c)\ge376$.

It remains to exclude the equality case.
Suppose $N(b,c)=376$.
Then \eqref{ellipticcurves:eq:count} reads
\[
3008=3126+S(b,c)-R(b,c),
\]
so $S(b,c)=R(b,c)-118$.
As $R(b,c)\in\{0,4,8\}$ this forces $S(b,c)\in\{-118,-114,-110\}$, and the bound $S(b,c)\ge-111$ leaves only
\[
R(b,c)=8,
\qquad
S(b,c)=-110.
\]
The curve $E_{b,c}$ would then have trace $t=-S(b,c)=110$.
But $110$ is a nonzero multiple of $5$, and $q=5^{5}$ has $p=5>3$ with $r=5$ odd, so Corollary~\ref{ellipticcurves:cor:trace} rules this out.
Therefore $N(b,c)\ne376$, and the preceding bound gives $N(b,c)\ge377$ for every pair of distinct $b,c\in\F_q^{*}$.
Hence $G_{3125}\in\cG(1,2,377)$.

Finally,
\[
\bigl(1+2\sqrt{2\cdot377}\bigr)^{2}
=\bigl(1+2\sqrt{754}\bigr)^{2}
=3017+4\sqrt{754}
>3017+108
=3125,
\]
because $27^{2}=729<754$.
This completes the proof.
\end{proof}

\putbib

\end{bibunit}
\endgroup
\subsection[(Budden--Penland) A 4-uniform tree that is not 5-good]{A 4-uniform tree that is not 5-good}
\label{sol:ngood_trees}
\begingroup
\begin{bibunit}

This problem asks whether $r$-uniform trees are optimal in the Ramsey problem against a complete $r$-uniform hypergraph, as ordinary trees are against a clique.
They are not.
The unique $4$-uniform tree on seven vertices fails to be $5$-good, and the obstruction is a red/blue coloring of $K_8^{(4)}$ whose red edges are exactly the fourteen affine planes of $\mathbb{F}_2^3$.

\subsubsection{Introduction}

All hypergraphs here are finite and simple.
For an $r$-uniform hypergraph $H$ we write $V(H)$ and $E(H)$ for its vertex and edge sets, and $K_n^{(r)}$ for the complete $r$-uniform hypergraph on $n$ vertices.
The Ramsey number $R(G,H;r)$ is the least $p$ such that every red/blue coloring of $E(K_p^{(r)})$ contains a red copy of $G$ or a blue copy of $H$.
A \emph{Berge cycle} of length $k\ge2$ in $H$ is an alternating sequence $v_1,e_1,v_2,e_2,\dots,v_k,e_k$ of distinct vertices $v_i$ and distinct edges $e_i$ with $v_i,v_{i+1}\in e_i$ for every $i$, indices read modulo $k$.
An \emph{$r$-uniform tree} is a connected $r$-uniform hypergraph containing no Berge cycle.
Equivalently \citep[Theorem~2.1]{budden2017trees}, it is a hypergraph that can be assembled starting from a single edge, each subsequent edge meeting the union of the previous ones in exactly one vertex; in particular an $r$-uniform tree on $m$ vertices has exactly $(m-1)/(r-1)$ edges.

A \emph{weak coloring} of $H$ is a coloring of $V(H)$ in which no edge is monochromatic, the \emph{weak chromatic number} $\chi_w(H)$ is the least number of colors in a weak coloring, and the \emph{chromatic surplus} $s(H)$ is the least size of a color class over all weak colorings of $H$ with exactly $\chi_w(H)$ colors.
\citet[Theorem~3.1]{budden2017trees} proved that every connected $r$-uniform hypergraph $G$ of order $m\ge r$ satisfies
\begin{equation}\label{ngood_trees:eq:ngood}
  R\bigl(G,K_n^{(r)};r\bigr)\ \ge\ (m-1)\bigl(\chi_w(K_n^{(r)})-1\bigr)+s\bigl(K_n^{(r)}\bigr),
\end{equation}
and they call $G$ \emph{$n$-good} when equality holds; their notation for the chromatic surplus is $t(\cdot)$, and we follow the later $s(\cdot)$ of \citet{budden2022hypergraph}.
For an arbitrary $r$-uniform target $H$ with $s(H)\le m$ the same theorem gives $R(G,H;r)\ge(m-1)(\chi_w(H)-1)+s(H)$, and $G$ is \emph{$H$-good} when equality holds; $n$-goodness is the case $H=K_n^{(r)}$.
Here $\chi_w(K_n^{(r)})=\lceil n/(r-1)\rceil$, since a color class of a weak coloring of $K_n^{(r)}$ is exactly a set of at most $r-1$ vertices.
For $r=2$, \eqref{ngood_trees:eq:ngood} reads $R(G,K_n)\ge(m-1)(n-1)+1$, and \citet{chvatal1977tree} proved that every tree attains it: every tree $T_m$ on $m$ vertices satisfies $R(T_m,K_n)=(m-1)(n-1)+1$.
The term $n$-good goes back to \citet{burr1980extremal}, and the systematic study of goodness to \citet{burr1983generalizations}.

\citet[Conjecture~6.2]{budden2017trees} conjectured that hypertrees are optimal in the same sense:

\medskip
\noindent
\emph{If $r\ge2$ and $T$ is any $r$-uniform tree, then $T$ is $n$-good} [for every $n\ge r$].
\medskip

\noindent
We disprove this.

The evidence behind the conjecture is substantial.
\citet{budden2017trees} record the hypertree embedding bound of \citet{loh2009note}, namely
\[
  R\bigl(T_m^{(r)},K_n^{(r)};r\bigr)\ \le\ \frac{(m-1)(n-1)}{r-1}+1
\]
for every $r$-uniform tree $T_m^{(r)}$ on $m$ vertices, and deduce that the conjecture holds whenever $(r-1)\mid(n-1)$, because the two bounds then agree.
For $r=3$ the divisibility condition says exactly that $n$ is odd, and for the smallest nontrivial tree \citet{budden2017trees} push past it: $R(T_5^{(3)},K_n^{(3)};3)=2n-2$ for even $n$ as well, so $T_5^{(3)}$ is $n$-good for every $n\ge3$.
Beyond $T_5^{(3)}$ their bounds for even $n$ leave a range of possible values rather than a single one, and they report finding no counterexample.
\citet{budden2022hypergraph} reduce the conjecture to the residues $n\equiv0\pmod{r-1}$: a tree that is $n$-good for every multiple $n$ of $r-1$ is $n$-good for every $n\ge r$.
In the asymptotic direction, \citet{boyadzhiyska2025ramsey} prove that for every $r\ge3$ and every $r$-uniform hypergraph $H$ with $s(H)\le2r-1$, all sufficiently long $r$-uniform loose paths are $H$-good, the threshold depending on $H$.
Every complete target satisfies the hypothesis, since $s(K_n^{(r)})\le r-1$, so for each fixed $n$ all sufficiently long loose paths are $n$-good.
A loose path that is not $n$-good is therefore short, and ours is the shortest nontrivial one.

Connected hypergraphs that are not $n$-good were known, but no tree was among them; \citet{budden2017trees} show that the $3$-uniform loose cycle $C_4^{(3)}$, two edges meeting in two vertices, is not $5$-good, and record that $K_4^{(3)}$, $K_5^{(3)}$ and $K_6^{(3)}$ are not $4$-good and that $K_5^{(4)}$ is not $5$-good.
Two nearby negative results are of a different kind.
The $\ell$-paths with $\ell\ge2$ that \citet{boyadzhiyska2025ramsey} rule out are not trees, since consecutive edges meet in $\ell\ge2$ vertices and hence span a Berge cycle of length $2$.
The loose paths they rule out are trees, but the targets there are incomplete hypergraphs $H$ with $s(H)>2r-1$, outside the hypothesis above.
Radziszowski's dynamic survey describes the subsequent literature as ``further results towards the conjecture that all $r$-uniform trees are $n$-good'' \citep{Radziszowski_1994}.
(A corrigendum \citep{budden2019corrigendum} amends Theorem~3.10 of the original paper, on disjoint unions of $n$-good hypergraphs; Conjecture~6.2 is unaffected.)

The counterexample is the smallest $4$-uniform tree other than a single edge, equivalently the $4$-uniform loose path with two edges, and it is taken against $K_5^{(4)}$, a pair of parameters that the divisibility criterion does not cover.
Let $T$ be the $4$-uniform hypergraph on the vertex set $[7]$ with the two edges
\begin{equation}\label{ngood_trees:eq:tree}
  e_1=\{1,2,3,4\},\qquad e_2=\{4,5,6,7\}.
\end{equation}
It is connected, and it has no Berge cycle: a Berge cycle uses at least two distinct edges, so with only $e_1$ and $e_2$ available it would need two distinct vertices in $e_1\cap e_2$, whereas $|e_1\cap e_2|=1$.
Hence $T$ is a $4$-uniform tree, and it is the unique one of order $7$ up to isomorphism, since a $4$-uniform tree on $7$ vertices has exactly $(7-1)/(4-1)=2$ edges and those edges meet in exactly one vertex.

\begin{theorem}\label{ngood_trees:thm:main}
Let $T$ be the $4$-uniform tree of \eqref{ngood_trees:eq:tree}.
Then
\[
  R\bigl(T,K_5^{(4)};4\bigr)=9,
\]
while the right-hand side of \eqref{ngood_trees:eq:ngood} equals $8$ for $G=T$, $r=4$ and $n=5$.
In particular $T$ is not $5$-good, and Conjecture~6.2 of \citet{budden2017trees} is false.
\end{theorem}

\begin{remark}\label{ngood_trees:rem:scope}
The $3$-uniform case of Conjecture~6.2, which is the bulk of \citet{budden2017trees}, remains open.
So does their conjecture that any two $r$-uniform trees of the same order have the same Ramsey number against $K_n^{(r)}$: $T$ is the only $4$-uniform tree of order $7$.

The divisibility hypothesis of \citet{budden2017trees} is not an artifact of a lossy estimate.
Fix $n\ge r$, write $n-1=q(r-1)+j$ with $0\le j\le r-2$, and let $k=(m-1)/(r-1)$ be the number of edges of an $r$-uniform tree of order $m$.
Then $\chi_w(K_n^{(r)})=q+1$ and $s(K_n^{(r)})=j+1$, so the Loh upper bound and the Burr-type lower bound in \eqref{ngood_trees:eq:ngood} are $(m-1)q+kj+1$ and $(m-1)q+j+1$ respectively, and the window between them has width exactly $j(k-1)$.
It closes precisely when $(r-1)\mid(n-1)$ or the tree is a single edge.
At $(r,m,n)=(4,7,5)$ the window is $\{8,9\}$ and Theorem~\ref{ngood_trees:thm:main} places the truth at the top of it, whereas at $(r,m,n)=(3,5,4)$ the window is again a single step, $\{6,7\}$, and there the truth is at the bottom: $R(T_5^{(3)},K_4^{(3)};3)=6$ \citep{budden2017trees}.
\end{remark}

The lower bound is a single explicit coloring.
The key point is that the affine planes of $\mathbb{F}_2^3$, regarded as $4$-element subsets, meet one another in $0$ or $2$ points and never in $1$, while the two edges of $T$ meet in exactly $1$ point; at the same time the planes are numerous enough that every $5$ points contain one.
The matching upper bound is Loh's hypertree embedding theorem, which at $(r,m,n)=(4,7,5)$ gives exactly $9$.

\subsubsection{The affine-plane coloring}

Take the vertex set of $K_8^{(4)}$ to be $V=\mathbb{F}_2^3$.
For $a\in\mathbb{F}_2^3\setminus\{0\}$ and $b\in\mathbb{F}_2$ put
\[
  P(a,b):=\{x\in\mathbb{F}_2^3:\ a\cdot x=b\},
\]
where $a\cdot x$ denotes the standard bilinear form; such a set is an \emph{affine plane}, that is, a coset of a two-dimensional subspace.
Color a $4$-subset of $V$ red if it is an affine plane, and blue otherwise.
Figure~\ref{ngood_trees:fig:planes} depicts the two intersection patterns at issue.

\begin{figure}[t]
\centering
\begin{tikzpicture}[
  pt/.style={circle,fill=black,inner sep=1.5pt},
  lbl/.style={font=\scriptsize},
  cptn/.style={font=\small,align=center}
]
\begin{scope}
  \draw[fill=blue!14,draw=blue!60!black,line width=0.7pt,fill opacity=0.7] (-1.05,0) ellipse (1.6 and 0.85);
  \draw[fill=orange!25,draw=orange!70!black,line width=0.7pt,fill opacity=0.6] (1.05,0) ellipse (1.6 and 0.85);
  \node[pt] at (-2.05,0.4) {};
  \node[pt] at (-2.05,-0.4) {};
  \node[pt] at (-1.05,0) {};
  \node[pt] at (0,0) {};
  \node[pt] at (1.05,0) {};
  \node[pt] at (2.05,0.4) {};
  \node[pt] at (2.05,-0.4) {};
  \node[lbl] at (-2.05,0.74) {$1$};
  \node[lbl] at (-2.05,-0.74) {$2$};
  \node[lbl] at (-1.05,0.32) {$3$};
  \node[lbl] at (0,0.34) {$4$};
  \node[lbl] at (1.05,0.32) {$5$};
  \node[lbl] at (2.05,0.74) {$6$};
  \node[lbl] at (2.05,-0.74) {$7$};
  \node[lbl] at (-1.75,1.1) {$e_1$};
  \node[lbl] at (1.75,1.1) {$e_2$};
  \node[cptn] at (0,-1.5) {the tree $T$};
\end{scope}
\begin{scope}[xshift=5.9cm,yshift=-0.85cm]
  \coordinate (a000) at (0,0);
  \coordinate (a100) at (1.7,0);
  \coordinate (a010) at (0,1.7);
  \coordinate (a110) at (1.7,1.7);
  \coordinate (a001) at (0.8,0.7);
  \coordinate (a101) at (2.5,0.7);
  \coordinate (a011) at (0.8,2.4);
  \coordinate (a111) at (2.5,2.4);
  \fill[orange!35,fill opacity=0.55] (a000) -- (a100) -- (a101) -- (a001) -- cycle;
  \fill[blue!35,fill opacity=0.55] (a000) -- (a100) -- (a110) -- (a010) -- cycle;
  \draw[gray!70,line width=0.6pt] (a000)--(a100) (a000)--(a010) (a000)--(a001)
    (a100)--(a110) (a100)--(a101) (a010)--(a110) (a010)--(a011)
    (a001)--(a101) (a001)--(a011) (a110)--(a111) (a101)--(a111) (a011)--(a111);
  \draw[red!70!black,line width=1.5pt] (a000)--(a100);
  \node[pt] at (a000) {};
  \node[pt] at (a100) {};
  \node[pt] at (a010) {};
  \node[pt] at (a110) {};
  \node[pt] at (a001) {};
  \node[pt] at (a101) {};
  \node[pt] at (a011) {};
  \node[pt] at (a111) {};
  \node[lbl,below left] at (a000) {$000$};
  \node[lbl,below right] at (a100) {$100$};
  \node[lbl,left] at (a010) {$010$};
  \node[lbl,below right] at (a110) {$110$};
  \node[lbl,above left] at (a001) {$001$};
  \node[lbl,right] at (a101) {$101$};
  \node[lbl,above left] at (a011) {$011$};
  \node[lbl,above right] at (a111) {$111$};
  \node[cptn] at (1.25,-0.65) {$P(001,0)$ and $P(010,0)$};
\end{scope}
\end{tikzpicture}
\caption{On the left, the unique $4$-uniform tree $T$ of order $7$: its two edges meet in the single vertex $4$.
On the right, two of the fourteen affine planes of $\mathbb{F}_2^3$, drawn as faces of the cube; they meet in the two points $000$ and $100$, marked by the thick segment.
Lemma~\ref{ngood_trees:lem:nored} says that two red edges can never meet the way $e_1$ and $e_2$ do.}
\label{ngood_trees:fig:planes}
\end{figure}
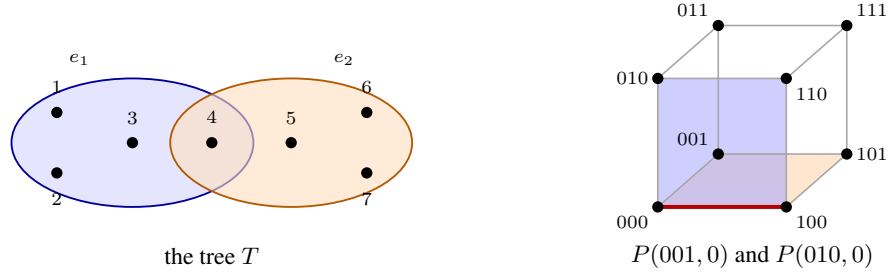

\begin{lemma}\label{ngood_trees:lem:planes}
There are exactly $14$ affine planes in $\mathbb{F}_2^3$, and a $4$-subset $S\subseteq\mathbb{F}_2^3$ is an affine plane if and only if $\sum_{x\in S}x=0$.
\end{lemma}

\begin{proof}
The map $a\mapsto P(a,0)$ is a bijection from $\mathbb{F}_2^3\setminus\{0\}$ onto the set of two-dimensional subspaces of $\mathbb{F}_2^3$, and each such subspace has exactly the two cosets $P(a,0)$ and $P(a,1)$.
Hence there are exactly $14$ affine planes, each of size $4$.

A two-dimensional subspace has the form $\{0,p,q,p+q\}$ and therefore sums to $0$, and translating a $4$-set by $x$ changes its sum by $4x=0$; so every affine plane sums to $0$.
Conversely, let $S=\{x_1,x_2,x_3,x_4\}$ satisfy $\sum_{i}x_i=0$.
Then $S+x_4=\{x_1+x_4,\,x_2+x_4,\,x_3+x_4,\,0\}$ consists of $0$ together with three distinct nonzero elements $p,q,u$ whose sum is $0$.
Thus $u=p+q$, and $p,q$ are distinct nonzero vectors and hence linearly independent, so $S+x_4=\{0,p,q,p+q\}$ is a two-dimensional subspace.
Therefore $S$ is one of its cosets.
\end{proof}

\subsubsection{The lower bound}

We now check that the coloring contains neither of the two forbidden configurations.

\begin{lemma}\label{ngood_trees:lem:nored}
Any two distinct red edges meet in $0$ or $2$ vertices.
Consequently the coloring contains no red copy of $T$.
\end{lemma}

\begin{proof}
Let $P(a,b)\neq P(c,d)$ be red edges.
If $a=c$ then $b\neq d$, and the two sets are disjoint.
If $a\neq c$ then $a$ and $c$ are distinct nonzero vectors, hence linearly independent over $\mathbb{F}_2$, so the linear map $x\mapsto(a\cdot x,\,c\cdot x)$ from $\mathbb{F}_2^3$ to $\mathbb{F}_2^2$ is surjective with kernel of dimension $1$.
Every one of its fibers, in particular $P(a,b)\cap P(c,d)$, therefore has exactly $2$ elements.

A red copy of $T$ would consist of two red edges meeting in exactly one vertex, which the above excludes.
\end{proof}

\begin{lemma}\label{ngood_trees:lem:noblue}
Every $5$-subset of $\mathbb{F}_2^3$ contains a red $4$-subset.
Consequently the coloring contains no blue $K_5^{(4)}$.
\end{lemma}

\begin{proof}
Let $X\subseteq\mathbb{F}_2^3$ with $|X|=5$, write $Y=\mathbb{F}_2^3\setminus X=\{u,v,w\}$, and put $\sigma=u+v+w$.
If $\sigma\in Y$, say $\sigma=u$, then $v+w=0$ and hence $v=w$, a contradiction; so $\sigma\in X$.
In each coordinate exactly four of the eight vectors of $\mathbb{F}_2^3$ have entry $1$, so $\sum_{x\in\mathbb{F}_2^3}x=0$ and therefore
\[
  \sum_{x\in X}x=\sum_{y\in Y}y=\sigma .
\]
It follows that $\sum_{x\in X\setminus\{\sigma\}}x=0$, so $X\setminus\{\sigma\}$ is red by Lemma~\ref{ngood_trees:lem:planes}.
A blue $K_5^{(4)}$ would be a $5$-set all of whose $4$-subsets are blue, which the above excludes.
\end{proof}

Lemmas~\ref{ngood_trees:lem:nored} and~\ref{ngood_trees:lem:noblue} exhibit a red/blue coloring of $K_8^{(4)}$ with no red $T$ and no blue $K_5^{(4)}$, so
\begin{equation}\label{ngood_trees:eq:lower}
  R\bigl(T,K_5^{(4)};4\bigr)\ \ge\ 9 .
\end{equation}

\subsubsection{Proof of the main theorem}

The remaining input is the hypertree embedding theorem of \citet[Theorem~1]{loh2009note}, which answered a question of \citet{bohman2010coloring} by removing all dependence on the uniformity $r$.

\begin{theorem}\label{ngood_trees:thm:loh}
Every $r$-uniform hypergraph with weak chromatic number greater than $k$ contains a copy of every $r$-uniform tree with $k$ edges.
\end{theorem}

\begin{proof}[Proof of Theorem~\ref{ngood_trees:thm:main}]
For the upper bound, consider any red/blue coloring of $E(K_9^{(4)})$ and let $H_R$ be the spanning subhypergraph formed by the red edges.
The tree $T$ has exactly $2$ edges.
If $\chi_w(H_R)>2$, then Theorem~\ref{ngood_trees:thm:loh} produces a red copy of $T$.
Otherwise fix a weak coloring of $H_R$ using at most two colors.
Some color class $C$ satisfies $|C|\ge\lceil9/2\rceil=5$, and no red edge lies inside a color class, so all $4$-subsets of any five vertices of $C$ are blue.
That is a blue $K_5^{(4)}$.
Hence $R(T,K_5^{(4)};4)\le9$, and with \eqref{ngood_trees:eq:lower} we obtain $R(T,K_5^{(4)};4)=9$.

It remains to evaluate the right-hand side of \eqref{ngood_trees:eq:ngood}.
A weak coloring of $K_5^{(4)}$ is exactly a partition of a $5$-set into classes of size at most $3$, so two classes are needed and $3+2$ is the only partition into two such classes.
Hence $\chi_w(K_5^{(4)})=2$ and $s(K_5^{(4)})=2$, and since $|V(T)|=7$ the right-hand side of \eqref{ngood_trees:eq:ngood} equals $(7-1)(2-1)+2=8$.
As $9\neq8$, the tree $T$ is not $5$-good.
\end{proof}

\begin{remark}\label{ngood_trees:rem:consequences}
Theorem~\ref{ngood_trees:thm:main} also settles a value left open in \citet{budden2017trees}.
\citet[Theorem~4.1]{budden2017trees} prove that $2r\le R(T_{2r-1}^{(r)},K_{r+1}^{(r)};r)\le2r+1$ for all $r\ge3$, which at $r=4$ is precisely the window $8\le R(T,K_5^{(4)};4)\le9$; the value is $9$.
The failure also propagates upward.
By \citet[Theorem~10]{budden2022hypergraph}, a $4$-uniform tree that is $6$-good is also $5$-good and $4$-good, so $T$ is not $6$-good either.
\end{remark}

\putbib
\end{bibunit}
\endgroup
\subsection[(Davies--Jenssen--Perkins--Roberts) Maximum versus average independent set size in triangle-free graphs]{Maximum versus average independent set size in triangle-free graphs}
\label{sol:perkins_gang}
\begingroup
\begin{bibunit}
\def\I{\mathcal{I}}
\def\avgalpha{\overline{\alpha}}

This problem asks by how much the largest independent set of a triangle-free graph exceeds a typical one.
\citet{davies2018average} conjectured that a factor $2-o_d(1)$ is forced once the minimum degree $d$ is large.
We disprove this.
The Cartesian products $C_5\,\square\,K_{m,m}$ are triangle-free and $(m+2)$-regular, and for them the ratio of the maximum to the average size of an independent set tends to $24/13=1.846\ldots$, which is bounded away from $2$.
The mechanism is that the two sides of $K_{m,m}$ must draw their independent sets from disjoint parts of the pentagon.

\subsubsection{Introduction}

For a finite graph $G$ let $\I(G)$ be its family of independent sets and $\alpha(G)=\max_{I\in\I(G)}|I|$ its independence number.
The \emph{hard-core model} on $G$ at fugacity $\lambda>0$ is the probability distribution on $\I(G)$ giving each $I$ mass proportional to $\lambda^{|I|}$, and
\[
  \avgalpha_G(\lambda)=\frac{\sum_{I\in\I(G)}|I|\,\lambda^{|I|}}{\sum_{I\in\I(G)}\lambda^{|I|}}
\]
is the expected size of a set drawn from it.
At $\lambda=1$ the distribution is uniform on $\I(G)$, so $\avgalpha_G(1)$ is the average size of an independent set of $G$; it lies between $0$ and $\alpha(G)$ and is not normalized by $|V(G)|$.

\citet{davies2018average} proved that a triangle-free graph $G$ on $n$ vertices with $\Delta(G)\le d$ satisfies $\avgalpha_G(1)\ge(1+o_d(1))\frac{\log d}{d}\,n$, so that the average independent set is already as large as the lower bound Shearer's theorem \citep{shearer1983note} gives for the maximum one; in particular this gives a second proof of $R(3,k)\le(1+o(1))k^2/\log k$.
In Section~5 of the same paper they ask by how much the maximum exceeds the average and record four conjectures: three on triangle-free graphs and a fourth, Conjecture~4, for $K_r$-free graphs.
The second of them, \citet[Conjecture~2]{davies2018average}, reads as follows:

\medskip
\noindent
\emph{For every triangle-free graph $G$ of minimum degree $d$, $\alpha(G)/\avgalpha_G(1)\ge 2-o_d(1)$.}
\medskip

\noindent
\citet{davies2018average} deduce from it that $R(3,k)\le(\tfrac12+o(1))k^2/\log k$, halving the constant in Shearer's bound, which is still the best known.

The general lower bound available for all graphs is \citet[Theorem~6]{davies2018average}, which states that $\alpha(G)/\avgalpha_G(\lambda)\ge 1+\alpha(G)/(\lambda n)$ for every graph $G$ on $n$ vertices, with equality for a disjoint union of copies of a single $K_r$.
That bound degenerates as soon as $\alpha(G)=o(n)$.
The constant $4/3$ of Conjecture~1 is the exact value of $\alpha(G)/\avgalpha_G(1)$ at $G=K_3$, and it also falls below $197136/137585=1.43283\ldots$, the smallest ratio \citet{davies2018average} report, attained by the cyclic triangle-free graph witnessing $R(3,9)\ge36$ \citep{grinstead1982ramsey}.

The evidence offered for Conjecture~2 is the expectation that graphs produced by the triangle-free process have ratio tending to $2$.
The conjecture remained open: \citet{davies2025hard} record it as Conjecture~B in the open problems section of their survey of the hard-core model in graph theory, and \citet{morris2026some} restates it in his ICM survey of Ramsey theory, noting that any lower bound better than $1+o(1)$ would be a very significant breakthrough.
The one route toward it proposed in \citet{davies2018average} is the identity
\[
  \alpha(G)=\avgalpha_G(1)+\int_1^{\infty}\frac{\Var_{\lambda}(|I|)}{\lambda}\,d\lambda,
\]
which reduces the conjecture to lower bounds on the variance of the hard-core model; \citet{davies2025expectations} prove such bounds, but only for fugacities that are small in terms of the number of vertices.
In a neighboring circle of hard-core conjectures, \citet{cambie2023counterexamples} disproved five conjectures on the extremal graphs for the occupancy fraction and the independence polynomial among regular graphs of given girth, by computer search over small graphs.

\begin{theorem}\label{perkins_gang:thm:main}
For $m\ge1$ let $G_m=C_5\,\square\,K_{m,m}$.
Then $G_m$ is triangle-free and $(m+2)$-regular on $10m$ vertices, $\alpha(G_m)=4m$, and
\[
  \avgalpha_{G_m}(1)=\frac{13}{6}\,m\Bigl(1+O\bigl((2/3)^m\bigr)\Bigr),
  \qquad\text{so that}\qquad
  \frac{\alpha(G_m)}{\avgalpha_{G_m}(1)}\longrightarrow \frac{24}{13}=1.846\ldots<2 .
\]
\end{theorem}

\begin{corollary}\label{perkins_gang:cor:refutation}
Conjecture~2 of \citet{davies2018average} is false.
Indeed $G_m$ is triangle-free of minimum degree $d=m+2\to\infty$ while $\alpha(G_m)/\avgalpha_{G_m}(1)\to24/13$, so no bound of the form $2-o_d(1)$ can hold.
Since a triangle-free graph is $K_r$-free for every $r\ge3$, the same graphs refute the minimum-degree assertion of \citet[Conjecture~4]{davies2018average} for every fixed $r\ge3$.
\end{corollary}

The product structure is essential: $K_{m,m}$ alone is triangle-free and $m$-regular with $\alpha=m$, $|\I(K_{m,m})|=2^{m+1}-1$ and $\sum_I|I|=m2^m$, so its ratio is exactly $2-2^{-m}$.
Conjecture~1 of \citet{davies2018average}, that $\alpha(G)/\avgalpha_G(1)\ge4/3$ for every triangle-free graph $G$, is untouched by these examples, since both $24/13$ and the value $32/19$ obtained below exceed $4/3$; so is the first assertion of Conjecture~4, which asks only for $1+1/r$ in the $K_r$-free case.
Conjecture~3 of \citet{davies2018average}, the version at general fugacity, is also untouched: we compute only at $\lambda=1$, whereas that conjecture is free to choose $\lambda$ small, and by the same deduction it still implies $R(3,k)\le(\tfrac12+o(1))k^2/\log k$.
Finally, $\alpha(G_m)/|V(G_m)|=2/5$, so these graphs lie in the range where the general bound of \citet{davies2018average} already forces the ratio to be at least $7/5$; they say nothing about triangle-free graphs with $\alpha(G)=o(|V(G)|)$, which is the range relevant to $R(3,k)$.

We now summarize the computation.
An independent set of $C_5\,\square\,K_{m,m}$ is a family of independent sets of $C_5$, one over each vertex of $K_{m,m}$, subject only to the requirement that the union of those over one side be disjoint from the union of those over the other.
Classifying such a family by its pair of unions and inverting over the Boolean lattice $2^{\mathbb{Z}_5}$ writes the independence polynomial of $G_m$ as a signed sum of $3^5$ terms $(F_XF_Y)^m$ indexed by the disjoint pairs $X,Y\subseteq\mathbb{Z}_5$, where $F_X$ is the independence polynomial of the subgraph of $C_5$ induced on $X$.
The key point is that $F_X(1)F_Y(1)$ is maximized only when one of $X,Y$ is a nonadjacent pair of $C_5$ and the other is its complement.
The mean $\tfrac32+\tfrac23=\tfrac{13}{6}$ attached to that configuration is the average number of chosen vertices per pair of opposite fibers, which gives $\avgalpha_{G_m}(1)\approx\tfrac{13}{6}m$.

\subsubsection{The construction}

In the \emph{Cartesian product} $G\,\square\,H$ the vertex set is $V(G)\times V(H)$, and $(g,h)$ is adjacent to $(g',h')$ exactly when $g=g'$ and $hh'\in E(H)$, or $h=h'$ and $gg'\in E(G)$.
Degrees add, so $G_m=C_5\,\square\,K_{m,m}$ is $(m+2)$-regular on $5\cdot 2m=10m$ vertices, and in particular $\delta(G_m)=m+2$.

\begin{lemma}\label{perkins_gang:lem:trianglefree}
The Cartesian product of two triangle-free graphs is triangle-free.
In particular $G_m$ is triangle-free.
\end{lemma}

\begin{proof}
Every edge of $G\,\square\,H$ changes exactly one coordinate.
Let $u,v,w$ span a triangle.
By the pigeonhole principle two of its three edges change the same coordinate, say the first, and these two edges share a vertex, say $uv$ and $vw$.
Then $u,v,w$ all have the same second coordinate, so the third edge $uw$ also changes only the first coordinate.
Hence the first coordinates of $u,v,w$ are pairwise distinct and pairwise adjacent, giving a triangle in $G$.
The same argument with the coordinates exchanged gives a triangle in $H$ when the repeated coordinate is the second.
Since $C_5$ and $K_{m,m}$ are triangle-free, so is $G_m$.
\end{proof}

Throughout, $\mathbb{Z}_5=\{0,1,2,3,4\}$ is the vertex set of $C_5$, with $i$ adjacent to $i\pm1$, and $L$ and $R$ are the two sides of $K_{m,m}$, each of size $m$.
For $X\subseteq\mathbb{Z}_5$ we write $\I(X)$ for the family of independent sets of the induced subgraph $C_5[X]$ and put
\[
  F_X(y)=\sum_{S\in\I(X)}y^{|S|},
  \qquad
  f(X)=F_X(1)=|\I(X)| ,
\]
so that $F_X$ is the \emph{independence polynomial} of $C_5[X]$.
Observe that $f$ is monotone: $X\subseteq X'$ implies $f(X)\le f(X')$.

The \emph{fiber} $\mathbb{Z}_5\times\{x\}$ over a vertex $x$ of $K_{m,m}$ induces a copy of $C_5$, and the two fibers over adjacent $x,x'$ are joined by a perfect matching that preserves the $\mathbb{Z}_5$-coordinate, as in Figure~\ref{perkins_gang:fig:fibers}.
Hence an independent set $I$ of $G_m$ is the same thing as a family $(S_x)_{x\in L\cup R}$ with
\[
  S_x\in\I(\mathbb{Z}_5)\ \text{for every }x,
  \qquad
  S_x\cap S_{x'}=\varnothing\ \text{whenever }xx'\in E(K_{m,m}),
\]
and then $|I|=\sum_x|S_x|$.
Since $K_{m,m}$ is complete bipartite, the disjointness constraints say precisely that
\begin{equation}\label{perkins_gang:eq:constraint}
  \Bigl(\bigcup_{x\in L}S_x\Bigr)\cap\Bigl(\bigcup_{x\in R}S_x\Bigr)=\varnothing .
\end{equation}

\begin{figure}[t]
\centering
\begin{tikzpicture}[
    scale=1,
    onv/.style={circle,draw,fill=black!65,inner sep=1.7pt},
    offv/.style={circle,draw,fill=white,inner sep=1.7pt},
    halo/.style={circle,fill=black!12,minimum size=15pt,inner sep=0pt}
]
  \node[halo] at (0,2) {};
  \node[halo] at (0,0) {};
  \node[halo] at (3.6,1) {};
  \node[halo] at (3.6,-1) {};
  \node[halo] at (3.6,-2) {};
  \node[onv,label=left:$0$]  (L0) at (0,2)  {};
  \node[offv,label=left:$1$] (L1) at (0,1)  {};
  \node[onv,label=left:$2$]  (L2) at (0,0)  {};
  \node[offv,label=left:$3$] (L3) at (0,-1) {};
  \node[offv,label=left:$4$] (L4) at (0,-2) {};
  \node[offv,label=right:$0$] (R0) at (3.6,2)  {};
  \node[onv,label=right:$1$]  (R1) at (3.6,1)  {};
  \node[offv,label=right:$2$] (R2) at (3.6,0)  {};
  \node[onv,label=right:$3$]  (R3) at (3.6,-1) {};
  \node[offv,label=right:$4$] (R4) at (3.6,-2) {};
  \draw (L0) -- (L1) -- (L2) -- (L3) -- (L4);
  \draw (L0) .. controls (-1.1,2.3) and (-1.1,-2.3) .. (L4);
  \draw (R0) -- (R1) -- (R2) -- (R3) -- (R4);
  \draw (R0) .. controls (4.7,2.3) and (4.7,-2.3) .. (R4);
  \foreach \i in {0,1,2,3,4}{\draw[gray] (L\i) -- (R\i);}
  \node at (0,2.9) {$x\in L$};
  \node at (3.6,2.9) {$x'\in R$};
\end{tikzpicture}
\caption{Two fibers of $G_m=C_5\,\square\,K_{m,m}$, over adjacent vertices $x\in L$ and $x'\in R$.
Each fiber induces a copy of $C_5$, and the gray matching between them preserves the $\mathbb{Z}_5$-coordinate, so the independent sets chosen in fibers on opposite sides must be disjoint.
The gray halos mark a dominant splitting from Lemma~\ref{perkins_gang:lem:extremal}: every fiber over $L$ takes its independent set inside the nonadjacent pair $X=\{0,2\}$ and every fiber over $R$ inside the complement $Y=\{1,3,4\}$, and this is the splitting of $\mathbb{Z}_5$ that maximizes $f(X)f(Y)$.
The filled vertices show one admissible choice inside them, namely $\{0,2\}$ over $x$ and $\{1,3\}$ over $x'$.}
\label{perkins_gang:fig:fibers}
\end{figure}
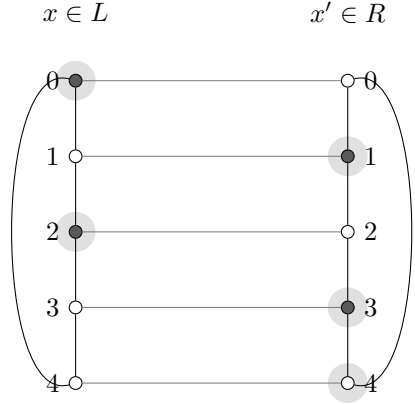

\begin{lemma}\label{perkins_gang:lem:alpha}
$\alpha(G_m)=4m$.
\end{lemma}

\begin{proof}
Each $S_x$ is an independent set of $C_5$, hence $|S_x|\le2$, and summing over the $2m$ fibers gives $\alpha(G_m)\le4m$.
For the lower bound take $S_x=\{0,2\}$ for every $x\in L$ and $S_x=\{1,3\}$ for every $x\in R$.
Both sets are independent in $C_5$ and they are disjoint, so~\eqref{perkins_gang:eq:constraint} holds and the resulting independent set has size $4m$.
\end{proof}

\subsubsection{Exact counting by M\"obius inversion}

Let
\[
  Z_m(y)=\sum_{I\in\I(G_m)}y^{|I|}
\]
be the independence polynomial of $G_m$, so that $|\I(G_m)|=Z_m(1)$ and $\avgalpha_{G_m}(1)=Z_m'(1)/Z_m(1)$.

One is tempted to sum $\bigl(F_X(y)F_Y(y)\bigr)^m$ over the disjoint pairs $(X,Y)$, reading $X$ and $Y$ as the two unions in~\eqref{perkins_gang:eq:constraint}.
This overcounts: the quantity $\bigl(F_X(y)F_Y(y)\bigr)^m$ records the families with $S_x\subseteq X$ for $x\in L$ and $S_x\subseteq Y$ for $x\in R$, so a single independent set is counted once for every pair $(X,Y)$ whose two parts contain its two unions.
The remedy is to force the unions to be attained exactly.
For $X\subseteq\mathbb{Z}_5$ define
\[
  g_X(y)=\sum_{\substack{(S_1,\dots,S_m)\in\I(\mathbb{Z}_5)^m\\ S_1\cup\dots\cup S_m=X}}
        y^{|S_1|+\dots+|S_m|},
\]
the generating function of the $m$-tuples of independent sets of $C_5$ whose union is exactly $X$.
A tuple has union contained in $X$ if and only if every $S_i$ lies in $\I(X)$, so
\[
  \sum_{X'\subseteq X}g_{X'}(y)=F_X(y)^m .
\]
M\"obius inversion in the Boolean lattice $2^{\mathbb{Z}_5}$ therefore gives
\begin{equation}\label{perkins_gang:eq:mobius}
  g_X(y)=\sum_{X'\subseteq X}(-1)^{|X\setminus X'|}F_{X'}(y)^m .
\end{equation}
Classifying an independent set of $G_m$ by the ordered pair of unions in~\eqref{perkins_gang:eq:constraint} partitions $\I(G_m)$, and yields the exact identity
\begin{equation}\label{perkins_gang:eq:partition}
  Z_m(y)=\sum_{\substack{X,Y\subseteq\mathbb{Z}_5\\ X\cap Y=\varnothing}}g_X(y)\,g_Y(y).
\end{equation}
Substituting~\eqref{perkins_gang:eq:mobius} into~\eqref{perkins_gang:eq:partition} collapses to a sum over the same index set with explicit signs.

\begin{lemma}\label{perkins_gang:lem:closed}
For every $m\ge1$,
\[
  Z_m(y)=\sum_{\substack{X,Y\subseteq\mathbb{Z}_5\\ X\cap Y=\varnothing}}
  (-1)^{\,5-|X|-|Y|}\bigl(F_X(y)F_Y(y)\bigr)^{m}.
\]
\end{lemma}

\begin{proof}
If $X\cap Y=\varnothing$ and $X'\subseteq X$, $Y'\subseteq Y$, then $X'\cap Y'=\varnothing$, so expanding~\eqref{perkins_gang:eq:mobius} in~\eqref{perkins_gang:eq:partition} produces a linear combination of the terms $\bigl(F_{X'}F_{Y'}\bigr)^m$ indexed by disjoint pairs $(X',Y')$.
Fix such a pair and set $W=\mathbb{Z}_5\setminus(X'\cup Y')$.
The pairs $(X,Y)$ contributing to it are exactly $X=X'\cup A$ and $Y=Y'\cup B$ with $A,B$ disjoint subsets of $W$, and each contributes the sign $(-1)^{|A|+|B|}$.
Assigning each $w\in W$ independently to $A$, to $B$, or to neither, with respective weights $-1,-1,+1$, the coefficient equals
\[
  \sum_{\substack{A,B\subseteq W\\ A\cap B=\varnothing}}(-1)^{|A|+|B|}
  =\prod_{w\in W}(1-1-1)=(-1)^{|W|}=(-1)^{\,5-|X'|-|Y'|}. \qedhere
\]
\end{proof}

\subsubsection{The dominant pairs}

By Lemma~\ref{perkins_gang:lem:closed} the exponential growth of $Z_m(1)$ is governed by $\max f(X)f(Y)$ over disjoint pairs.
We determine both the maximum and the next value, the latter being what controls the error term.

\begin{lemma}\label{perkins_gang:lem:extremal}
Let $X,Y\subseteq\mathbb{Z}_5$ be disjoint.
Then $f(X)f(Y)\le24$, with equality exactly when one of $X,Y$ is a pair of nonadjacent vertices of $C_5$ and the other is its complement, which happens for $10$ ordered pairs.
Moreover, if $f(X)f(Y)\neq24$ then $f(X)f(Y)\le16$.
\end{lemma}

\begin{proof}
A direct count gives the value of $f$ on every induced subgraph of $C_5$: it is $1$ on $\varnothing$, $2$ on a single vertex, $3$ on an edge and $4$ on a nonedge, $5$ on $P_3$ and $6$ on $K_2\cup K_1$, $8$ on $P_4$, and $11$ on $C_5$ itself.
Since $f$ is monotone and $X\cap Y=\varnothing$, replacing $Y$ by $\mathbb{Z}_5\setminus X$ can only increase $f(X)f(Y)$.
It therefore suffices to tabulate the complementary pairs $(X,\mathbb{Z}_5\setminus X)$, up to the rotational symmetry of $C_5$ and up to swapping the two parts:
\[
  \begin{array}{c|c|c|c}
    |X| & C_5[X] & C_5[\mathbb{Z}_5\setminus X] & f(X)f(\mathbb{Z}_5\setminus X)\\\hline
    0 & \varnothing & C_5 & 1\cdot11=11\\
    1 & K_1 & P_4 & 2\cdot8=16\\
    2,\ X\ \text{an edge} & K_2 & P_3 & 3\cdot5=15\\
    2,\ X\ \text{a nonedge} & \overline{K_2} & K_2\cup K_1 & 4\cdot6=24
  \end{array}
\]
The cases $|X|\ge3$ are obtained by swapping the two parts.
Hence $\max f(X)f(Y)=24$ over all disjoint pairs, attained on complementary pairs only in the stated configuration; there are $5$ nonadjacent pairs in $C_5$, giving $10$ ordered pairs.

Suppose now $f(X)f(Y)>16$.
Enlarging $Y$ to $\mathbb{Z}_5\setminus X$ gives $f(X)f(\mathbb{Z}_5\setminus X)>16$, so by the table $X$ is either a nonadjacent pair or the complement of one.
If $X=\{a,a+2\}$ is the nonadjacent pair then $f(X)=4$, so $f(Y)>4$; but $Y$ is contained in $\mathbb{Z}_5\setminus X$, an edge plus an isolated vertex, whose proper subsets all have $f\le4$, so $Y=\mathbb{Z}_5\setminus X$ and $f(X)f(Y)=24$.
If instead $f(X)=6$ then $f(Y)>16/6$ with $Y$ contained in a nonadjacent pair, whence $f(Y)\in\{1,2,4\}$ and only $f(Y)=4$ survives, again giving $24$.
Therefore every value other than $24$ is at most $16$.
\end{proof}

For a dominant pair, say $X=\{0,2\}$ and $Y=\{1,3,4\}$, the induced subgraphs are $\overline{K_2}$ and $K_2\cup K_1$, so
\[
  F_X(y)=(1+y)^2,
  \qquad
  F_Y(y)=(1+y)(1+2y),
  \qquad
  P(y):=F_X(y)F_Y(y)=(1+y)^3(1+2y).
\]
Thus $P(1)=8\cdot3=24$ and
\begin{equation}\label{perkins_gang:eq:mean}
  \frac{P'(1)}{P(1)}=\Bigl.\frac{3}{1+y}\Bigr|_{y=1}+\Bigl.\frac{2}{1+2y}\Bigr|_{y=1}
  =\frac32+\frac23=\frac{13}{6}.
\end{equation}
All five rotations of the pair $(X,Y)$, in both orders, give the same polynomial $P$.
Since $P'(1)/P(1)$ is the sum of the mean sizes of a uniformly random element of $\I(X)$ and of $\I(Y)$, the constant $13/6$ has a per-vertex reading: a vertex of $L$ contributes on average $1$, the mean of $|S|$ over the four independent subsets of a nonadjacent pair, and a vertex of $R$ contributes $7/6$, the mean of $|S|$ over the six independent subsets of an edge plus an isolated vertex.

\subsubsection{Proof of the main theorem}

\begin{proof}[Proof of Theorem~\ref{perkins_gang:thm:main}]
The graph-theoretic assertions are Lemmas~\ref{perkins_gang:lem:trianglefree} and~\ref{perkins_gang:lem:alpha} together with the degree count above.
Split the sum of Lemma~\ref{perkins_gang:lem:closed} into the $10$ dominant pairs, each of which contributes $+P(y)^m$ by Lemma~\ref{perkins_gang:lem:extremal} and the sign $(-1)^{5-2-3}=+1$, and a remainder:
\[
  Z_m(y)=10\,P(y)^m+E_m(y),
  \qquad
  E_m(y)=\sum_{\substack{(X,Y)\ \text{disjoint}\\ \text{not dominant}}}
  (-1)^{5-|X|-|Y|}\bigl(F_X(y)F_Y(y)\bigr)^m.
\]
Assigning each element of $\mathbb{Z}_5$ to $X$, to $Y$, or to neither shows that there are $3^5=243$ ordered disjoint pairs in all, and by Lemma~\ref{perkins_gang:lem:extremal} every nondominant one has $F_X(1)F_Y(1)\le16$.
Hence
\[
  |E_m(1)|\le 243\cdot16^m .
\]
For the derivative, each nondominant term satisfies
\[
  \Bigl|\tfrac{d}{dy}\bigl(F_X(y)F_Y(y)\bigr)^m\Bigr|_{y=1}
  = m\bigl(F_X(1)F_Y(1)\bigr)^{m}\cdot\frac{(F_XF_Y)'(1)}{F_X(1)F_Y(1)}
  \le 4m\cdot16^m,
\]
because $(F_XF_Y)'(1)/(F_XF_Y)(1)$ is a sum of two mean independent set sizes in $C_5$, each at most $2$.
Hence $|E_m'(1)|\le 972\,m\,16^m$.
Using $P(1)=24$ and~\eqref{perkins_gang:eq:mean},
\[
  Z_m(1)=10\cdot24^m\Bigl(1+O\bigl((2/3)^m\bigr)\Bigr),
  \qquad
  Z_m'(1)=\frac{13}{6}\,m\cdot10\cdot24^m\Bigl(1+O\bigl((2/3)^m\bigr)\Bigr),
\]
since $972\,m\,16^m$ divided by $\tfrac{13}{6}m\cdot10\cdot24^m$ is $O\bigl((2/3)^m\bigr)$.
Dividing,
\[
  \avgalpha_{G_m}(1)=\frac{Z_m'(1)}{Z_m(1)}
  =\frac{13}{6}\,m\Bigl(1+O\bigl((2/3)^m\bigr)\Bigr),
\]
and with $\alpha(G_m)=4m$ the ratio tends to $4/(13/6)=24/13$, completing the proof.
\end{proof}

\begin{remark}\label{perkins_gang:rem:small}
The exact values, from Lemma~\ref{perkins_gang:lem:closed}, are
\[
  \begin{array}{c|c|c|c|c}
    m & |\I(G_m)| & \sum_{I}|I| & \avgalpha_{G_m}(1) & \alpha(G_m)/\avgalpha_{G_m}(1)\\\hline
    1 & 81      & 200      & 200/81      & 81/50=1.6200\ldots\\
    2 & 3561    & 16260    & 5420/1187   & 2374/1355=1.7520\ldots\\
    3 & 113541  & 746850   & 248950/37847& 227082/124475=1.8243\ldots
  \end{array}
\]
The ratio is not monotone in $m$: it climbs above its limit $24/13=1.8461\ldots$, and its largest value $426826898/228757635=1.8658\ldots$ is attained at $m=6$.
It never exceeds $1.87$, since the bounds on $E_m(1)$ and $E_m'(1)$ in the proof of Theorem~\ref{perkins_gang:thm:main} give
\[
  \frac{\alpha(G_m)}{\avgalpha_{G_m}(1)}
  \le\frac{4\bigl(10+243(2/3)^m\bigr)}{\tfrac{65}{3}-972(2/3)^m}<1.864
  \qquad(m\ge22),
\]
and the remaining $m\le21$ are evaluated directly from Lemma~\ref{perkins_gang:lem:closed}.
\end{remark}

\begin{remark}\label{perkins_gang:rem:variants}
The same computation applies with $C_5$ replaced by any triangle-free graph $H$, the sum in Lemma~\ref{perkins_gang:lem:closed} then running over the disjoint pairs of subsets of $V(H)$ with sign $(-1)^{|V(H)|-|X|-|Y|}$.
Take for $H$ the circulant $C_{13}(1,5)$ on $\mathbb{Z}_{13}$, in which $i$ is joined to $i\pm1$ and $i\pm5$; it is $4$-regular and triangle-free with $\alpha(H)=4$, so it witnesses $R(3,5)\ge14$, and the matching upper bound is due to \citet{greenwood1955combinatorial}.
Then $C_{13}(1,5)\,\square\,K_{m,m}$ is triangle-free and $(m+4)$-regular, and it has independence number $8m$ because $H$ contains two disjoint independent $4$-sets.
Over disjoint pairs $X,Y\subseteq\mathbb{Z}_{13}$ the maximum of $f(X)f(Y)$ is $1728=36\cdot48$, attained at exactly $52$ ordered pairs, all complementary with $\{|X|,|Y|\}=\{6,7\}$ and hence of sign $+1$; the next value is $1440$, and every dominant pair has mean $19/4$ in place of $13/6$.
The argument above then gives
\[
  \frac{\alpha}{\avgalpha}\longrightarrow\frac{8}{19/4}=\frac{32}{19}=1.6842\ldots
\]
with relative error $O\bigl((5/6)^m\bigr)$.
Replacing $C_5$ by $[q]$ with clique fibers for $q\ge2$, that is taking the graph on $(L\cup R)\times[q]$ in which each fiber $\{x\}\times[q]$ is a $K_q$ and $(x,c)$ is joined to $(x',c)$ for all $xx'\in E(K_{m,m})$, gives $K_{q+1}$-free graphs of minimum degree $m+q-1$ whose ratio tends to $\rho_q=2/\bigl(\frac{a}{a+1}+\frac{b}{b+1}\bigr)$ with $a=\lfloor q/2\rfloor$ and $b=\lceil q/2\rceil$; here $q=2$ gives the triangle-free graph $K_2\,\square\,K_{m,m}$ and the value $2$, while $q=3$ gives $12/7$.
Finally, for a disjoint union independence numbers add and independence polynomials multiply, so $\avgalpha$ adds as well and the ratio of a disjoint union of $t$ copies of $G_m$ equals that of $G_m$; consequently, for each fixed minimum degree $d=m+2$ that same ratio is realized by triangle-free graphs on arbitrarily many vertices, and the counterexample is not an artifact of the number of vertices being tied to the degree.
\end{remark}

\putbib

\end{bibunit}
\endgroup
\subsection[(Erd\H{o}s) Sets with no large divisor difference]{Sets with no large divisor difference}
\label{sol:erdos}
\begingroup
\begin{bibunit}

This problem asks how large a set $A\subseteq[n]$ can be if no difference $b-a\ge t$ between two of its elements divides $b$.
The odd numbers form such a set of size $\lceil n/2\rceil$, and Erd\H{o}s asked whether $\lvert A\rvert\le\bigl(\tfrac12+o_t(1)\bigr)n$ must hold.
We prove that it does, with the explicit error term $3n/\sqrt{\log\log n}$.
The proof weights each integer by the number of its prime factors drawn from the odd primes in $(t,n^{1/3}]$, compares the even elements of $A$ with the odd integers omitted from it through the injection $a\mapsto a-p$, and converts the resulting weighted inequality into a bound on $\lvert A\rvert$ by Cauchy--Schwarz.

\subsubsection{Introduction}

Erd\H{o}s asked, in a letter to Ruzsa written around 1980, how dense a set of integers can be if no sufficiently large difference between two of its elements divides one of them.
The question is recorded in the miscellany of Erd\H{o}s problems of \citet{guy1983unsolved} and in Ruzsa's survey of Erd\H{o}s's work on the integers \citep{ruzsa1999erdHos}, and is Problem~635 on the Erd\H{o}s problems website \citep{erdosproblems635}, where it reads as follows.

\medskip
\noindent
\emph{Let $t\geq 1$ and $A\subseteq \{1,\ldots,N\}$ be such that whenever $a,b\in A$ with $b-a\geq t$ we have $b-a\nmid b$.
How large can $\lvert A\rvert$ be?
Is it true that $\lvert A\rvert\leq\left(\frac12+o_t(1)\right)N$?}
\medskip

We write $n$ throughout for the integer called $N$ there.
We answer the second question affirmatively, with an explicit rate.

Fix an integer $t\ge1$.
Call a set $A\subseteq[n]$ \emph{$t$-admissible} if no two elements $a<b$ of $A$ satisfy $b-a\ge t$ and $b-a\mid b$, and write $F(n;t)$ for the largest size of a $t$-admissible subset of $[n]$.
Since $b=a+(b-a)$, the divisibility $b-a\mid b$ is equivalent to $b-a\mid a$, so the forbidden configuration is symmetric in the two elements; we call such a difference a \emph{divisor difference} of the pair.
Enlarging $t$ constrains fewer pairs, so $F(n;t)$ is nondecreasing in $t$.

Two cases are immediate.
A $1$-admissible set contains no two consecutive integers, since a difference of $1$ is at least $t=1$ and divides everything, so $F(n;1)\le\lceil n/2\rceil$.
The odd numbers in $[n]$ are $1$-admissible, because the difference of two odd numbers is even and an even number cannot divide an odd one, so in fact $F(n;1)=\lceil n/2\rceil$, and $F(n;t)\ge\lceil n/2\rceil$ for every $t$.
Once $n\ge2$, the odd numbers stop being optimal as soon as $t\ge2$.
Erd\H{o}s observed that the set
\[
  A=\{m\le n:m\text{ odd}\}\cup\{2^k\le n:k\text{ odd}\}
\]
is $2$-admissible \citep{erdosproblems635}.
Indeed, pairs of odd elements are handled by the previous paragraph.
Two powers $2^k<2^j$ have $j-k\ge2$, since $j$ and $k$ are distinct odd numbers, so their difference $2^k(2^{j-k}-1)$ has odd part $2^{j-k}-1>1$ and therefore does not divide $2^j$.
If exactly one of $a<b$ is an added power, then $b-a$ is odd, while the equivalent divisibilities $b-a\mid a$ and $b-a\mid b$ make $b-a$ a divisor of that power of two, forcing $b-a=1<2$.
Since the two sets above are disjoint, this gives
\[
  F(n;t)\ \ge\ \Bigl\lceil\frac n2\Bigr\rceil+\frac{\log_2 n-1}{2}
  \qquad\text{for every }t\ge 2,
\]
so the inequality asked for cannot hold with $o_t(1)$ replaced by $0$.

The second question was answered affirmatively in January 2026 by Liam Price with ChatGPT-5.2, as recorded on the problem's page \citep{erdosproblems635}; the first question, which asks for the exact value of $F(n;t)$, remains open.
Tao observed that an affirmative answer also follows quickly from an inequality of \citet{elliott2012probabilistic}, which bounds a weighted mean square, over the small primes $p$, of the difference between the average of an arbitrary function on an interval $I$ and its average along the multiples of $p$ in $I$, by the mean square of that function on $I$.
Elliott's inequality can be substituted for Lemma~\ref{erdos:lem:moments} below.
Tao also pointed out that the graph in which $a$ and $b$ are joined whenever $b-a$ divides $b$, once $b-a$ is restricted to primes or almost primes, is closely related to the divisibility graphs through which pair correlations of multiplicative functions are studied \citep{matomaki2016sign,helfgott2021expansion}; that literature is concerned with connectivity and expansion rather than with independent sets, which is what Erd\H{o}s's question asks about \citep{erdosproblems635}.

\begin{theorem}
\label{erdos:thm:main}
For every integer $t\ge1$ there is an $n_0(t)$ such that
\[
  \Bigl\lceil\frac n2\Bigr\rceil
  \ \le\ F(n;t)\ \le\
  \frac n2+\frac{3n}{\sqrt{\log\log n}}
\]
for all $n\ge n_0(t)$.
In particular $F(n;t)=\bigl(\tfrac12+o_t(1)\bigr)n$ for each fixed $t$.
\end{theorem}

\begin{remark}
\label{erdos:rem:independence}
The proof of Theorem~\ref{erdos:thm:main} given below was found independently of, and later than, the answer recorded on the problem's page \citep{erdosproblems635}, which has priority.
The inequality behind Tao's alternative route, due to \citet{elliott2012probabilistic}, is contemporaneous with Erd\H{o}s's question, and since the deduction from it is short Tao notes that the statement may already be present in the literature \citep{erdosproblems635}.
\end{remark}

\begin{remark}
\label{erdos:rem:rate}
Erd\H{o}s's construction above gives $F(n;t)\ge n/2+\Omega(\log n)$ for $t\ge2$, whereas the error term proved here is of size $n/\sqrt{\log\log n}$.
Tao has noted that arguments of Elliott's type give an error term $O_t(n/\log\log n)$, still far above the $\Omega(\log n)$ of the construction \citep{erdosproblems635}.
\end{remark}

The upper bound rests on a single injection.
Let $p$ be an odd prime with $p>t$.
If $a$ is an even multiple of $p$, then $a-p$ is an odd multiple of $p$ at distance $p\ge t$ from $a$, and $p$ divides $a$; so $a$ and $a-p$ cannot both lie in a $t$-admissible set.
Summing this over a set $P$ of odd primes larger than $t$ weights each even element of a $t$-admissible set $A$, and each odd element excluded from $A$, by its number of prime factors drawn from $P$.
That weight has mean $\kappa=\sum_{p\in P}1/p$ on each parity class, and its total square deviation there is at most $n\kappa/2+O(\lvert P\rvert^2)$, so Cauchy--Schwarz converts the weighted inequality into a bound of size $n/\sqrt\kappa$ for the excess of even elements of $A$ over missing odd elements.
Taking $P$ to be the odd primes in $(t,n^{1/3}]$ makes $\kappa$ as large as $\log\log n$.

\subsubsection{A second moment estimate}

The following elementary estimate will be used.

\begin{lemma}
\label{erdos:lem:moments}
Let $P$ be a finite set of odd primes, and put
\[
  \kappa=\sum_{p\in P}\frac1p,
  \qquad
  \omega_P(m)=\#\{p\in P:p\mid m\}.
\]
Then for every $n\ge1$,
\[
  \sum_{\substack{m\le n\\ m\text{ even}}}\bigl(\omega_P(m)-\kappa\bigr)^2\le\frac{n\kappa}{2}+2\lvert P\rvert^2
  \qquad\text{and}\qquad
  \sum_{\substack{m\le n\\ m\text{ odd}}}\bigl(\omega_P(m)-\kappa\bigr)^2\le\frac{n\kappa}{2}+2\lvert P\rvert^2 .
\]
\end{lemma}

\begin{proof}
Let $\mathcal M$ be either the set of even integers in $[n]$ or the set of odd ones, and for an odd squarefree $d$ put $N(d)=\#\{m\in\mathcal M:d\mid m\}$.
Since $d$ is odd, $N(d)=\lfloor n/2d\rfloor$ in the even case and $N(d)=\lfloor n/d\rfloor-\lfloor n/2d\rfloor$ in the odd case.
In both cases
\[
  N(d)=\frac{n}{2d}+\delta_d,
  \qquad
  \lvert\delta_d\rvert\le1 .
\]
For distinct $p,q\in P$ we have $\mathbf 1_{\{p\mid m\}}\mathbf 1_{\{q\mid m\}}=\mathbf 1_{\{pq\mid m\}}$, so expanding the square gives
\begin{align*}
  \sum_{m\in\mathcal M}\bigl(\omega_P(m)-\kappa\bigr)^2
  &=\sum_{\substack{p,q\in P\\ p\ne q}}N(pq)+(1-2\kappa)\sum_{p\in P}N(p)+\kappa^2N(1)\\
  &=\frac n2\Bigl(\kappa^2-\sum_{p\in P}\frac1{p^2}\Bigr)
    +\frac{n\kappa}{2}-n\kappa^2
    +\frac{n\kappa^2}{2}
    +\mathcal{E}\\
  &=\frac{n\kappa}{2}-\frac n2\sum_{p\in P}\frac1{p^2}+\mathcal{E},
\end{align*}
where
\[
  \mathcal{E}=\sum_{\substack{p,q\in P\\ p\ne q}}\delta_{pq}+(1-2\kappa)\sum_{p\in P}\delta_p+\kappa^2\delta_1 .
\]
Every $p\in P$ is odd, so $\kappa\le\lvert P\rvert/3$ and therefore
\[
  \lvert\mathcal{E}\rvert
  \le\bigl(\lvert P\rvert^2-\lvert P\rvert\bigr)+(1+2\kappa)\lvert P\rvert+\kappa^2
  \le\Bigl(1+\frac23+\frac19\Bigr)\lvert P\rvert^2
  \le 2\lvert P\rvert^2 .
\]
Discarding the negative term $-\frac n2\sum_{p\in P}p^{-2}$ completes the proof.
\end{proof}

\subsubsection{Proof of the main theorem}

\begin{proof}[Proof of Theorem~\ref{erdos:thm:main}]
The lower bound is the set of odd numbers in $[n]$, as noted above.

For the upper bound, fix $t\ge1$ and set
\[
  z=n^{1/3},
  \qquad
  P=\{p\text{ an odd prime}:t<p\le z\},
  \qquad
  \kappa=\sum_{p\in P}\frac1p .
\]
By Mertens' theorem $\sum_{p\le z}1/p=\log\log z+O(1)$, and $\log\log z=\log\log n-\log3$, so $\kappa=\log\log n+O_t(1)$.
Choose $n_0(t)$ so that $\kappa\ge\max\bigl(1,\tfrac12\log\log n\bigr)$, $n^{1/3}\ge4$, and $\sqrt{\log\log n}\le n$ for all $n\ge n_0(t)$.

Let $A\subseteq[n]$ be $t$-admissible and put $B=[n]\setminus A$.
We claim that for each $p\in P$,
\[
  \#\{a\in A:a\text{ even},\ p\mid a\}
  \ \le\
  \#\{b\in B:b\text{ odd},\ p\mid b\}.
\]
Indeed, let $a\in A$ be even with $p\mid a$.
As $p$ is odd, $a=2rp$ for some $r\ge1$, and $a-p=(2r-1)p$ is an odd multiple of $p$ lying in $[1,n]$.
The pair $a-p<a$ has difference $p>t$ and $p\mid a$, so $a-p\notin A$, that is, $a-p\in B$.
The map $a\mapsto a-p$ is injective, which proves the claim.

Summing the claim over $p\in P$ gives
\begin{equation}
\label{erdos:eq:weighted}
  \sum_{\substack{a\in A\\ a\text{ even}}}\omega_P(a)
  \ \le\
  \sum_{\substack{b\in B\\ b\text{ odd}}}\omega_P(b).
\end{equation}
Write
\[
  x=\#\{a\in A:a\text{ even}\},
  \qquad
  y=\#\{b\in B:b\text{ odd}\}.
\]
Subtracting $\kappa$ from every summand in \eqref{erdos:eq:weighted} and rearranging,
\[
  \kappa(x-y)
  \ \le\
  \sum_{\substack{b\in B\\ b\text{ odd}}}\bigl(\omega_P(b)-\kappa\bigr)
  -\sum_{\substack{a\in A\\ a\text{ even}}}\bigl(\omega_P(a)-\kappa\bigr).
\]
Apply Cauchy--Schwarz to each of the two sums, extend the resulting sums of squares to all integers of the relevant parity in $[n]$, and invoke Lemma~\ref{erdos:lem:moments}:
\[
  \kappa(x-y)
  \ \le\
  \bigl(\sqrt x+\sqrt y\bigr)\Bigl(\frac{n\kappa}{2}+2\lvert P\rvert^2\Bigr)^{1/2}.
\]
Now $\lvert P\rvert\le z=n^{1/3}$, and $n^{1/3}\ge4\ge4/\kappa$, so $2\lvert P\rvert^2\le2n^{2/3}\le n\kappa/2$.
Moreover $x\le n/2$ and $y\le(n+1)/2$, whence $\sqrt x+\sqrt y\le\sqrt{2(n+1)}\le\sqrt{3n}$.
Therefore $\kappa(x-y)\le\sqrt{3n}\cdot\sqrt{n\kappa}$, that is,
\[
  x-y\le n\sqrt{\frac3\kappa}.
\]
The set $A$ has exactly $\lceil n/2\rceil-y$ odd elements and $x$ even elements, so
\[
  \lvert A\rvert
  =\Bigl\lceil\frac n2\Bigr\rceil-y+x
  \le\Bigl\lceil\frac n2\Bigr\rceil+n\sqrt{\frac3\kappa}
  \le\frac{n+1}{2}+n\sqrt{\frac{6}{\log\log n}}
  \le\frac n2+\frac{3n}{\sqrt{\log\log n}},
\]
the last step because $\sqrt{\log\log n}\le n$ and $3-\sqrt6>\tfrac12$.
Since $A$ was an arbitrary $t$-admissible subset of $[n]$, this is the desired bound on $F(n;t)$.
\end{proof}

\putbib
\end{bibunit}
\endgroup
\subsection[(Erd\H{o}s--Straus) Divisibility among binomial coefficients]{Divisibility among binomial coefficients}
\label{sol:erdos_straus}
\begingroup
\begin{bibunit}
\def\N{\mathbb{N}}

This problem asks for the density of the set of $m$ for which some admissible $k$ makes a product of $k$ consecutive integers starting just above $n$ divide the product of $k$ consecutive integers starting just above $m$.
\citet{erdos1977products} asked whether all, or almost all, large $m$ admit such a $k$, and if not what the density of those that do is; we show that the density exists and equals $1$ for every fixed $n\ge2$.
The proof produces, for each $B$, a single $k$ for which $\binom{n+k}{n}$ has no prime factor below $B$, so that the required divisibility is forced by one congruence condition on $m$ modulo each prime divisor of that coefficient.

\subsubsection{Introduction}

For positive integers $n$ and $k$ write $A(n,k)=(n+k)!/n!$ for the product of the $k$ consecutive integers $n+1,\dots,n+k$.
\citet{erdos1977products} study the divisibility relation $A(n,k)\mid A(m,k)$ for $m>n$.
Since
\[
  \frac{A(m,k)}{A(n,k)}=\binom{m+k}{k}\Bigm/\binom{n+k}{k},
\]
that relation is exactly the divisibility
\begin{equation}\label{erdos_straus:eq:14}
  \binom{n+k}{n}\ \Bigm|\ \binom{m+k}{k}
\end{equation}
between two binomial coefficients.
Bounding $k$ is essential here.
Indeed, for $k>m-n$ one has
\[
  \frac{A(m,k)}{A(n,k)}=\frac{A(n+k,m-n)}{A(n,m-n)},
\]
and \citet{erdos1977products} note that this quotient is an integer for $k=A(n,m-n)-m$, a value exceeding $m-n$ as soon as $m\ge n+2$.
For $m=n+1$ that value is $0$; there one may take $k=n+1$ instead, with quotient $2$.
Without a restriction on $k$ the problem is therefore vacuous.
With the natural restriction $1\le k\le m-n$ in place, they ask the following question in their \S1, in which the displayed divisibility is their (1.4).

\medskip
\noindent
\emph{Given $n>1$ is it true that for all (almost all) large $m$ there exists a $k$, $1\le k\le m-n$ so that $\binom{k+n}{n}\mid\binom{m+k}{k}$?
If not, what is the density $d^*(n)$ of integers $m$ for which \textup{(1.4)} has a solution with $1\le k\le m-n$?}
\medskip

For a fixed integer $n\ge1$ let
\[
  G_n:=\Bigl\{m\in\N:\ m>n,\ \exists\,k\in\N,\ 1\le k\le m-n,\
    \binom{n+k}{n}\Bigm|\binom{m+k}{k}\Bigr\},
\]
and for a set $S\subseteq\N$ write
\[
  d(S):=\lim_{x\to\infty}\frac{|S\cap[1,x]|}{x}
\]
for its natural density, when the limit exists, so that $d^*(n)=d(G_n)$.
We prove that $d^*(n)$ exists and equals $1$ for every fixed $n\ge2$, so that the ``almost all'' alternative holds.

The case $n=1$, excluded from the question above, is settled outright in \citet{erdos1977products}, and its proof is the carry argument used below.
For a prime $p$, taking $k=p-1$ turns \eqref{erdos_straus:eq:14} into the assertion $p\mid\binom{m+p-1}{p-1}$.
The base-$p$ digits of $p-1$ are $(p-1,0,0,\dots)$, so adding $m$ and $p-1$ in base $p$ produces a carry out of the units place as soon as $p\nmid m$, and Kummer's theorem then gives $p\mid\binom{m+p-1}{p-1}$.
No $m>2$ is divisible by every prime $p\le m$, and any prime $p\le m$ with $p\nmid m$ supplies an admissible $k=p-1$ in $[1,m-1]$, so every $m>2$ lies in $G_1$.
Erd\H{o}s and Straus write that already the next case, $n=2$, ``seems much more difficult to decide''.

Erd\H{o}s and Straus themselves prove two results about \eqref{erdos_straus:eq:14}, both in the regime where $m/n$ is bounded.
\citet[Theorem~2.1]{erdos1977products} show that for fixed $c>0$ and $\Lambda>1$ only finitely many triples $n,k,m$ with $k\ge cn$ and $n+k\le m\le\Lambda n$ satisfy it, and their Theorem~3.3 removes the hypothesis $k\ge cn$ at the cost of requiring $k\ge2$ and a prime in $[n+1,n+k]$.
Both leave the range studied here untouched, since we fix $n$ and let $m\to\infty$, so that $m/n\to\infty$.
Apart from these, work on \eqref{erdos_straus:eq:14} has concentrated on the boundary case $k=m-n$, that is, on the question whether for each fixed $n$ some $k$ satisfies $\binom{n+k}{n}\mid\binom{n+2k}{k}$.
That case is raised in \citet{erdos1977products}, recorded by \citet{erdos1980old}, and appears as Problem 389 on the Erd\H{o}s problems website \citep{erdosproblems389}, where it is listed as open; \citet{ulas2013note} gives computational results for it and for a companion question of Erd\H{o}s and Graham, verifying that a suitable $k$ exists for every $n\le20$, and settling the companion Erd\H{o}s--Graham question for every $n\le9$.

\citet[Theorem~2]{pomerance2015divisors} proves that for each fixed $k\ge1$ the set of $M$ with $M+k\mid\binom{2M}{M}$ has asymptotic density $1$, and notes that the same proof yields density one for the divisibility $(M+1)(M+2)\cdots(M+k)\mid\binom{2M}{M}$.
\citet[Theorems~1 and~2]{pomerance2026remarks} allows $k$ to grow with $M$: the product $(M+1)\cdots(M+k)$ divides $\binom{2M}{M}$ for all $k\le\eta\log M$, for any fixed $\eta<1/\log4$, and $\binom{M+k}{k}$ divides $\binom{2M}{M}$ for all $k\le e^{0.8\sqrt{\log M}}$, in both cases for a set of $M$ of density $1$.
The mechanism there is the one we use: a high power of $p$ dividing $M+k$ forces the low base-$p$ digits of $M$ to be large, and Kummer's theorem converts this into carries.
None of these statements implies our theorem: in each of them the dividend is the central binomial coefficient $\binom{2M}{M}$ and $k$ is small compared with $M$, whereas in \eqref{erdos_straus:eq:14} the dividend is $\binom{m+k}{k}$ and the divisor is $\binom{n+k}{n}$ with $n$ fixed and $k$ unbounded.
In a similar vein \citet{harborth1979divisibility} showed that for each fixed $k$ almost all entries $\binom{M}{j}$ of Pascal's triangle are divisible by $M(M-1)\cdots(M-k+1)$; see the discussion in \citet{pomerance2015divisors}.
We have found no prior source asserting that $G_n$ has positive lower density, let alone density one.

Throughout, $p$ and $q$ denote primes, $v_p(\cdot)$ is the $p$-adic valuation, $\pi(B)=\sum_{p\le B}1$ is the prime-counting function, $\vartheta(B)=\sum_{p\le B}\log p$ is Chebyshev's function, and $\omega(N)$ is the number of distinct prime divisors of $N$.

\begin{theorem}\label{erdos_straus:thm:main}
For every fixed integer $n\ge2$ the natural density $d^*(n)$ exists and
\[
  d^*(n)=1 .
\]
\end{theorem}

Since $G_n\subseteq\N$, its upper density is at most $1$, so the entire content of the theorem is a lower bound on the lower density.

\begin{remark}\label{erdos_straus:rem:scope}
Whether \emph{all} large $m$ lie in $G_n$ remains open, since we show that $\N\setminus G_n$ has density zero and not that it is finite.
The proof is also silent about the least admissible $k$ for a given $m$.
To achieve density $1-\varepsilon$ it uses a single $k=k_B$ of size $e^{O_n(B)}$ and covers only those $m$ with $m\ge n+k_B$, so it gives no information about small $k$, and in particular none about the boundary case $k=m-n$, which is Problem 389.
\end{remark}

We now summarize the proof.
For each integer parameter $B$ we exhibit one value $k=k_B$ of the free variable in \eqref{erdos_straus:eq:14} that works simultaneously for every $m$ in an explicit periodic set $T_B$, and we show that the density of $T_B$ tends to $1$ as $B\to\infty$.
The key point is to choose $k_B$ so that $N_B=\binom{n+k_B}{n}$ has no prime factor below $B$.
Every prime $q\mid N_B$ then divides exactly one of the $n$ integers $k_B+1,\dots,k_B+n$, say $k_B+i(q)$, and does so to the full multiplicity $v_q(N_B)$.
The single congruence condition $m\bmod q\ge i(q)$ is then enough to force $q^{v_q(N_B)}\mid\binom{m+k_B}{k_B}$, because $k_B$ has residue $q^{s}-i(q)$ modulo $q^{s}$ for every $s\le v_q(N_B)$ and each such $s$ therefore contributes a carry.
Only Kummer's theorem, Legendre's formula, the Chinese remainder theorem and Chebyshev's elementary bounds on $\pi$ and $\vartheta$ are needed.

\subsubsection{A binomial coefficient free of small primes}

Fix $n\ge2$.
For an integer parameter $B>n$ define
\begin{equation}\label{erdos_straus:eq:kB}
  k_B:=\prod_{p\le B}p^{e_p},
  \qquad
  e_p:=\min\{e\ge1:\ p^{e}>n\},
\end{equation}
and set
\[
  N_B:=\binom{n+k_B}{n}.
\]

\begin{lemma}\label{erdos_straus:lem:nosmall}
Every prime divisor of $N_B$ is larger than $B$.
\end{lemma}

\begin{proof}
Let $p\le B$.
By \eqref{erdos_straus:eq:kB} we have $p^{e_p}\mid k_B$, so the $e_p$ lowest base-$p$ digits of $k_B$ vanish, while $n<p^{e_p}$, so every base-$p$ digit of $n$ of index $e_p$ or higher vanishes.
Adding $n$ and $k_B$ in base $p$ therefore produces no carry at all: in each of the positions $0,\dots,e_p-1$ the digit of $k_B$ is $0$, and in each higher position the digit of $n$ is $0$, so no position ever sums to $p$ or more.
By Kummer's theorem $v_p\binom{n+k_B}{n}$ equals the number of carries in this addition, whence
\[
  v_p(N_B)=0 .
\]
Thus no prime $p\le B$ divides $N_B$.
\end{proof}

\begin{lemma}\label{erdos_straus:lem:index}
Let $q$ be a prime with $q\mid N_B$ and put $a_q:=v_q(N_B)$.
Then there is a unique index $i(q)\in\{1,\dots,n\}$ with $q\mid k_B+i(q)$, and for this index
\[
  q^{a_q}\ \big\|\ k_B+i(q),
\]
that is, $a_q=v_q\bigl(k_B+i(q)\bigr)$.
\end{lemma}

\begin{proof}
By Lemma~\ref{erdos_straus:lem:nosmall} we have $q>B>n$, so $q\nmid n!$ and hence
\[
  v_q(N_B)=v_q\Bigl(\prod_{i=1}^{n}(k_B+i)\Bigr)-v_q(n!)
         =\sum_{i=1}^{n}v_q(k_B+i).
\]
The left-hand side is positive, so at least one summand is positive.
Two distinct indices $1\le i<j\le n$ with $q\mid k_B+i$ and $q\mid k_B+j$ would give $q\mid j-i$ with $0<j-i<n<q$, which is impossible.
Hence exactly one index $i(q)$ contributes, and the displayed identity reduces to $a_q=v_q(k_B+i(q))$.
\end{proof}

\subsubsection{A congruence forcing the divisibility}

\begin{lemma}\label{erdos_straus:lem:legendre}
Let $q\mid N_B$, write $i=i(q)$ and $a=a_q$.
If $m$ is a nonnegative integer with
\[
  m\bmod q\in\{i,i+1,\dots,q-1\},
\]
then
\[
  v_q\binom{m+k_B}{k_B}\ \ge\ a .
\]
\end{lemma}

\begin{proof}
Fix $s$ with $1\le s\le a$.
By Lemma~\ref{erdos_straus:lem:index} we have $q^{a}\mid k_B+i$, hence $q^{s}\mid k_B+i$, so $k_B\equiv -i \pmod{q^{s}}$.
Since $1\le i\le n<q\le q^{s}$, the residue of $k_B$ modulo $q^{s}$ is exactly $q^{s}-i$, that is,
\[
  k_B=u_sq^{s}+(q^{s}-i)
  \qquad\text{for some integer } u_s=\lfloor k_B/q^{s}\rfloor\ge 0 .
\]
Write $m=v_sq^{s}+t_s$ with $0\le t_s<q^{s}$.
Since $t_s\equiv m\pmod{q}$ and $t_s\ge0$, we have $t_s\ge t_s\bmod q=m\bmod q\ge i$.
Consequently
\[
  i\le t_s<q^{s},
  \qquad\text{so}\qquad
  q^{s}\ \le\ t_s+q^{s}-i\ <\ 2q^{s},
\]
and therefore
\[
  \Bigl\lfloor\frac{m+k_B}{q^{s}}\Bigr\rfloor
  -\Bigl\lfloor\frac{m}{q^{s}}\Bigr\rfloor
  -\Bigl\lfloor\frac{k_B}{q^{s}}\Bigr\rfloor
  =(v_s+u_s)+\Bigl\lfloor\frac{t_s+q^{s}-i}{q^{s}}\Bigr\rfloor-v_s-u_s
  =\Bigl\lfloor\frac{t_s+q^{s}-i}{q^{s}}\Bigr\rfloor
  =1 .
\]
By Legendre's formula applied to $(m+k_B)!$, $m!$ and $k_B!$,
\[
  v_q\binom{m+k_B}{k_B}
  =\sum_{s\ge1}
    \Bigl(\Bigl\lfloor\frac{m+k_B}{q^{s}}\Bigr\rfloor
      -\Bigl\lfloor\frac{m}{q^{s}}\Bigr\rfloor
      -\Bigl\lfloor\frac{k_B}{q^{s}}\Bigr\rfloor\Bigr).
\]
Each summand is nonnegative, since $\lfloor\alpha+\beta\rfloor\ge\lfloor\alpha\rfloor+\lfloor\beta\rfloor$, and the $a$ summands with $1\le s\le a$ each equal $1$.
Hence the sum is at least $a$.
\end{proof}

Observe that the hypothesis of Lemma~\ref{erdos_straus:lem:legendre} is a condition modulo $q$ alone, even though its conclusion concerns the prime power $q^{a}$.
This is what makes the Chinese remainder theorem cheap to apply.
Define
\begin{equation}\label{erdos_straus:eq:TB}
  T_B:=\Bigl\{m\in\N:\ m\bmod q\in\{i(q),i(q)+1,\dots,q-1\}
    \ \text{ for every prime } q\mid N_B\Bigr\}.
\end{equation}

\begin{corollary}\label{erdos_straus:cor:TB}
For every $m\in T_B$ we have $N_B\mid\binom{m+k_B}{k_B}$.
Consequently
\[
  T_B\cap[\,n+k_B,\infty)\ \subseteq\ G_n .
\]
\end{corollary}

\begin{proof}
Let $m\in T_B$.
For each prime $q\mid N_B$, Lemma~\ref{erdos_straus:lem:legendre} gives $v_q\binom{m+k_B}{k_B}\ge a_q=v_q(N_B)$; hence $N_B\mid\binom{m+k_B}{k_B}$.
If moreover $m\ge n+k_B$, then $k:=k_B$ satisfies $1\le k\le m-n$ and $\binom{n+k}{n}=N_B$ divides $\binom{m+k}{k}$, so $m\in G_n$.
\end{proof}

\begin{example}\label{erdos_straus:ex:n2}
Take $n=2$ and $B=5$.
Then $e_2=2$ and $e_3=e_5=1$, so $k_5=4\cdot3\cdot5=60$ and
\[
  N_5=\binom{62}{2}=1891=31\cdot61 ,
\]
whose prime factors indeed exceed $5$.
Here $31\mid 62=k_5+2$ and $61\mid 61=k_5+1$, so $i(31)=2$ and $i(61)=1$, both with $a_q=1$.
Corollary~\ref{erdos_straus:cor:TB} therefore says that
\[
  \binom{62}{2}\ \Bigm|\ \binom{m+60}{60}
  \qquad\text{whenever } m\ge62,\quad m\not\equiv0,1 \pmod{31},\quad m\not\equiv0\pmod{61},
\]
a set of density $\tfrac{29}{31}\cdot\tfrac{60}{61}=\tfrac{1740}{1891}>0.92$.
The congruence conditions are sufficient but not necessary: for instance $m=930$ is excluded by the condition at $31$, yet $\binom{62}{2}\mid\binom{990}{60}$.
Indeed $930$ and $60$ have base-$31$ digits $(0,30)$ and $(29,1)$, written from least to most significant, so their addition still carries out of the second place and $31\mid\binom{990}{60}$, while $930\bmod61=15\ge i(61)$ leaves $930$ inside the condition at $61$.
\end{example}

\subsubsection{The density of the congruence set}

The definition \eqref{erdos_straus:eq:TB} imposes congruence conditions modulo the finitely many distinct primes $q\mid N_B$, so $T_B$ is a union of residue classes modulo $\prod_{q\mid N_B}q$ and its natural density $d(T_B)$ exists.
Since $i(q)\le n<q$, the condition at $q$ excludes exactly the $i(q)$ residues $0,1,\dots,i(q)-1$, so by the Chinese remainder theorem
\begin{equation}\label{erdos_straus:eq:crt}
  d(T_B)=\prod_{q\mid N_B}\Bigl(1-\frac{i(q)}{q}\Bigr)
             \ \ge\ \prod_{q\mid N_B}\Bigl(1-\frac{n}{q}\Bigr)\ >\ 0 .
\end{equation}

\begin{lemma}\label{erdos_straus:lem:omega}
For every integer $B>n$ we have $\log k_B=O_n(B)$ and
\[
  \omega(N_B)=O_n\!\Bigl(\frac{B}{\log B}\Bigr).
\]
\end{lemma}

\begin{proof}
By the minimality of $e_p$ in \eqref{erdos_straus:eq:kB} we have $p^{e_p-1}\le n$, whence $e_p\log p\le \log n+\log p$, and therefore
\[
  \log k_B=\sum_{p\le B}e_p\log p
  \le \pi(B)\log n+\vartheta(B).
\]
Chebyshev's bounds $\pi(B)=O(B/\log B)$ and $\vartheta(B)=O(B)$ give $\log k_B=O_n(B)$.

For the second assertion, note that $N_B\cdot n!=\prod_{i=1}^{n}(k_B+i)$, so every prime divisor of $N_B$ divides some $k_B+i$ with $1\le i\le n$.
Fix such an $i$ and let $t$ be the number of distinct primes exceeding $B$ that divide $k_B+i$.
Each such prime is at least $B+1$ because $B$ is an integer, so their product divides $k_B+i$ and is at least $(B+1)^{t}$, whence
\[
  t\le \frac{\log(k_B+n)}{\log(B+1)} .
\]
By Lemma~\ref{erdos_straus:lem:nosmall} every prime divisor of $N_B$ exceeds $B$, so summing over the $n$ values of $i$ gives
\[
  \omega(N_B)\le n\,\frac{\log(k_B+n)}{\log(B+1)}
             =O_n\!\Bigl(\frac{B}{\log B}\Bigr),
\]
where the last step uses $\log k_B=O_n(B)$.
\end{proof}

\begin{lemma}\label{erdos_straus:lem:density}
We have $d(T_B)\to1$ as $B\to\infty$.
\end{lemma}

\begin{proof}
Assume that $B\ge 2n$.
Every prime $q\mid N_B$ satisfies $q\ge B+1>2n$ by Lemma~\ref{erdos_straus:lem:nosmall} and the integrality of $B$, so $n/q<1/2$.
For $0\le x\le 1/2$ one has $\log(1-x)\ge -2x$; indeed $h(x):=\log(1-x)+2x$ satisfies $h(0)=0$ and $h'(x)=2-\tfrac1{1-x}\ge0$ on $[0,1/2]$, so $h\ge0$ there.
Applying this to each factor of \eqref{erdos_straus:eq:crt},
\[
  \log d(T_B)\ \ge\ \sum_{q\mid N_B}\log\Bigl(1-\frac nq\Bigr)
  \ \ge\ -2\sum_{q\mid N_B}\frac nq
  \ \ge\ -\frac{2n\,\omega(N_B)}{B+1},
\]
where the last step uses $q\ge B+1$ for every such $q$.
By Lemma~\ref{erdos_straus:lem:omega} the right-hand side is $O_n(1/\log B)$ in absolute value, hence $\log d(T_B)\to0$ and $d(T_B)\to1$.
\end{proof}

Note that $d(T_B)$ need not increase with $B$.
Continuing Example~\ref{erdos_straus:ex:n2}, we have $k_{13}=4\cdot3\cdot5\cdot7\cdot11\cdot13=60060$ and
\[
  N_{13}=\binom{60062}{2}=17\cdot59\cdot509\cdot3533 ,
\]
with $i(17)=i(3533)=1$ and $i(59)=i(509)=2$, so that \eqref{erdos_straus:eq:crt} gives
\[
  d(T_{13})=\frac{16}{17}\cdot\frac{57}{59}\cdot\frac{507}{509}\cdot\frac{3532}{3533}
  =0.9054\ldots\ <\ 0.9201\ldots=d(T_5).
\]
Raising $B$ from $5$ to $13$ has replaced two prime factors by four, one of them barely above $B$.
Only the limit is asserted.

\subsubsection{Proof of the main theorem}

\begin{proof}[Proof of Theorem~\ref{erdos_straus:thm:main}]
Let $\varepsilon>0$.
By Lemma~\ref{erdos_straus:lem:density} choose an integer $B\ge2n$ with $d(T_B)>1-\varepsilon$.
By Corollary~\ref{erdos_straus:cor:TB} every element of $T_B$ that is at least $n+k_B$ lies in $G_n$, so for every $x\ge n+k_B$,
\[
  |G_n\cap[1,x]|
  \ \ge\ \bigl|T_B\cap[n+k_B,x]\bigr|
  \ \ge\ |T_B\cap[1,x]|-(n+k_B).
\]
The subtracted quantity is independent of $x$, so dividing by $x$ and letting $x\to\infty$ gives
\[
  \liminf_{x\to\infty}\frac{|G_n\cap[1,x]|}{x}
  \ \ge\ \lim_{x\to\infty}\frac{|T_B\cap[1,x]|}{x}
  =d(T_B)>1-\varepsilon .
\]
Since $\varepsilon>0$ was arbitrary, the lower density of $G_n$ is $1$, while its upper density is trivially at most $1$.
Hence the natural density exists and $d^*(n)=1$, completing the proof.
\end{proof}

\putbib
\end{bibunit}
\endgroup
\subsection[(Ikenmeyer--Pak--Panova) Many-one \texorpdfstring{$\mathsf{GapP}$}{GapP}-completeness for binary symmetric group characters]{Many-one \texorpdfstring{$\mathsf{GapP}$}{GapP}-completeness for binary symmetric group characters}
\label{sol:pak}
\begingroup
\begin{bibunit}
\def\SharpP{\#\mathsf{P}}
\def\ComputeCharBinary{\textnormal{\textsc{ComputeCharBinary}}}
\def\CSat{\#\textnormal{\textsc{CircuitSat}}}
\def\EC{\#\textnormal{\textsc{ExactCover}}}
\def\DiffEC{\textnormal{\textsc{Diff}}\#\textnormal{\textsc{ExactCover}}}
\def\ec{\operatorname{ec}}
\def\Sn{\mathfrak{S}}

This problem concerns the complexity of evaluating an irreducible character of the symmetric group when both partitions are written in binary.
\citet{ikenmeyer2024positivity} proved that this evaluation is $\GapP$-complete under Turing reductions and conjectured that many-one reductions already suffice; we prove the conjecture, in the sharper form that the partition indexing the character may always be taken to have at most two parts.
The reduction encodes an arbitrary $\GapP$ function as a subset-sum count in a large base and reads the value off a two-row character, which by a classical identity is a difference of two such counts at consecutive targets.

\subsubsection{Introduction}

For a partition $\lambda\vdash n$ let $\chi^\lambda$ denote the irreducible character of the symmetric group $\Sn_n$ indexed by $\lambda$.
Its value at a permutation depends only on the cycle type of that permutation, so for $\mu\vdash n$ we write $\chi^\lambda(\mu)$ for the common value of $\chi^\lambda$ on the conjugacy class of cycle type $\mu$.
The problem \ComputeCharBinary\ takes as input two partitions $\lambda,\mu\vdash n$, each presented as a list of parts written in binary, and outputs the integer $\chi^\lambda(\mu)$.
Because a character value can be negative, the natural home for this function is not the counting class $\SharpP$ of \citet{valiant1979complexity} but the gap class $\GapP=\SharpP-\SharpP$ of differences of two $\SharpP$ functions, introduced by \citet{fenner1994gap}.
A function $F$ is \emph{$\GapP$-hard under many-one reductions} if for every $f\in\GapP$ there is a polynomial-time computable map $R$ with $f(x)=F(R(x))$ for all $x$, and \emph{$\GapP$-complete} if moreover $F\in\GapP$.
This is stronger than hardness under Turing reductions, where $f$ is only required to be computable in polynomial time given an oracle for $F$; we follow \citet{papadimitriou2003computational} for the standard conventions.

\citet{ikenmeyer2024positivity} proved that deciding $\chi^\lambda(\mu)=0$ is $\mathsf{C}_{=}\mathsf{P}$-complete and that deciding $\chi^\lambda(\mu)\geq 0$ is $\mathsf{PP}$-complete, both under many-one reductions, and deduced that neither $|\chi^\lambda(\mu)|$ nor $\chi^\lambda(\mu)^2$ lies in $\SharpP$ unless the polynomial hierarchy collapses to its second level.
As a byproduct of the same reduction they obtained that \ComputeCharBinary\ is $\GapP$-complete under Turing reductions \citep{ikenmeyer2024positivity}, and they noted that their route cannot be pushed as far as a parsimonious reduction.
They then stated the following \citep[Conjecture~5.2]{ikenmeyer2024positivity}.

\medskip
\noindent
\emph{The problem \ComputeCharBinary\ is $\GapP$-complete under many-one reductions.}
\medskip

We prove this conjecture, in the sharper form that the reduction may always be taken to output a partition $\lambda$ with at most two parts.

The closest previous result is the Turing-reduction completeness stated above.
Before that, \citet{hepler1994complexity} proved that computing $\chi^\lambda(\mu)$ is $\SharpP$-hard under many-one reductions already for unary input, hence also for binary input; $\SharpP$-hardness constrains only the nonnegative part of the character and does not by itself give $\GapP$-hardness.
\citet{pak2017complexity} showed that the positivity of a Kronecker coefficient can be decided in time $O(\log N)$ for partitions with a bounded number of parts and largest part $N$, while \citet{ikenmeyer2017vanishing} showed that deciding positivity of a Kronecker coefficient is NP-hard; \citet{panova2023computational} surveys this circle of questions.
On the algorithmic side, \citet{bravyi2025classical} give an algorithm computing a matrix product state that encodes the column $(\chi^\lambda(\mu))_{\lambda\vdash n}$ of the character table, and record the worst-case $\SharpP$-hardness of a single entry as the obstruction to a general polynomial-time algorithm.
We are not aware of a previous many-one $\GapP$-hardness result for this function.

\begin{theorem}
\label{pak:thm:main}
The function \ComputeCharBinary\ is $\GapP$-complete under polynomial-time many-one reductions.
More precisely, for every $f\in\GapP$ there is a polynomial-time computable map $x\mapsto(\lambda_x,\mu_x)$, whose values are pairs of partitions of a common integer $n_x$ with $\lambda_x$ having at most two parts, such that
\[
f(x)=\chi^{\lambda_x}(\mu_x)
\]
for every input $x$.
\end{theorem}

\begin{remark}
\label{pak:rem:scope}
The reduction below computes $f(x)$ exactly, in the many-one sense of the definition above, but it exhibits the value as a difference of two subset-sum counts rather than as a single unsigned count, so it gives no combinatorial interpretation of the character value.
\citet{ikenmeyer2024positivity} record a different obstruction to parsimoniousness in their own reduction, namely that its passage from matchings to counts of ordered set partitions multiplies the count by a fixed factor, since each matching arises from the same number of ordered set partitions, obtained from one another by permuting the bins and the repeated items.
The binary encoding is also essential.
For unary input the quantity $N_\mu(t)$ defined below obeys a subset-sum recursion over the parts of $\mu$ with $0\leq t\leq n$, so two-row character values are then computable in polynomial time, and the family of instances produced below can carry no hardness at all in the unary model unless every $\GapP$ function is polynomial-time computable.
An alternative reduction runs through the identities of \citet{ikenmeyer2024positivity}, which realize as a single character value the difference of two counts of ordered set partitions with the same item sizes and with bin sizes agreeing outside two distinguished bins, whose sizes are $2$ and $4$ in the first count and $1$ and $5$ in the second; the two-row route above avoids the padding that this requires.
\end{remark}

The proof rests on one classical identity and one gadget.
For a two-row shape the character value is a difference of two subset-sum counts at consecutive targets,
\[
\chi^{(n-s,s)}(\mu)=N_\mu(s)-N_\mu(s-1),
\]
where $N_\mu(t)$ is the number of subsets of the parts of $\mu$ with total size $t$.
It therefore suffices to manufacture, out of two given counting problems, a single multiset of binary integers whose subset sums realize the first count at $s$ and the second at $s-1$.
The key point is that the two targets differ by exactly $1$, so the two problems have to be separated inside a single units digit.
We take the two problems to be counts of exact covers of a finite set, write every part in a large base $Q$, give the two instances disjoint blocks of base-$Q$ digits, and adjoin two selector parts that differ in the units digit and each of which pre-fills the digit block of the other instance.

\subsubsection{Characters of two-row shape}

Throughout, a partition $\mu=(\mu_1,\ldots,\mu_m)\vdash n$ has positive parts, and parts of equal size are regarded as distinct, indexed by $[m]$.
For $t\in\mathbb{Z}$ set
\[
N_\mu(t):=\#\Big\{I\subseteq[m]:\sum_{i\in I}\mu_i=t\Big\},
\]
so that $N_\mu(t)=0$ for $t<0$ and $N_\mu(0)=1$.

\begin{lemma}
\label{pak:lem:tworow}
Let $\mu\vdash n$ and let $0\leq s\leq n/2$.
Then
\[
\chi^{(n-s,s)}(\mu)=N_\mu(s)-N_\mu(s-1).
\]
\end{lemma}

\begin{proof}
For an integer vector $\alpha=(\alpha_1,\alpha_2,\ldots)$ with entries summing to $n$ let $\phi^\alpha$ be the character of the representation of $\Sn_n$ induced from the trivial representation of the Young subgroup $\Sn_{\alpha_1}\times\Sn_{\alpha_2}\times\cdots$, with the convention $\phi^\alpha=0$ when some $\alpha_i<0$.
Equivalently, $\phi^\alpha$ is the character of the action of $\Sn_n$ on the words containing exactly $\alpha_i$ letters equal to $i$.
Such a word is fixed by a permutation $\pi$ precisely when it is constant on each cycle of $\pi$, so for $\pi$ of cycle type $\mu$ the value $\phi^\alpha(\mu)$ is the number of ways to label the $m$ cycles by letters so that the cycles labeled $i$ have total length $\alpha_i$ \citep{ikenmeyer2024positivity}.
In particular, for a vector with two entries,
\[
\phi^{(n-t,t)}(\mu)=N_\mu(t).
\]
The Frobenius character formula \citep[Eq.~2.3.8]{james2006representation}, equivalently Young's rule \citep{sagan2001symmetric}, gives
\[
\chi^{(n-s,s)}=\phi^{(n-s,s)}-\phi^{(n-s+1,s-1)},
\]
and evaluating at cycle type $\mu$ yields the identity.
\end{proof}

\subsubsection{A difference of two exact cover counts}

An instance of $\EC$ is a pair $(X,C)$ in which $X=[k]$ with $k\geq 0$ and $C=(S_1,\ldots,S_r)$ is a list of nonempty subsets of $X$, members of equal content being regarded as distinct and indexed by $[r]$.
A \emph{selection} is a subset $F\subseteq[r]$, and it is an \emph{exact cover} if every $u\in X$ lies in $S_i$ for exactly one $i\in F$.
The value of the instance is the number
\[
\ec(X,C):=\#\{F\subseteq[r]:F\text{ is an exact cover}\}
\]
of exact covers, so that $\ec(\varnothing,())=1$, the empty selection covering the empty ground set.
Deciding whether an exact cover exists is NP-complete already when every member of $C$ has three elements \citep{garey2002computers}; we put no bound on the sizes of the members, which the gadget below does not need.
An element of $X$ lying in no member of $C$ is covered by no selection, so an instance containing such an element has value $0$.
Replacing every such instance by the fixed instance $\big([3],(\{1,2\},\{2,3\})\big)$, whose value is also $0$, we may and do assume
\[
k\leq\sum_{S\in C}|S|,
\]
so that $k$ is bounded by the length of the instance.

\begin{lemma}
\label{pak:lem:parsimonious}
There is a polynomial-time computable map sending a conjunctive normal form formula $\varphi$ with clauses of between one and three literals to an instance $(X_\varphi,C_\varphi)$ of $\EC$ with $\ec(X_\varphi,C_\varphi)$ equal to the number of satisfying assignments of $\varphi$.
\end{lemma}

\begin{proof}
Let $\varphi$ have variables $x_1,\ldots,x_N$ and clauses $c_1,\ldots,c_M$, and let $c_j$ be the disjunction of the literals $\ell_{j,1},\ldots,\ell_{j,w_j}$ with $1\leq w_j\leq 3$; repeated and complementary literals inside a clause are allowed.
Take the ground set
\[
X_\varphi:=\{v_i:i\in[N]\}\cup\{z_j:j\in[M]\}\cup\{p_{j,q}:j\in[M],\ q\in[w_j]\},
\]
one element for each variable, one for each clause and one for each position inside a clause, and let $C_\varphi$ consist of the sets
\[
V_{i,t}:=\{v_i\}\cup\{p_{j,q}:\ell_{j,q}\in\{x_i,\neg x_i\}\text{ is falsified by }x_i=t\}
\qquad(i\in[N],\ t\in\{0,1\})
\]
together with the sets
\[
W_{j,T}:=\{z_j\}\cup\{p_{j,q}:q\in T\}
\qquad(j\in[M],\ \varnothing\neq T\subseteq[w_j]).
\]
All of them are nonempty, there are $2N+\sum_{j}(2^{w_j}-1)\leq 2N+7M$ of them, and the list is written down in time linear in the length of $\varphi$.

Let $F$ be an exact cover.
The element $v_i$ lies only in $V_{i,0}$ and $V_{i,1}$, so exactly one of the two is selected; let $\alpha(i)$ be the corresponding value of $x_i$.
The element $z_j$ lies only in the sets $W_{j,T}$, so for each $j$ exactly one $T=T_j$ is selected.
The selected sets $V_{i,\alpha(i)}$ cover between them precisely those $p_{j,q}$ whose literal is false under $\alpha$, so exactness at the elements $p_{j,q}$ forces $T_j$ to be the set of positions $q$ with $\ell_{j,q}$ true under $\alpha$.
As $T_j$ is nonempty, $\alpha$ satisfies every clause.
Conversely, let $\alpha$ satisfy $\varphi$ and let $T_j$ be the set of positions of $c_j$ holding a literal true under $\alpha$, which is nonempty.
The sets $V_{i,\alpha(i)}$ with $i\in[N]$ and $W_{j,T_j}$ with $j\in[M]$ then cover each element of $X_\varphi$ exactly once.
The two constructions are mutually inverse, so the exact covers correspond bijectively to the satisfying assignments of $\varphi$.
\end{proof}

Given two instances of $\EC$, write
\[
\DiffEC\big((X,C),(X',C')\big):=\ec(X,C)-\ec(X',C').
\]

\begin{lemma}
\label{pak:lem:diffec}
The function $\DiffEC$ is $\GapP$-hard under polynomial-time many-one reductions.
\end{lemma}

\begin{proof}
Let $f\in\GapP$ and write $f=g-h$ with $g,h\in\SharpP$.
If $g(x)=\#\{w\in\{0,1\}^{q(|x|)}:V(x,w)=1\}$ for a polynomial $q$ and a polynomial-time predicate $V$, then the standard simulation of a polynomial-time machine by a Boolean circuit \citep{papadimitriou2003computational} produces in polynomial time a circuit $\Gamma_x$ with gates of fan-in at most two and with $\CSat(\Gamma_x)=g(x)$; hence $\CSat$ is $\SharpP$-complete under parsimonious reductions.
The Tseitin transformation \citep{tseitin1983complexity} turns a circuit into a formula in conjunctive normal form with clauses of between one and three literals: it introduces one variable for each gate, adds the clauses forcing that variable to equal the value of the gate, and adds a unit clause forcing the output gate to take the value $1$.
The gate variables are determined by the input variables, so every satisfying assignment of the circuit extends to exactly one satisfying assignment of the formula, and the transformation is parsimonious.
Composing it with Lemma~\ref{pak:lem:parsimonious} and normalizing as above, we obtain polynomial-time computable instances with
\[
\ec(X_x,C_x)=g(x),
\qquad
\ec(X'_x,C'_x)=h(x),
\]
and therefore $f(x)=\DiffEC((X_x,C_x),(X'_x,C'_x))$.
\end{proof}

\subsubsection{The digit gadget}

The construction is a subset-sum encoding in a large base, in the style of the strong NP-hardness proof of $4$-Partition \citep{garey2002computers}.
Fix two instances $(X,C)$ and $(X',C')$ as above, with $X=[k]$, $X'=[k']$, $C=(S_1,\ldots,S_r)$ and $C'=(S'_1,\ldots,S'_{r'})$, and set
\[
Q:=r+r'+4.
\]
We index base-$Q$ digit positions by $0,1,\ldots,k+k'+1$, and call position $0$ the \emph{units digit}, positions $1,\ldots,k$ the \emph{$X$-block}, positions $k+1,\ldots,k+k'$ the \emph{$X'$-block}, and position $k+k'+1$ the \emph{selector digit}.
For $S\in C$ and $S'\in C'$ put
\[
a_S:=\sum_{u\in S}Q^{u},
\qquad
b_{S'}:=\sum_{u\in S'}Q^{k+u},
\]
so that $a_S$ marks the elements covered by $S$ inside the $X$-block, and likewise for $b_{S'}$.
Let
\[
A:=\sum_{u\in X}Q^{u},
\qquad
B:=\sum_{u\in X'}Q^{k+u}
\]
be the all-ones patterns on the two blocks, and let $D:=Q^{k+k'+1}$.
Introduce the two \emph{selectors} and the \emph{guard}
\[
c_A:=B+D+1,
\qquad
c_B:=A+D,
\qquad
H:=A+B+D+2,
\]
and set
\[
s:=A+B+D+1,
\]
so that $H=s+1$.
Let $\mu$ be the partition obtained by sorting the multiset
\[
\{a_S:S\in C\}\cup\{b_{S'}:S'\in C'\}\cup\{c_A,c_B,H\}
\]
into weakly decreasing order, let $n:=|\mu|$, and put
\[
\lambda:=(n-s,s).
\]
Table~\ref{pak:tab:digits} displays the base-$Q$ digits of every quantity involved.

\begin{table}[t]
\centering
\begin{tabular}{lcccc}
\toprule
& units digit & $X$-block & $X'$-block & selector digit\\
& $Q^0$ & $Q^1,\ldots,Q^{k}$ & $Q^{k+1},\ldots,Q^{k+k'}$ & $Q^{k+k'+1}$\\
\midrule
$a_S$ \ $(S\in C)$   & $0$ & $\mathbf{1}_S$ & $\mathbf{0}$ & $0$\\
$b_{S'}$ \ $(S'\in C')$ & $0$ & $\mathbf{0}$ & $\mathbf{1}_{S'}$ & $0$\\
$c_A$                & $1$ & $\mathbf{0}$ & $\mathbf{1}$ & $1$\\
$c_B$                & $0$ & $\mathbf{1}$ & $\mathbf{0}$ & $1$\\
$H$                  & $2$ & $\mathbf{1}$ & $\mathbf{1}$ & $1$\\
\midrule
$s$                  & $1$ & $\mathbf{1}$ & $\mathbf{1}$ & $1$\\
$s-1$                & $0$ & $\mathbf{1}$ & $\mathbf{1}$ & $1$\\
\bottomrule
\end{tabular}
\caption{The base-$Q$ digits of the parts of $\mu$ and of the two targets $s$ and $s-1$.
Here $\mathbf{1}$ and $\mathbf{0}$ denote the all-ones and all-zeros patterns on the block indicated, and $\mathbf{1}_S$ denotes the indicator pattern of $S$ on the $X$-block, with $\mathbf{1}_{S'}$ defined analogously on the $X'$-block.
No column can reach $Q$, so subset sums may be compared digit by digit: a subset of parts summing to $s$ must take $c_A$, hence miss the $X'$-block entirely and pick out an exact cover of $(X,C)$, while a subset summing to $s-1$ must take $c_B$ and pick out an exact cover of $(X',C')$.}
\label{pak:tab:digits}
\end{table}

\begin{lemma}
\label{pak:lem:nocarry}
Let $m=r+r'+3$ be the number of parts of $\mu$.
For every $I\subseteq[m]$ the base-$Q$ digits of $\sum_{i\in I}\mu_i$ are obtained by adding the digit patterns of the parts indexed by $I$, without carrying.
Consequently, for $t\in\{s-1,s\}$ one has $\sum_{i\in I}\mu_i=t$ if and only if those digit patterns sum to the digit pattern of $t$.
\end{lemma}

\begin{proof}
By Table~\ref{pak:tab:digits}, every part of $\mu$ has all its base-$Q$ digits in $\{0,1,2\}$.
At a position of the $X$-block the parts with a nonzero digit there are among the $r$ parts $a_S$, the selector $c_B$ and the guard $H$, each contributing $1$, so the total over any $I$ is at most $r+2$.
Symmetrically a position of the $X'$-block receives at most $r'+2$.
The selector digit receives at most $3$, from $c_A$, $c_B$ and $H$, and the units digit receives at most $3$, namely $1$ from $c_A$ and $2$ from $H$.
Every other position receives $0$.
All of these totals are smaller than $Q=r+r'+4$, so no carrying occurs.
Finally $s$ has digit $1$ at the units digit, at every position of the two blocks, and at the selector digit, while $s-1$ has the same digits except $0$ at the units digit; both patterns have entries below $Q$, so two such integers agree if and only if their digit patterns agree.
\end{proof}

\begin{lemma}
\label{pak:lem:valid}
The pair $(\lambda,\mu)$ is a valid input to \ComputeCharBinary, with $\lambda$ a partition of $n$ having at most two parts and $0\leq s\leq n/2$, and it is computable from $(X,C)$ and $(X',C')$ in polynomial time.
\end{lemma}

\begin{proof}
Since $c_A+c_B+H=2A+2B+3D+3$, we get
\[
n=\sum_{S\in C}a_S+\sum_{S'\in C'}b_{S'}+2A+2B+3D+3,
\]
and therefore
\[
n-2s=\sum_{S\in C}a_S+\sum_{S'\in C'}b_{S'}+D+1>0.
\]
Hence $n-s>s\geq 0$, so $\lambda=(n-s,s)$ is a partition of $n$ with at most two parts and $s\leq n/2$.
For the running time, the normalization gives $k\leq\sum_{S\in C}|S|$ and $k'\leq\sum_{S'\in C'}|S'|$, so the largest exponent $k+k'+1$ is linear in the length of the input, while $\log_2 Q=O(\log(r+r'+4))$.
Thus each of the $r+r'+3$ parts of $\mu$, and each of $n$ and $s$, is an integer of $O\big((k+k'+1)\log(r+r'+4)\big)$ bits, and all of them are produced by polynomially many arithmetic operations on integers of that size.
\end{proof}

\begin{lemma}
\label{pak:lem:counts}
For the partition $\mu$ constructed above,
\[
N_\mu(s)=\ec(X,C),
\qquad
N_\mu(s-1)=\ec(X',C').
\]
\end{lemma}

\begin{proof}
Let $I$ index a subset of the parts with $\sum_{i\in I}\mu_i\in\{s-1,s\}$, which by Lemma~\ref{pak:lem:nocarry} we may analyze digit by digit.
Since $H=s+1>s$, the guard is not taken.
Both $s$ and $s-1$ have selector digit $1$, and apart from $H$ only $c_A$ and $c_B$ have a nonzero selector digit, each equal to $1$; hence exactly one of $c_A$ and $c_B$ is taken.

Suppose first that $\sum_{i\in I}\mu_i=s$.
The units digit of $s$ is $1$, and among the remaining parts only $c_A$ has a nonzero units digit, so $c_A$ is taken and $c_B$ is not.
Now $c_A$ already contributes $1$ at every position of the $X'$-block, which is exactly the digit of $s$ there, so no part $b_{S'}$ is taken.
The parts taken are therefore $\{c_A\}\cup\{a_{S_i}:i\in F\}$ for some $F\subseteq[r]$, and the one remaining requirement is that $\sum_{i\in F}a_{S_i}$ have digit $1$ at every position of the $X$-block.
By the definition of $a_S$ this says exactly that the selected sets cover each element of the ground set exactly once, that is, that $F$ is an exact cover of $(X,C)$.
Hence the subsets of parts summing to $s$ correspond bijectively to the exact covers of $(X,C)$, and $N_\mu(s)=\ec(X,C)$.

Suppose now that $\sum_{i\in I}\mu_i=s-1$.
The units digit of $s-1$ is $0$, so $c_A$ is not taken, and therefore $c_B$ is.
Then $c_B$ contributes $1$ at every position of the $X$-block, so no part $a_S$ is taken, and the parts taken are $\{c_B\}\cup\{b_{S'_i}:i\in F'\}$ with $F'\subseteq[r']$ subject to the same condition on the $X'$-block.
Hence $N_\mu(s-1)=\ec(X',C')$.
\end{proof}

\subsubsection{Membership in \texorpdfstring{$\mathsf{GapP}$}{GapP}}

\begin{proposition}
\label{pak:prop:membership}
\ComputeCharBinary\ belongs to $\GapP$.
\end{proposition}

\begin{proof}
Let the input be $\lambda=(\lambda_1,\ldots,\lambda_p)$ and $\mu=(\mu_1,\ldots,\mu_m)$, both partitions of $n$ written as binary lists of parts, so that $p$, $m$ and all the bit lengths involved are bounded by the length of the input; an input not of this form is recognized in polynomial time and given the value $0$.
For an integer vector $\beta=(\beta_1,\ldots,\beta_p)$ let $P_\mu(\beta)$ denote the number of ordered tuples $(K_1,\ldots,K_p)$ of pairwise disjoint, possibly empty subsets of $[m]$ whose union is $[m]$, such that $\sum_{j\in K_i}\mu_j=\beta_i$ for every $i$; this is $0$ if some $\beta_i<0$, because the parts of $\mu$ are positive.
As in the proof of Lemma~\ref{pak:lem:tworow} one has $\phi^\beta(\mu)=P_\mu(\beta)$, so the Frobenius character formula \citep[Eq.~2.3.8]{james2006representation} reads
\[
\chi^\lambda(\mu)=\sum_{\sigma\in\Sn_p}\operatorname{sgn}(\sigma)\,P_\mu(\lambda+\sigma-\mathrm{id}),
\]
where $\lambda+\sigma-\mathrm{id}$ is the vector whose $i$th entry is $\lambda_i+\sigma(i)-i$; see also \citet{ikenmeyer2024positivity}.
Let $g$ and $h$ count the pairs $(\sigma,K)$ in which $\sigma\in\Sn_p$ has $\operatorname{sgn}(\sigma)=+1$ and $\operatorname{sgn}(\sigma)=-1$ respectively, and $K=(K_1,\ldots,K_p)$ is a tuple as above whose blocks have $\mu$-weights $\sum_{j\in K_i}\mu_j$ equal to the entries of $\lambda+\sigma-\mathrm{id}$.
Such a pair is specified by $O\big((p+m)\log(p+1)\big)$ bits and is verified by computing $\operatorname{sgn}(\sigma)$ and comparing $p$ sums of binary integers of $O(\log n)$ bits each, so $g,h\in\SharpP$.
Therefore $\chi^\lambda(\mu)=g-h$ lies in $\GapP$.
\end{proof}

\subsubsection{Proof of the main theorem}

\begin{proof}[Proof of Theorem~\ref{pak:thm:main}]
Membership in $\GapP$ is Proposition~\ref{pak:prop:membership}.
For hardness, fix $f\in\GapP$.
By Lemma~\ref{pak:lem:diffec} there is a polynomial-time computable map sending an input $x$ to two exact cover instances with
\[
f(x)=\ec(X_x,C_x)-\ec(X'_x,C'_x).
\]
Apply the digit gadget to those two instances, and let $(\lambda_x,\mu_x)$ be the resulting pair, with $\lambda_x=(n_x-s_x,s_x)$.
By Lemma~\ref{pak:lem:valid} the map $x\mapsto(\lambda_x,\mu_x)$ is computable in polynomial time, its values are partitions of $n_x$, the shape $\lambda_x$ has at most two parts, and $0\leq s_x\leq n_x/2$.
By Lemma~\ref{pak:lem:tworow} and Lemma~\ref{pak:lem:counts},
\[
\chi^{\lambda_x}(\mu_x)
=N_{\mu_x}(s_x)-N_{\mu_x}(s_x-1)
=\ec(X_x,C_x)-\ec(X'_x,C'_x)
=f(x).
\]
Hence every $f\in\GapP$ many-one reduces to \ComputeCharBinary, completing the proof.
\end{proof}

\FloatBarrier
\putbib

\end{bibunit}
\endgroup
\subsection[(Kahn) Matching variance versus the residual matching number]{Matching variance versus the residual matching number}
\label{sol:kahn_matchings}
\begingroup
\begin{bibunit}
\def\Mat{\mathcal{M}}
\def\Gna{G_{n,a}}
\def\Prob{\mathbb{P}}

Kahn's normal law for matchings characterizes asymptotic normality of the size of a uniformly random matching through five graph statistics that are bounded or unbounded together, and asks how tightly two of them, the variance $\sigma^2$ and the residual matching number $\lambda$, are tied to one another.
We show that in one direction they are not tied at all: along cliques carrying private pendant leaves the ratio $\sigma^2/\lambda$ grows at least linearly in the number of leaves at each clique vertex, while both parameters tend to infinity.
The mechanism is that extra leaves inflate the fluctuation of the clique matching without inflating the residue.

\subsubsection{Introduction}

For a finite simple graph $G$ let $\Mat(G)$ be the set of matchings of $G$, let $M$ be drawn uniformly from $\Mat(G)$, and put $\xi_G=|M|$.
Write $\mu(G)=\E[\xi_G]$ and $\sigma^2(G)=\Var[\xi_G]$, and write $\nu(G)$ and $\tau(G)$ for the matching number and the vertex cover number of $G$.
For $x\in V(G)$ let $p(x)$ denote the probability that $x$ is not covered by $M$.

By a theorem of \citet{godsil1981matching}, resting on the real-rootedness of the matching polynomial of \citet{heilmann1972theory}, the distribution of $\xi_{G_n}$ along a sequence of graphs is asymptotically normal if and only if $\sigma(G_n)\to\infty$.
\citet{kahn2000normal} identified four combinatorial statistics with exactly the same threshold behavior, two of which are introduced there for the purpose.
The first is the \emph{cover defect}
\[
  \kappa(G)=\min\Bigl\{\sum_{y\in Y}p(y):\ Y\text{ a vertex cover of }G\Bigr\},
\]
and the second, writing $F_M$ for the set of edges of $G$ meeting no edge of $M$, is the \emph{residual matching number}
\[
  \lambda(G)=\E\bigl[\nu(F_M)\bigr].
\]
Thus $F_M$ is the edge set induced by the vertices left uncovered by $M$, and $\lambda$ measures how much of a matching still fits into the residue.
Kahn's theorem states that for any sequence $(G_n)$ with $|V(G_n)|\to\infty$ and $\delta(G_n)\ge1$, and with $\sigma_n=\sigma(G_n)$ and so on, the five conditions
\[
  \sigma_n=O(1),
  \qquad
  \nu_n-\mu_n=O(1),
  \qquad
  \tau_n-\mu_n=O(1),
  \qquad
  \kappa_n=O(1),
  \qquad
  \lambda_n=O(1)
\]
are equivalent \citep[Theorem~1.10]{kahn2000normal}.
He then asks how tight the comparison between the first and the last of these is \citep[Question~7.3]{kahn2000normal}:

\medskip
\noindent
\emph{How closely related are $\sigma^2$ and $\lambda$?
In particular, is it true that $\lambda=\Theta(\sigma^2)$ (that is, are there bounds on the ratios $\lambda/\sigma^2$ and $\sigma^2/\lambda$)?}
\medskip

\noindent
We show that the answer is no, by breaking the direction $\sigma^2=O(\lambda)$.

What Kahn's own inequalities give is the following.
He proves $\sigma^2,\lambda\le\nu-\mu\le\tau-\mu\le\kappa^2/2+O(\kappa)$, and also $\kappa=O(\lambda^2+\lambda)$ and $\kappa=O(\sigma^4+\sigma^2)$ \citep{kahn2000normal}.
Combining these gives
\[
  \sigma^2=O(\lambda^4)\ \text{ whenever }\lambda=\Omega(1),
  \qquad
  \lambda\le\nu-\mu=O(\sigma^8)\ \text{ whenever }\sigma=\Omega(1),
\]
and Kahn conjectures that the second of these can be improved to $\nu-\mu=O(\sigma^6)$, which would be best possible \citep{kahn2000normal}.
The examples he records all have $\sigma^2\asymp\lambda$.
For his Example 7.1, a clique $K_n$ with one private pendant leaf at each clique vertex, one has $\sigma^2=\tfrac12\sqrt n+O(1)$ and $\lambda=\kappa=\sqrt n+O(1)$, and for his Example 7.2 both $\sigma^2$ and $\lambda$ are of order $n^{1/3}$ \citep{kahn2000normal}.
For a disjoint union of copies of $K_{1,m}$ Kahn computes $\lambda=\nu-\mu$ and $\sigma^2=m(\nu-\mu)/(m+1)$, so that $\sigma^2/\lambda=m/(m+1)<1$ \citep{kahn2000normal}; there the leaf count moves $\sigma^2$ and $\lambda$ together, and in the construction below it is the clique that decouples them.
General criteria for central limit theorems of this kind, in terms of the location of the zeros of graph-counting polynomials, were later given by \citet{lebowitz2016central}; they say nothing about the size of $\lambda$.
Apart from Kahn's own examples we are aware of no work that bears on Question 7.3.

\subsubsection{The construction}

Our graphs are the graphs of Kahn's Example 7.1 with more leaves.

\begin{definition}\label{kahn_matchings:def:family}
For integers $n\ge0$ and $a\ge2$ let $\Gna$ be the graph with vertex set
\[
  V(\Gna)=\{v_1,\dots,v_n\}\cup\{y_{i,t}:i\in[n],\ t\in[a-1]\}
\]
and edge set
\[
  E(\Gna)=\bigl\{\{v_i,v_j\}:1\le i<j\le n\bigr\}
   \cup\bigl\{\{v_i,y_{i,t}\}:i\in[n],\ t\in[a-1]\bigr\}.
\]
That is, $\Gna$ is a clique $K_n$ carrying $a-1$ private pendant leaves at each clique vertex.
We call $v_1,\dots,v_n$ the \emph{clique vertices}.
\end{definition}

The case $a=2$ is Kahn's Example 7.1.
For $n\ge1$ every leaf has degree $1$ and every clique vertex has degree $n-1+(a-1)\ge1$, so $\Gna$ is a finite simple graph with $\delta(\Gna)\ge1$ and the family satisfies the standing hypothesis of Kahn's theorem.
The parameter $a$ is the number of configurations available at a clique vertex that is not matched inside the clique: it may stay uncovered, or take any one of its $a-1$ leaves.
Figure~\ref{kahn_matchings:fig:residual} shows a matching of $G_{5,3}$ and the residual graph it leaves behind.

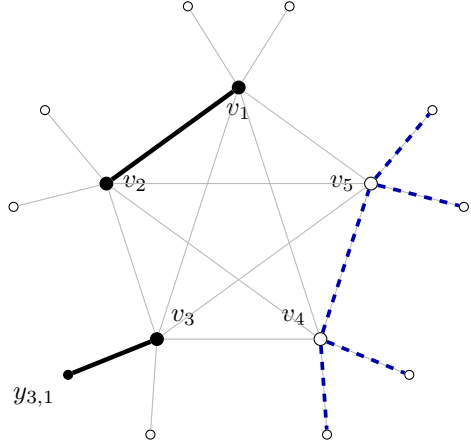
\begin{figure}[t]
\centering
\begin{tikzpicture}[scale=1.15,
  cl/.style={circle,draw,fill=black,inner sep=1.7pt},
  fr/.style={circle,draw,fill=white,inner sep=1.7pt},
  lc/.style={circle,draw,fill=black,inner sep=1.2pt},
  lu/.style={circle,draw,fill=white,inner sep=1.2pt},
  bg/.style={gray!55,line width=0.4pt},
  mt/.style={black,line width=1.7pt},
  rs/.style={blue!65!black,line width=1.4pt,dashed}]
\node[cl,label=below:{$v_1$}] (v1) at (90:1.6) {};
\node[cl,label=right:{$v_2$}] (v2) at (162:1.6) {};
\node[cl,label=above right:{$v_3$}] (v3) at (234:1.6) {};
\node[fr,label=above left:{$v_4$}] (v4) at (306:1.6) {};
\node[fr,label=left:{$v_5$}] (v5) at (18:1.6) {};
\node[lu] (l1a) at (77:2.6) {};
\node[lu] (l1b) at (103:2.6) {};
\node[lu] (l2a) at (149:2.6) {};
\node[lu] (l2b) at (175:2.6) {};
\node[lc,label=below left:{$y_{3,1}$}] (l3a) at (221:2.6) {};
\node[lu] (l3b) at (247:2.6) {};
\node[lu] (l4a) at (293:2.6) {};
\node[lu] (l4b) at (319:2.6) {};
\node[lu] (l5a) at (5:2.6) {};
\node[lu] (l5b) at (31:2.6) {};
\foreach \i/\j in {v1/v2,v1/v3,v1/v4,v1/v5,v2/v3,v2/v4,v2/v5,v3/v4,v3/v5,v4/v5}
  \draw[bg] (\i) -- (\j);
\foreach \i/\j in {v1/l1a,v1/l1b,v2/l2a,v2/l2b,v3/l3a,v3/l3b,v4/l4a,v4/l4b,v5/l5a,v5/l5b}
  \draw[bg] (\i) -- (\j);
\draw[mt] (v1) -- (v2);
\draw[mt] (v3) -- (l3a);
\draw[rs] (v4) -- (v5);
\draw[rs] (v4) -- (l4a);
\draw[rs] (v4) -- (l4b);
\draw[rs] (v5) -- (l5a);
\draw[rs] (v5) -- (l5b);
\end{tikzpicture}
\caption{The graph $G_{5,3}$, with its ten clique edges and its two private leaves at each clique vertex drawn in gray, and with the vertices covered by the matching drawn solid.
The matching $M$ consists of the two heavy black edges $v_1v_2$ and $v_3y_{3,1}$, so the set of clique vertices missed by the clique part of $M$ is $\{v_3,v_4,v_5\}$ and exactly one of them takes a leaf.
The residual graph $F_M$, the set of edges meeting no edge of $M$, is drawn heavy and dashed: it is the copy of $G_{2,3}$ carried by the two clique vertices that took no leaf, and $\nu(F_M)=2$.}
\label{kahn_matchings:fig:residual}
\end{figure}

\begin{theorem}\label{kahn_matchings:thm:main}
There are absolute constants $n_0$ and $c>0$ such that for all integers $n\ge n_0$ and $a$ with $4\le a\le n^{1/4}$,
\[
  \tfrac14\sqrt n\ \le\ \lambda(\Gna)\ \le\ 5\sqrt n,
  \qquad
  \sigma^2(\Gna)\ \ge\ c\,a\sqrt n .
\]
In particular
\[
  \frac{\sigma^2(\Gna)}{\lambda(\Gna)}\ \ge\ \frac{c}{5}\,a .
\]
\end{theorem}

\begin{corollary}\label{kahn_matchings:cor:ratio}
Put $a_n=\lfloor n^{1/4}\rfloor$ and $H_n=G_{n,a_n}$ for $n\ge16$.
Then $\lambda(H_n)\to\infty$, $\sigma^2(H_n)\to\infty$, and
\[
  \frac{\sigma^2(H_n)}{\lambda(H_n)}=\Omega\bigl(n^{1/4}\bigr)\longrightarrow\infty .
\]
Hence there is no absolute constant $C$ with $\sigma^2(G)\le C\lambda(G)$ for all finite simple graphs $G$, and $\lambda=\Theta(\sigma^2)$ fails.
\end{corollary}

Both parameters diverge along $(H_n)$, so this is not the degenerate kind of counterexample in which $\sigma^2$ and $\lambda$ are both $O(1)$.

\begin{remark}\label{kahn_matchings:rem:onesided}
Corollary~\ref{kahn_matchings:cor:ratio} says nothing about the reverse bound $\lambda=O(\sigma^2)$, which remains open, and the family $(H_n)$ gives no information about it: there $\lambda/\sigma^2\to0$, so the family is consistent with that direction.
The best bound we know in that direction is still the one implied by Kahn's inequalities, namely $\lambda\le\nu-\mu=O(\sigma^8)$ for $\sigma=\Omega(1)$.
\end{remark}

\begin{remark}\label{kahn_matchings:rem:exponent}
The family also limits how much the bound $\sigma^2=O(\lambda^4)$ recorded above can be improved.
Indeed $a_n\ge\tfrac12n^{1/4}$, so for $n$ large Theorem~\ref{kahn_matchings:thm:main} gives $\sigma^2(H_n)\ge\tfrac12c\,n^{3/4}$ while $\lambda(H_n)\le5\sqrt n$, whence
\[
  \sigma^2(H_n)\ \ge\ \frac{c}{2\cdot5^{3/2}}\,\lambda(H_n)^{3/2}.
\]
Since $\lambda(H_n)\to\infty$, no bound of the form $\sigma^2=O(\lambda^{\beta})$ can hold for graphs with $\lambda=\Omega(1)$ unless $\beta\ge3/2$.
\end{remark}

The proof is a direct computation with the exact law of $\xi_{\Gna}$.
A matching of $\Gna$ consists of a matching of the clique together with, at each clique vertex left uncovered by it, one of $a$ choices: no leaf, or one of the $a-1$ leaves.
Writing $S$ for the number of clique vertices missed by the clique part of the matching and $B$ for the number of those that do take a leaf, one has $\xi=(n-S)/2+B$ and $\nu(F_M)=S-B$, and $B$ is binomial with parameters $S$ and $(a-1)/a$.
Hence $\lambda=\E[S]/a$, while for $a\ge4$ the variance $\sigma^2$ retains a fixed fraction of $\Var[S]$, and we show that $\E[S]\asymp a\sqrt n$ and $\Var[S]\gg a\sqrt n$.
The key point is that the leaf count $a$ divides $\lambda$ but not $\sigma^2$: multiplying the number of admissible configurations at a free clique vertex by $a$ shifts the equilibrium of the clique matching, pushing $\E[S]$ up to order $a\sqrt n$ and $\Var[S]$ up to at least that order, whereas the residue is governed by the free clique vertices that took \emph{no} leaf, and their expected number $\E[S]/a$ stays of order $\sqrt n$.

\subsubsection{The law of a uniform matching}

Fix $n$ and $a\ge2$, write $G=\Gna$, and abbreviate $\xi=\xi_G=|M|$.
Given a matching $M$ of $G$, let
\[
  S_M=\{i\in[n]:\ v_i\text{ is not covered by an edge of }M\cap E(K_n)\},
  \qquad S=|S_M|,
\]
and let $B$ be the number of $i\in S_M$ for which $v_i$ is matched by $M$ to one of its leaves.
Set $q=(a-1)/a$.

\begin{lemma}\label{kahn_matchings:lem:law}
Let $M$ be uniform on $\Mat(G)$.
Then $S\equiv n\pmod 2$, and for $0\le s\le n$ with $s\equiv n\pmod2$,
\begin{equation}\label{kahn_matchings:eq:weights}
  \Prob(S=s)=\frac{w_s}{W},
  \qquad
  w_s:=\frac{n!}{s!\,2^{k}\,k!}\,a^{s},
  \qquad
  k:=\frac{n-s}{2},
  \qquad
  W:=\sum_{s}w_s .
\end{equation}
Conditionally on $S=s$, the variable $B$ is binomial $\mathrm{Bin}(s,q)$.
Moreover, for every matching $M$,
\begin{equation}\label{kahn_matchings:eq:decomp}
  |M|=\frac{n-S}{2}+B,
  \qquad
  \nu(F_M)=S-B .
\end{equation}
\end{lemma}

\begin{proof}
A matching of $G$ is determined by two successive choices: a matching $M_0\subseteq E(K_n)$ of the clique, and then, for each clique vertex left uncovered by $M_0$, either nothing or one of its $a-1$ leaves.
Indeed leaf edges at distinct clique vertices are disjoint, and a leaf edge at $v_i$ is compatible with $M_0$ exactly when $v_i$ is uncovered by $M_0$.

The number of matchings $M_0$ of $K_n$ leaving a prescribed set of $s$ clique vertices uncovered is the number of perfect matchings of $K_{n-s}$, namely $(n-s)!/(2^{k}k!)$ with $k=(n-s)/2$; this forces $s\equiv n\pmod 2$.
Choosing the set of uncovered vertices in $\binom ns$ ways and then making the leaf choices in $a^{s}$ ways gives the number of matchings $M$ with $S=s$ as
\[
  \binom ns\frac{(n-s)!}{2^{k}k!}\,a^{s}=\frac{n!}{s!\,2^{k}k!}\,a^{s}=w_s,
\]
which is \eqref{kahn_matchings:eq:weights}.
Since the $s$ leaf choices are made independently and uniformly among $a$ options, of which $a-1$ produce a leaf edge, $B$ is $\mathrm{Bin}(s,q)$ given $S=s$.

For the first identity in \eqref{kahn_matchings:eq:decomp}, note that $M$ consists of $(n-S)/2$ clique edges and $B$ leaf edges.
For the second, a clique vertex $v_i$ is uncovered by $M$ precisely when $i\in S_M$ and $v_i$ took no leaf, which happens for exactly $S-B$ indices, while a leaf $y_{i,t}$ is uncovered precisely when it was not chosen.
An edge lies in $F_M$ if and only if both its ends are uncovered, so $F_M$ consists of all clique edges between the $S-B$ uncovered clique vertices together with all $a-1$ leaf edges at each of them: that is, $F_M\cong G_{S-B,\,a}$.
In $G_{r,a}$ with $a\ge2$ the $r$ clique vertices form a vertex cover, so $\nu\le r$, while matching each clique vertex to a private leaf shows $\nu\ge r$.
Hence $\nu(F_M)=S-B$.
\end{proof}

\begin{corollary}\label{kahn_matchings:cor:formulas}
With the notation of Lemma~\ref{kahn_matchings:lem:law},
\begin{equation}\label{kahn_matchings:eq:lambda-sigma}
  \lambda(\Gna)=\frac{\E[S]}{a},
  \qquad
  \sigma^2(\Gna)=\Bigl(\frac{a-2}{2a}\Bigr)^{2}\Var[S]
    +\frac{a-1}{a^{2}}\,\E[S] .
\end{equation}
\end{corollary}

\begin{proof}
By \eqref{kahn_matchings:eq:decomp} and $\E[B\mid S]=qS$ we get $\lambda(\Gna)=\E[\nu(F_M)]=\E[S-B]=(1-q)\E[S]=\E[S]/a$.
Also $\E[\xi\mid S]=\tfrac n2+(q-\tfrac12)S$ and $\Var[\xi\mid S]=q(1-q)S$, so decomposing the variance over $S$,
\[
  \sigma^2=\bigl(q-\tfrac12\bigr)^{2}\Var[S]+q(1-q)\E[S],
\]
and $q-\tfrac12=(a-2)/(2a)$ while $q(1-q)=(a-1)/a^{2}$.
\end{proof}

Everything therefore reduces to the mean and the variance of the weight sequence \eqref{kahn_matchings:eq:weights}.

\subsubsection{The distribution of the free set}

Throughout this subsection and the next we assume
\begin{equation}\label{kahn_matchings:eq:regime}
  4\le a\le n^{1/4},
  \qquad
  L:=a\sqrt n,
\end{equation}
and that $n$ is larger than a suitable absolute constant; all implied constants below are absolute.
Since $a\le n^{1/4}$ gives $a^{3}\le a\sqrt n$, and $a\ge4$, we have
\begin{equation}\label{kahn_matchings:eq:asq}
  a^{2}\le \frac La\le\frac L4,
  \qquad
  L\le n^{3/4},
  \qquad
  L\ge 4\sqrt n .
\end{equation}
In particular $L\to\infty$ with $n$, so every hypothesis of the form ``$L$ larger than an absolute constant'' below is implied by ``$n$ larger than an absolute constant''; we use this without further comment.
We also record that $2L+4\le n/2$ for $n$ large, by the middle bound in \eqref{kahn_matchings:eq:asq}.

From \eqref{kahn_matchings:eq:weights}, for $0\le s\le n-2$ with $s\equiv n\pmod 2$,
\begin{equation}\label{kahn_matchings:eq:ratio}
  \rho(s):=\frac{w_{s+2}}{w_s}=\frac{a^{2}(n-s)}{(s+1)(s+2)} .
\end{equation}
The function $\rho$ is strictly decreasing in $s$, so the sequence $(w_s)$ is strictly log-concave along the arithmetic progression $s\equiv n\pmod2$, and in particular unimodal.
Fix a mode $s_0$, that is, an admissible $s_0$ with $w_{s_0}=\max_s w_s$.

\begin{lemma}\label{kahn_matchings:lem:mode}
Assume \eqref{kahn_matchings:eq:regime} and $n$ large.
Then $\tfrac12 L\le s_0\le 2L+2$.
\end{lemma}

\begin{proof}
If $s\ge2L$ is admissible then $(s+1)(s+2)>s^{2}\ge4L^{2}=4a^{2}n$, so \eqref{kahn_matchings:eq:ratio} gives
\begin{equation}\label{kahn_matchings:eq:tail-ratio}
  \rho(s)\ \le\ \frac{a^{2}n}{s^{2}}\ \le\ \frac14 .
\end{equation}
Hence $w_{s+2}<w_s$ for every admissible $s\ge2L$.
If we had $s_0>2L+2$ then $s_0-2\ge2L$ would be admissible and $w_{s_0}<w_{s_0-2}$, contradicting maximality; therefore $s_0\le2L+2$.

Since $s_0+2\le2L+4\le n$, maximality also gives $w_{s_0+2}\le w_{s_0}$, that is $\rho(s_0)\le1$, that is
\[
  (s_0+1)(s_0+2)\ \ge\ a^{2}(n-s_0)\ \ge\ \tfrac12a^{2}n=\tfrac12L^{2},
\]
where we used $s_0\le2L+2\le n/2$.
Hence $s_0+2\ge L/\sqrt2$, and since $L$ is large this yields $s_0\ge L/\sqrt2-2\ge L/2$.
\end{proof}

\begin{lemma}\label{kahn_matchings:lem:mean-ub}
Assume \eqref{kahn_matchings:eq:regime} and $n$ large.
Then $\E[S]\le 5L$, and consequently $\lambda(\Gna)\le5\sqrt n$.
\end{lemma}

\begin{proof}
Let $T$ be the least integer with $T\ge2L$ and $T\equiv n\pmod2$, so that $T\le2L+2$.
By \eqref{kahn_matchings:eq:tail-ratio} we have $\rho(s)\le\frac14$ for every admissible $s\ge T$, whence $w_{T+2j}\le4^{-j}w_T$ for all $j\ge0$.
Splitting the expectation at $T$ and using $w_T\le W$,
\[
  \E[S]\ \le\ T+\sum_{j\ge0}(T+2j)\frac{w_{T+2j}}{W}
       \ \le\ T+\sum_{j\ge0}(T+2j)4^{-j}
       = \tfrac73T+\tfrac89 .
\]
Since $T\le2L+2$ this is at most $\tfrac{14}3L+\tfrac{50}9\le5L$, because $L$ is large.
Finally $\lambda(\Gna)=\E[S]/a\le5L/a=5\sqrt n$ by Corollary~\ref{kahn_matchings:cor:formulas}.
\end{proof}

The next lemma is the technical heart of the argument: the distribution of $S$ has no heavy atom.
It is what forces $\Var[S]$ to be large, and it also rules out the degenerate possibility that $\lambda$ stays bounded.

\begin{lemma}\label{kahn_matchings:lem:maxatom}
Assume \eqref{kahn_matchings:eq:regime} and $n$ large.
Then
\[
  \max_s\Prob(S=s)\ \le\ \frac{22}{\sqrt L},
  \qquad L=a\sqrt n .
\]
\end{lemma}

\begin{proof}
Put $D=\lfloor\sqrt L/4\rfloor$; since $L$ is large we have $D\ge2$ and $D\ge\sqrt L/8$.
By Lemma~\ref{kahn_matchings:lem:mode} and \eqref{kahn_matchings:eq:asq}, $s_0+2D\le2L+2+\sqrt L/2\le n$, so the weights $w_{s_0},\dots,w_{s_0+2D}$ are all defined and positive.
As $s_0$ is a mode and $(w_s)$ is unimodal, $w_{s_0+2j}\ge w_{s_0+2D}$ for $0\le j\le D$, whence
\[
  W\ \ge\ \sum_{j=0}^{D}w_{s_0+2j}\ \ge\ (D+1)\,w_{s_0+2D}
\]
and therefore
\begin{equation}\label{kahn_matchings:eq:atom-bound}
  \max_s\Prob(S=s)=\frac{w_{s_0}}{W}
   \ \le\ \frac1{D+1}\cdot\frac{w_{s_0}}{w_{s_0+2D}}
   =\frac1{D+1}\prod_{j=0}^{D-1}\rho(s_0+2j)^{-1}.
\end{equation}
Since $\rho$ is decreasing, $\prod_{j=0}^{D-1}\rho(s_0+2j)\ge\rho(s_0+2D)^{D}$, so it remains to bound $\rho(s_0+2D)$ from below.

As $s_0$ is a mode and $s_0\ge\tfrac12L\ge2$, we have $w_{s_0-2}\le w_{s_0}$, that is $\rho(s_0-2)\ge1$, which reads $a^{2}(n-s_0+2)\ge(s_0-1)s_0$ and hence
\[
  a^{2}(n-s_0)\ \ge\ s_0^{2}-s_0-2a^{2}.
\]
Subtracting $2a^2D$ from both sides,
\[
  a^{2}(n-s_0-2D)\ \ge\ s_0^{2}-s_0-2a^{2}(D+1).
\]
Now $s_0\ge L/2$ gives $s_0\le 2s_0^{2}/L$, while $a^{2}\le L/4$ and $D\le\sqrt L/4$ give
\[
  2a^{2}(D+1)\ \le\ \frac L2\Bigl(\frac{\sqrt L}4+1\Bigr)
  =\frac{L^{3/2}}8+\frac L2
  \ \le\ s_0^{2}\Bigl(\frac1{2\sqrt L}+\frac2L\Bigr),
\]
using $s_0^{2}\ge L^{2}/4$.
Therefore
\[
  a^{2}(n-s_0-2D)\ \ge\ s_0^{2}\Bigl(1-\frac1{2\sqrt L}-\frac4L\Bigr).
\]
On the other hand $(2D+2)/s_0\le(\tfrac12\sqrt L+2)/(\tfrac12L)=1/\sqrt L+4/L$, so
\[
  (s_0+2D+1)(s_0+2D+2)\ \le\ (s_0+2D+2)^{2}
  \ \le\ s_0^{2}\Bigl(1+\frac1{\sqrt L}+\frac4L\Bigr)^{2}.
\]
Combining the last two displays,
\[
  \rho(s_0+2D)\ \ge\
  \frac{1-\tfrac1{2\sqrt L}-\tfrac4L}{\bigl(1+\tfrac1{\sqrt L}+\tfrac4L\bigr)^{2}}
  \ \ge\ 1-\frac{3}{\sqrt L}
\]
for $L$ large.
Using $\log(1-x)\ge-\tfrac43x$ for $0\le x\le\tfrac14$ together with $D\le\sqrt L/4$, we get
\[
  \prod_{j=0}^{D-1}\rho(s_0+2j)\ \ge\ \Bigl(1-\frac3{\sqrt L}\Bigr)^{D}
  \ \ge\ \exp\Bigl(-\frac{4D}{\sqrt L}\Bigr)\ \ge\ e^{-1}.
\]
Substituting into \eqref{kahn_matchings:eq:atom-bound} and using $D+1\ge\sqrt L/8$ gives $\max_s\Prob(S=s)\le 8e/\sqrt L\le22/\sqrt L$.
\end{proof}

\begin{lemma}\label{kahn_matchings:lem:mean-lb}
Assume \eqref{kahn_matchings:eq:regime} and $n$ large.
Then $\E[S]\ge L/4$, and consequently $\lambda(\Gna)\ge\sqrt n/4$.
\end{lemma}

\begin{proof}
Let $T'$ be the largest admissible $s$ with $s\le L/2$, so that $T'\ge L/2-2$.
For admissible $s\le T'$ we have $n-s\ge n-L/2\ge n/2$ by \eqref{kahn_matchings:eq:asq}, hence
\[
  \rho(s)=\frac{a^{2}(n-s)}{(s+1)(s+2)}
  \ \ge\ \frac{a^{2}n/2}{(L/2+2)^{2}}
  =\frac{L^{2}/2}{(L/2+2)^{2}}\ \ge\ \frac32
\]
for $L$ large.
Therefore $w_{T'-2j}\le(2/3)^{j}w_{T'}$ for all $j\ge0$, so that
\[
  \Prob(S\le T')=\sum_{j\ge0}\frac{w_{T'-2j}}{W}\ \le\ 3\,\frac{w_{T'}}{W}
  \ \le\ 3\max_s\Prob(S=s)\ \le\ \frac{66}{\sqrt L}
\]
by Lemma~\ref{kahn_matchings:lem:maxatom}.
Consequently $\E[S]\ge T'\Prob(S>T')\ge(L/2-2)(1-66/\sqrt L)\ge L/4$ for $L$ large.
Dividing by $a$ gives $\lambda(\Gna)=\E[S]/a\ge\sqrt n/4$.
\end{proof}

\subsubsection{A variance lower bound}

We use the following elementary fact, which converts a bound on the largest atom of an integer-valued random variable into a lower bound on its variance.

\begin{lemma}\label{kahn_matchings:lem:atomvar}
Let $X$ be an integer-valued random variable and put $\theta:=\max_{k\in\mathbb Z}\Prob(X=k)$.
Then
\[
  \Var[X]\ \ge\ \frac{(1-\theta)^{3}}{12\,\theta^{2}} .
\]
\end{lemma}

\begin{proof}
Fix $c\in\mathbb R$.
For $u\ge0$ the interval $[c-u,c+u]$ contains at most $2u+1$ integers, so $\Prob(|X-c|\le u)\le(2u+1)\theta$ and $\Prob(|X-c|>u)\ge1-(2u+1)\theta$.
Put $U=(1-\theta)/(2\theta)$, the point at which the last expression vanishes, so that the integrand below is nonnegative on $[0,U]$.
Then
\[
  \E\bigl[(X-c)^{2}\bigr]=\int_0^{\infty}2u\,\Prob(|X-c|>u)\,du
   \ \ge\ \int_0^{U}2u\bigl(1-(2u+1)\theta\bigr)\,du
   =U^{2}(1-\theta)-\tfrac43\theta U^{3}.
\]
Substituting $U=(1-\theta)/(2\theta)$ gives
\[
  \frac{(1-\theta)^{3}}{4\theta^{2}}-\frac{(1-\theta)^{3}}{6\theta^{2}}=\frac{(1-\theta)^{3}}{12\theta^{2}},
\]
and taking $c=\E[X]$ proves the claim.
\end{proof}

\begin{corollary}\label{kahn_matchings:cor:varS}
Assume \eqref{kahn_matchings:eq:regime} and $n$ large.
Then $\Var[S]\ge L/8000$.
\end{corollary}

\begin{proof}
Write $\theta=\max_s\Prob(S=s)$, so $\theta\le22/\sqrt L$ by Lemma~\ref{kahn_matchings:lem:maxatom}, and $\theta\le\tfrac1{10}$ since $L$ is large.
The function $x\mapsto(1-x)^{3}/(12x^{2})$ is decreasing on $(0,1)$, so Lemma~\ref{kahn_matchings:lem:atomvar} gives
\[
  \Var[S]\ \ge\ \frac{(1-\theta)^{3}}{12\theta^{2}}
  \ \ge\ \frac{(9/10)^{3}}{12}\cdot\frac{L}{484}\ \ge\ \frac L{8000}. \qedhere
\]
\end{proof}

\subsubsection{Proof of the main theorem}

\begin{proof}[Proof of Theorem~\ref{kahn_matchings:thm:main}]
Let $n_0$ be an absolute constant large enough for the finitely many hypotheses ``$n$ large'' invoked above, and assume $n\ge n_0$ and $4\le a\le n^{1/4}$.
Lemmas~\ref{kahn_matchings:lem:mean-ub} and~\ref{kahn_matchings:lem:mean-lb} give $\tfrac14\sqrt n\le\lambda(\Gna)\le5\sqrt n$.
For the variance, $a\ge4$ gives $\bigl((a-2)/(2a)\bigr)^{2}\ge(1/4)^{2}=1/16$, so \eqref{kahn_matchings:eq:lambda-sigma} and Corollary~\ref{kahn_matchings:cor:varS} give
\[
  \sigma^2(\Gna)\ \ge\ \frac1{16}\Var[S]\ \ge\ \frac{L}{128000}\ \ge\ 10^{-6}a\sqrt n,
\]
which is the bound on $\sigma^2(\Gna)$ with $c=10^{-6}$.
Dividing the two bounds gives $\sigma^2(\Gna)/\lambda(\Gna)\ge 2\cdot10^{-7}a$.
\end{proof}

\begin{proof}[Proof of Corollary~\ref{kahn_matchings:cor:ratio}]
For $n\ge\max\{n_0,256\}$ the choice $a=a_n=\lfloor n^{1/4}\rfloor$ satisfies $4\le a_n\le n^{1/4}$, so Theorem~\ref{kahn_matchings:thm:main} applies to $H_n=G_{n,a_n}$.
It gives $\lambda(H_n)\ge\sqrt n/4\to\infty$, then $\sigma^2(H_n)\ge c\,a_n\sqrt n\to\infty$, and finally $\sigma^2(H_n)/\lambda(H_n)\ge c\,a_n/5=\Omega(n^{1/4})$.
A bound $\sigma^2(G)\le C\lambda(G)$ valid for all finite simple graphs $G$ would force $a_n\le5C/c$ for every $n$, which is absurd.
\end{proof}

\putbib

\end{bibunit}
\endgroup
\subsection[(Klazar) An ordered hypergraph extremal function of order \texorpdfstring{$n\log n$}{n log n}]{An ordered hypergraph extremal function of order \texorpdfstring{$n\log n$}{n log n}}
\label{sol:hypergraph_turan}
\begingroup
\begin{bibunit}
\def\exe{\operatorname{ex}_{e}}
\def\exi{\operatorname{ex}_{i}}
\def\gex{\operatorname{gex}}

For a fixed pattern $F$, Klazar's extremal function $\exe(F,n)$ is the largest number of edges of a simple hypergraph on at most $n$ linearly ordered vertices that contains no order-preserving Berge copy of $F$.
For the five-vertex ordered graph $G_1$ below, the corresponding ordered \emph{graph} extremal function is $\Theta(n\log n)$, while Klazar's bounds for the hypergraph function differ by a factor of $\log n(\log\log n)^3$.
We show that the two functions have the same order of magnitude, so that $\exe(G_1,n)=\Theta(n\log n)$.
The proof compresses a $G_1$-free hypergraph into a single $G_1$-free ordered graph, at the cost of a factor of five and an additive $n$.

\subsubsection{Introduction}

We follow the setup of \citet{klazar2004extremal}.
A \emph{hypergraph} is a finite list $H=(E_j:j\in I)$ of finite nonempty subsets of $\mathbb{N}$, called \emph{edges}.
Its vertex set is $V(H)=\bigcup_{j\in I}E_j$, so that $H$ has no isolated vertices, and we write $e(H)=|I|$ for the number of edges and $i(H)=\sum_{j\in I}|E_j|$ for the number of vertex--edge incidences.
The hypergraph $H$ is \emph{simple} if its edges are pairwise distinct, and it is a \emph{graph} if every edge has exactly two elements.
Vertices always carry the linear order inherited from $\mathbb{N}$.

Given hypergraphs $H=(E_j:j\in I)$ and $F=(F_k:k\in K)$, we write $H\succ F$, and say that $H$ \emph{contains} $F$, if there are an increasing injection $\phi\colon V(F)\to V(H)$ and an injection $\psi\colon K\to I$ such that
\[
\phi(F_k)\subseteq E_{\psi(k)}\qquad\text{for every }k\in K.
\]
Otherwise $H$ is \emph{$F$-free}, written $H\not\succ F$.
Thus a copy of $F$ in $H$ consists of an order-preserving image of the vertices of $F$ together with a choice of pairwise distinct hyperedges of $H$, one for each edge of $F$, each containing the image of the edge assigned to it; the chosen hyperedges are allowed to contain further vertices.
This is a Berge-type containment that respects the vertex order, and it restricts to ordinary ordered subgraph containment when $H$ and $F$ are both graphs, since then $\phi(F_k)$ and $E_{\psi(k)}$ both have two elements and the inclusion forces equality.
The associated extremal functions are
\begin{align*}
\exe(F,n)&=\max\{e(H):\ H\text{ simple},\ |V(H)|\le n,\ H\not\succ F\},\\
\exi(F,n)&=\max\{i(H):\ H\text{ simple},\ |V(H)|\le n,\ H\not\succ F\},
\end{align*}
and $\gex(F,n)$ denotes the same maximum of $e(G)$ taken over simple ordered \emph{graphs} $G$ alone.
Of course $\gex(F,n)\le\exe(F,n)\le\exi(F,n)$ for every $F$ and $n$.

Throughout, $G_1$ denotes the ordered graph
\[
G_1=\bigl(\{1,3\},\{1,5\},\{2,3\},\{2,4\}\bigr)
\]
on the vertices $1<2<3<4<5$.
Reading $\{1,2\}$ as rows and $\{3,4,5\}$ as columns, $G_1$ is the ordered bipartite graph of the $2\times3$ zero-one matrix $\left(\begin{smallmatrix}1&0&1\\1&1&0\end{smallmatrix}\right)$, and \citet{furedi1990maximum} determined the extremal function of that matrix in the course of bounding the number of unit distances among the vertices of a convex $n$-gon.
The upper bound was obtained independently by \citet{bienstock1991extremal}.
\citet{furedi1992davenport} began the systematic study of the resulting matrix extremal problems, and \citet{tardos2019extremal} surveys the ordered graph theory that grew out of them.
As \citet{klazar2004extremal} records, in this particular case the bounds pass from ordered bipartite graphs to all ordered graphs, so that
\begin{equation}\label{hypergraph_turan:eq:gex}
\gex(G_1,n)=\Theta(n\log n).
\end{equation}

\citet{klazar2004extremal} introduced the containment $\succ$ and the functions $\exe$ and $\exi$, and asked how much of \eqref{hypergraph_turan:eq:gex} survives the passage from ordered graphs to ordered hypergraphs.
His Theorem~3.3 gives, in its two parts,
\[
n\log n\ll\exe(G_1,n)\le\exi(G_1,n)\ll n(\log n)^2(\log\log n)^3,
\]
the upper bound being obtained by iterating a recursion that controls $\exe(G_1,n)$ in terms of the ordered graph extremal function of blow-ups of $G_1$.
He then asked, as Problem~3.4:

\medskip
\noindent
\emph{What is the exact asymptotics of $\exe(G_1,n)$?}
\medskip

We answer this up to the implied constants: the hypergraph function has the same order of magnitude as the graph function.

The closest previous work we are aware of on the hypergraph function is Klazar's own.
In a companion paper, \citet{klazar2004extremalb} determines $\exe(F,n)$ and $\exi(F,n)$ exactly for the $55$ patterns $F$ with at most four incidences, whereas $G_1$ has eight.
That paper normalizes by $|V(H)|=n$ in place of $|V(H)|\le n$; by its Proposition~2.4 the two conventions give the same $\exe(F,n)$ unless $F$ consists of distinct singleton edges.
\citet{klazar2007extensions} and, independently, \citet{balogh2006hereditary} carry the linear bound of \citet{marcus2004excluded} for excluded permutation matrices from matrices over to ordered hypergraphs, but that theorem applies to permutation patterns, and $G_1$ is not one: by \eqref{hypergraph_turan:eq:gex} its extremal function is already superlinear at the graph level.

\begin{theorem}\label{hypergraph_turan:thm:main}
For every $n\ge1$,
\[
\gex(G_1,n)\ \le\ \exe(G_1,n)\ \le\ n+5\gex(G_1,n).
\]
In particular $\exe(G_1,n)=\Theta(n\log n)$.
\end{theorem}

The key point is that a $G_1$-free simple ordered hypergraph is controlled by a single $G_1$-free ordered graph.
Process the hyperedges one at a time and record, for each, one pair of its vertices that has not been recorded before.
The recorded pairs form an ordered graph $B$, and a copy of $G_1$ in $B$ pulls back to four distinct hyperedges forming a copy of $G_1$ in the host, so $B$ is $G_1$-free.
The hyperedges from which nothing new was recorded are cliques of $B$, and the crucial observation is that a $G_1$-free ordered graph $B$ has at most $4e(B)$ cliques of size at least two, because in such a graph any edge $\{a,b\}$ with $a<b$ has at most two common neighbors to the right of $b$.

\subsubsection{Cliques in a \texorpdfstring{$G_1$}{G1}-free ordered graph}

We use ``clique'' to mean any set of pairwise adjacent vertices, not necessarily a maximal one.

\begin{lemma}\label{hypergraph_turan:lem:cliques}
Let $B$ be a $G_1$-free simple ordered graph.
Then for every edge $\{a,b\}$ of $B$ with $a<b$, the vertices $a$ and $b$ have at most two common neighbors greater than $b$.
Consequently $B$ has at most $4e(B)$ cliques of size at least two.
\end{lemma}

\begin{proof}
Fix an edge $\{a,b\}$ of $B$ with $a<b$, and suppose that $a$ and $b$ have three common neighbors greater than $b$.
Choose three such vertices $c<d<f$, so that the five vertices
\[
a<b<c<d<f
\]
of $B$ carry in particular the four edges
\[
\{a,c\},\quad\{a,f\},\quad\{b,c\},\quad\{b,d\}.
\]
The increasing injection $1\mapsto a$, $2\mapsto b$, $3\mapsto c$, $4\mapsto d$, $5\mapsto f$ carries the four edges of $G_1$ to these four edges of $B$, so $B\succ G_1$, a contradiction.

Now let $K$ be a clique of $B$ with $|K|\ge2$ and let $a<b$ be its two smallest vertices.
Then $\{a,b\}$ is an edge of $B$, and every remaining vertex of $K$ is a common neighbor of $a$ and $b$ greater than $b$.
Hence $K$ is determined by the edge $\{a,b\}$ together with a subset of the set of common neighbors of $a$ and $b$ to the right of $b$, a set of size at most two.
So at most $2^2=4$ cliques of size at least two have $a$ and $b$ as their two smallest vertices, and summing over the edges of $B$ gives the bound.
\end{proof}

\subsubsection{Proof of the main theorem}

\begin{proof}[Proof of Theorem~\ref{hypergraph_turan:thm:main}]
The lower bound is immediate: a simple ordered graph is a simple ordered hypergraph, and the containment relation $\succ$ used to define $\gex$ is the one used to define $\exe$, so the maximum defining $\exe(G_1,n)$ is taken over a larger family.

For the upper bound, let $H$ be a simple ordered hypergraph with $|V(H)|\le n$ and $H\not\succ G_1$.
Since $H$ is simple, its singleton edges are pairwise distinct one-element subsets of $V(H)$, so there are at most $n$ of them.

We build an ordered graph $B$ from the remaining edges.
List the edges of $H$ of size at least two in an arbitrary order and start with $B$ empty.
Processing them one at a time, if the current edge $E$ contains a pair $\{a,b\}$ that is not yet an edge of $B$, then add one such pair to $B$ and call $E$ \emph{assigned} to that pair; otherwise call $E$ \emph{unassigned} and change nothing.
Every assigned edge creates exactly one new edge of $B$, so assignment is a bijection between the assigned edges of $H$ and the edges of $B$, and in particular the number of assigned edges is $e(B)$.
Also $V(B)\subseteq V(H)$, so $|V(B)|\le n$.

We claim that $B$ is $G_1$-free.
Suppose instead that there are vertices $x_1<x_2<x_3<x_4<x_5$ of $B$ such that
\[
\{x_1,x_3\},\quad\{x_1,x_5\},\quad\{x_2,x_3\},\quad\{x_2,x_4\}
\]
are all edges of $B$.
These are four distinct edges of $B$, so the hyperedges assigned to them are four distinct edges $D_1,D_2,D_3,D_4$ of $H$ with
\[
\{x_1,x_3\}\subseteq D_1,\quad\{x_1,x_5\}\subseteq D_2,\quad\{x_2,x_3\}\subseteq D_3,\quad\{x_2,x_4\}\subseteq D_4.
\]
Taking $\phi(i)=x_i$ for $1\le i\le5$ and sending the edges $\{1,3\},\{1,5\},\{2,3\},\{2,4\}$ of $G_1$ to $D_1,D_2,D_3,D_4$ respectively gives $H\succ G_1$, contrary to hypothesis.
Hence $B$ is $G_1$-free, and therefore
\[
e(B)\le\gex(G_1,n).
\]

It remains to count the unassigned edges.
If an edge $E$ of $H$ with $|E|\ge2$ is unassigned, then at the moment $E$ was processed every pair contained in $E$ was already an edge of $B$; since edges are only ever added to $B$, the same holds at the end of the process, so $E$ is a clique of $B$ of size at least two.
Distinct edges of $H$ are distinct sets because $H$ is simple, so distinct unassigned edges give distinct cliques of $B$.
By Lemma~\ref{hypergraph_turan:lem:cliques} the number of unassigned edges is therefore at most $4e(B)$.

Collecting the three types of edge,
\[
e(H)\ \le\ n+e(B)+4e(B)\ \le\ n+5\gex(G_1,n),
\]
which is the upper bound.
The final assertion follows by combining the two bounds with \eqref{hypergraph_turan:eq:gex}.
\end{proof}

\begin{remark}\label{hypergraph_turan:rem:incidences}
The argument above bounds only the number of edges, since the edges assigned in the auxiliary construction may have large cardinality.
The incidence count is nevertheless controlled by \citet[Theorem~2.3]{klazar2004extremal}, which states that $\exi(F,n)\le(2p-1)(q-1)\exe(F,n)$ whenever $F$ has no two separated edges, where $p=|V(F)|$ and $q=e(F)>1$; here two edges are separated if every vertex of the first precedes every vertex of the second.
Every edge of $G_1$ has largest vertex at least $3$ and smallest vertex at most $2$, so no edge of $G_1$ lies entirely to the left of another and the hypothesis holds; with $p=5$ and $q=4$ it gives $\exi(G_1,n)\le27\exe(G_1,n)$.
This is Klazar's own route from part~1 to part~2 of his Theorem~3.3, and applied to Theorem~\ref{hypergraph_turan:thm:main} it yields $\exi(G_1,n)=\Theta(n\log n)$ as well.
\end{remark}

\putbib

\end{bibunit}
\endgroup
\subsection[(Kumar--Mohar--Pragada--Zhan) Eigenvalues below \texorpdfstring{$-2$}{-2} under repeated subdivision]{Eigenvalues below \texorpdfstring{$-2$}{-2} under repeated subdivision}
\label{sol:subdivision_eigenvalue}
\begingroup
\begin{bibunit}

This problem concerns the eigenvalues that a graph retains below $-2$ when a fixed set of its edges is subdivided over and over.
We show that their number is eventually constant, and that the constant is the number of negative eigenvalues of a single fixed matrix on the original vertex set, namely $2I_n+A_R-D_S$, where $A_R$ records the edges that are never subdivided and $D_S$ the degrees in the edges that are.
The proof eliminates the subdivision vertices by a Schur complement, which is legitimate because the block they contribute is positive definite, and then compares the resulting $n\times n$ matrices in the Loewner order.

\subsubsection{Introduction}

Let $G$ be a finite simple graph on $n=|V(G)|$ vertices, let $S\subseteq E(G)$ be a fixed set of edges, and write $R=E(G)\setminus S$ for the rest.
For $t\ge1$ let $G_t=G_t(S)$ be the graph obtained from $G$ by replacing every edge $uv\in S$ with a path of length $t$ from $u$ to $v$, the \emph{$t$-stretch} of $uv$, whose $t-1$ internal vertices are new and lie on no other stretch.
The edges of $R$ are left untouched, and $G_1=G$.
We write $\lambda_1(X)\ge\cdots\ge\lambda_{|V(X)|}(X)$ for the adjacency eigenvalues of a graph $X$ and
\[
m_X(a,b):=\#\{i:\lambda_i(X)\in(a,b)\}
\]
for the number of them in an interval, counted with multiplicity.
For a real symmetric matrix $M$ we write $n_-(M)$ for its number of negative eigenvalues, counted with multiplicity, that is, for its \emph{negative index of inertia}.
Finally $A_R$ and $A_S$ denote the adjacency matrices of the spanning subgraphs $(V(G),R)$ and $(V(G),S)$, and $D_S$ the diagonal matrix of the degrees $\deg_S(v)$ in $(V(G),S)$, all indexed by $V(G)$; the degree of $v$ in $(V(G),R)$ is written $\deg_R(v)$.

The window $[-2,2]$ is the one that matters here.
Deleting $Q=\{v\in V(G):\deg_G(v)\ge3\}$ from $G_t$ leaves a disjoint union of paths and cycles, whose eigenvalues all lie in $[-2,2]$, so interlacing bounds the number of eigenvalues of $G_t$ above $2$, and likewise the number below $-2$, by $|Q|$, uniformly in $t$ \citep{kumar2025subdivision}.
The value $-2$ is also the classical threshold on the negative side: a connected graph with $m_G(-\infty,-2)=0$ is a generalized line graph or one of finitely many exceptional graphs representable in the root system $E_8$, by the classification of \citet{cameron1991line}, and the structure theory of graphs with least eigenvalue $-2$ is built around that dichotomy \citep{cvetkovic2004spectral}.
So $m_{G_t}(-\infty,-2)$ measures how far the subdivided graph is from being a generalized line graph, and the question is whether repeated subdivision settles this quantity down.

Subdivision has been studied spectrally since the theorem of \citet{hoffman1974spectral} on the spectral radii of topologically equivalent graphs, and it is the basic operation in Hoffman's program on limit points of spectral radii, surveyed by \citet{wang2025developments}.
Subdividing a \emph{subset} of the edges is exactly the operation used by \citet{haiman2022graphs} to construct graphs with high approximate second eigenvalue multiplicity, showing that the multiplicity bound of \citet{jiang2021equiangular} behind the resolution of the equiangular lines problem is sharp in that relaxed sense.
Motivated by this, \citet{kumar2025subdivision} analyze the whole spectrum of $G_t$ as $t\to\infty$.
They study the four sequences $m_{G_t}(2,\infty)$, $m_{G_t}(-\infty,-2)$, $m_{H_t}(2,\infty)$, and $m_{H_t}(-\infty,-2)$, where $H_t=H_t(S)$ is obtained from $G_{2t+1}(S)$ by deleting the middle edge of every stretch.
Since $H_t$ is an induced subgraph of $H_{t+1}$, the two $H$-sequences are nondecreasing, and subdividing a single edge cannot decrease the number of eigenvalues above $2$, so $m_{G_t}(2,\infty)$ is nondecreasing as well.
Being nondecreasing and bounded by $|Q|$, those three sequences are eventually constant.
The fourth sequence is not monotone, and is left open as \citet[Conjecture~20]{kumar2025subdivision}:

\medskip
\noindent
\emph{There exists $t_0\in\mathbb{N}$ such that $m_{G_t}(-\infty,-2)$ is constant for all $t\ge t_0$.}
\medskip

\noindent
The failure of monotonicity is genuine and is caused by parity: subdividing changes the length of every cycle through a stretch, and Figure~1 of \citet{kumar2025subdivision} exhibits a graph whose successive subdivisions have $1$, $0$, and $1$ eigenvalues below $-2$.
Example~\ref{subdivision_eigenvalue:ex:pendant} below returns to that graph and follows the sequence to the end.

In the special case $S=E(G)$ the conjecture already follows from the same paper.
There \citet{kumar2025subdivision} show that for each fixed $k\le n$ the eigenvalue $\lambda_{|V(G_t)|-k+1}(G_t)$ tends to $-d_k/\sqrt{d_k-1}$ when $d_k\ge3$ and to $-2$ otherwise, where $d_1\ge\cdots\ge d_n$ is the degree sequence of $G$, so at least $|Q|$ eigenvalues of $G_t$ lie below $-2$ once $t$ is large, while their interlacing bound caps the count at $|Q|$ for every $t$.
For a general $S$ the upper half of this argument survives, since that cap is proved for an arbitrary $S$; the lower half does not.
They do prove that each individual eigenvalue $\lambda_{|V(G_t)|-k+1}(G_t)$ converges, but for a general $S$ the limit is not known in closed form and may be $-2$ itself, and an eigenvalue converging to $-2$ is free to cross the threshold again and again.
That boundary case is the whole difficulty.
Since the count never exceeds $|Q|$, it equals the number of $k\le|Q|$ with $\lambda_{|V(G_t)|-k+1}(G_t)<-2$; the indices whose limit differs from $-2$ contribute a constant from some point on, and Conjecture~20 asserts that the remaining ones do too.
We are aware of no work on the conjecture itself; the papers citing \citet{kumar2025subdivision} pursue other lines, among them the Hoffman program \citep{wang2025developments} and the effect of subdivision on other spectral parameters, such as the Perron eigenvalue of the Ricci matrix of a tree \citep{bai2026edge}, which is also handled by eliminating the internal structure with a Schur complement.

We prove the conjecture for every $G$ and every $S$, and identify the constant.

\begin{theorem}
\label{subdivision_eigenvalue:thm:main}
Let $G$ be a finite simple graph on $n$ vertices, let $S\subseteq E(G)$, let $R=E(G)\setminus S$, and put
\begin{equation}
\label{subdivision_eigenvalue:eq:kinf}
K_\infty:=2I_n+A_R-D_S.
\end{equation}
Then
\[
m_{G_t}(-\infty,-2)\le n_-(K_\infty)
\qquad\text{for every }t\ge2,
\]
with equality for all sufficiently large $t$.
In particular $m_{G_t}(-\infty,-2)$ is eventually constant, and its eventual value is $n_-(K_\infty)$.
\end{theorem}

When every edge is subdivided the matrix $K_\infty$ is diagonal and the eventual value can be read off the degree sequence.

\begin{corollary}
\label{subdivision_eigenvalue:cor:alledges}
If $S=E(G)$, then $m_{G_t}(-\infty,-2)=|Q|$ for all sufficiently large $t$, where $Q=\{v\in V(G):\deg_G(v)\ge3\}$.
\end{corollary}

This is the value predicted by \citet{kumar2025subdivision}, and it makes explicit the negative-side case of the tightness they assert for their interlacing bound.

\begin{remark}
\label{subdivision_eigenvalue:rem:effective}
The threshold in Theorem~\ref{subdivision_eigenvalue:thm:main} is effective.
Write $r=n_-(K_\infty)$, and note that for $r=0$ the upper bound already forces equality for every $t\ge2$.
For $r\ge1$ let $\gamma:=-\lambda_{n-r+1}(K_\infty)>0$ be the smallest absolute value of a negative eigenvalue of $K_\infty$.
The proof gives equality for every $t\ge2$ with $t>2\Delta(G)/\gamma$, where $\Delta(G)$ is the maximum degree of $G$.
Moreover $K_\infty$ has integer entries and spectral radius at most $2+\Delta(G)$, which forces $\gamma\ge(2+\Delta(G))^{-(n-1)}$.
Hence equality holds for every $t\ge2$ with $t>2\Delta(G)(2+\Delta(G))^{n-1}$.
\end{remark}

\begin{example}
\label{subdivision_eigenvalue:ex:pendant}
Let $G$ be the cycle $v_1v_2v_3v_4$ together with a pendant edge $v_1v_5$, and let $S=\{v_2v_3\}$, so that $G_t$ is the cycle $C_{t+3}$ with a pendant edge attached.
This is the family in Figure~1 of \citet{kumar2025subdivision}, whose first three members have $1$, $0$, and $1$ eigenvalues below $-2$.
Here $D_S=\operatorname{diag}(0,1,1,0,0)$ and
\[
K_\infty=
\begin{pmatrix}
2&1&0&1&1\\
1&1&0&0&0\\
0&0&1&1&0\\
1&0&1&2&0\\
1&0&0&0&2
\end{pmatrix},
\]
whose eigenvalues are approximately $-0.1388$, $0.5401$, $1.5102$, $2.3835$ and $3.7050$, so $n_-(K_\infty)=1$.
Theorem~\ref{subdivision_eigenvalue:thm:main} therefore predicts the value $1$, and indeed the sequence $m_{G_t}(-\infty,-2)$ for $t=1,2,3,\dots$ is
\[
1,\;0,\;1,\;0,\;1,\;1,\;1,\;1,\;\dots,
\]
oscillating with the parity of the cycle length until $t=5$ and constant thereafter.
The effective threshold of Remark~\ref{subdivision_eigenvalue:rem:effective} gives only $t>2\cdot3/\gamma$ with $\gamma\approx0.1388$, that is $t\ge44$, so the bound is far from sharp on this example.
\end{example}

The proof has two steps.
The first is exact and holds for every $t\ge2$: eliminating the subdivision vertices replaces $A(G_t)+2I$ with an $n\times n$ matrix
\begin{equation}
\label{subdivision_eigenvalue:eq:kt}
K_t=K_\infty+\frac1t\bigl(D_S-(-1)^tA_S\bigr)
\end{equation}
of the same negative inertia.
The second is a comparison: the perturbation in \eqref{subdivision_eigenvalue:eq:kt} is positive semidefinite for every parity of $t$, being the Laplacian of $(V(G),S)$ for even $t$ and its signless Laplacian for odd $t$, and it tends to $0$.
So $K_t$ approaches $K_\infty$ from above in the Loewner order, which pins the negative inertia from both sides.

\subsubsection{Eliminating the subdivision vertices}

Write $P_\ell$ for the path on $\ell$ vertices and put $T_\ell:=2I_\ell+A(P_\ell)$.

\begin{lemma}
\label{subdivision_eigenvalue:lem:path}
For every $\ell\ge1$ the matrix $T_\ell$ is positive definite, $\det T_\ell=\ell+1$, and
\[
(T_\ell^{-1})_{1,1}=(T_\ell^{-1})_{\ell,\ell}=\frac{\ell}{\ell+1},
\qquad
(T_\ell^{-1})_{1,\ell}=(T_\ell^{-1})_{\ell,1}=\frac{(-1)^{\ell+1}}{\ell+1}.
\]
\end{lemma}

\begin{proof}
For $y=(y_1,\dots,y_\ell)^{\mathsf T}\in\R^\ell$,
\[
y^{\mathsf T}T_\ell y
=
2\sum_{i=1}^\ell y_i^2+2\sum_{i=1}^{\ell-1}y_iy_{i+1}
=
y_1^2+y_\ell^2+\sum_{i=1}^{\ell-1}(y_i+y_{i+1})^2 .
\]
If this vanishes then $y_1=0$ and $y_{i+1}=-y_i$ for all $i$, so $y=0$; hence $T_\ell$ is positive definite.
Expanding $\det T_\ell$ along the first row gives $\det T_\ell=2\det T_{\ell-1}-\det T_{\ell-2}$ with $\det T_0=1$ and $\det T_1=2$, whence $\det T_\ell=\ell+1$.
By Cramer's rule, $(T_\ell^{-1})_{1,1}$ is the determinant of the matrix obtained by deleting the first row and column, namely $\det T_{\ell-1}=\ell$, divided by $\det T_\ell=\ell+1$; the entry $(T_\ell^{-1})_{\ell,\ell}$ is the same by symmetry.
For the corner entry, $(T_\ell^{-1})_{1,\ell}=(-1)^{\ell+1}\det(N')/\det T_\ell$, where $N'$ is $T_\ell$ with its last row and first column deleted.
The rows of $N'$ are indexed by $1\le i\le\ell-1$ and its columns by $2\le j\le\ell$, and the entry in position $(i,j)$ vanishes unless $|i-j|\le1$; reindexing the columns by $k=j-1$ makes $N'$ lower triangular with every diagonal entry equal to $(T_\ell)_{i,i+1}=1$.
Hence $\det N'=1$ and $(T_\ell^{-1})_{1,\ell}=(-1)^{\ell+1}/(\ell+1)$.
\end{proof}

\begin{figure}[t]
\centering
\begin{tikzpicture}[
  orig/.style={circle,fill=black,inner sep=1.7pt},
  sub/.style={circle,draw,fill=white,inner sep=1.9pt}
]
\node[orig,label=below:$u$] (u) at (0,0) {};
\node[sub,label=below:$w_1$] (w1) at (1.1,0) {};
\node[sub,label=below:$w_2$] (w2) at (2.2,0) {};
\node (mid) at (3.3,0) {$\cdots$};
\node[sub,label=below:$w_{t-1}$] (wm) at (4.4,0) {};
\node[orig,label=below:$v$] (v) at (5.5,0) {};
\draw (u) -- (w1) -- (w2) -- (mid) -- (wm) -- (v);
\node at (2.75,0.95) {$T_{t-1}$};

\node[orig,label=below:$u$] (u2) at (8.2,0) {};
\node[orig,label=below:$v$] (v2) at (10.4,0) {};
\draw[dashed] (u2) to[bend left=45] (v2);
\node at (9.3,1.05) {$-(-1)^t/t$};
\node at (8.2,-0.95) {$-(t-1)/t$};
\node at (10.4,-0.95) {$-(t-1)/t$};
\end{tikzpicture}
\caption{The $t$-stretch of an edge $uv\in S$ inside $G_t$, on the left, and its contribution to the Schur complement, on the right.
The $t-1$ internal vertices, drawn hollow, carry the positive definite block $T_{t-1}$ of $A(G_t)+2I$; eliminating them subtracts $(t-1)/t$ from each of the diagonal entries at $u$ and at $v$ and subtracts $(-1)^t/t$ from the entries at $uv$ and at $vu$, which is the content of Proposition~\ref{subdivision_eigenvalue:prop:schur}.}
\label{subdivision_eigenvalue:fig:stretch}
\end{figure}
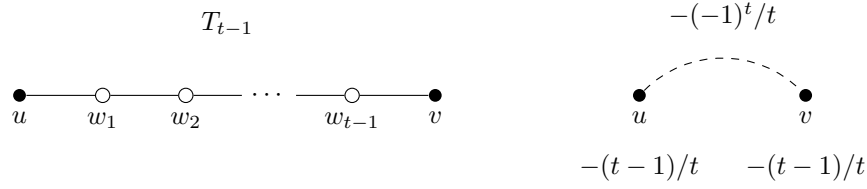

\begin{proposition}
\label{subdivision_eigenvalue:prop:schur}
For every $t\ge2$,
\[
m_{G_t}(-\infty,-2)=n_-(K_t),
\qquad
K_t:=2I_n+A_R-\frac{t-1}{t}D_S-\frac{(-1)^t}{t}A_S,
\]
and $K_t$ satisfies \eqref{subdivision_eigenvalue:eq:kt}.
\end{proposition}

\begin{proof}
Put $B_t:=A(G_t)+2I$.
An eigenvalue $\lambda$ of $A(G_t)$ satisfies $\lambda<-2$ exactly when $\lambda+2<0$, so
\[
m_{G_t}(-\infty,-2)=n_-(B_t).
\]
Partition $V(G_t)$ into the original vertices $V=V(G)$ and the set $W_t$ of internal vertices of the stretches.
Two internal vertices are adjacent in $G_t$ only if they lie on the same stretch and are consecutive on it, so $B_t[W_t,W_t]$ is block diagonal with one block $T_{t-1}$ for each edge of $S$, as in Figure~\ref{subdivision_eigenvalue:fig:stretch}.
By Lemma~\ref{subdivision_eigenvalue:lem:path} this block is positive definite, so Haynsworth's inertia additivity formula \citep{haynsworth1968determination} applies and gives
\[
n_-(B_t)=n_-\bigl(B_t[V,V]-B_t[V,W_t]\,B_t[W_t,W_t]^{-1}\,B_t[W_t,V]\bigr).
\]
It remains to identify the Schur complement on the right.

No edge of $S$ survives in $G_t$ for $t\ge2$, so $B_t[V,V]=2I_n+A_R$.
Label the stretch of $uv\in S$ as $u,w_1,\dots,w_\ell,v$ with $\ell=t-1$, so that the column of $B_t[V,W_t]$ at $w_1$ is $e_u$, the column at $w_\ell$ is $e_v$, and all other columns of that stretch vanish; here $e_u,e_v\in\R^{V}$ are standard basis vectors, and for $\ell=1$ the single column is $e_u+e_v$.
Writing $C=T_\ell^{-1}$, the stretch of $uv$ therefore contributes
\[
C_{1,1}e_ue_u^{\mathsf T}+C_{\ell,\ell}e_ve_v^{\mathsf T}+C_{1,\ell}\bigl(e_ue_v^{\mathsf T}+e_ve_u^{\mathsf T}\bigr)
\]
to $B_t[V,W_t]B_t[W_t,W_t]^{-1}B_t[W_t,V]$, and this formula is correct for $\ell=1$ as well, since then $C_{1,1}=C_{\ell,\ell}=C_{1,\ell}=\tfrac12$.
By Lemma~\ref{subdivision_eigenvalue:lem:path} with $\ell=t-1$ we have $C_{1,1}=C_{\ell,\ell}=(t-1)/t$ and $C_{1,\ell}=(-1)^t/t$.
Summing over $S$ and using $\sum_{uv\in S}(e_ue_u^{\mathsf T}+e_ve_v^{\mathsf T})=D_S$ and $\sum_{uv\in S}(e_ue_v^{\mathsf T}+e_ve_u^{\mathsf T})=A_S$, the Schur complement equals $K_t$ as displayed.
Finally
\[
K_t
=
2I_n+A_R-D_S+\frac1tD_S-\frac{(-1)^t}{t}A_S
=
K_\infty+\frac1t\bigl(D_S-(-1)^tA_S\bigr),
\]
which is \eqref{subdivision_eigenvalue:eq:kt}.
\end{proof}

\begin{remark}
\label{subdivision_eigenvalue:rem:subdivisiongraph}
For $t=2$ and $S=E(G)$ the graph $G_2$ is the subdivision graph of $G$ and Proposition~\ref{subdivision_eigenvalue:prop:schur} reads $K_2=\tfrac12\bigl(4I_n-(D+A(G))\bigr)$, where $D$ is the degree matrix of $G$ and $D+A(G)$ is the signless Laplacian of $G$.
So $m_{G_2}(-\infty,-2)$ is the number of signless Laplacian eigenvalues of $G$ exceeding $4$.
Indeed, if $N$ is the vertex-edge incidence matrix of $G$, then in the block form given by $V(G)$ and the subdivision vertices
\[
A(G_2)=
\begin{pmatrix}
0&N\\
N^{\mathsf T}&0
\end{pmatrix},
\]
so the nonzero eigenvalues of $A(G_2)$ are the numbers $\pm\sqrt{q}$ with $q$ a nonzero eigenvalue of $NN^{\mathsf T}=D+A(G)$, and $-\sqrt{q}<-2$ exactly when $q>4$.
\end{remark}

\subsubsection{Comparison in the Loewner order}

\begin{lemma}
\label{subdivision_eigenvalue:lem:psd}
For every $t\ge2$ we have $K_t\succeq K_\infty$ in the Loewner order and $\|K_t-K_\infty\|\le2\Delta(G)/t$, where $\|\cdot\|$ is the spectral norm.
\end{lemma}

\begin{proof}
Fix an arbitrary orientation of each edge of $S$.
For $x\in\R^{V(G)}$,
\[
x^{\mathsf T}\bigl(D_S-(-1)^tA_S\bigr)x
=
\sum_{uv\in S}\bigl(x_u-(-1)^tx_v\bigr)^2
\ge0,
\]
since expanding the squares returns $x^{\mathsf T}D_Sx-(-1)^tx^{\mathsf T}A_Sx$ and the summands do not depend on the chosen orientation.
Thus $D_S-(-1)^tA_S$ is positive semidefinite, being the Laplacian of $(V(G),S)$ for even $t$ and its signless Laplacian for odd $t$, and \eqref{subdivision_eigenvalue:eq:kt} gives $K_t-K_\infty=\frac1t\bigl(D_S-(-1)^tA_S\bigr)\succeq0$.
For the norm bound, the spectral norm of a real symmetric matrix is at most its largest absolute row sum, and the row of $D_S-(-1)^tA_S$ indexed by $v$ has diagonal entry $\deg_S(v)$ and a further $\deg_S(v)$ entries of absolute value $1$, one for each $S$-neighbor of $v$, so that row sum is $2\deg_S(v)\le2\Delta(G)$, and the factor $1/t$ yields the stated bound.
\end{proof}

\subsubsection{Proof of the main theorem}

\begin{proof}[Proof of Theorem~\ref{subdivision_eigenvalue:thm:main}]
Put $r=n_-(K_\infty)$ and fix $t\ge2$.
By Proposition~\ref{subdivision_eigenvalue:prop:schur} it suffices to prove that $n_-(K_t)\le r$ always, and that $n_-(K_t)\ge r$ once $t$ is large.

For the upper bound we may assume $r<n$, since otherwise there is nothing to prove.
By definition of $r$ we have $\lambda_{n-r}(K_\infty)\ge0$, and $K_t\succeq K_\infty$ gives $\lambda_j(K_t)\ge\lambda_j(K_\infty)$ for every $j$ by Weyl monotonicity \citep{horn2012matrix}.
Hence $\lambda_{n-r}(K_t)\ge0$, so at most $r$ eigenvalues of $K_t$ are negative, that is, $n_-(K_t)\le r$.
This holds for every $t\ge2$ and is the asserted inequality.

For the lower bound we may assume $r\ge1$, since for $r=0$ the upper bound already gives $n_-(K_t)=0$ for every $t\ge2$.
Let $\gamma:=-\lambda_{n-r+1}(K_\infty)>0$ be the gap of Remark~\ref{subdivision_eigenvalue:rem:effective}, so that $\lambda_j(K_\infty)\le-\gamma$ for every $j\ge n-r+1$.
Weyl's perturbation inequality together with Lemma~\ref{subdivision_eigenvalue:lem:psd} gives
\[
\lambda_j(K_t)\le\lambda_j(K_\infty)+\|K_t-K_\infty\|\le-\gamma+\frac{2\Delta(G)}{t}
\]
for those $j$, which is negative as soon as $t>2\Delta(G)/\gamma$.
For such $t$ at least $r$ eigenvalues of $K_t$ are negative, so $n_-(K_t)\ge r$.

Combining the two bounds, $n_-(K_t)=r$ for every $t>\max\{2,2\Delta(G)/\gamma\}$, and Proposition~\ref{subdivision_eigenvalue:prop:schur} turns this into $m_{G_t}(-\infty,-2)=r=n_-(K_\infty)$, as desired.
\end{proof}

\begin{proof}[Proof of Remark~\ref{subdivision_eigenvalue:rem:effective}]
The first assertion is the threshold $t>2\Delta(G)/\gamma$ displayed in the proof above.
For the second, the spectral norm of $K_\infty$ is at most its largest absolute row sum, which is $\max_v\bigl(|2-\deg_S(v)|+\deg_R(v)\bigr)\le2+\Delta(G)$, so every eigenvalue of $K_\infty$ has absolute value at most $2+\Delta(G)$.
On the other hand $K_\infty$ has integer entries, so the product of its nonzero eigenvalues is, up to sign, the last nonvanishing coefficient of its characteristic polynomial and hence a nonzero integer.
The absolute values of those at most $n$ eigenvalues therefore multiply to at least $1$, and each factor is at most $2+\Delta(G)$, so the smallest of them, and in particular $\gamma$, is at least $(2+\Delta(G))^{-(n-1)}$.
\end{proof}

\begin{proof}[Proof of Corollary~\ref{subdivision_eigenvalue:cor:alledges}]
When $S=E(G)$ we have $A_R=0$ and $D_S$ equal to the degree matrix $D$ of $G$, so \eqref{subdivision_eigenvalue:eq:kinf} makes $K_\infty=2I_n-D$ diagonal with entries $2-\deg_G(v)$.
Such an entry is negative exactly when $\deg_G(v)\ge3$, so $n_-(K_\infty)=|Q|$.
\end{proof}

\putbib
\end{bibunit}
\endgroup
\subsection[(Lund--Saraf--Wolf) Small unions of lines closing a route to the Nikodym bound]{Small unions of lines closing a route to the Nikodym bound}
\label{sol:finkakeya}
\begingroup
\begin{bibunit}
\def\F{\mathbb{F}}
\def\PG{\mathrm{PG}}
\def\sq{(\F_q^{*})^{2}}

\citet{lund2018finite} conjectured that a family of $\Omega(q^3)$ lines in $\F_q^3$, no plane of which contains a superlinear number of them, must cover all but a vanishing proportion of $\F_q^3$.
We show that the union can have density $1/2+o(1)$.
For every odd prime power $q$ there is a family $L$ of $q^2(q+1)/2$ lines, at most $q+1$ of which lie in any single plane, whose union has exactly $q^2(q+1)/2$ points.
The family consists of one half of the tangent lines to the paraboloid $z=x^2-\nu y^2$, where $\nu$ is a fixed nonsquare, the half being selected by the quadratic character of the direction.
The same family also refutes the sharper quantitative conjecture that they state alongside the first one.

\subsubsection{Introduction}

For a set $L$ of affine lines in $\F_q^3$, write
\[
P(L)=\bigcup_{\ell\in L}\ell
\]
for the set of points lying on some line of $L$.
The problem is to bound $|P(L)|$ from below when $L$ is large and no plane of $\F_q^3$ carries too much of it.
In Section~1.2.1 of their paper on Kakeya and Nikodym sets in three dimensions, \citet{lund2018finite} propose the following, their Conjecture~1.4.

\medskip
\noindent
\emph{Let $C>0$ be a constant independent of $q$, and let $\alpha(q)\in\omega(q)$.
If $L$ is a set of at least $Cq^3$ lines and no plane contains $\alpha(q)$ lines of $L$, then $|P(L)|\ge(1-o(1))q^3$.
The $o(1)$ is a function of $q$ that depends on $C$ and $\alpha$.}
\medskip

\noindent
In the same place they propose the following sharper form, their Conjecture~1.5.

\medskip
\noindent
\emph{Let $\varepsilon>0$ be any constant and let $q$ be a sufficiently large prime power.
Let $L$ be a set of at least $q^{5/2+\varepsilon}$ lines in $\F_q^3$ such that no plane contains more than $(1/2)q^{3/2}$ lines of $L$.
Then, $|P(L)|\ge q^3-O(q^{5/2})$.}
\medskip

\noindent
These two conjectures are the geometric heart of \citet{lund2018finite}.
For $q$ sufficiently large, their Theorem~1.1 gives $|K|\ge0.2107q^3$ for every Kakeya set $K\subseteq\F_q^3$ and their Theorem~1.3 gives $|\mathcal N|\ge0.38q^3$ for every Nikodym set $\mathcal N\subseteq\F_q^3$, while their Theorem~3.8 shows that Conjecture~1.4 would upgrade the second of these to the conjectured optimal bound $(1-o(1))q^3$.
Conjecture~1.5 is carefully calibrated against two constructions of \citet{lund2018finite}: from a nondegenerate Hermitian variety, for square $q=p^2$ and each $0<\theta<1$, they produce a set $L$ of $(\theta+o(1))q^{7/2}$ lines with no plane containing more than $(\theta+o(1))q^{3/2}$ of them and with $|P(L)|\le q^3-(1-\theta+o(1))q^{5/2}$, so the error term $O(q^{5/2})$ in the conclusion of Conjecture~1.5 cannot be improved; a second construction of theirs, the union of $O(q^{1/2})$ Hermitian varieties, shows that the hypothesis $|L|\ge q^{5/2+\varepsilon}$ cannot be substantially relaxed.
They also prove that any $0.62q^3$ lines in $\F_q^3$ already satisfy $|P(L)|\ge(0.38-o(1))q^3$ with no hypothesis on planes at all, and that without such a hypothesis this is close to sharp: for large $q$, they take all lines lying in a union of $0.62q$ planes through a common point and obtain $(1-o(1))0.62q^3$ lines whose union has fewer than $0.43q^3$ points.
It is this flat example that the plane hypothesis of Conjecture~1.4 is designed to exclude, and our family shows that excluding it does not help: the half-tangent family has at most $q+1$ lines in any plane, against order $q^2$ for that example.

The same question for the much smaller count $|L|=q^2$ has a longer history, coming from work on Kakeya sets.
\citet{wolff1999recent} showed that a set $L$ of $q^2$ lines in $\F_q^3$, at most $O(q)$ of which lie in any plane, has $|P(L)|=\Omega(q^{5/2})$, and \citet{mockenhaupt2004restriction} showed that this is sharp when $q$ is a square.
Over prime fields the exponent improves: \citet{ellenberg2016incidence} prove that such an $L$ has $|P(L)|\ge cq^3$ for an absolute constant $c>0$, once no plane contains more than $q$ of its lines.
Conjecture~1.4 assumes many more lines, $Cq^3$ rather than $q^2$, allows $\omega(q)$ rather than $O(q)$ of them in a plane, and asks for the constant $1-o(1)$ in place of $c$.
We are aware of no counterexample in the literature respecting its plane hypothesis, and \citet{tao2025new}, who constructs a Nikodym set in $\F_q^d$ of size $q^d-\bigl((d-2)/\log2+1+o(1)\bigr)q^{d-1}\log q$ for each fixed $d\ge3$ and each odd prime power $q$, still records Conjecture~1.4 as the route to the three-dimensional case of the Nikodym conjecture.

The construction below is not new as a piece of finite geometry.
The projective closure of the paraboloid is the quadric $X^2-\nu Y^2-ZW=0$ of $\PG(3,q)$, an elliptic quadric with $q^2+1$ points, and the splitting of the $q+1$ tangent lines at each of its points into two halves of size $(q+1)/2$, according to whether the quadratic form takes square or nonsquare values on the tangent line away from the quadric, is precisely the splitting used by \citet{bruen1999construction} to build the first infinite family of Cameron--Liebler line classes of $\PG(3,q)$, $q$ odd, with parameter $(q^2+1)/2$.
\citet{gavrilyuk2018derivation} write the splitting out explicitly as a partition of the tangent lines into two classes, either of which yields a Cameron--Liebler line class of that parameter when adjoined to the secant lines or to the external lines, and they derive a further such family from it.
\citet{cossidente2017new} construct yet more families with the same parameter for odd $q\ge7$.
What is new here is only the observation that this classical half-tangent family, read in affine coordinates, is a counterexample to the line-union conjecture.

Our counterexample is at the opposite extreme from the Hermitian family of \citet{lund2018finite}: its union misses a positive proportion of $\F_q^3$ rather than a $q^{-1/2}$ proportion, and it exists for every odd prime power, including the prime fields for which no Hermitian variety is available.

\begin{theorem}\label{finkakeya:thm:main}
Let $q$ be an odd prime power.
There is a set $L$ of affine lines in $\F_q^3$ such that
\[
|L|=\frac{q^2(q+1)}{2},
\qquad
\max_{\Pi}\bigl|\{\ell\in L:\ell\subseteq\Pi\}\bigr|\le q+1,
\qquad
|P(L)|=\frac{q^2(q+1)}{2},
\]
where the maximum is taken over all affine planes $\Pi\subseteq\F_q^3$.
\end{theorem}

\begin{corollary}\label{finkakeya:cor:refutation}
Conjecture~1.4 of \citet{lund2018finite} is false, and so is Conjecture~1.5 of \citet{lund2018finite} for every fixed $0<\varepsilon<1/2$.
\end{corollary}

\begin{remark}\label{finkakeya:rem:odd}
The construction requires $q$ odd, since it selects half of the directions by a quadratic character and in characteristic $2$ every element of $\F_q$ is a square.
This is no restriction for the purpose at hand: both conjectures are asymptotic statements as $q\to\infty$ over prime powers, so a counterexample along the odd prime powers refutes them.
\end{remark}

\begin{remark}\label{finkakeya:rem:nikodym}
Theorem~\ref{finkakeya:thm:main} closes the route to the three-dimensional Nikodym conjecture through Conjecture~1.4 of \citet{lund2018finite}, but says nothing about the Nikodym conjecture itself, which remains open.
It shows only that the constant $1$ in the conclusion of Conjecture~1.4 cannot be replaced by anything larger than $1/2$; whether the conclusion survives with some absolute constant $c>0$ in place of $1$ is not addressed here, though \citet{lund2018finite} already give $c=0.38$ once $C\ge0.62$.
\end{remark}

We now briefly summarize the construction.
Fix a nonsquare $\nu\in\F_q^{*}$ and let $Q(x,y)=x^2-\nu y^2$, an anisotropic binary form, and slice $\F_q^3$ into the $q$ level sets of $Q(x,y)-z$.
The key point is that the tangent line to the paraboloid $z=Q(x,y)$ in a direction $[u:v]$ meets only the level sets $\{Q(x,y)-z=s\}$ with $s=0$ or $s$ in the square class of $Q(u,v)$, so keeping only the $(q+1)/2$ directions with $Q(u,v)$ a square confines the whole family to the level sets indexed by $0$ and by the nonzero squares.
The plane bound comes from the fact that a plane meets the paraboloid in at most $q+1$ points and, unless it is a tangent plane, determines the direction of a tangent line at each of them.

\subsubsection{The half-tangent family}

Fix an odd prime power $q$ and a nonsquare $\nu\in\F_q^{*}$, and set
\[
Q(x,y)=x^2-\nu y^2.
\]
The binary form $Q$ is \emph{anisotropic}: if $Q(x,y)=0$ with $(x,y)\ne(0,0)$, then $y\ne0$ and $\nu=(x/y)^2$ would be a square.
Let
\[
S=\{(a,b,Q(a,b)):a,b\in\F_q\}\subseteq\F_q^3
\]
be the associated affine paraboloid, a set of $q^2$ points.

Call a projective direction $[u:v]\in\PG(1,q)$ \emph{square} if $Q(u,v)$ is a nonzero square in $\F_q$, and let $D\subseteq\PG(1,q)$ be the set of square directions.
This is well defined: $Q(u,v)\ne0$ for $(u,v)\ne(0,0)$ by anisotropy, and replacing $(u,v)$ by $\lambda(u,v)$ multiplies $Q(u,v)$ by the square $\lambda^2$.

\begin{lemma}\label{finkakeya:lem:directions}
We have $|D|=(q+1)/2$.
\end{lemma}

\begin{proof}
Since $\nu$ is a nonsquare we have $\F_q(\sqrt\nu)=\F_{q^2}$, and $Q$ is the norm form of this extension:
\[
N(u+v\sqrt\nu)=(u+v\sqrt\nu)(u-v\sqrt\nu)=u^2-\nu v^2=Q(u,v).
\]
The norm map $N:\F_{q^2}^{*}\to\F_q^{*}$ is surjective with kernel of size $q+1$, so every $c\in\F_q^{*}$ has exactly $q+1$ preimages.
Hence exactly $\tfrac{q-1}{2}(q+1)$ pairs $(u,v)\ne(0,0)$ have $Q(u,v)$ a nonzero square.
Each projective direction accounts for exactly $q-1$ such pairs, so $|D|=(q+1)/2$.
\end{proof}

For $(a,b)\in\F_q^2$ and $[u:v]\in\PG(1,q)$ set
\[
\ell_{a,b,[u:v]}=\bigl\{\bigl(a+tu,\ b+tv,\ Q(a,b)+2t(au-\nu bv)\bigr):t\in\F_q\bigr\}.
\]
Replacing $(u,v)$ by $\lambda(u,v)$ only reparametrizes this set, so it depends on $[u:v]$ alone, and it is a line because $(u,v)\ne(0,0)$.
Expanding $Q$ gives the identity
\begin{equation}
\label{finkakeya:eq:tangent}
Q(a+tu,\,b+tv)-\bigl(Q(a,b)+2t(au-\nu bv)\bigr)=t^2Q(u,v),
\end{equation}
so along $\ell_{a,b,[u:v]}$ the function $Q(x,y)-z$ equals $t^2Q(u,v)$ and vanishes to order two at $t=0$.
In other words, $\ell_{a,b,[u:v]}$ is the tangent line to $S$ at $(a,b,Q(a,b))$ in the direction $[u:v]$.
Define the \emph{half-tangent family}
\[
L=\bigl\{\ell_{a,b,[u:v]}:(a,b)\in\F_q^2,\ [u:v]\in D\bigr\}.
\]

\begin{lemma}\label{finkakeya:lem:size}
Every $\ell\in L$ meets $S$ exactly in its point of tangency, and $|L|=q^2(q+1)/2$.
\end{lemma}

\begin{proof}
By \eqref{finkakeya:eq:tangent}, the point of $\ell_{a,b,[u:v]}$ with parameter $t$ lies on $S$ if and only if $t^2Q(u,v)=0$.
Since $Q(u,v)\ne0$, this forces $t=0$.
So $\ell\cap S$ is the single point $(a,b,Q(a,b))$, which is therefore determined by $\ell$, as is the direction $[u:v]$.
Hence $(a,b,[u:v])\mapsto\ell_{a,b,[u:v]}$ is injective on $\F_q^2\times D$, and Lemma~\ref{finkakeya:lem:directions} gives $|L|=q^2\cdot\tfrac{q+1}{2}$.
\end{proof}

Figure~\ref{finkakeya:fig:halftangent} shows the selected directions at a point of $S$, and the resulting union, in the case $q=7$.

\begin{figure}[t]
\centering
\begin{tikzpicture}[scale=1.0]
\begin{scope}
  \draw[dashed,gray] (202.5:1.6) -- (22.5:1.6);
  \draw[dashed,gray] (225:1.6) -- (45:1.6);
  \draw[dashed,gray] (292.5:1.6) -- (112.5:1.6);
  \draw[dashed,gray] (315:1.6) -- (135:1.6);
  \draw[thick] (180:1.6) -- (0:1.6);
  \draw[thick] (247.5:1.6) -- (67.5:1.6);
  \draw[thick] (270:1.6) -- (90:1.6);
  \draw[thick] (337.5:1.6) -- (157.5:1.6);
  \fill (0,0) circle (2.2pt);
  \node[below right] at (0.04,-0.04) {\small $p$};
  \node at (0:2.05) {\scriptsize $[1{:}0]$};
  \node at (22.5:2.05) {\scriptsize $[1{:}1]$};
  \node at (45:2.05) {\scriptsize $[1{:}2]$};
  \node at (67.5:2.05) {\scriptsize $[1{:}3]$};
  \node at (90:1.9) {\scriptsize $[1{:}4]$};
  \node at (112.5:2.05) {\scriptsize $[1{:}5]$};
  \node at (135:2.05) {\scriptsize $[1{:}6]$};
  \node at (157.5:2.05) {\scriptsize $[0{:}1]$};
\end{scope}
\begin{scope}[xshift=6.4cm]
  \draw[gray] (0,-1.75) -- (0,1.75);
  \fill (0,1.5) circle (2.4pt);
  \node[right] at (0.18,1.5) {\scriptsize $s=0$};
  \fill (0,1.0) circle (2.4pt);
  \node[right] at (0.18,1.0) {\scriptsize $s=1$};
  \fill (0,0.5) circle (2.4pt);
  \node[right] at (0.18,0.5) {\scriptsize $s=2$};
  \draw (0,0) circle (2.4pt);
  \node[right] at (0.18,0) {\scriptsize $s=3$};
  \fill (0,-0.5) circle (2.4pt);
  \node[right] at (0.18,-0.5) {\scriptsize $s=4$};
  \draw (0,-1.0) circle (2.4pt);
  \node[right] at (0.18,-1.0) {\scriptsize $s=5$};
  \draw (0,-1.5) circle (2.4pt);
  \node[right] at (0.18,-1.5) {\scriptsize $s=6$};
\end{scope}
\end{tikzpicture}
\caption{The two counts behind Theorem~\ref{finkakeya:thm:main}, drawn for $q=7$ and $\nu=3$.
On the left are the eight tangent lines to $S$ at a point $p$, all lying in the tangent plane at $p$; the four solid ones are those whose direction $[u:v]$ has $Q(u,v)=u^2-3v^2$ a nonzero square, and these are exactly the lines of $L$ through $p$.
On the right are the seven points of $\F_q^3$ lying over a fixed $(x,y)$, labeled by the value $s=Q(x,y)-z$; the four solid ones are those belonging to $P(L)$, namely those with $s\in\{0,1,2,4\}$.}
\label{finkakeya:fig:halftangent}
\end{figure}
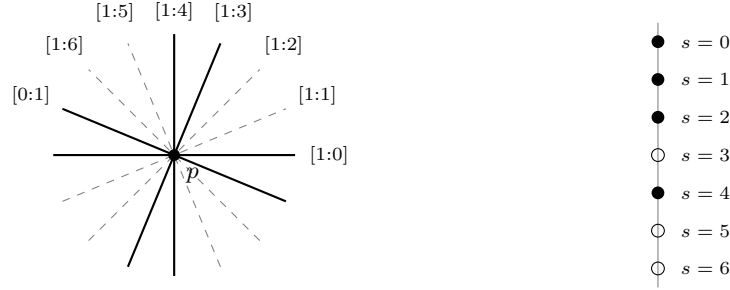

\subsubsection{Lines contained in a plane}

\begin{lemma}\label{finkakeya:lem:plane}
Every affine plane $\Pi\subseteq\F_q^3$ contains at most $q+1$ lines of $L$.
\end{lemma}

\begin{proof}
Write
\[
\Pi=\{(x,y,z)\in\F_q^3:c_1x+c_2y+c_3z=c_0\},
\qquad
(c_1,c_2,c_3)\ne(0,0,0).
\]
If $\ell_{a,b,[u:v]}\subseteq\Pi$, then its point of tangency lies in $\Pi\cap S$ and its direction vector $(u,v,2(au-\nu bv))$ lies in the direction plane of $\Pi$, the latter condition reading
\begin{equation}
\label{finkakeya:eq:direction}
(c_1+2c_3a)u+(c_2-2\nu c_3b)v=0.
\end{equation}

\smallskip
\noindent\textit{The size of $\Pi\cap S$.}
If $c_3=0$, then $\Pi\cap S$ is parametrized by the $q$ solutions $(x,y)$ of $c_1x+c_2y=c_0$, so $|\Pi\cap S|=q$.
If $c_3\ne0$, then substituting $z=Q(x,y)$ into the equation of $\Pi$ and completing the square gives
\begin{equation}
\label{finkakeya:eq:section}
Q\Bigl(x+\frac{c_1}{2c_3},\ y-\frac{c_2}{2\nu c_3}\Bigr)
=\frac{c_0}{c_3}+\frac{c_1^2}{4c_3^2}-\frac{c_2^2}{4\nu c_3^2}.
\end{equation}
If the right-hand side of \eqref{finkakeya:eq:section} is zero, then anisotropy of $Q$ gives exactly one solution, and otherwise the norm count in the proof of Lemma~\ref{finkakeya:lem:directions} gives exactly $q+1$ solutions.
Hence $|\Pi\cap S|\in\{1,q,q+1\}$, and in particular $|\Pi\cap S|\le q+1$.

\smallskip
\noindent\textit{Tangent planes.}
Since $Q(x,y)-z$ has gradient $(2a,-2\nu b,-1)$ at $(a,b,Q(a,b))$, the tangent plane to $S$ at that point is
\[
T_{a,b}:\qquad 2ax-2\nu by-z=Q(a,b).
\]
We claim that, for $(a,b,Q(a,b))\in\Pi\cap S$, the two coefficients in \eqref{finkakeya:eq:direction} both vanish if and only if $\Pi=T_{a,b}$.
Indeed, they vanish exactly when $c_1=-2c_3a$ and $c_2=2\nu c_3b$, which forces $c_3\ne0$; dividing the equation of $\Pi$ by $-c_3$ then turns it into $2ax-2\nu by-z=-c_0/c_3$, and evaluating at $(a,b,Q(a,b))\in\Pi$ gives $-c_0/c_3=Q(a,b)$, so $\Pi=T_{a,b}$.
Conversely, if $\Pi=T_{a,b}$ then $(c_1,c_2,c_3)=\lambda(2a,-2\nu b,-1)$ for some $\lambda\in\F_q^{*}$, whence $c_1=-2c_3a$ and $c_2=2\nu c_3b$.
Comparing with \eqref{finkakeya:eq:section}, whose right-hand side vanishes exactly when the unique point of $\Pi\cap S$ has $(a,b)=(-c_1/(2c_3),\,c_2/(2\nu c_3))$, we conclude that $\Pi$ is a tangent plane of $S$ if and only if $|\Pi\cap S|=1$.

\smallskip
\noindent\textit{The two cases.}
Suppose first that $\Pi$ is not a tangent plane of $S$.
Then for every $(a,b,Q(a,b))\in\Pi\cap S$ the two coefficients in \eqref{finkakeya:eq:direction} are not both zero, so at most one $[u:v]\in\PG(1,q)$ satisfies \eqref{finkakeya:eq:direction}, and therefore at most one line of $L$ tangent at that point is contained in $\Pi$.
Since every line of $L$ contained in $\Pi$ is tangent at some point of $\Pi\cap S$, the plane $\Pi$ contains at most $|\Pi\cap S|\le q+1$ lines of $L$.
Suppose instead that $\Pi=T_{p}$ is the tangent plane at a point $p\in S$.
Then $\Pi\cap S=\{p\}$, so every line of $L$ inside $\Pi$ is tangent at $p$, and there are exactly $|D|=(q+1)/2$ of these.
This exhausts all possibilities, and in every case $\Pi$ contains at most $q+1$ lines of $L$.
\end{proof}

\subsubsection{The union}

\begin{lemma}\label{finkakeya:lem:union}
We have
\[
P(L)=\bigl\{(x,y,z)\in\F_q^3:\ Q(x,y)-z\in\{0\}\cup\sq\bigr\},
\]
and consequently $|P(L)|=q^2(q+1)/2$.
\end{lemma}

\begin{proof}
Write $s=Q(x,y)-z$ for the value of the defining function at a point $(x,y,z)$.
By \eqref{finkakeya:eq:tangent}, at the point of $\ell_{a,b,[u:v]}$ with parameter $t$ we have $s=t^2Q(u,v)$.
For $[u:v]\in D$ the value $Q(u,v)$ is a nonzero square, so $s\in\{0\}\cup\sq$, which gives one inclusion.

Conversely, suppose $s=Q(x,y)-z\in\{0\}\cup\sq$, and fix any $[u:v]\in D$, which exists by Lemma~\ref{finkakeya:lem:directions}.
If $s=0$, then $(x,y,z)\in S$ and $(x,y,z)$ is the point of $\ell_{x,y,[u:v]}$ with $t=0$.
If $s\ne0$, then $s/Q(u,v)$ is a nonzero square, say $s/Q(u,v)=\lambda^2$ with $\lambda\in\F_q^{*}$, and we set
\[
a=x-\lambda u,
\qquad
b=y-\lambda v.
\]
At parameter $t=\lambda$ the line $\ell_{a,b,[u:v]}$ has first two coordinates $(a+\lambda u,b+\lambda v)=(x,y)$, and by \eqref{finkakeya:eq:tangent} its third coordinate is
\[
Q(a,b)+2\lambda(au-\nu bv)=Q(x,y)-\lambda^2Q(u,v)=Q(x,y)-s=z.
\]
Hence $(x,y,z)\in\ell_{a,b,[u:v]}\subseteq P(L)$, which gives the other inclusion.

Finally, for each fixed $(x,y)\in\F_q^2$ the map $z\mapsto Q(x,y)-z$ is a bijection of $\F_q$, and
\[
\bigl|\{0\}\cup\sq\bigr|=1+\frac{q-1}{2}=\frac{q+1}{2}.
\]
Therefore $|P(L)|=q^2\cdot\tfrac{q+1}{2}$, as desired.
\end{proof}

\subsubsection{Proof of the main theorem}

\begin{proof}[Proof of Theorem~\ref{finkakeya:thm:main}]
Let $L$ be the half-tangent family.
Lemma~\ref{finkakeya:lem:size} gives $|L|=q^2(q+1)/2$, Lemma~\ref{finkakeya:lem:plane} gives $\max_\Pi|\{\ell\in L:\ell\subseteq\Pi\}|\le q+1$, and Lemma~\ref{finkakeya:lem:union} gives $|P(L)|=q^2(q+1)/2$.
\end{proof}

\begin{proof}[Proof of Corollary~\ref{finkakeya:cor:refutation}]
Let $q$ range over the odd prime powers and let $L=L_q$ be the family of Theorem~\ref{finkakeya:thm:main}, so that
\[
|L_q|=\frac{q^2(q+1)}{2}\ge\frac{q^3}{2},
\qquad
|P(L_q)|=\frac{q^2(q+1)}{2}=\Bigl(\frac12+o(1)\Bigr)q^3.
\]
For Conjecture~1.4, take $C=1/2$.
Given any function $\alpha$ with $\alpha(q)/q\to\infty$ we have $\alpha(q)>q+1$ for all large $q$, so by Theorem~\ref{finkakeya:thm:main} no plane contains $\alpha(q)$ lines of $L_q$.
The hypotheses therefore hold along the odd prime powers, while the conclusion $|P(L_q)|\ge(1-o(1))q^3$ fails.

For Conjecture~1.5, fix $0<\varepsilon<1/2$.
Then $|L_q|\ge q^3/2\ge q^{5/2+\varepsilon}$ as soon as $q^{1/2-\varepsilon}\ge2$, and $q+1\le\tfrac12q^{3/2}$ for every $q\ge7$, so the hypotheses hold for all sufficiently large odd $q$.
But
\[
q^3-|P(L_q)|=\frac{q^3}{2}-\frac{q^2}{2},
\]
which is not $O(q^{5/2})$, so the conclusion fails.
\end{proof}

\putbib

\end{bibunit}
\endgroup
\subsection[(Naserasr--Wang--Zhu) High-girth graphs attaining \texorpdfstring{$\chi^s_c(G)=2\chi(G)$}{chi\textasciicircum s\_c(G)=2chi(G)}]{High-girth graphs attaining \texorpdfstring{$\chi^s_c(G)=2\chi(G)$}{chi\textasciicircum s\_c(G)=2chi(G)}}
\label{sol:signed_circular}
\begingroup
\begin{bibunit}
\def\cs{\chi^{s}_{c}}
\def\cc{\chi_{c}}

For every graph $G$ the signed circular chromatic number satisfies $\cc(G)\le\cs(G)\le 2\cc(G)$, and \citet{naserasr2020circular} showed that the upper bound is approached by $k$-chromatic graphs of arbitrarily large girth.
Whether it is reached by a finite such graph was left open.
We show that it is: for all integers $k,g\ge2$ there is a finite simple graph of chromatic number $k$ and girth at least $g$ whose signed circular chromatic number is exactly $2k$.
The witness is a sparse random $k$-partite signed graph with its short cycles deleted, and the point of the proof is a union bound that rules out every admissible ratio $p/q<2k$ at once.

\subsubsection{Introduction}

A \emph{signed graph} $(G,\sigma)$ is a graph $G$ together with a \emph{signature} $\sigma:E(G)\to\{+,-\}$.
Following \citet[Definition~6]{naserasr2020circular}, for an even integer $p$ and an integer $q$ with $1\le q\le p/2$, a \emph{$(p,q)$-coloring} of $(G,\sigma)$ is a map $f:V(G)\to\mathbb{Z}_p$ such that
\[
  d_p\bigl(f(u),f(v)\bigr)\ge q \quad\text{for every positive edge } uv,
\]
\[
  d_p\bigl(f(u),f(v)+\tfrac p2\bigr)\ge q \quad\text{for every negative edge } uv,
\]
where $d_p(a,b)=\min\{|a-b|,\,p-|a-b|\}$ is the distance in the cycle $\mathbb{Z}_p$.
The \emph{circular chromatic number} of $(G,\sigma)$ is $\cc(G,\sigma)=\inf\{p/q:(G,\sigma)\text{ has a }(p,q)\text{-coloring}\}$, and the \emph{signed circular chromatic number} of a graph $G$ is
\[
  \cs(G)=\max\bigl\{\cc(G,\sigma):\sigma\text{ a signature of }G\bigr\}.
\]
For the all-positive signature one recovers the ordinary circular chromatic number surveyed by \citet{zhu2001circular}, so $\cc(G)\le\cs(G)$, and \citet{naserasr2020circular} give the matching upper bound $\cs(G)\le2\cc(G)$.
Switching at a vertex set $A$, that is, reversing the signs of the edges of the cut $(A,V(G)\setminus A)$, in the sense of \citet{zaslavsky1982signed}, does not change $\cc(G,\sigma)$, since it amounts to replacing $f(v)$ by $f(v)+p/2$ on $A$; quantifying over \emph{all} maps $f:V(G)\to\mathbb{Z}_p$ therefore already accounts for switching.
Throughout, the girth of a signed graph is the girth of its underlying graph, and all graphs are simple, so girth at least $3$ is automatic and digons do not arise.

Naserasr, Wang and Zhu proved that the bound $\cs(G)\le2\cc(G)$ remains tight when the girth is prescribed: for all integers $k,g\ge2$ and every $\varepsilon>0$ there is a graph $G$ of girth at least $g$ with $\chi(G)=k$ and $\cs(G)>2k-\varepsilon$ \citep[Theorem~30]{naserasr2020circular}.
For each integer $p$ they produce a graph on the vertex set of a bipartite augmented tree of \citet{alon2016coloring}, with one new edge for each leaf joining two of that leaf's ancestors, carrying a signature that admits no $(2kp,p+1)$-coloring.
The resulting lower bounds $2kp/(p+1)$ increase to $2k$ but never equal it.
In the remark immediately following that proof they ask whether the supremum is attained:

\medskip
\noindent
\emph{It is not known whether there is a finite $k$-chromatic graph of girth at least $g$ and with $\cs(G)=2k$.}
\medskip

\noindent
We answer this affirmatively, for every pair $k,g\ge2$, by a random construction.

Attainment questions of this shape go both ways in this subject.
On the one hand, \citet{naserasr2020circular} show that the supremum of $\cc(G,\sigma)$ over signed $d$-degenerate simple graphs equals $2\lfloor d/2\rfloor+2$, and that it is attained by $(K_{d+1},+)$ for odd $d$ and by an explicit signed graph $\Omega_d$ for even $d\ge4$; for $d=2$ they produce only a sequence of signed graphs whose circular chromatic numbers tend to $4$.
\citet{kardovs2023circular} then showed that in the case $d=2$ the supremum is genuinely never attained, by proving that every signed $2$-degenerate simple graph on $n$ vertices has circular chromatic number at most $4-2/\lfloor (n+1)/2\rfloor$, and that this bound is tight for every $n\ge2$.
On the other hand \citet{naserasr2020circular} attain the supremum $10/3$ for signed series-parallel simple graphs with an explicit signed outerplanar graph, and \citet{pan2022circular} went on to show that every rational in $[2,10/3]$ occurs; \citet{zhu2023circular} carry the same analysis out for signed series-parallel graphs in which every cycle with an odd number of positive edges is long, a signed analogue of odd girth rather than the girth of the underlying graph.
The girth family of \citet{naserasr2020circular} sat on neither side of this divide.

Coloring of signed graphs goes back to \citet{zaslavsky1982coloring}, whose $0$-free $2k$-colorings are exactly the circular $2k$-colorings \citep{naserasr2020circular}, and to \citet{mavcajova2014chromatic}, who proposed a chromatic number for signed graphs and conjectured that every signed planar simple graph is $4$-colorable, equivalently that $\cs(G)\le4$ for every planar $G$ \citep{naserasr2020circular}; \citet{kardovs20214} refuted this, and \citet{naserasr2020circular} exhibit a signed planar simple graph with $\cc=4+\tfrac23$.
A different circular refinement of signed graph coloring was introduced earlier by \citet{kang2018circular}; the two notions are compared by \citet{naserasr2020circular}, and the area as a whole is surveyed in \citet{wang2022circular}.
We are not aware of any work that settles the attainment question above for every pair $k,g$.
Our construction adapts the deletion argument of \citet{erdos1959graph} to a random $k$-partite signed graph.

\subsubsection{Statement and small cases}

One half of the problem is immediate.

\begin{lemma}\label{signed_circular:lem:upper}
If $\chi(G)\le k$ then every signature $\sigma$ of $G$ admits a $(2k,1)$-coloring, and hence $\cs(G)\le 2k$.
\end{lemma}

\begin{proof}
Let $V_1,\dots,V_k$ be the color classes of a proper $k$-coloring of $G$ and put $f(v)=i-1\in\mathbb{Z}_{2k}$ for $v\in V_i$.
Let $uv\in E(G)$.
Then $f(u)\ne f(v)$, so $d_{2k}(f(u),f(v))\ge1$; and $f(v)+k\in\{k,\dots,2k-1\}$ is distinct from $f(u)\in\{0,\dots,k-1\}$, so $d_{2k}(f(u),f(v)+k)\ge1$.
Both edge conditions hold regardless of $\sigma$, so $\cc(G,\sigma)\le 2k$ for every $\sigma$.
\end{proof}

For small $k$ and $g$ the remaining half can already be settled by a finite search.

\begin{example}\label{signed_circular:ex:k34}
Take $k=2$ and $g=4$, and let $K_{3,4}$ have parts $\{u_1,u_2,u_3\}$ and $\{w_1,w_2,w_3,w_4\}$.
The signature whose negative edges are exactly $u_2w_2$, $u_2w_4$, $u_3w_2$ and $u_3w_3$ has circular chromatic number $4$, so $\cs(K_{3,4})=4=2\chi(K_{3,4})$.
Such a value is certified by an exhaustive check over the finitely many pairs $(p,q)$ left admissible by \citet[Corollary~23]{naserasr2020circular}, which states that $\cc(G,\sigma)=p/q$ for some even $p\le 2|V(G)|$ and that the infimum in its definition is a minimum.
No graph on fewer than seven vertices has $\chi=2$ and $\cs=4$.
Indeed, if $H\subseteq G$ then every signature of $H$ extends to $G$ and every $(p,q)$-coloring of the extension restricts to $H$, so $\cs(H)\le\cs(G)$; every bipartite graph on at most six vertices is a subgraph of $K_{3,3}$, of $K_{2,4}$ or of $K_{1,5}$; and $\cs(K_{3,3})=3$, $\cs(K_{2,4})=8/3$ and $\cs(K_{1,5})=2$.
The value $\cs(K_{3,4})=4$ is also due to \citet{gujgiczer2023winding}, whose signed graph $\widehat{BQ}(2,3)$ is $K_{3,4}$ with a maximum matching positive and the remaining nine edges negative; naming that matching $u_1w_1$, $u_2w_3$, $u_3w_4$ and switching at $\{u_1,w_1\}$ turns it into the signature above.
At $k=3$ and $g=3$ the same happens: if $K_{3,3,4}$ has parts $\{x_1,x_2,x_3\}$, $\{y_1,y_2,y_3\}$ and $\{z_1,z_2,z_3,z_4\}$, then the signature whose negative edges are exactly
\[
  x_1y_1,\ x_1y_3,\ x_3y_3,\ x_1z_4,\ x_2z_1,\ x_2z_4,\ y_1z_2,\ y_1z_4,\ y_2z_2
\]
has circular chromatic number $6$, so $\cs(K_{3,3,4})=6=2\chi(K_{3,3,4})$.
\end{example}

\begin{theorem}\label{signed_circular:thm:main}
For all integers $k,g\ge2$ there is a finite simple graph $G$ with $\chi(G)=k$, girth at least $g$, and $\cs(G)=2k$.
\end{theorem}

\begin{remark}\label{signed_circular:rem:scope}
The proof is a probabilistic existence argument and exhibits no explicit graph.
All three estimates it uses are effective, so a bound on the least admissible $|V(G)|$ in terms of $k$ and $g$ could be extracted.
Beyond the witnesses of Example~\ref{signed_circular:ex:k34}, which have girth $3$ and $4$, we know of no explicit witness of girth at least $5$.
The same remark of \citet{naserasr2020circular} asks a second question, namely whether for every rational $p/q$, every integer $g$ and every $\varepsilon>0$ there is a graph $G$ of girth at least $g$ with $\cc(G)\le p/q$ and $\cs(G)>2p/q-\varepsilon$.
The construction below is tied to an integer $k$ through the $k$-partition and through Lemma~\ref{signed_circular:lem:upper}, and does not address that question.
\end{remark}

By Lemma~\ref{signed_circular:lem:upper} the whole content of Theorem~\ref{signed_circular:thm:main} is the lower bound: one must exhibit a \emph{single} signature $\sigma$ on a $k$-chromatic graph of girth at least $g$ with $\cc(G,\sigma)\ge2k$, that is, with no $(p,q)$-coloring for any admissible $(p,q)$ with $p/q<2k$.
Two features of the problem make this delicate.
First, a $k$-critical graph satisfies $\cs(G)\le 2k-2$ \citep{naserasr2020circular}, so a witness must be far from critical; the construction below in fact produces graphs in which every single edge can be deleted without lowering the chromatic number.
Second, by \citet[Corollary~23]{naserasr2020circular} the value $\cc(G,\sigma)$ is a ratio $p/q$ with $p$ even and $p\le 2|V(G)|$, so the number of ratios below $2k$ that have to be excluded grows with the graph.
Hence an argument that excludes one fixed target ratio per construction cannot suffice, and the union bound has to range over all admissible $(p,q)$ at once.

The proof splits into two independent pieces.
The first is a deterministic statement about colors: if $p/q<2k$ then among any $k$ colors in $\mathbb{Z}_p$ there are two that one of the two signs forbids, so every $k$-tuple of colors carries a blocking pair.
The second is the usual Erd\H{o}s deletion argument, applied to a $k$-partite random signed graph sparse enough that short cycles are rare but dense enough that the failure probability for a fixed coloring beats the number of colorings.

\subsubsection{Neutral color pairs}

Fix an even integer $p$ and an integer $q$ with $1\le q\le p/2$.
Call an unordered pair $\{a,b\}\subseteq\mathbb{Z}_p$ \emph{neutral} if it obstructs neither sign, that is, if $d_p(a,b)\ge q$ and $d_p(a,b+\frac p2)\ge q$.
Since $d_p(a,b+\frac p2)=\frac p2-d_p(a,b)$, neutrality says exactly that
\begin{equation}\label{signed_circular:eq:neutral}
  q\ \le\ d_p(a,b)\ \le\ \tfrac p2-q .
\end{equation}
A pair that is not neutral is \emph{blocking}: at least one of the two signs placed on it violates the $(p,q)$-condition.
Note that $\{a,a\}$ is blocking, because $d_p(a,a)=0<q$.

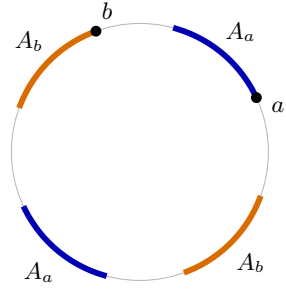
\begin{figure}[t]
\centering
\begin{tikzpicture}[scale=1.7]
  \draw[gray!55] (0,0) circle (1);
  \draw[line width=2.2pt,blue!70!black] (25:1) arc (25:75:1);
  \draw[line width=2.2pt,blue!70!black] (205:1) arc (205:255:1);
  \draw[line width=2.2pt,orange!85!black] (110:1) arc (110:160:1);
  \draw[line width=2.2pt,orange!85!black] (290:1) arc (290:340:1);
  \fill (25:1) circle (1.2pt);
  \fill (110:1) circle (1.2pt);
  \node[font=\small] at (18:1.13) {$a$};
  \node[font=\small] at (103:1.13) {$b$};
  \node[font=\small] at (50:1.22) {$A_a$};
  \node[font=\small] at (230:1.22) {$A_a$};
  \node[font=\small] at (135:1.22) {$A_b$};
  \node[font=\small] at (315:1.22) {$A_b$};
\end{tikzpicture}
\caption{The two arcs making up $A_a$ and the two arcs making up $A_b$, drawn on the cycle $\mathbb{Z}_p$.
Each arc has length $q$, so $|A_a|=|A_b|=2q$; as in the proof of Lemma~\ref{signed_circular:lem:neutral}, the pair $\{a,b\}$ is neutral precisely when the four arcs are pairwise disjoint, which is the situation drawn.}
\label{signed_circular:fig:arcs}
\end{figure}

\begin{lemma}\label{signed_circular:lem:neutral}
The graph on vertex set $\mathbb{Z}_p$ whose edges are the neutral pairs has clique number at most $\lfloor p/(2q)\rfloor$.
In particular, if $p/q<2k$ then it contains no clique of size $k$.
\end{lemma}

\begin{proof}
For $a\in\mathbb{Z}_p$ let $A_a=[a,a+q)\cup[a+\frac p2,a+\frac p2+q)$, a union of two arcs of $\mathbb{Z}_p$ of length $q$ each, as in Figure~\ref{signed_circular:fig:arcs}.
The two arcs are disjoint because $q\le p/2$, so $|A_a|=2q$.
Suppose $\{a,b\}$ is neutral.
The arcs $[a,a+q)$ and $[b,b+q)$ meet only if $d_p(a,b)<q$, which \eqref{signed_circular:eq:neutral} forbids.
The arcs $[a,a+q)$ and $[b+\frac p2,b+\frac p2+q)$ meet only if $d_p(a,b+\frac p2)<q$, likewise forbidden, and the two remaining pairs of arcs give back these same two conditions.
Hence $A_a\cap A_b=\varnothing$.
If $a_1,\dots,a_t$ is a neutral clique, the sets $A_{a_1},\dots,A_{a_t}$ are therefore pairwise disjoint subsets of $\mathbb{Z}_p$, so $2tq\le p$ and $t\le\lfloor p/(2q)\rfloor$.
Finally $p/q<2k$ gives $\lfloor p/(2q)\rfloor\le k-1$.
\end{proof}

Lemma~\ref{signed_circular:lem:neutral} is the only place where the hypothesis $p/q<2k$ is used, and it is what forces every $k$-tuple of colors to contain a blocking pair.

\subsubsection{The random construction}

Fix $k\ge2$ and $g\ge2$, and fix a real number $\alpha$ with
\[
  0<\alpha<1,\qquad \alpha(g-2)<1 .
\]
For a large integer $m$ set $n=km$ and $\rho=m^{-1+\alpha}\in(0,1]$.
Let $V=V_1\cup\dots\cup V_k$ with $|V_i|=m$, so $|V|=n$, and call a pair of vertices lying in distinct parts a \emph{cross pair}.
Independently for each cross pair $\{u,v\}$, place
\begin{gather*}
\text{a positive edge with probability }\tfrac\rho2,
\qquad
\text{a negative edge with probability }\tfrac\rho2,\\
\text{no edge with probability }1-\rho .
\end{gather*}
No edge is placed inside a part.
This yields a random signed simple graph $(G_0,\sigma_0)$ on $V$ whose underlying graph is $k$-partite.

\begin{lemma}\label{signed_circular:lem:count}
Let $p$ be even, let $1\le q\le p/2$ with $p/q<2k$, and let $f:V\to\mathbb{Z}_p$ be arbitrary.
Then at least $m^2$ cross pairs $\{u,v\}$ have $\{f(u),f(v)\}$ blocking.
Likewise, for every map $c:V\to[k-1]$ at least $m^2$ cross pairs are monochromatic under $c$.
\end{lemma}

\begin{proof}
Call a tuple $(v_1,\dots,v_k)\in V_1\times\dots\times V_k$ a transversal; there are $m^k$ of them.
If all $\binom k2$ pairs of a transversal were neutral, then the colors $f(v_1),\dots,f(v_k)$ would be pairwise distinct, since a repeated color gives a blocking pair, and they would form a neutral clique of size $k$, contradicting Lemma~\ref{signed_circular:lem:neutral}.
So every transversal contains a blocking cross pair.
Each cross pair lies in exactly $m^{k-2}$ transversals, obtained by choosing one vertex from each of the remaining $k-2$ parts.
Hence the number of blocking cross pairs is at least $m^k/m^{k-2}=m^2$.
The second statement is the same count, with the pigeonhole principle in place of Lemma~\ref{signed_circular:lem:neutral}.
\end{proof}

\subsubsection{Proof of the main theorem}

\begin{proof}[Proof of Theorem~\ref{signed_circular:thm:main}]
We consider three events for $(G_0,\sigma_0)$.

\smallskip
\noindent\textit{No cheap circular coloring.}
Fix an even $p$ and an integer $1\le q\le p/2$ with $p/q<2k$, and fix $f:V\to\mathbb{Z}_p$.
Let $X_f$ be the number of edges of $(G_0,\sigma_0)$ violating $f$.
By Lemma~\ref{signed_circular:lem:count} we may fix a set of exactly $m^2$ blocking cross pairs; for each of them one of the two signs violates $f$, and that particular signed edge is present with probability $\rho/2$, independently over pairs.
Hence $X_f$ stochastically dominates a $\mathrm{Bin}(m^2,\rho/2)$ variable, whose mean is $\tfrac12 m^{1+\alpha}$, and the Chernoff bound gives
\[
  \mathbb{P}\Bigl[X_f<\tfrac14 m^{1+\alpha}\Bigr]\ \le\ \exp\bigl(-\tfrac1{16}m^{1+\alpha}\bigr).
\]
By \citet[Corollary~23]{naserasr2020circular} the circular chromatic number of any signed graph on $n$ vertices equals $p/q$ for some even $p\le 2n$, and this applies to every spanning subgraph of $G_0$, so it suffices to range over even $p\le 2n$.
The number of triples $(p,q,f)$ with $p$ even, $p\le 2n$, $1\le q\le p/2$ and $f:V\to\mathbb{Z}_p$ is at most $n\cdot n\cdot(2n)^{n}=\exp(O(m\log m))$, and $m^{1+\alpha}\gg m\log m$, so with probability $1-o(1)$ every such triple with $p/q<2k$ satisfies $X_f\ge\frac14m^{1+\alpha}$.

\smallskip
\noindent\textit{No cheap $(k-1)$-coloring.}
For a map $c:V\to[k-1]$ let $Y_c$ be the number of $c$-monochromatic edges of $G_0$.
By Lemma~\ref{signed_circular:lem:count} and the same argument, $Y_c$ stochastically dominates $\mathrm{Bin}(m^2,\rho)$, whose mean is $m^{1+\alpha}$, so $\mathbb{P}[Y_c<\frac14m^{1+\alpha}]\le\exp(-\frac18m^{1+\alpha})$.
There are $(k-1)^{n}=\exp(O(m))$ such maps, so with probability $1-o(1)$ every $c$ has $Y_c\ge\frac14m^{1+\alpha}$.

\smallskip
\noindent\textit{Few short cycles.}
Let $Z$ be the number of cycles of $G_0$ of length less than $g$.
Since a cycle of length $\ell$ is present only if all $\ell$ of its pairs receive edges,
\[
  \E[Z]\ \le\ \sum_{\ell=3}^{g-1}\frac{n^{\ell}\rho^{\ell}}{2\ell}
  \ \le\ \sum_{\ell=3}^{g-1}(km^{\alpha})^{\ell}
  \ =\ O\bigl(m^{\alpha(g-1)}\bigr),
\]
where the implicit constant depends on $k$ and $g$ only.
Now $\alpha(g-1)<1+\alpha$ precisely because $\alpha(g-2)<1$, so $\E[Z]=o(m^{1+\alpha})$, and Markov's inequality gives $Z<\frac18m^{1+\alpha}$ with probability $1-o(1)$.

\medskip
For $m$ large all three events hold simultaneously, and we fix such an outcome.
Delete one edge from each cycle of $G_0$ of length less than $g$, obtaining a signed graph $(G,\sigma)$ on the same vertex set $V$ with fewer than $\frac18m^{1+\alpha}$ edges removed.
A cycle of $G$ of length less than $g$ would be a cycle of $G_0$ of length less than $g$ all of whose edges survived the deletion, and there is none, so $G$ has girth at least $g$; moreover $G$ is simple and $k$-partite.

Let $p$ be even with $p\le 2n$, let $1\le q\le p/2$ satisfy $p/q<2k$, and let $f:V\to\mathbb{Z}_p$ be arbitrary.
Then $f$ still violates at least $\frac14m^{1+\alpha}-\frac18m^{1+\alpha}>0$ edges of $(G,\sigma)$, so $(G,\sigma)$ has no $(p,q)$-coloring for any such pair.
Since $|V(G)|=n$, \citet[Corollary~23]{naserasr2020circular} gives $\cc(G,\sigma)=p_0/q_0$ for a pair $(p_0,q_0)$ with $p_0$ even, $p_0\le2n$ and $1\le q_0\le p_0/2$ for which $(G,\sigma)$ has a $(p_0,q_0)$-coloring, and the previous sentence forces $p_0/q_0\ge 2k$, so $\cc(G,\sigma)\ge2k$.
Similarly every $c:V\to[k-1]$ leaves a monochromatic edge, so $\chi(G)\ge k$, while $\chi(G)\le k$ because $G$ is $k$-partite.
Finally Lemma~\ref{signed_circular:lem:upper} gives $\cs(G)\le2k$, whence
\[
  2k\ \le\ \cc(G,\sigma)\ \le\ \cs(G)\ \le\ 2k . \qedhere
\]
\end{proof}

\begin{remark}\label{signed_circular:rem:noncritical}
The graph produced above is not merely non-critical.
Indeed, every map $c:V\to[k-1]$ leaves at least $\tfrac18m^{1+\alpha}>1$ monochromatic edges of $G$, so $\chi(G-e)=k$ for every edge $e$ of $G$.
\end{remark}

\putbib

\end{bibunit}
\endgroup
\subsection[(R\"odl--Siggers) Counting 4-critical linear triple systems]{Counting 4-critical linear triple systems}
\label{sol:count_critical}
\begingroup
\begin{bibunit}
\def\cA{\mathcal{A}}

R\"odl and Siggers conjectured that the number of $k$-critical $(r,l)$-systems on $n$ vertices is exponential in $n^{l}$.
We disprove this for $(k,r,l)=(4,3,2)$: there is a constant $c>0$ such that for infinitely many $n$ there are at least $\exp(cn^{2}\log n)$ pairwise nonisomorphic $4$-critical linear triple systems on $n$ vertices.
The construction lifts a dense $4$-critical graph to a triple system along a proper edge coloring of that graph, in such a way that the edge coloring can be read off from any minimal non-$3$-colorable subsystem of the lift.
The extra $\log n$ in the exponent is exactly the entropy of the edge coloring, and it survives the passage from labeled systems to isomorphism classes.

\subsubsection{Introduction}

A hypergraph $H$ is \emph{$k$-colorable} if its vertices can be colored with $k$ colors so that no edge is monochromatic, and \emph{$k$-chromatic} if $k$ is the least such number.
Following \citet{rodl2006color}, $H$ is \emph{$k$-critical} if it is $k$-chromatic, has no isolated vertices, and $H-e$ is $(k-1)$-colorable for every edge $e$ of $H$.
An \emph{$(r,l)$-system} is an $r$-uniform hypergraph in which no $l$-set of vertices lies in more than one edge; a $(3,2)$-system is a \emph{linear triple system}.
Write $T(k,r,l,n)$ for the number of nonisomorphic $k$-critical $(r,l)$-systems on $n$ vertices.

The first bound on this count is due to \citet{abbott1980enumeration}, who showed that for all $k,r\ge3$ there is a constant $b=b(k,r)>1$ with $T(k,r,2,n)>b^{\,n}$ for all large $n$.
\citet{rodl2006color} proved that for all $k\ge3$ and $r>l\ge2$ and all large $n$ there is a $k$-critical $(r,l)$-system on $n$ vertices with at least $c_0\,n^{l}$ edges, which is optimal up to the constant since an $(r,l)$-system on $n$ vertices has at most $\binom{n}{l}\big/\binom{r}{l}$ edges.
Feeding that density into a construction indexed by the bipartitions of the edge set, they improved the Abbott--Liu--Toft bound to
\[
T(k,r,l,n)>\alpha^{\,n^{l}}
\]
for all $k\ge3$, $r>l\ge2$ and all large $n$, with a constant $\alpha=\alpha(k,r,l)>1$ \citep[Theorem~6.1]{rodl2006color}.
Every system they produce has at least $c'n^{l}$ edges for a constant $c'>0$, and Section~6 of \citet{rodl2006color} closes with a trivial upper bound and a conjecture:

\medskip
\noindent
\emph{A trivial upper bound for the number of $(r,l)$-systems with $c'n^{l}$ edges is $\binom{\binom{n}{r}}{c'n^{l}}<\binom{n^{r}}{c'n^{l}}\approx (n^{r-l})^{n^{l}}=O(d^{\,n^{l}\log n^{r-l}})$. We conjecture that the actual number is in fact exponential in $n^{l}$.}
\medskip

\noindent
We read ``the actual number'' as $T(k,r,l,n)$, which the preceding theorem bounds from below by $\alpha^{\,n^{l}}$; the content of the conjecture is then the matching upper bound $T(k,r,l,n)\le C^{n^{l}}$ for some constant $C=C(k,r,l)$, that is, that the factor $\log n^{r-l}$ in the trivial bound is an artifact of the counting.
For graphs the corresponding assertion is immediate, since there are only $2^{\binom{n}{2}}$ graphs on $n$ labeled vertices.
Once $r>l$, the trivial count of $(r,l)$-systems with $\Theta(n^{l})$ edges carries a logarithm in the exponent, and the conjecture asserts that criticality removes it.
We show that criticality does not remove it, already for linear triple systems and $k=4$.

On the competing reading, in which ``the actual number'' counts all $(r,l)$-systems on $n$ vertices with $\Omega(n^{l})$ edges, criticality plays no role and the logarithm is present for a direct counting reason, so it is the reading above that is at issue.

Beyond the lower bound $\alpha^{\,n^{l}}$ of \citet{rodl2006color} we are aware of no further work on the growth of $T(k,r,l,n)$, and the problem is not recorded in the survey of color-critical hypergraphs of \citet{kostochka2006color}, which concerns critical systems with few edges rather than their number.

\begin{theorem}\label{count_critical:thm:main}
There is a constant $c>0$ and there are infinitely many integers $n$ such that
\[
T(4,3,2,n)\ge \exp\bigl(c\,n^{2}\log n\bigr).
\]
In particular there is no constant $C$ for which $T(4,3,2,n)\le C^{n^{2}}$ for all $n$.
\end{theorem}

We now summarize the construction.
Fix a $4$-critical graph $G$ on $m$ vertices with $\Omega(m^{2})$ edges, which exists by a theorem of Toft, and fix a palette $[t]$ with $t=4m$.
Every proper edge coloring $\beta\colon E(G)\to[t]$ is turned into a linear triple system $J_\beta$ on a vertex set that does not depend on $\beta$: take three copies $v^{1},v^{2},v^{3}$ of each vertex $v$ of $G$ together with one \emph{palette vertex} $w_j$ for each $j\in[t]$, and place over each edge $uv$ of $G$ the three triples $\{u^{i},v^{i},w_{\beta(uv)}\}$.
Two forcing gadgets of \citet{rodl2006color} are then attached: one chains the palette vertices so that they all receive a common color, and one ties $v^{1},v^{2},v^{3}$ together so that they receive three distinct colors.
A proper $3$-coloring of $J_\beta$ would therefore select, for each $v$, the unique copy $v^{i}$ carrying the palette color, and $v\mapsto i$ would be a proper $3$-coloring of $G$.
The key point is that a minimal non-$3$-colorable subsystem $K_\beta$ of $J_\beta$ must retain a triple over every edge of $G$, since $G$ is $4$-critical; that triple exhibits the value $\beta(uv)$, so $\beta\mapsto K_\beta$ is injective on labeled systems.
There are $\exp(\Omega(m^{2}\log m))$ choices of $\beta$, while $J_\beta$ has only $O(m)$ vertices, so passing to isomorphism classes costs a factor $\exp(O(m\log m))$ and the bound survives.

\subsubsection{The construction}

We use two known ingredients.
The first is a theorem of \citet{toft1970maximal}: there is a constant $a>0$ such that for every sufficiently large $m$ there is a $4$-critical graph on $m$ vertices with at least $a\,m^{2}$ edges.
As recorded by \citet{rodl2006color}, one may take $a=1/16$ here.

The second is the pair of forcing gadgets of \citet{rodl2006color}.
We recall only the properties used below.

\begin{lemma}[Forcing gadgets]\label{count_critical:lem:gadgets}
There are fixed finite linear triple systems $S$ and $D$ with the following properties.
\begin{enumerate}
\item $S$ has two designated vertices $x,y$, and no edge of $S$ contains both of them.
In every proper $3$-coloring of $S$ the vertices $x,y$ receive the same color, and conversely every assignment of a common color to $x$ and $y$ extends to a proper $3$-coloring of $S$.
\item $D$ has three designated vertices $u_1,u_2,u_3$.
In every proper $3$-coloring of $D$ these three vertices receive three distinct colors, and conversely every assignment of three distinct colors to $u_1,u_2,u_3$ extends to a proper $3$-coloring of $D$.
\end{enumerate}
\end{lemma}

\begin{proof}
Take $S=S(3,3)$ and $D=D(3,3)$ from \citet[Construction~3.1]{rodl2006color}.
Both are $3$-chromatic $(3,2)$-systems, so each has at least one proper $3$-coloring, and the two forward implications are exactly the forcing properties for which those gadgets are built.
That no edge of $S$ contains both designated vertices is noted in the proof of that construction: $S(3,3)$ is obtained by deleting edges from the system left when one edge through $x$ and $y$ is removed from a $4$-critical linear triple system, and such a system exists by \citet[Theorem~1]{abbott1978existence}.

For the converse statements, fix a proper $3$-coloring $\psi_0$ of $S$.
Then $\psi_0(x)=\psi_0(y)$, and given any target color $\gamma$ we may compose $\psi_0$ with a permutation of $[3]$ carrying $\psi_0(x)$ to $\gamma$.
Likewise fix a proper $3$-coloring $\psi_0$ of $D$; then $\psi_0(u_1),\psi_0(u_2),\psi_0(u_3)$ are distinct, and given distinct target colors $\gamma_1,\gamma_2,\gamma_3$ the assignment $\psi_0(u_i)\mapsto\gamma_i$ is a well-defined permutation of $[3]$, with which we compose $\psi_0$.
\end{proof}

Fix from now on an integer $m$ large enough for Toft's theorem, and let $G$ be a $4$-critical graph on $m$ vertices with
\[
e(G)\ge a\,m^{2}.
\]
Set $t=4m$, and let $\cA$ be the set of proper edge colorings $\beta\colon E(G)\to[t]$, that is, of maps assigning distinct colors to any two edges sharing an endpoint.

\begin{lemma}\label{count_critical:lem:colorings}
$|\cA|\ge (2m)^{e(G)}$.
\end{lemma}

\begin{proof}
Order $E(G)$ arbitrarily and color greedily.
When the edge $uv$ is reached, the colors already used on edges meeting $uv$ number at most
\[
(\deg u-1)+(\deg v-1)\le 2m-4 ,
\]
since every degree is at most $m-1$.
As $t=4m$, at least $2m+4$ colors are available at every step.
Distinct sequences of choices produce distinct colorings, since the edge order is fixed, so $|\cA|\ge(2m)^{e(G)}$.
\end{proof}

We now define the ambient vertex set.
Let
\[
V^{\ast}=\{v^{1},v^{2},v^{3}:v\in V(G)\}\cup\{w_1,\dots,w_t\}\cup W,
\]
where $W$ consists of $t-1$ disjoint sets of fresh internal vertices, one for a copy of $S$ attached to each consecutive pair $w_j,w_{j+1}$, together with $m$ further disjoint sets of fresh internal vertices, one for a copy of $D$ attached to each triple $v^{1},v^{2},v^{3}$.
The set $V^{\ast}$ does not depend on $\beta$.

For $\beta\in\cA$ let $J_\beta$ be the hypergraph on $V^{\ast}$ whose edges are the following.
For every $uv\in E(G)$ and every $i\in[3]$ there is the \emph{lifted triple}
\begin{equation}\label{count_critical:eq:lift}
\{u^{i},v^{i},w_{\beta(uv)}\}.
\end{equation}
For every $j\in[t-1]$ there are the edges of a copy of $S$ whose designated vertices are identified with $w_j$ and $w_{j+1}$ and whose remaining vertices are the corresponding fresh set in $W$.
For every $v\in V(G)$ there are the edges of a copy of $D$ whose designated vertices are identified with $v^{1},v^{2},v^{3}$ and whose remaining vertices are the corresponding fresh set in $W$.
Figure~\ref{count_critical:fig:lift} shows the three pieces.

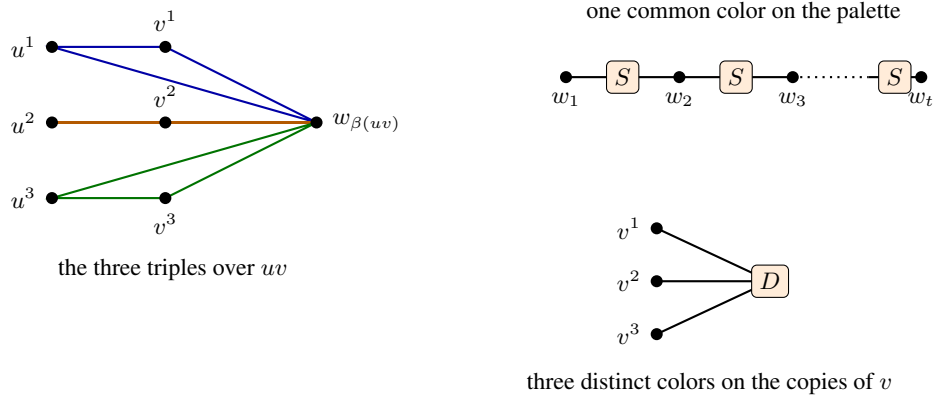
\begin{figure}[t]
\centering
\begin{tikzpicture}[
  pt/.style={circle,fill=black,inner sep=1.6pt},
  gad/.style={rectangle,draw,rounded corners=2pt,fill=orange!15,inner sep=3pt,font=\small},
  cptn/.style={font=\small,align=center}
]
\begin{scope}
  \coordinate (u1) at (0,2.0);
  \coordinate (v1) at (1.5,2.0);
  \coordinate (u2) at (0,1.0);
  \coordinate (v2) at (1.5,1.0);
  \coordinate (u3) at (0,0.0);
  \coordinate (v3) at (1.5,0.0);
  \coordinate (wa) at (3.5,1.0);
  \draw[blue!65!black,line width=0.8pt] (u1) -- (v1) -- (wa) -- cycle;
  \draw[orange!70!black,line width=0.8pt] (u2) -- (v2) -- (wa) -- cycle;
  \draw[green!45!black,line width=0.8pt] (u3) -- (v3) -- (wa) -- cycle;
  \node[pt,label=left:{\small $u^{1}$}] at (u1) {};
  \node[pt,label=left:{\small $u^{2}$}] at (u2) {};
  \node[pt,label=left:{\small $u^{3}$}] at (u3) {};
  \node[pt,label=above:{\small $v^{1}$}] at (v1) {};
  \node[pt,label=above:{\small $v^{2}$}] at (v2) {};
  \node[pt,label=below:{\small $v^{3}$}] at (v3) {};
  \node[pt,label=right:{\small $w_{\beta(uv)}$}] at (wa) {};
  \node[cptn] at (1.6,-0.95) {the three triples over $uv$};
\end{scope}

\begin{scope}[xshift=6.8cm,yshift=1.6cm]
  \coordinate (w1) at (0,0);
  \coordinate (w2) at (1.5,0);
  \coordinate (w3) at (3.0,0);
  \coordinate (w4) at (4.7,0);
  \draw[line width=0.8pt] (w1) -- (w2) -- (w3);
  \draw[line width=0.8pt,dotted] (w3) -- (4.0,0);
  \draw[line width=0.8pt] (4.0,0) -- (w4);
  \node[gad] at (0.75,0) {$S$};
  \node[gad] at (2.25,0) {$S$};
  \node[gad] at (4.35,0) {$S$};
  \node[pt,label=below:{\small $w_1$}] at (w1) {};
  \node[pt,label=below:{\small $w_2$}] at (w2) {};
  \node[pt,label=below:{\small $w_3$}] at (w3) {};
  \node[pt,label=below:{\small $w_t$}] at (w4) {};
  \node[cptn] at (2.35,0.85) {one common color on the palette};
\end{scope}

\begin{scope}[xshift=8.0cm,yshift=-1.1cm]
  \coordinate (d1) at (0,0.7);
  \coordinate (d2) at (0,0.0);
  \coordinate (d3) at (0,-0.7);
  \coordinate (dc) at (1.5,0.0);
  \draw[line width=0.8pt] (d1) -- (dc);
  \draw[line width=0.8pt] (d2) -- (dc);
  \draw[line width=0.8pt] (d3) -- (dc);
  \node[gad] at (dc) {$D$};
  \node[pt,label=left:{\small $v^{1}$}] at (d1) {};
  \node[pt,label=left:{\small $v^{2}$}] at (d2) {};
  \node[pt,label=left:{\small $v^{3}$}] at (d3) {};
  \node[cptn] at (0.7,-1.35) {three distinct colors on the copies of $v$};
\end{scope}
\end{tikzpicture}
\caption{The three ingredients of $J_\beta$.
On the left, the three lifted triples \eqref{count_critical:eq:lift} sitting over one edge $uv$ of $G$, drawn as triangles sharing the palette vertex $w_{\beta(uv)}$.
On the upper right, the chain of copies of $S$ that forces $w_1,\dots,w_t$ to receive a single common color.
On the lower right, the copy of $D$ that forces the three copies of a vertex $v$ to receive three distinct colors.}
\label{count_critical:fig:lift}
\end{figure}

\begin{lemma}\label{count_critical:lem:linear}
For every $\beta\in\cA$ the hypergraph $J_\beta$ is a linear triple system, and $|V^{\ast}|\le C_0m$ for a constant $C_0$ depending only on $S$ and $D$.
\end{lemma}

\begin{proof}
Write $s=|V(S)|$ and $d_0=|V(D)|$.
Counting the vertices of $V^{\ast}$ layer by layer,
\[
|V^{\ast}|=3m+t+(t-1)(s-2)+m(d_0-3)\le (4s+d_0-4)\,m ,
\]
since $t=4m$; take $C_0=4s+d_0-4$.

Every edge of $J_\beta$ has three vertices, so it remains to check that no pair of vertices lies in two edges.
Consider first a pair of the form $\{u^{i},v^{i}\}$ with $u\ne v$.
No edge of a gadget copy contains copies of two distinct vertices of $G$, and a lifted triple containing both $u^{i}$ and $v^{i}$ lies over the edge $uv$ in layer $i$, so at most one edge contains the pair.
Next consider a pair $\{v^{i},w_j\}$.
Gadget copies contribute no such edge, and a lifted triple containing this pair lies over an edge of $G$ incident to $v$ and colored $j$; since $\beta$ is proper there is at most one such edge, and the layer $i$ is determined as well.
A pair $\{w_j,w_{j'}\}$ lies in no lifted triple, and it lies in no gadget edge either, because the copies of $S$ meet each other only in single palette vertices and no edge of $S$ contains both designated vertices.
A pair $\{v^{i},v^{i'}\}$ with $i\ne i'$ lies in no lifted triple, and the only gadget copy containing it is the copy of $D$ attached to $v$.
A pair $\{u^{i},v^{i'}\}$ with $u\ne v$ and $i\ne i'$ lies in no edge at all: every lifted triple uses a single layer, and no gadget edge contains copies of two distinct vertices of $G$.
Every remaining pair contains a fresh internal vertex, which lies in a single gadget copy, so every edge containing the pair is an edge of that copy; each copy is itself a linear triple system, so there is at most one such edge.
Hence no pair of vertices lies in two edges of $J_\beta$.
\end{proof}

\begin{lemma}\label{count_critical:lem:nothree}
For every $\beta\in\cA$ the hypergraph $J_\beta$ is not $3$-colorable.
\end{lemma}

\begin{proof}
Suppose $\psi$ is a proper $3$-coloring of $J_\beta$.
The copies of $S$ give $\psi(w_j)=\psi(w_{j+1})$ for every $j\in[t-1]$ by Lemma~\ref{count_critical:lem:gadgets}, so
\[
\psi(w_1)=\psi(w_2)=\cdots=\psi(w_t),
\]
and after relabeling the colors we may assume this common value is $1$.
For each $v\in V(G)$ the copy of $D$ attached to $v^{1},v^{2},v^{3}$ forces those three vertices to receive three distinct colors, so there is a unique index $\phi(v)\in[3]$ with $\psi(v^{\phi(v)})=1$.

Let $uv\in E(G)$ and suppose $\phi(u)=\phi(v)=i$.
Then the lifted triple $\{u^{i},v^{i},w_{\beta(uv)}\}$ is monochromatic of color $1$, a contradiction.
Hence $\phi$ is a proper $3$-coloring of $G$, which is impossible because $G$ is $4$-chromatic.
\end{proof}

\subsubsection{Recovering the edge coloring}

For $uv\in E(G)$ write
\[
L_{uv}=\bigl\{\{u^{i},v^{i},w_{\beta(uv)}\}:i\in[3]\bigr\}
\]
for the set of three lifted triples over $uv$.
These sets are pairwise disjoint as $uv$ ranges over $E(G)$.

\begin{lemma}\label{count_critical:lem:retain}
Let $K\subseteq J_\beta$ be a subhypergraph that is not $3$-colorable.
Then $K$ contains a triple from $L_{uv}$ for every $uv\in E(G)$.
\end{lemma}

\begin{proof}
Suppose instead that $K$ contains no triple of $L_{uv}$ for some $uv\in E(G)$, so that $K$ is a subhypergraph of $J_\beta-L_{uv}$.
We exhibit a proper $3$-coloring of $J_\beta-L_{uv}$, which restricts to one of $K$.
Since $G$ is $4$-critical, $G-uv$ has a proper $3$-coloring $\phi\colon V(G)\to[3]$.
Give every palette vertex $w_j$ the color $1$, and for each $z\in V(G)$ color $z^{1},z^{2},z^{3}$ with the three colors in such a way that $z^{\phi(z)}$ receives the color $1$.

Let $\{x^{i},y^{i},w_{\beta(xy)}\}$ be a retained lifted triple, so $xy\in E(G)$ and $xy\ne uv$.
If this triple were monochromatic then, as $w_{\beta(xy)}$ has color $1$, we would have $\phi(x)=\phi(y)=i$, contradicting $\phi(x)\ne\phi(y)$.
So no retained lifted triple is monochromatic.
On each copy of $S$ the two designated vertices have received the same color, and on each copy of $D$ the three designated vertices have received three distinct colors, so by Lemma~\ref{count_critical:lem:gadgets} the coloring extends over the fresh internal vertices of every gadget copy.
Thus $J_\beta-L_{uv}$ is $3$-colorable, contradicting the hypothesis on $K$.
\end{proof}

For each $\beta\in\cA$ choose an edge-minimal non-$3$-colorable subhypergraph of $J_\beta$ and delete its isolated vertices; call the result $K_\beta$.
Such a subhypergraph exists by Lemma~\ref{count_critical:lem:nothree}.

\begin{lemma}\label{count_critical:lem:critical}
Each $K_\beta$ is a $4$-critical linear triple system with at least $e(G)$ edges.
Moreover the map $\beta\mapsto K_\beta$ is injective.
\end{lemma}

\begin{proof}
Linearity is inherited from $J_\beta$ by Lemma~\ref{count_critical:lem:linear}.
By the choice of $K_\beta$ it is not $3$-colorable while $K_\beta-e$ is $3$-colorable for every edge $e$, and it has no isolated vertices.
It is $4$-colorable: delete an edge $e$, take a proper $3$-coloring of $K_\beta-e$, and if $e$ is monochromatic recolor one vertex of $e$ with a fourth color.
The recolored vertex then lies in no monochromatic edge, since it is the only vertex of that color, and every other edge is unaffected.
Hence $K_\beta$ is $4$-chromatic and $4$-critical.
By Lemma~\ref{count_critical:lem:retain} it contains a triple of $L_{uv}$ for each of the $e(G)$ edges $uv$ of $G$, and these sets are pairwise disjoint, so $K_\beta$ has at least $e(G)$ edges.

For injectivity, fix $uv\in E(G)$ and read off $\beta(uv)$ from the labeled hypergraph $K_\beta$ as follows.
By Lemma~\ref{count_critical:lem:retain} some triple of the form $\{u^{i},v^{i},w_j\}$ belongs to $K_\beta$, and as observed in the proof of Lemma~\ref{count_critical:lem:linear} the only edges of $J_\beta$ containing $u^{i}$ and $v^{i}$ are the lifted triples over $uv$, all of which use the palette vertex $w_{\beta(uv)}$.
So $j=\beta(uv)$.
Therefore $K_\beta$ determines $\beta$, and distinct edge colorings give distinct labeled hypergraphs on $V^{\ast}$.
\end{proof}

\subsubsection{Proof of the main theorem}

\begin{proof}[Proof of Theorem~\ref{count_critical:thm:main}]
Let $m$ be large, let $G$, $\cA$ and $\{K_\beta:\beta\in\cA\}$ be as above, and put $M=|V^{\ast}|$, so that $M\le C_0m$ by Lemma~\ref{count_critical:lem:linear}.
By Lemmas~\ref{count_critical:lem:colorings} and~\ref{count_critical:lem:critical} the family $\{K_\beta:\beta\in\cA\}$ consists of at least $(2m)^{e(G)}$ distinct $4$-critical linear triple systems, all labeled inside the single vertex set $V^{\ast}$.
Each has between $1$ and $M$ vertices, so there is an integer $q\le M$ such that at least
\[
\frac{(2m)^{e(G)}}{M}
\]
of them have exactly $q$ vertices.
A labeled system inside $V^{\ast}$ isomorphic to a fixed $q$-vertex system is determined by an injection of the latter's vertex set into $V^{\ast}$, so each isomorphism class accounts for at most $M(M-1)\cdots(M-q+1)\le M^{M}$ of them.
Therefore
\[
T(4,3,2,q)\ \ge\ \frac{(2m)^{e(G)}}{M\cdot M^{M}} .
\]
Using $e(G)\ge am^{2}$ and $M\le C_0m$,
\[
\log\left(\frac{(2m)^{e(G)}}{M\cdot M^{M}}\right)
\ \ge\ a\,m^{2}\log (2m)-(C_0m+1)\log (C_0m)
\ \ge\ \frac{a}{2}\,m^{2}\log m
\]
for all sufficiently large $m$.

It remains to compare $q$ with $m$.
On one hand $q\le M\le C_0m$.
On the other hand a linear triple system on $q$ vertices has at most $\binom{q}{2}\big/3$ edges, since its edges contain three vertex pairs each and no pair is repeated, so Lemma~\ref{count_critical:lem:critical} gives
\[
\frac{q^{2}}{6}\ \ge\ \frac{1}{3}\binom{q}{2}\ \ge\ e(G)\ \ge\ a\,m^{2},
\]
whence $q\ge\sqrt{6a}\,m$.
Thus $q=q(m)=\Theta(m)$, and in particular $q(m)\to\infty$ as $m\to\infty$, so the integers $q(m)$ take infinitely many distinct values.
Finally $m\ge q/C_0$ and $\log m\ge \log q-\log C_0\ge\tfrac12\log q$ once $q\ge C_0^{2}$, so
\[
T(4,3,2,q)\ \ge\ \exp\left(\frac{a}{2}m^{2}\log m\right)\ \ge\ \exp\left(\frac{a}{4C_0^{2}}\,q^{2}\log q\right).
\]
Writing $n=q(m)$ and $c=a/(4C_0^{2})$ completes the proof.
\end{proof}

\begin{remark}\label{count_critical:rem:optimize}
The constant $c$ is not optimized here, and the argument settles only the case $(k,r,l)=(4,3,2)$, and only along a sequence of values of $n$.
Any palette of bounded size would give only $\exp(O(m^{2}))$ colorings and would not contradict the conjecture.
\end{remark}

\FloatBarrier
\putbib
\end{bibunit}
\endgroup
\subsection[(Spiro) Strongly connected digraphs with no small \texorpdfstring{$k$}{k}-kernel]{Strongly connected digraphs with no small \texorpdfstring{$k$}{k}-kernel}
\label{sol:digraph_kernels}
\begingroup
\begin{bibunit}

This problem asks how small a $k$-kernel a strongly connected digraph is guaranteed to have.
Spiro asked whether $|V(G)|/(k+1)+O_k(1)$ is always enough once $k\ge3$; we show that it is not, and that the coefficient $1/(k+1)$ is itself wrong, the truth lying between $1/k$ and $1/(k-1)$.
The digraphs witnessing this are a hub feeding $m$ parallel directed paths that all return through one common tail, long enough to force the hub out of every $k$-kernel.
Two refutations of the same question were published after our run, and the construction below was obtained independently of both.

\subsubsection{Introduction}

All digraphs here are finite and have no loops or parallel arcs; directed cycles of length two are allowed.
A set $X\subseteq V(G)$ is \emph{stable} if no arc of $G$ has both ends in $X$, and for an integer $k\ge1$ a \emph{$k$-kernel} of $G$ is a stable set $X\subseteq V(G)$ such that every vertex of $G$ can be joined from $X$ by a directed path of length at most $k$.
Paths of length $0$ are allowed, so a $k$-kernel covers itself.
A $1$-kernel is a kernel and a $2$-kernel is a quasikernel. In a tournament every stable set is a single vertex, so a quasikernel is exactly a king \citep{post2023common}.
We write $\kappa_k(G)$ for the minimum size of a $k$-kernel of $G$.

Not every digraph has a kernel, but every digraph has a quasikernel by the theorem of \citet{chvatal2006every}, so $\kappa_k(G)$ is defined for every $k\ge2$.
How small a quasikernel one is entitled to expect is the subject of the \emph{small quasi-kernel conjecture}, that $\kappa_2(G)\le|V(G)|/2$ whenever $G$ has no source.
It was posed by P.~L.~Erd\H{o}s and Sz\'ekely in 1976 and stated in print by \citet{erdos2010two}; \citet{erdHos2023small} survey what is known.
After reversing every arc, $k$-kernels become the $(2,k)$-kernels of \citet{kwasnik2006k}, that is, the stable sets that absorb every vertex within distance $k$.
\citet{spiro2026generalized} initiated their extremal study and asked what the source-free hypothesis buys when it is strengthened to strong connectivity.

\medskip
\noindent
\emph{If $D$ is a strongly connected digraph and $q\ge3$ is an integer, does there exist a $q$-kernel $Q$ of $D$ such that $|Q|\le\frac{|V(D)|}{q+1}+O_q(1)$?}
\medskip

\noindent
This is Question~7.7 of \citet{spiro2026generalized}, and it is recorded as Conjecture~1.4 by \citet{nguyen2024distant}, who attribute it to Spiro and leave it open.
We write $k$ throughout for the integer called $q$ there.
Spiro observes that the coefficient $1/(k+1)$ would be best possible, by considering a collection of directed cycles $C_{k+2}$ all sharing a single vertex.
The restriction to $k\ge3$ is necessary, and is what distinguishes the published form of the question from the first preprint version, which asked it for $k\ge2$: Example~17 of \citet{erdHos2023small} exhibits strongly connected digraphs whose smallest quasikernel has size $(\tfrac12-o(1))|V(G)|$, which is far above $|V(G)|/3$.

On the positive side, \citet{spiro2026generalized} showed that under the hypotheses of the question there is always a $k$-kernel of size at most about $|V(G)|/\log k$, by producing about $\log k$ pairwise disjoint $k$-kernels.
This was improved by \citet[Theorem~1.5]{nguyen2024distant}, who proved that every digraph $G$ with $|V(G)|>1$ admitting a spanning out-arborescence, and in particular every strongly connected digraph, satisfies
\begin{equation}\label{digraph_kernels:eq:nss}
  \kappa_k(G)\ \le\ 1+\frac{|V(G)|-2}{k-1}
  \qquad(k\ge2).
\end{equation}
They deduce \eqref{digraph_kernels:eq:nss} from the acyclic case, where the coefficient improves to $1/k$: an acyclic digraph on at least two vertices with only one source has a $k$-kernel of size at most $1+(|V(G)|-2)/k$.
The digraph of their Figure~1, a source $y$ with a single arc to a hub $x$ together with $m$ disjoint directed paths on $k$ vertices leaving $x$, shows that this is tight.
We show that the answer to Question~7.7 is negative for every $k\ge3$, and that it fails by a linear rather than a constant margin.

\begin{theorem}\label{digraph_kernels:thm:main}
For every integer $k\ge2$ and every integer $m\ge1$ there is a strongly connected digraph $G_{k,m}$ on $mk+k+2$ vertices whose smallest $k$-kernel has size exactly
\[
  \kappa_k(G_{k,m})\ =\ m+1\ =\ \frac{|V(G_{k,m})|-2}{k}.
\]
Consequently, for every $k\ge2$ and every constant $c$ there is a strongly connected digraph $G$ with $\kappa_k(G)>|V(G)|/(k+1)+c$.
\end{theorem}

Theorem~\ref{digraph_kernels:thm:main} does more than defeat the additive term $O_k(1)$: it shows that the coefficient $1/(k+1)$ is itself wrong.
Write $c_k^{*}$ for the infimum of the constants $c$ for which $\kappa_k(G)\le c|V(G)|+O_k(1)$ holds for all strongly connected $G$.
Theorem~\ref{digraph_kernels:thm:main} gives $c_k^{*}\ge1/k$ and \eqref{digraph_kernels:eq:nss} gives $c_k^{*}\le1/(k-1)$.
For $k\ge3$, whether the value $1/k$ is itself admissible is asked by \citet{wang2026counterexamples}; for $k=2$ the corresponding statement $c_2^{*}=1/2$ would follow from the small quasi-kernel conjecture, since every strongly connected digraph is source-free.
The excess is already substantial for moderate parameters: $\kappa_3(G_{3,50})=51$ on $155$ vertices against $155/4=38.75$, and $\kappa_{10}(G_{10,50})=51$ on $512$ vertices against $512/11\approx46.55$.
Taking $k=2$ gives strongly connected digraphs whose smallest quasikernel has size exactly $\tfrac12|V(G)|-1$, one less than the bound of the small quasi-kernel conjecture; digraphs achieving this up to $o(|V(G)|)$ were already given by \citet{erdHos2023small}.

Two refutations have since appeared.
\citet{wang2026counterexamples} give for each $k\ge3$ strongly connected digraphs every $k$-kernel of which has size at least $(|V(G)|-2)/k$; this is the bound of Theorem~\ref{digraph_kernels:thm:main}, which also covers $k=2$.
\citet{penev2026small} obtain the weaker bound $\kappa_k(G)>(1/k-\varepsilon)|V(G)|$, but with the extra feature that the digraphs may be taken of arbitrarily large directed girth.
By contrast, every directed cycle of $G_{k,m}$ passes through the hub $x$ and hence traverses a whole arm, so these cycles have exactly the lengths $k+2,k+3,\dots,2k+2$ and $G_{k,m}$ has directed girth only $k+2$.
The same paper shows that for odd $k\ge3$ every source-free bipartite digraph with a spanning out-arborescence, and in particular every strongly connected bipartite digraph, has a $k$-kernel of size at most $|V(G)|/(k+1)+1$.
For odd $k$, then, no bipartite digraph can refute Question~7.7, and indeed the lengths just listed include odd ones, so $G_{k,m}$ is not bipartite.

The construction is the acyclic example of Figure~1 of \citet{nguyen2024distant} with a return path attached: the $m$ parallel directed paths leaving the hub $x$, each on $k$ vertices, are brought back to $x$ through one common tail.
In the acyclic example the hub is unusable because the source $y$ lies in every $k$-kernel and sends an arc to $x$; here $y$ is no longer a source, and the tail takes over that role.
The tail is long enough that some vertex of it must be used to cover its own far end, and every vertex of the tail sends an arc to the hub; this forces the hub out of every $k$-kernel and thereby forces each of the $m$ paths to pay for itself.

\subsubsection{The construction}

Fix integers $k\ge2$ and $m\ge1$.
The digraph $G_{k,m}$ has vertex set
\[
  V(G_{k,m})=\{x,y,p_1,\dots,p_k\}\ \cup\ \{a_{i,j}:1\le i\le m,\ 1\le j\le k\}
\]
and arc set consisting of
\[
  y\to x,\qquad
  p_j\to x\ \ (1\le j\le k),\qquad
  p_j\to p_{j+1}\ \ (1\le j<k),\qquad
  p_k\to y,
\]
together with, for each $i\in\{1,\dots,m\}$, the \emph{arm}
\[
  x\to a_{i,1}\to a_{i,2}\to\cdots\to a_{i,k}\to p_1 .
\]
There are no other arcs.
Write $A_i:=\{a_{i,1},\dots,a_{i,k}\}$ for the $i$th arm and $P:=\{y,p_1,\dots,p_k\}$ for the \emph{tail}.
Thus $V(G_{k,m})$ is the disjoint union of $\{x\}$, $P$ and $A_1,\dots,A_m$, and
\begin{equation}\label{digraph_kernels:eq:size}
  |V(G_{k,m})|=1+(k+1)+mk=mk+k+2 .
\end{equation}
Figure~\ref{digraph_kernels:fig:construction} shows the digraph.

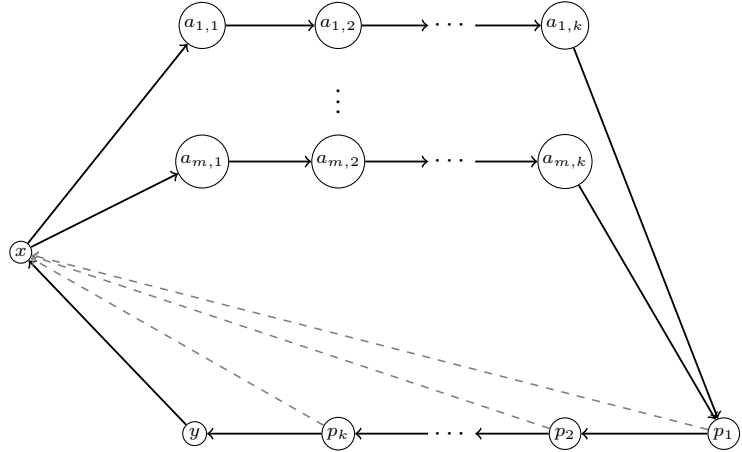
\begin{figure}[t]
\centering
\begin{tikzpicture}[
    v/.style={circle,draw,inner sep=1pt,font=\scriptsize},
    a/.style={->,line width=0.7pt},
    b/.style={->,line width=0.6pt,gray,dashed}
]
\node[v] (x)   at (0,0)      {$x$};
\node[v] (a11) at (2.4,3.0)  {$a_{1,1}$};
\node[v] (a12) at (4.2,3.0)  {$a_{1,2}$};
\node[inner sep=2pt] (d1) at (5.7,3.0) {$\cdots$};
\node[v] (a1k) at (7.2,3.0)  {$a_{1,k}$};
\node[inner sep=2pt] at (4.2,2.1) {$\vdots$};
\node[v] (am1) at (2.4,1.2)  {$a_{m,1}$};
\node[v] (am2) at (4.2,1.2)  {$a_{m,2}$};
\node[inner sep=2pt] (dm) at (5.7,1.2) {$\cdots$};
\node[v] (amk) at (7.2,1.2)  {$a_{m,k}$};
\node[v] (p1)  at (9.3,-2.4) {$p_1$};
\node[v] (p2)  at (7.2,-2.4) {$p_2$};
\node[inner sep=2pt] (dp) at (5.7,-2.4) {$\cdots$};
\node[v] (pk)  at (4.2,-2.4) {$p_k$};
\node[v] (yy)  at (2.3,-2.4) {$y$};

\draw[a] (x)   -- (a11);
\draw[a] (a11) -- (a12);
\draw[a] (a12) -- (d1);
\draw[a] (d1)  -- (a1k);
\draw[a] (a1k) -- (p1);
\draw[a] (x)   -- (am1);
\draw[a] (am1) -- (am2);
\draw[a] (am2) -- (dm);
\draw[a] (dm)  -- (amk);
\draw[a] (amk) -- (p1);
\draw[a] (p1)  -- (p2);
\draw[a] (p2)  -- (dp);
\draw[a] (dp)  -- (pk);
\draw[a] (pk)  -- (yy);
\draw[a] (yy)  -- (x);

\draw[b] (p1) -- (x);
\draw[b] (p2) -- (x);
\draw[b] (pk) -- (x);
\end{tikzpicture}
\caption{The digraph $G_{k,m}$.
The hub $x$ feeds $m$ arms, each a directed path with $k$ internal vertices running from $x$ to $p_1$, and the tail $p_1\to p_2\to\cdots\to p_k\to y\to x$ closes the digraph up.
The dashed arcs $p_j\to x$, present for every $j$, are what keep $x$ out of every $k$-kernel.}
\label{digraph_kernels:fig:construction}
\end{figure}

\begin{lemma}\label{digraph_kernels:lem:strong}
$G_{k,m}$ is strongly connected.
\end{lemma}

\begin{proof}
Every vertex reaches $x$: the vertex $y$ and each $p_j$ do so along a single arc, and $a_{i,j}$ reaches $p_1$ along its own arm and then uses $p_1\to x$.
Conversely $x$ reaches every vertex: it reaches all of $A_i$ along the $i$th arm and then $p_1$, then $p_2,\dots,p_k$ along the tail, and finally $y$.
Hence any vertex reaches any other through $x$.
\end{proof}

\subsubsection{A small \texorpdfstring{$k$}{k}-kernel}

We first exhibit a $k$-kernel of the size claimed in Theorem~\ref{digraph_kernels:thm:main}.

\begin{lemma}\label{digraph_kernels:lem:upper}
The set $K_0:=\{y\}\cup\{a_{i,k}:1\le i\le m\}$ is a $k$-kernel of $G_{k,m}$, and $|K_0|=m+1$.
\end{lemma}

\begin{proof}
First, $K_0$ is stable.
Indeed the only arcs incident with $y$ are $y\to x$ and $p_k\to y$, and the only arcs incident with $a_{i,k}$ are $a_{i,k-1}\to a_{i,k}$ and $a_{i,k}\to p_1$; none of $x$, $p_k$, $p_1$, $a_{i,k-1}$ lies in $K_0$.
In particular distinct vertices $a_{i,k}$ and $a_{i',k}$ are nonadjacent, as they lie on different arms.

For the covering condition, $y\to x$ gives $\mathrm{dist}(y,x)=1$, and $y\to x\to a_{i,1}\to\cdots\to a_{i,j}$ is a directed path of length $j+1$, so $\mathrm{dist}(y,a_{i,j})\le k$ for every $j\le k-1$.
Each $a_{i,k}$ covers itself, and $a_{i,k}\to p_1\to p_2\to\cdots\to p_j$ has length $j\le k$, so any one $a_{i,k}$ covers all of $p_1,\dots,p_k$.
Every vertex of $G_{k,m}$ is therefore joined from $K_0$ by a directed path of length at most $k$.
\end{proof}

\subsubsection{Every \texorpdfstring{$k$}{k}-kernel is large}

We now show that no $k$-kernel of $G_{k,m}$ can be smaller.

\begin{lemma}\label{digraph_kernels:lem:lower}
Every $k$-kernel $K$ of $G_{k,m}$ satisfies $|K|\ge m+1$.
\end{lemma}

\begin{proof}
The crucial observation is that the far end of the tail and the far end of each arm can each be covered from only one place.

First consider $y$.
The unique in-neighbor of $y$ is $p_k$, the unique in-neighbor of $p_j$ is $p_{j-1}$ for $2\le j\le k$, and the in-neighbors of $p_1$ are the $m$ vertices $a_{i,k}$.
Walking backwards from $y$ therefore gives $\mathrm{dist}(p_{k+1-\ell},y)=\ell$ for $1\le\ell\le k$, while every remaining vertex is at distance at least $k+1$ from $y$.
So the vertices $u$ with $\mathrm{dist}(u,y)\le k$ are exactly the $k+1$ vertices of $P$, and hence $K\cap P\ne\varnothing$.

Next, every vertex of $P$ has an arc to $x$, by $y\to x$ and $p_j\to x$.
Picking $u\in K\cap P$, the pair $u,x$ is joined by an arc, so stability of $K$ forces $x\notin K$.

Now fix $i$ and consider $a_{i,k}$.
The unique in-neighbor of $a_{i,j}$ is $a_{i,j-1}$ for $2\le j\le k$, and the unique in-neighbor of $a_{i,1}$ is $x$.
Walking backwards from $a_{i,k}$ therefore gives $\mathrm{dist}(a_{i,k-\ell},a_{i,k})=\ell$ for $1\le\ell\le k-1$ and $\mathrm{dist}(x,a_{i,k})=k$, while every remaining vertex is at distance at least $k+1$ from $a_{i,k}$.
So the vertices $u$ with $\mathrm{dist}(u,a_{i,k})\le k$ are exactly the $k+1$ vertices of $A_i\cup\{x\}$, and since $x\notin K$ we get $K\cap A_i\ne\varnothing$.

The sets $P,A_1,\dots,A_m$ are pairwise disjoint, so $|K|\ge m+1$.
\end{proof}

\subsubsection{Proof of the main theorem}

\begin{proof}[Proof of Theorem~\ref{digraph_kernels:thm:main}]
The digraph $G_{k,m}$ is strongly connected by Lemma~\ref{digraph_kernels:lem:strong} and has $mk+k+2$ vertices by \eqref{digraph_kernels:eq:size}.
Lemmas~\ref{digraph_kernels:lem:upper} and~\ref{digraph_kernels:lem:lower} give $\kappa_k(G_{k,m})=m+1$, and
\[
  m+1=\frac{mk+k}{k}=\frac{|V(G_{k,m})|-2}{k}
\]
by \eqref{digraph_kernels:eq:size}.
For the last assertion, fix $k\ge2$ and compute
\[
  \kappa_k(G_{k,m})-\frac{|V(G_{k,m})|}{k+1}
  =(m+1)-\frac{mk+k+2}{k+1}
  =\frac{(m+1)(k+1)-mk-k-2}{k+1}
  =\frac{m-1}{k+1}.
\]
This tends to infinity as $m\to\infty$, so given $c$ we may choose $m$ with $(m-1)/(k+1)>c$ and take $G=G_{k,m}$.
\end{proof}

\begin{remark}\label{digraph_kernels:rem:variant}
The hub-and-arms mechanism is not tied to the particular tail used above.
Fix integers $k\ge2$ and $t\ge1$, and write $C=\{c_0,\dots,c_k\}$, $R=\{r_1,\dots,r_{k-1}\}$ and $B_i=\{b_{i,1},\dots,b_{i,k}\}$ for $1\le i\le t$.
The digraph $H_{k,t}$ has vertex set
\[
  V(H_{k,t})=C\cup\{x,z\}\cup R\cup B_1\cup\cdots\cup B_t
\]
and arc set consisting of
\[
  c_0\to c_1\to\cdots\to c_k\to c_0,\qquad
  c_j\to x\ \ (0\le j\le k),\qquad
  z\to r_1\to\cdots\to r_{k-1}\to c_0,
\]
together with, for each $i\in\{1,\dots,t\}$, the arm
\[
  x\to b_{i,1}\to\cdots\to b_{i,k}\to z .
\]
Then $H_{k,t}$ is strongly connected on $kt+2k+2$ vertices, and $\kappa_k(H_{k,t})=t+1$.
The upper bound comes from the $k$-kernel $\{c_0\}\cup\{b_{i,k}:1\le i\le t\}$, and the lower bound repeats the proof of Lemma~\ref{digraph_kernels:lem:lower}.
Indeed the only arc entering $C$ from outside is $r_{k-1}\to c_0$, so a path from outside $C$ to $c_k$ has length at least $1+\mathrm{dist}(c_0,c_k)=k+1$ and $c_k$ is covered only from within $C$; every $c_j$ has an arc to $x$, so $x$ lies in no $k$-kernel; and the only arc entering $B_i$ from outside is $x\to b_{i,1}$, so $b_{i,k}$ is covered only from $B_i\cup\{x\}$ and every $k$-kernel meets every $B_i$.
Since $t+1=(|V(H_{k,t})|-k-2)/k$, this family gives the same coefficient $1/k$ as Theorem~\ref{digraph_kernels:thm:main} but loses $1$ in the additive constant.
\end{remark}

\putbib

\end{bibunit}
\endgroup
\subsection[(Taylor) \texorpdfstring{$F$}{F}-positivity of chromatic symmetric functions of hypertrees]{\texorpdfstring{$F$}{F}-positivity of chromatic symmetric functions of hypertrees}
\label{sol:hypertree}
\begingroup
\begin{bibunit}
\def\N{\mathbb{N}}
\def\QSym{\mathrm{QSym}}
\def\Des{\operatorname{Des}}
\def\one{\mathbf{1}}

The chromatic symmetric function of a hypergraph is the generating function for the colorings under which no hyperedge is monochromatic.
Unlike its graph analogue, it need not expand nonnegatively in Gessel's fundamental quasisymmetric basis, but \citet{taylor2015chromatic} proved that it does whenever the hypergraph is a hypertree all of whose hyperedges have prime size, and conjectured that primality is superfluous.
We prove the conjecture.
The prime case rests on partitioning the nonconstant colorings of an $r$-element hyperedge into $r!$ blocks indexed by the linear orders of that hyperedge, and no general construction of such a partition is known.
We replace the partition by a decomposition of the indicator function of ``nonconstant'' into chain conditions with nonnegative rational weights; such a decomposition exists for every $r$, and it glues across a hypertree exactly as a partition would.

\subsubsection{Introduction}

A \emph{hypergraph} is a pair $H=(V,E)$ with $V$ finite and $E$ a family of subsets of $V$, each of size at least $2$, called \emph{hyperedges}.
A \emph{coloring} of $H$ is a map $\kappa:V\to\N$, where $\N=\{1,2,\dots\}$, and $\kappa$ is \emph{proper} if no hyperedge is monochromatic, that is, if $\kappa$ is nonconstant on every $e\in E$.
The \emph{chromatic symmetric function} of $H$ is
\[
   X_H=\sum_{\kappa\ \text{proper}}\ \prod_{v\in V}x_{\kappa(v)},
\]
introduced for ordinary graphs by \citet{stanley1995symmetric} and extended to hypergraphs by \citet{stanley1998graph}.

Write $n=|V|$ and $[n]=\{1,\dots,n\}$.
For $S\subseteq[n-1]$ the \emph{fundamental quasisymmetric function} of \citet{gessel1984multipartite} is
\[
   F_S^{(n)}=\sum_{\substack{i_1\leq\cdots\leq i_n\\ j\in S\Rightarrow i_j<i_{j+1}}}x_{i_1}\cdots x_{i_n},
\]
and the $F_S^{(n)}$ with $S\subseteq[n-1]$ form a basis of the degree-$n$ component of the ring $\QSym$ of quasisymmetric functions.
A homogeneous quasisymmetric function of degree $n$ is \emph{$F$-positive} if all of its coefficients in this basis are nonnegative.

A \emph{path} in $H$ is a sequence $v_1,e_1,v_2,e_2,\dots,e_m,v_{m+1}$ with $v_i,v_{i+1}\in e_i$ for each $i$, in which the hyperedges $e_i$ are distinct and the vertices $v_i$ are distinct except that $v_1=v_{m+1}$ is allowed.
It is a \emph{cycle} if $v_1=v_{m+1}$ and $m\geq2$.
The hypergraph $H$ is \emph{connected} if any two vertices are joined by a path, and a \emph{hypertree} is a connected hypergraph with no cycles; this is the convention of \citet{gessel2005hypergraphs} adopted by \citet{taylor2015chromatic}.

For an ordinary graph, $X_G$ is always $F$-positive: \citet{stanley1995symmetric} splits the proper colorings according to the acyclic orientation each one induces and identifies every piece as a $(P,\omega)$-partition enumerator.
For hypergraphs this fails.
\citet{taylor2015chromatic} records that the hypergraph on $V=[4]$ with hyperedges $\{1,2,3\}$ and $\{2,3,4\}$ has
\[
   X_H=2F_{\{1\}}+6F_{\{2\}}+2F_{\{3\}}+4F_{\{1,2\}}+8F_{\{1,3\}}+4F_{\{2,3\}}-2F_{\{1,2,3\}},
\]
and that the hypergraph with hyperedges $\{1,2,3\},\{1,4\},\{2,4\},\{3,4,5\}$ is not $F$-positive either, although distinct hyperedges of it meet in at most one vertex.
Both of these contain cycles.
On the other hand \citet{taylor2015chromatic} proves that $X_H$ is $F$-positive whenever $H$ is a hypertree each of whose hyperedges has prime size, and in that case obtains the explicit expansion $X_H=\sum_{\pi}F^{(n)}_{\Des_H(\pi)}$ over all $n!$ bijections $\pi:V\to[n]$, where the $H$-descents $\Des_H(\pi)$ are read off from the unique path between consecutively labeled vertices.
\citet[Conjecture~A]{taylor2015chromatic} asks for the removal of the primality hypothesis:

\medskip
\noindent
\emph{Let $H$ be a hypertree.
Then $X_H$ is $F$-positive.}
\medskip

\noindent
The role of primality is isolated by \citet{taylor2015chromatic}.
Call an integer $r\geq2$ \emph{splittable} if the nonconstant colorings of an $r$-element set can be partitioned into $r!$ blocks indexed by the linear orders of that set, the block of a linear order consisting of the colorings that are weakly increasing along it and strictly increasing at some prescribed set of steps; he proves that $X_H$ is $F$-positive for every hypertree all of whose hyperedge sizes are splittable.
He shows that every prime is splittable, by the cyclic standardization of \citet{gessel1993counting}, and reports a splitting for $r=4$ found by computer search; he also observes that splittability of $r$ is equivalent to partitionability of the Coxeter complex of type $A_{r-1}$ with the empty face removed, a scheduling problem in the sense of \citet{breuer2016scheduling}.
Conjecture~A was therefore known for every hypertree whose hyperedge sizes are all prime or equal to $4$, and in particular for every hypertree with no hyperedge on more than $5$ vertices.
Each further composite size requires a splitting of its own.
In a different direction, \citet{pawlowski2018chromatic} proves that $X_H$ is Schur-positive, hence $F$-positive, for every hyperforest whose line graph is bipartite, which settles Conjecture~B of \citet{taylor2015chromatic}; this does not reach all hypertrees, since three $6$-element hyperedges through a common vertex form a hypertree whose line graph is $K_3$ and whose hyperedge sizes are neither prime nor $4$.

\begin{theorem}
\label{hypertree:thm:main}
Let $H=(V,E)$ be a finite hypertree and $n=|V|$.
Then
\[
   X_H=\sum_{S\subseteq[n-1]}a_S\,F_S^{(n)}
   \qquad\text{with } a_S\in\mathbb{Z}_{\geq0}\ \text{for every } S\subseteq[n-1].
\]
\end{theorem}

\noindent
This is Conjecture~A of \citet{taylor2015chromatic}.
Every bijection $V\to[n]$ is a proper coloring, and the coefficient of $x_1x_2\cdots x_n$ in $F_S^{(n)}$ equals $1$ for every $S$, so the coefficients of Theorem~\ref{hypertree:thm:main} satisfy $\sum_{S\subseteq[n-1]}a_S=n!$, as \citet{taylor2015chromatic} observes.
Theorem~\ref{hypertree:thm:main} therefore says that $X_H$ is a sum of $n!$ fundamental quasisymmetric functions counted with multiplicity.

Certain cases are immediate.
If $n=1$ then $E=\varnothing$, since every hyperedge has at least two vertices, and $X_H=x_1+x_2+\cdots=F_\varnothing^{(1)}$.
We assume $n\geq2$ from now on, so that every vertex of $H$ lies in a hyperedge.

We now summarize the proof.
Fix a hyperedge $e$ of size $r$ and, for a linear order $\tau$ of $e$ and a set $B\subseteq[r-1]$, consider the colorings of $e$ that are weakly increasing along $\tau$ and strictly increasing at the steps in $B$.
The key point is that the indicator function of ``$\kappa$ is nonconstant on $e$'' is a nonnegative rational combination of the indicator functions of these conditions, taken over all $r!$ orders $\tau$ and all nonempty $B$, with weights depending only on $r$ and $B$.
The weights are forced by the fundamental expansion of the chromatic symmetric function of a single hyperedge, and their nonnegativity is the elementary fact that every subset of $[r-1]$ is the descent set of at least one permutation of $[r]$.
Multiplying this local identity over the hyperedges expresses $X_H$ as a nonnegative rational combination of generating functions for systems of weak and strict inequalities, one inequality per consecutive pair inside each chosen order.
The hypertree hypothesis makes each such system a $(P,\omega)$-partition condition on an ordinary tree, so each of these generating functions is $F$-positive, and the fundamental coefficients of $X_H$ are nonnegative rationals.
They are integers because the monomial quasisymmetric coefficients of $X_H$ are integers and the two bases are related by an integral unitriangular matrix.

\subsubsection{A weighted local decomposition for one hyperedge}

Fix $r\geq2$.
For $\pi\in\mathfrak{S}_r$ write $\Des(\pi)=\{j\in[r-1]:\pi(j)>\pi(j+1)\}$, and for $B\subseteq[r-1]$ put
\[
   A_r(B)=\#\{\pi\in\mathfrak{S}_r:\Des(\pi)=B\},
   \qquad
   w_r(B)=\frac{A_r(B)-(-1)^{|B|}}{r!}.
\]
Since $A_r(\varnothing)=1$ we have $w_r(\varnothing)=0$, so the empty set never contributes below.

\begin{lemma}
\label{hypertree:lem:weights}
For every $B\subseteq[r-1]$ one has $w_r(B)\geq0$.
\end{lemma}

\begin{proof}
Every subset of $[r-1]$ is the descent set of at least one permutation in $\mathfrak{S}_r$.
Indeed, let $B$ cut $[r]$ into consecutive blocks, fill the first block increasingly with the largest available values of $[r]$, the second block increasingly with the next largest available values, and so on.
This permutation ascends inside each block and descends exactly at the cuts, so its descent set is $B$ and $A_r(B)\geq1$.
If $|B|$ is odd then $r!\,w_r(B)=A_r(B)+1>0$, and if $|B|$ is even then $r!\,w_r(B)=A_r(B)-1\geq0$.
\end{proof}

For $P=\{p_1<\cdots<p_k\}\subseteq[r-1]$ let
\[
   \alpha(P)=(\alpha_1(P),\dots,\alpha_{k+1}(P))=(p_1,\,p_2-p_1,\,\dots,\,p_k-p_{k-1},\,r-p_k)
\]
be the composition of $r$ whose partial sums are the elements of $P$, with $\alpha(\varnothing)=(r)$.

\begin{lemma}
\label{hypertree:lem:subsetsum}
If $\varnothing\neq P\subseteq[r-1]$ then
\[
   \sum_{\varnothing\neq B\subseteq P}w_r(B)=\frac{1}{\prod_i\alpha_i(P)!}.
\]
\end{lemma}

\begin{proof}
A permutation of $[r]$ has descent set contained in $P$ exactly when it is increasing on each of the blocks that $P$ cuts $[r]$ into, and such a permutation is determined by the unordered choice of values placed in those blocks.
Hence
\[
   \sum_{B\subseteq P}A_r(B)=\frac{r!}{\prod_i\alpha_i(P)!}.
\]
Since $P\neq\varnothing$ we also have $\sum_{B\subseteq P}(-1)^{|B|}=0$, and therefore
\[
   \sum_{B\subseteq P}w_r(B)=\frac{1}{r!}\left(\frac{r!}{\prod_i\alpha_i(P)!}-0\right)=\frac{1}{\prod_i\alpha_i(P)!}.
\]
The summand $w_r(\varnothing)$ vanishes, so it may be omitted.
\end{proof}

Let $e$ be a set of size $r$.
For a linear order $\tau=(u_1,\dots,u_r)$ of the elements of $e$ and a set $B\subseteq[r-1]$, let $C(\tau,B)$ be the set of colorings $\kappa:e\to\N$ with
\[
   \kappa(u_1)\leq\cdots\leq\kappa(u_r),
   \qquad
   \kappa(u_j)<\kappa(u_{j+1})\ \text{ for every } j\in B.
\]

\begin{proposition}
\label{hypertree:prop:local}
For every coloring $\kappa:e\to\N$,
\[
   \one_{\{\kappa\ \text{is nonconstant on}\ e\}}
   =\sum_{\tau}\ \sum_{\varnothing\neq B\subseteq[r-1]}w_r(B)\,\one_{\{\kappa\in C(\tau,B)\}},
\]
where $\tau$ runs over all $r!$ linear orders of $e$.
\end{proposition}

\begin{proof}
Suppose first that $\kappa$ is constant.
Then no strict inequality holds, so $\kappa\notin C(\tau,B)$ for every $\tau$ and every nonempty $B$, and the right hand side is $0$.

Suppose now that $\kappa$ is nonconstant, and let the fibers of $\kappa$ have sizes $\alpha_1,\dots,\alpha_k$ listed in increasing order of color, so that $k\geq2$ and $\alpha_1+\cdots+\alpha_k=r$.
There are exactly $\prod_i\alpha_i!$ linear orders $\tau=(u_1,\dots,u_r)$ along which $\kappa$ is weakly increasing, namely those obtained by listing the fibers in increasing order of color and ordering each fiber arbitrarily.
For any other order $\tau$ the coloring $\kappa$ lies in no $C(\tau,B)$ at all.
Fix one of the $\prod_i\alpha_i!$ weakly increasing orders.
Along it the strict jumps occur exactly at the set
\[
   P=\{\alpha_1,\ \alpha_1+\alpha_2,\ \dots,\ \alpha_1+\cdots+\alpha_{k-1}\},
\]
which is nonempty because $k\geq2$, and $\alpha(P)=(\alpha_1,\dots,\alpha_k)$.
Hence $\kappa\in C(\tau,B)$ if and only if $B\subseteq P$, so the inner sum for this $\tau$ equals
\[
   \sum_{\varnothing\neq B\subseteq P}w_r(B)=\frac{1}{\prod_i\alpha_i!}
\]
by Lemma~\ref{hypertree:lem:subsetsum}.
Summing over the $\prod_i\alpha_i!$ weakly increasing orders gives $1$, as desired.
\end{proof}

\begin{example}
\label{hypertree:ex:four}
Take $r=4$, the smallest hyperedge size not covered by primality.
The descent-set counts on $\mathfrak{S}_4$ are $A_4(\varnothing)=A_4(\{1,2,3\})=1$, $A_4(\{1\})=A_4(\{3\})=A_4(\{1,2\})=A_4(\{2,3\})=3$ and $A_4(\{2\})=A_4(\{1,3\})=5$, so
\[
   w_4(\{1\})=w_4(\{3\})=w_4(\{1,3\})=\tfrac16,
   \qquad
   w_4(\{2\})=\tfrac14,
\]
\[
   w_4(\{1,2\})=w_4(\{2,3\})=w_4(\{1,2,3\})=\tfrac1{12}.
\]
If $\kappa$ takes two distinct values on $e$, each twice, then $\prod_i\alpha_i!=4$ orders are weakly increasing, each with $P=\{2\}$, and the total contribution is $4\cdot w_4(\{2\})=1$.
If $\kappa$ is injective then a single order is weakly increasing, with $P=\{1,2,3\}$, and the total contribution is $\tfrac16+\tfrac14+\tfrac16+\tfrac1{12}+\tfrac16+\tfrac1{12}+\tfrac1{12}=1$.
\end{example}

\subsubsection{Gluing the local identities over a hypertree}

For each hyperedge $e$ choose a linear order $\tau_e=(u_{e,1},\dots,u_{e,|e|})$ of its vertices together with a nonempty set $B_e\subseteq[|e|-1]$, and let $\Omega$ denote such a collection of choices.
Put
\[
   W_\Omega=\prod_{e\in E}w_{|e|}(B_e),
\]
which is nonnegative by Lemma~\ref{hypertree:lem:weights}, and let $K_\Omega$ be the generating function $\sum_\kappa\prod_{v\in V}x_{\kappa(v)}$ over the colorings $\kappa:V\to\N$ satisfying, for every $e\in E$,
\begin{equation}
\label{hypertree:eq:omega-conditions}
   \kappa(u_{e,1})\leq\cdots\leq\kappa(u_{e,|e|}),
   \qquad
   \kappa(u_{e,j})<\kappa(u_{e,j+1})\ \text{ for every } j\in B_e.
\end{equation}
A coloring is proper exactly when it is nonconstant on every hyperedge, so multiplying the identity of Proposition~\ref{hypertree:prop:local} over the hyperedges of $H$ gives the pointwise identity
\[
   \one_{\{\kappa\ \text{proper}\}}=\sum_\Omega W_\Omega\,\one_{\{\kappa\ \text{satisfies the inequalities of}\ \Omega\}},
\]
in which ``the inequalities of $\Omega$'' abbreviates \eqref{hypertree:eq:omega-conditions} for every $e\in E$ and the sum over $\Omega$ is finite.
Weighting by $\prod_{v\in V}x_{\kappa(v)}$ and summing over all colorings therefore gives
\begin{equation}
\label{hypertree:eq:weighted-sum}
   X_H=\sum_\Omega W_\Omega K_\Omega .
\end{equation}
We now use the hypertree hypothesis to identify each $K_\Omega$ as a $(P,\omega)$-partition enumerator.

\begin{lemma}
\label{hypertree:lem:incidence}
Let $H=(V,E)$ be a hypertree with $|V|\geq2$.
Then its incidence graph is a tree, distinct hyperedges of $H$ meet in at most one vertex, and
\[
   \sum_{e\in E}(|e|-1)=|V|-1 .
\]
\end{lemma}

\begin{proof}
The incidence graph has vertex set $V\sqcup E$ and an edge joining $v$ to $e$ whenever $v\in e$.
Since $|V|\geq2$ and $H$ is connected, every vertex of $H$ lies in a hyperedge, and every path of $H$ from $v$ to $v'$ is a walk of the incidence graph from $v$ to $v'$; as each hyperedge is adjacent to its own vertices, the incidence graph is connected.
It is bipartite and simple, so each of its cycles has even length $2m$ with $m\geq2$ and reads $v_1,e_1,v_2,e_2,\dots,v_m,e_m,v_1$ with the $v_i$ distinct and the $e_i$ distinct, which is precisely a cycle of $H$.
As $H$ has no cycles the incidence graph is acyclic, hence a tree.
A tree on $|V|+|E|$ vertices has $|V|+|E|-1$ edges, and the incidence graph has $\sum_{e\in E}|e|$ edges, so $\sum_{e\in E}|e|=|V|+|E|-1$, which is the displayed identity.
Finally, if distinct hyperedges $e,e'$ both contained distinct vertices $u,v$, then $u,e,v,e',u$ would be a cycle of $H$.
\end{proof}

Given $\Omega$, let $T_\Omega$ be the graph on vertex set $V$ whose edges join consecutive vertices in the chosen order of each hyperedge,
\[
   E(T_\Omega)=\bigl\{\{u_{e,j},u_{e,j+1}\}\ :\ e\in E,\ 1\leq j<|e|\bigr\}.
\]

\begin{lemma}
\label{hypertree:lem:tree}
For every $\Omega$ the graph $T_\Omega$ is a tree on $V$.
\end{lemma}

\begin{proof}
It is connected.
Indeed, let $u,v\in V$ be distinct and let $u=v_1,e_1,\dots,e_m,v_{m+1}=v$ be a path of $H$.
For each $i$ the vertices $v_i$ and $v_{i+1}$ both lie in $e_i$, and the edges of $T_\Omega$ contributed by $e_i$ form a path through all of $e_i$, so $v_i$ and $v_{i+1}$ are joined in $T_\Omega$.
Next, the $\sum_{e\in E}(|e|-1)$ listed pairs are pairwise distinct: consecutive pairs inside one linear order are distinct, and two pairs coming from different hyperedges are distinct because distinct hyperedges share at most one vertex by Lemma~\ref{hypertree:lem:incidence}.
Hence $T_\Omega$ is a connected graph on $|V|$ vertices with $\sum_{e\in E}(|e|-1)=|V|-1$ edges, again by Lemma~\ref{hypertree:lem:incidence}, and is therefore a tree.
\end{proof}

Orient every edge of $T_\Omega$ as $u_{e,j}\to u_{e,j+1}$, and call this oriented edge \emph{strict} if $j\in B_e$ and \emph{weak} otherwise.
By Lemma~\ref{hypertree:lem:tree} the underlying graph is a tree, so this orientation is acyclic and the transitive closure of these relations is a partial order; let $P_\Omega$ be the resulting poset on $V$, generated by the relations $u_{e,j}<_{P_\Omega}u_{e,j+1}$.
Let $Q_\Omega$ be the second orientation of the same tree obtained by keeping every weak edge and reversing every strict edge.
It is again acyclic, so we may choose a bijection $\omega_\Omega:V\to[n]$ that is increasing along $Q_\Omega$; equivalently, along an oriented edge $u\to v$ of $T_\Omega$,
\begin{equation}
\label{hypertree:eq:omega}
   \omega_\Omega(u)<\omega_\Omega(v)\ \text{ if the edge is weak},
   \qquad
   \omega_\Omega(u)>\omega_\Omega(v)\ \text{ if the edge is strict}.
\end{equation}
Figure~\ref{hypertree:fig:gluing} shows the construction on a small hypertree.

\begin{figure}[t]
\centering
\begin{tikzpicture}[
    vtx/.style={circle,draw,fill=white,minimum size=6.5mm,inner sep=0.5pt,font=\small}
]
\begin{scope}
  \draw[line width=0.8pt,gray] (0,1) ellipse (0.55 and 1.55);
  \draw[line width=0.8pt,gray] (1.95,0) ellipse (2.75 and 0.6);
  \node[vtx] (a1) at (0,2)   {$1$};
  \node[vtx] (a2) at (0,1)   {$2$};
  \node[vtx] (a3) at (0,0)   {$3$};
  \node[vtx] (a4) at (1.3,0) {$4$};
  \node[vtx] (a5) at (2.6,0) {$5$};
  \node[vtx] (a6) at (3.9,0) {$6$};
  \node[font=\small] at (-1.0,2.0) {$e_1$};
  \node[font=\small] at (1.95,-1.05) {$e_2$};
  \node[font=\small] at (1.95,2.6) {the hypertree $H$};
\end{scope}
\begin{scope}[xshift=7.2cm]
  \node[vtx] (b1) at (0,2)   {$1$};
  \node[vtx] (b2) at (0,1)   {$2$};
  \node[vtx] (b3) at (0,0)   {$3$};
  \node[vtx] (b4) at (1.3,0) {$4$};
  \node[vtx] (b5) at (2.6,0) {$5$};
  \node[vtx] (b6) at (3.9,0) {$6$};
  \draw[->,line width=0.6pt,gray] (b1) -- (b2);
  \draw[->,line width=1.6pt] (b2) -- (b3);
  \draw[->,line width=1.6pt] (b3) -- (b4);
  \draw[->,line width=0.6pt,gray] (b4) -- (b5);
  \draw[->,line width=1.6pt] (b5) -- (b6);
  \node[font=\small] at (1.95,2.6) {the oriented tree $T_\Omega$};
\end{scope}
\end{tikzpicture}
\caption{On the left, the hypertree $H$ on $V=[6]$ with hyperedges $e_1=\{1,2,3\}$ and $e_2=\{3,4,5,6\}$ drawn as gray ovals.
On the right, the tree $T_\Omega$ for the choice $\tau_{e_1}=(1,2,3)$, $\tau_{e_2}=(3,4,5,6)$, $B_{e_1}=\{2\}$ and $B_{e_2}=\{1,3\}$: each hyperedge is threaded along its chosen order, and the three bold arrows are the strict steps, the two thin gray arrows the weak ones.
Reversing the bold arrows gives the acyclic orientation $Q_\Omega$, and the labeling $\omega_\Omega$ that assigns $1,2,3,4,5,6$ to the vertices $4,6,1,3,5,2$ in this order is increasing along it, as required by \eqref{hypertree:eq:omega}.}
\label{hypertree:fig:gluing}
\end{figure}
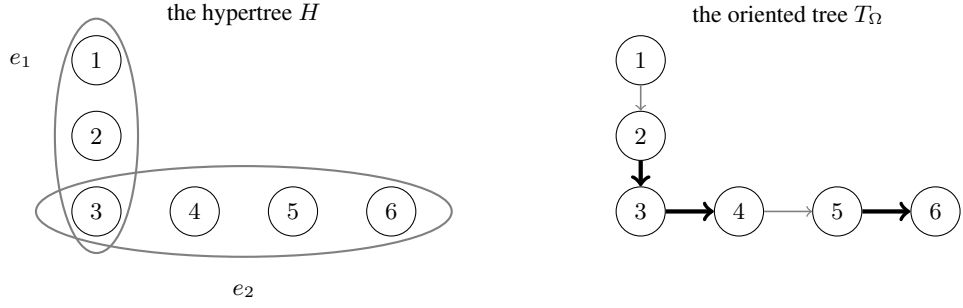

We use the order-preserving convention for $P$-partitions.
Given a finite poset $P$ on $n$ elements and a bijection $\omega:P\to[n]$, a \emph{$(P,\omega)$-partition} is a map $\sigma:P\to\N$ such that $x<_Py$ implies $\sigma(x)\leq\sigma(y)$, and implies $\sigma(x)<\sigma(y)$ when moreover $\omega(x)>\omega(y)$.
Write $\Gamma(P,\omega)=\sum_\sigma\prod_{v\in P}x_{\sigma(v)}$ for the generating function of the $(P,\omega)$-partitions.
The fundamental theorem of $(P,\omega)$-partitions, for which see \citet{gessel1984multipartite} or, in the order-reversing convention, \citet{stanley2011enumerative}, states that
\begin{equation}
\label{hypertree:eq:ftpp}
   \Gamma(P,\omega)=\sum_{\pi\in\mathcal{L}(P)}F^{(n)}_{D_\omega(\pi)},
   \qquad
   D_\omega(\pi)=\{i\in[n-1]:\omega(\pi_i)>\omega(\pi_{i+1})\},
\end{equation}
where $\mathcal{L}(P)$ is the set of linear extensions of $P$, each written as the word $\pi_1\pi_2\cdots\pi_n$ that lists the elements of $P$ in the corresponding order.
In particular $\Gamma(P,\omega)$ is $F$-positive with integer coefficients.

\begin{remark}
\label{hypertree:rem:weights}
The weights of Proposition~\ref{hypertree:prop:local} are forced by the shape of the decomposition.
Let $p_i=x_1^i+x_2^i+\cdots$ be the $i$th power sum symmetric function.
Multiplying the identity of Proposition~\ref{hypertree:prop:local} by $\prod_{v\in e}x_{\kappa(v)}$ and summing over all $\kappa:e\to\N$ turns it into an identity of symmetric functions: the left hand side becomes $X_e=p_1^r-p_r$ for the hypergraph consisting of the single hyperedge $e$, and the generating function of $C(\tau,B)$ is $F_B^{(r)}$ for each of the $r!$ orders $\tau$, so
\[
   p_1^r-p_r=\sum_{S\subseteq[r-1]}\bigl(A_r(S)-(-1)^{|S|}\bigr)F_S^{(r)}.
\]
Applying \eqref{hypertree:eq:ftpp} to an antichain on $r$ elements gives $p_1^r=\sum_{\pi\in\mathfrak{S}_r}F^{(r)}_{\Des(\pi)}=\sum_S A_r(S)F_S^{(r)}$, so the display above is the classical expansion $p_r=\sum_{S\subseteq[r-1]}(-1)^{|S|}F_S^{(r)}$ of the power sum in the fundamental basis.
The content of Proposition~\ref{hypertree:prop:local} is that the identity already holds coloring by coloring, which is what makes it glue, and the content of Lemma~\ref{hypertree:lem:weights} is that its coefficients are nonnegative.
\end{remark}

\begin{lemma}
\label{hypertree:lem:Kisgamma}
For every $\Omega$ one has $K_\Omega=\Gamma(P_\Omega,\omega_\Omega)$.
In particular $K_\Omega$ is a nonnegative integral combination of fundamental quasisymmetric functions.
\end{lemma}

\begin{proof}
Suppose first that $\kappa$ satisfies \eqref{hypertree:eq:omega-conditions}.
Then along every oriented edge $u\to v$ of $T_\Omega$ one has $\kappa(u)\leq\kappa(v)$, with strict inequality when the edge is strict.
If $x<_{P_\Omega}y$ there is a directed path from $x$ to $y$ in $T_\Omega$, and chaining the inequalities along it gives $\kappa(x)\leq\kappa(y)$.
Suppose in addition that $\omega_\Omega(x)>\omega_\Omega(y)$.
By \eqref{hypertree:eq:omega} the value of $\omega_\Omega$ increases along every weak edge, so if all edges of that directed path were weak we would get $\omega_\Omega(x)<\omega_\Omega(y)$.
Hence the path contains a strict edge, so at least one of the chained inequalities is strict and $\kappa(x)<\kappa(y)$.
Thus $\kappa$ is a $(P_\Omega,\omega_\Omega)$-partition.

Conversely, suppose $\kappa$ is a $(P_\Omega,\omega_\Omega)$-partition and let $u\to v$ be an oriented edge of $T_\Omega$, so that $u<_{P_\Omega}v$.
If the edge is weak then $\omega_\Omega(u)<\omega_\Omega(v)$ by \eqref{hypertree:eq:omega} and the definition gives $\kappa(u)\leq\kappa(v)$.
If it is strict then $\omega_\Omega(u)>\omega_\Omega(v)$ and the definition gives $\kappa(u)<\kappa(v)$.
Ranging over the edges contributed by a hyperedge $e$ recovers exactly the conditions \eqref{hypertree:eq:omega-conditions} for $e$.

The colorings counted by $K_\Omega$ are therefore precisely the $(P_\Omega,\omega_\Omega)$-partitions, and the last assertion follows from \eqref{hypertree:eq:ftpp}.
\end{proof}

\subsubsection{Proof of the main theorem}

\begin{proof}[Proof of Theorem~\ref{hypertree:thm:main}]
We may assume $n\geq2$.
Every weight $W_\Omega$ in \eqref{hypertree:eq:weighted-sum} is nonnegative by Lemma~\ref{hypertree:lem:weights}, and every $K_\Omega$ is a nonnegative integral combination of fundamental quasisymmetric functions by Lemma~\ref{hypertree:lem:Kisgamma}.
Hence \eqref{hypertree:eq:weighted-sum} exhibits $X_H$ as a nonnegative rational combination of the $F_S^{(n)}$, so every $a_S$ is a nonnegative rational number.

It remains to see that the $a_S$ are integers.
For $T\subseteq[n-1]$ let $\alpha(T)=(\alpha_1(T),\dots,\alpha_m(T))$ be the associated composition of $n$, defined as above with $r$ replaced by $n$, and let
\[
   M_T^{(n)}=\sum_{i_1<\cdots<i_m}x_{i_1}^{\alpha_1(T)}\cdots x_{i_m}^{\alpha_m(T)}
\]
be the corresponding monomial quasisymmetric function.
Splitting the defining sum of $F_S^{(n)}$ according to the exact set $T$ of indices $j$ with $i_j<i_{j+1}$ gives
\[
   F_S^{(n)}=\sum_{T\supseteq S}M_T^{(n)}.
\]
Write $X_H=\sum_T c_T M_T^{(n)}$.
Then $c_T=\sum_{S\subseteq T}a_S$, so M\"obius inversion in the Boolean lattice gives
\[
   a_S=\sum_{T\subseteq S}(-1)^{|S|-|T|}c_T .
\]
Now $c_T$ is the coefficient of $x_1^{\alpha_1(T)}\cdots x_m^{\alpha_m(T)}$ in $X_H$, which is the number of proper colorings $\kappa:V\to[m]$ with $|\kappa^{-1}(i)|=\alpha_i(T)$ for every $i\in[m]$, hence a nonnegative integer.
Therefore each $a_S$ is an integer, and being also nonnegative it lies in $\mathbb{Z}_{\geq0}$.
\end{proof}

\begin{remark}
\label{hypertree:rem:partition}
The weighted decomposition produces no splittings, and hence says nothing about the partitionability of the Coxeter complexes of type $A_{r-1}$ with the empty face removed.
Nor does it recover the combinatorial interpretation of the coefficients that \citet{taylor2015chromatic} gives in the prime case.
Theorem~\ref{hypertree:thm:main} establishes that the $a_S$ are nonnegative integers, but the decomposition \eqref{hypertree:eq:weighted-sum} realizes them only as nonnegative rational combinations of numbers of linear extensions of posets.
\end{remark}

\begin{remark}
\label{hypertree:rem:forest}
Connectedness enters the argument only in Lemma~\ref{hypertree:lem:incidence} and in the proof of Lemma~\ref{hypertree:lem:tree}, where it supplies a path of $H$ between any two vertices.
For a hyperforest, that is, a disjoint union of hypertrees, the same construction produces a spanning forest $T_\Omega$ whose components are the trees built inside the components of $H$, and $P_\Omega$ is the disjoint union of the corresponding posets; \eqref{hypertree:eq:ftpp} applies verbatim and Theorem~\ref{hypertree:thm:main} holds for hyperforests as well.
\end{remark}

\putbib

\end{bibunit}
\endgroup
\subsection[(Wanless) Quartically many Fano subsquares in Latin squares]{Quartically many Fano subsquares in Latin squares}
\label{sol:latin}
\begingroup
\begin{bibunit}
\def\F{\mathbb{F}}
\def\zst{\zeta^{*}}

This problem asks whether a Latin square of order $n$ can contain more than cubically many subsquares of order $7$.
We show that it can: for every $n\ge56$ there is a Latin square of order $n$ with more than $(n/28)^4$ subsquares isotopic to the Fano square $S_7$.
A quartic upper bound is already due to \citet{browning2015overlapping}, so it is the lower bound that settles the order of growth; we also prove the explicit upper bound $n^4$ by a short self-contained argument.
The construction is an affine lift of the Fano quasigroup over a finite field of characteristic $2$.
The upper bound comes from the observation that three rows and one column already generate the whole incidence structure of $S_7$.

\subsubsection{Introduction}

Let $M$ be a Latin square of order $n$.
A \emph{subsquare} of order $m$ in $M$ is a set of $m$ rows together with a set of $m$ columns whose induced $m\times m$ subarray contains only $m$ distinct symbols; that subarray is then itself a Latin square, and its symbol set is determined by the chosen rows and columns.
Two Latin squares are \emph{isotopic} if one is carried to the other by relabeling rows, columns and symbols independently.
Following \citet{browning2014bounds}, write $\zeta(n,m)$ for the largest number of subsquares of order $m$ in a Latin square of order $n$, and, for a fixed Latin square $L$, write $\zst(n,L)$ for the largest number of subsquares isotopic to $L$ in a Latin square of order $n$.

Let $V=\F_2^3$ and $P=V\setminus\{0\}$.
The \emph{Fano square} $S_7$ is the Cayley table of the quasigroup $(P,\circ)$ defined by
\begin{equation}\label{latin:eq:fano}
x\circ y=\begin{cases}
x, & x=y,\\
x+y, & x\neq y,
\end{cases}
\end{equation}
with addition in $V$.
Equivalently, $S_7$ is the Steiner quasigroup of the Fano plane: the unordered triples $\{x,y,x+y\}$ with $x\neq y$ are exactly the seven lines of $\mathrm{PG}(2,2)$.
Writing $1,\dots,7$ for the nonzero vectors of $V$ in binary notation, $S_7$ is the array
\[
\begin{array}{c|ccccccc}
\circ & 1 & 2 & 3 & 4 & 5 & 6 & 7\\\hline
1 & 1 & 3 & 2 & 5 & 4 & 7 & 6\\
2 & 3 & 2 & 1 & 6 & 7 & 4 & 5\\
3 & 2 & 1 & 3 & 7 & 6 & 5 & 4\\
4 & 5 & 6 & 7 & 4 & 1 & 2 & 3\\
5 & 4 & 7 & 6 & 1 & 5 & 3 & 2\\
6 & 7 & 4 & 5 & 2 & 3 & 6 & 1\\
7 & 6 & 5 & 4 & 3 & 2 & 1 & 7
\end{array}
\]

Wanless asked the following at the LOOPS '11 open problem session \citep[Problem~2.8]{loops2011session}.

\medskip
\noindent
\emph{Fix a prime $p$.
Is there a family of latin squares with more than cubically many subsquares of order $p$?
More precisely, is it true that for every constant $c$ there is a latin square $L$ of order $n$ such that there are more than $cn^3$ subsquares of order $p$ in $L$?}
\medskip

\noindent
The note accompanying the problem records that the answer is negative for $p\in\{2,3,5\}$, and that for $p=7$ the subsquares would have to be multiplication tables of a Steiner quasigroup.
Since the Steiner triple system of order $7$ is unique up to isomorphism, the only Steiner quasigroup of order $7$ is $(P,\circ)$.
The case $p=7$ of the problem is therefore precisely the question, left open by \citet{browning2014bounds}, of whether $\zst(n,S_7)$ grows faster than cubically.
We answer it affirmatively.

The results of \citet{browning2014bounds} give the exact order of growth of $\zeta(n,m)$ for several small $m$: one has $\zeta(n,m)=\Theta(n^3)$ for $m\in\{2,3,5\}$ and $\zeta(n,m)=\Theta(n^4)$ for $m\in\{4,6,9,10\}$, with the sharper bounds $n^3/8+O(n^2)\le\zeta(n,2)\le n^3/4+O(n^2)$ and $n^3/27+O(n^{5/2})\le\zeta(n,3)\le n^3/18+O(n^2)$.
For a fixed square they show that $\zst(n,L)=\Theta(n^3)$ when $L$ is cyclic, that $\zst(n,L)=O(n^3)$ for a large class of $L$, and that every $L$ admits some $\varepsilon\in(0,1)$ with $\zst(n,L)=\Omega(n^{2+\varepsilon})$.
The value $m=7$ is absent from both lists, and by the note quoted above the only order-$7$ isotopy class that can push $\zeta(n,7)$ past cubic growth is $S_7$.
On the upper bound side, \citet[Theorem~8]{browning2015overlapping} show that a Latin square of order $n$ has $O(n^{\psi(m,t)+t})$ subsquares of order $m$ for all positive integers $t\le m\le n$, where $\psi(m,t)=\lceil\tfrac12\lfloor m/t\rfloor\rceil$ for odd $m$; with $m=7$ and $t=2$ this gives $\zeta(n,7)=O(n^4)$.
For subsquares of unbounded order, \citet{browning2013bounds} prove that a Latin square of order $n$ has at most $n^{O(\log k)}$ subsquares of order $k$.

Extremal behavior here is very far from typical behavior.
\citet{mckay1999most} show that for every $\varepsilon>0$ almost all Latin squares of order $n$ have at least $n^{3/2-\varepsilon}$ subsquares of order $2$, and \citet{allsop2025subsquares} show that a uniformly random $k\times n$ Latin rectangle has no proper subsquare of order $4$ or more with probability $1-O(1/n)$.
In particular a random Latin square of order $n$ has, with probability $1-O(1/n)$, no proper subsquare of order $7$, so a positive answer must come from an explicit construction.
The main result of this section is the following.

\begin{theorem}\label{latin:thm:main}
For every $n\ge 56$,
\[
\left(\frac{n}{28}\right)^{4}<\zst(n,S_7)\le n^{4}.
\]
In particular $\zst(n,S_7)=\Theta(n^4)$.
\end{theorem}

Combined with the bound of \citet{browning2015overlapping} quoted above, this settles the order of growth of $\zeta(n,7)$ as well, adding $7$ to the list of $m$ for which $\zeta(n,m)$ is known.

\begin{corollary}\label{latin:cor:order7}
$\zeta(n,7)=\Theta(n^4)$.
\end{corollary}

We now summarize the construction.
The key point is that $S_7$ is almost an $\F_2$-linear object: off the diagonal the rule \eqref{latin:eq:fano} is simply addition in $V$, and the diagonal rule $x\circ x=x$ is the only obstruction to linearity.
We therefore work on $P\times\F_q$ with $q$ a power of two, keep the additive rule off the diagonal, and replace the diagonal rule by a nontrivial affine combination $\lambda u+(1+\lambda)v$ in the fiber coordinate.
The copies of $S_7$ are then the translated graphs of the $\F_2$-linear maps $V\to\F_q$: there are $q^3$ such maps and $q$ translations, giving $q^4$ subsquares in a Latin square of order $7q$.
The affine parameter $\lambda$ is what forces the diagonal cells of each graph to close up correctly.

\subsubsection{The construction}

Fix a power of two $q\ge4$ and an element $\lambda\in\F_q\setminus\{0,1\}$.
Define a binary operation $*$ on $P\times\F_q$ by
\begin{equation}\label{latin:eq:lift}
(x,u)*(y,v)=
\begin{cases}
(x,\ \lambda u+(1+\lambda)v), & x=y,\\
(x+y,\ u+v), & x\neq y,
\end{cases}
\end{equation}
where the first coordinate is computed in $V$ and the second in $\F_q$.
Write $L_q$ for the resulting array, with rows and columns indexed by $P\times\F_q$.

\begin{lemma}\label{latin:lem:latin-square}
$L_q$ is a Latin square of order $7q$.
\end{lemma}

\begin{proof}
Since $|P\times\F_q|=7q$, it suffices to check that $*$ is a quasigroup operation, that is, that for each row and each symbol there is a unique column producing that symbol, and likewise with the roles of rows and columns interchanged.

Fix a row $(x,u)$ and a symbol $(z,w)$, and seek a column $(y,v)$ with $(x,u)*(y,v)=(z,w)$.
Suppose first that $z=x$.
The second branch of \eqref{latin:eq:lift} would force $x+y=z=x$, hence $y=0\notin P$, so the first branch applies and $y=x$.
The remaining equation $\lambda u+(1+\lambda)v=w$ has the unique solution
\[
v=(1+\lambda)^{-1}(w+\lambda u),
\]
because $1+\lambda\neq0$.
Suppose next that $z\neq x$.
The first branch would force $z=x$, so the second branch applies and $y=x+z$, which is nonzero and distinct from $x$ because $z\neq x$ and $z\neq0$.
The remaining equation $u+v=w$ has the unique solution $v=w+u$.

Now fix a column $(y,v)$ and a symbol $(z,w)$, and seek a row $(x,u)$.
If $z=y$ then exactly as before $x=y$, and $\lambda u+(1+\lambda)v=w$ has the unique solution $u=\lambda^{-1}(w+(1+\lambda)v)$ because $\lambda\neq0$.
If $z\neq y$ then $x=y+z\in P$ with $x\neq y$, and $u=w+v$.
\end{proof}

We next exhibit many Fano subsquares of $L_q$.
Let $\phi\colon V\to\F_q$ be an $\F_2$-linear map and let $r\in\F_q$.
Put
\[
\tilde r=\lambda^{-1}(1+\lambda)r,
\]
and define
\begin{gather*}
R_{\phi,r}=\{(x,\phi(x)+r):x\in P\},
\qquad
C_{\phi,r}=\{(x,\phi(x)+\tilde r):x\in P\},\\
T_{\phi,r}=\{(x,\phi(x)+r+\tilde r):x\in P\}.
\end{gather*}
Each of these sets has exactly seven elements, since its members have distinct first coordinates.

\begin{lemma}\label{latin:lem:copies}
For every pair $(\phi,r)$ the subarray of $L_q$ on the rows $R_{\phi,r}$ and the columns $C_{\phi,r}$ is a subsquare with symbol set $T_{\phi,r}$, and it is isotopic to $S_7$.
\end{lemma}

\begin{proof}
Index the row $(x,\phi(x)+r)$, the column $(y,\phi(y)+\tilde r)$ and the symbol $(z,\phi(z)+r+\tilde r)$ by their first coordinates $x,y,z\in P$.
We claim that the row indexed by $x$ times the column indexed by $y$ is the symbol indexed by $x\circ y$.

Suppose $x\neq y$.
The second branch of \eqref{latin:eq:lift} applies, and since $\phi$ is $\F_2$-linear,
\[
(\phi(x)+r)+(\phi(y)+\tilde r)=\phi(x+y)+r+\tilde r .
\]
The product is therefore the symbol indexed by $x+y=x\circ y$.

Suppose instead $x=y$.
The first branch of \eqref{latin:eq:lift} applies, and in characteristic $2$,
\[
\lambda\bigl(\phi(x)+r\bigr)+(1+\lambda)\bigl(\phi(x)+\tilde r\bigr)
=(\lambda+1+\lambda)\phi(x)+\lambda r+(1+\lambda)\tilde r
=\phi(x)+\lambda r+(1+\lambda)\tilde r .
\]
The definition of $\tilde r$ says exactly that $\lambda\tilde r=(1+\lambda)r$, and adding $\lambda r+\tilde r$ to both sides of this identity turns it into
\[
\lambda r+(1+\lambda)\tilde r=r+\tilde r .
\]
The product is therefore the symbol indexed by $x=x\circ x$, which proves the claim.

Consequently the subarray on $R_{\phi,r}\times C_{\phi,r}$ uses exactly the seven symbols of $T_{\phi,r}$, so it is a subsquare of order $7$, and the three indexings by elements of $P$ exhibit an isotopism from $S_7$ onto it.
\end{proof}

\begin{lemma}\label{latin:lem:distinct}
The $q^4$ subsquares produced by Lemma~\ref{latin:lem:copies} are pairwise distinct.
\end{lemma}

\begin{proof}
There are $q^3$ $\F_2$-linear maps $V\to\F_q$ and $q$ choices of $r$, so there are $q^4$ pairs $(\phi,r)$; it suffices to show that distinct pairs give distinct row sets.
Suppose $R_{\phi,r}=R_{\psi,s}$.
Each of the two sets contains exactly one element with first coordinate $x$, for every $x\in P$, so $\phi(x)+r=\psi(x)+s$ for every $x\in P$.
Thus the $\F_2$-linear map $h=\phi+\psi$ takes the constant value $r+s$ on $P$.
Choosing linearly independent $a,b\in V$ gives
\[
r+s=h(a+b)=h(a)+h(b)=(r+s)+(r+s)=0 .
\]
Hence $r=s$, and $h$ vanishes on $P$ and at $0$, so $\phi=\psi$.
\end{proof}

\begin{corollary}\label{latin:cor:lower}
For every power of two $q\ge4$ one has $\zst(7q,S_7)\ge q^4$.
Moreover $\zst(n,S_7)>(n/28)^4$ for every $n\ge56$.
\end{corollary}

\begin{proof}
The first assertion is immediate from Lemmas~\ref{latin:lem:latin-square}, \ref{latin:lem:copies} and \ref{latin:lem:distinct}.

For the second, let $n\ge56$ and let $q$ be the largest power of two with $14q\le n$.
Then $q\ge4$, and maximality gives $28q>n$, so $q>n/28$.
By Evans' embedding theorem, every partial Latin square of order $N$ embeds in a Latin square of every order at least $2N$ \citep{evans1960embedding}.
Applying this to $L_q$, which is a Latin square of order $7q$ and in particular a partial Latin square of order $7q$, produces a Latin square $M$ of order $n\ge14q$ whose subarray on the first $7q$ rows and columns is $L_q$.
Every subsquare of $L_q$ is a subsquare of $M$, so $M$ contains at least $q^4>(n/28)^4$ subsquares isotopic to $S_7$.
\end{proof}

\begin{example}\label{latin:ex:q4}
Take $q=4$, so that $\F_4=\{0,1,\alpha,\alpha^{2}\}$ with $\alpha^{3}=1$, and $\lambda\in\{\alpha,\alpha^{2}\}$.
Then $L_4$ is a Latin square of order $28$ containing at least $4^4=256$ subsquares isotopic to $S_7$.
\end{example}

\subsubsection{The upper bound}

It is convenient to view a Latin square $M$ of order $n$ as a tripartite incidence structure.
Its vertices are the $n$ rows, the $n$ columns and the $n$ symbols of $M$, and each of the $n^2$ cells contributes the triple consisting of its row, its column and the symbol it carries.
The defining property of a Latin square is exactly that any two vertices lying in different parts belong to a unique common triple.
Isotopisms are precisely the isomorphisms of these structures that respect the three parts.

Observe that if $K$ is a subsquare of $M$ and two of the three vertices of some triple of $M$ belong to $K$, then so does the third.
Indeed, if a row and a column of $K$ are given then the symbol in that cell is a symbol of $K$; if a row and a symbol of $K$ are given then the column in which that symbol occurs in that row of $K$ is the unique such column in $M$; and the third case is symmetric.

For $x\in P$ write $R_x$, $C_x$ and $T_x$ for the row, column and symbol vertices of $S_7$ indexed by $x$, so that $R_x$, $C_y$ and $T_{x\circ y}$ form a triple for all $x,y\in P$.

\begin{lemma}\label{latin:lem:generated}
Let $e_1,e_2,e_3$ be a basis of $V=\F_2^3$ and put $e_4=e_1+e_2+e_3$.
Then the four vertices $R_{e_1},R_{e_2},R_{e_3},C_{e_4}$ generate all twenty-one vertices of $S_7$ under the rule that two known vertices in different parts determine the third vertex of their triple.
\end{lemma}

\begin{proof}
Throughout we use \eqref{latin:eq:fano}; note that $e_i\neq e_4$ for $i\in\{1,2,3\}$, since $e_4=e_i$ would force the sum of the other two basis vectors to vanish.

Pairing each of $R_{e_1},R_{e_2},R_{e_3}$ with $C_{e_4}$ produces
\[
T_{e_2+e_3},\qquad T_{e_1+e_3},\qquad T_{e_1+e_2},
\]
because $e_i\circ e_4=e_i+e_4$.
Next, if $x\in P$ and $x\circ y=w$ with $w\neq x$, then $y\neq x$ and $y=x+w$ is determined.
Applying this three times,
\[
R_{e_1},T_{e_1+e_3}\ \text{give}\ C_{e_3},\qquad
R_{e_1},T_{e_1+e_2}\ \text{give}\ C_{e_2},\qquad
R_{e_2},T_{e_1+e_2}\ \text{give}\ C_{e_1}.
\]
The pairs $R_{e_i},C_{e_i}$ then give $T_{e_i}$ for $i\in\{1,2,3\}$, since $e_i\circ e_i=e_i$.
Applying the same rule as before,
\[
R_{e_1},T_{e_2}\ \text{give}\ C_{e_1+e_2},\qquad
R_{e_1},T_{e_3}\ \text{give}\ C_{e_1+e_3},\qquad
R_{e_2},T_{e_3}\ \text{give}\ C_{e_2+e_3}.
\]
At this stage the columns
\[
C_{e_1},\ C_{e_2},\ C_{e_3},\ C_{e_1+e_2},\ C_{e_1+e_3},\ C_{e_2+e_3},\ C_{e_4}
\]
are all known, and these are all seven columns of $S_7$.
Pairing $R_{e_1}$ with each of them yields the symbols $T_{e_1\circ y}$ for $y\in P$; as $y$ runs over $P$ the element $e_1\circ y$ runs over all of $P$, so every symbol vertex is known.
Pairing $C_{e_1}$ with each symbol vertex then yields every row vertex, since for each $z\in P$ there is a unique $x\in P$ with $x\circ e_1=z$.
\end{proof}

\begin{corollary}\label{latin:cor:upper}
For every $n$ one has $\zst(n,S_7)\le n^4$.
\end{corollary}

\begin{proof}
Let $M$ be a Latin square of order $n$ and let $K$ be a subsquare of $M$ isotopic to $S_7$, and let $e_1,e_2,e_3,e_4$ be as in Lemma~\ref{latin:lem:generated}.
Choose an isotopism $\theta$ from $S_7$ onto $K$ and record the quadruple
\[
\bigl(\theta(R_{e_1}),\ \theta(R_{e_2}),\ \theta(R_{e_3}),\ \theta(C_{e_4})\bigr),
\]
which consists of three rows and one column of $M$.
There are at most $n^4$ such quadruples, so it suffices to show that the quadruple determines $K$.

Since $\theta$ is an isotopism onto $K$, it carries triples of $S_7$ to triples of $M$ that lie in $K$.
By the observation above, applying the rule ``two vertices in different parts determine the third vertex of their triple'' inside $M$ to vertices of $K$ never leaves $K$, and it agrees with the corresponding rule in $S_7$ transported by $\theta$.
By Lemma~\ref{latin:lem:generated}, starting from the recorded quadruple and iterating this rule inside $M$ therefore produces exactly the twenty-one vertices of $K$.
Hence $K$ is determined by the quadruple, as desired.
\end{proof}

\subsubsection{Proof of the main theorem}

\begin{proof}[Proof of Theorem~\ref{latin:thm:main}]
Corollary~\ref{latin:cor:lower} gives $\zst(n,S_7)>(n/28)^{4}$ for every $n\ge56$, and Corollary~\ref{latin:cor:upper} gives $\zst(n,S_7)\le n^{4}$ for every $n$.
\end{proof}

\begin{remark}
Along the sequence $n=7q$ with $q$ a power of two, Corollary~\ref{latin:cor:lower} gives the stronger bound $\zst(n,S_7)\ge(n/7)^4$, so that
\[
\frac{1}{2401}\le\limsup_{n\to\infty}\frac{\zst(n,S_7)}{n^{4}}\le1 .
\]
The correct constant remains undetermined.
\end{remark}

\begin{proof}[Proof of Corollary~\ref{latin:cor:order7}]
Every subsquare isotopic to $S_7$ has order $7$, so $\zeta(n,7)\ge\zst(n,S_7)=\Omega(n^4)$ by Theorem~\ref{latin:thm:main}.
In the other direction, the bound of \citet[Theorem~8]{browning2015overlapping} with $m=7$ and $t=2$ gives $\zeta(n,7)=O(n^4)$.
\end{proof}

\putbib

\end{bibunit}
\endgroup
\end{document}